\pdfoutput=1
\documentclass[10pt]{article}
\usepackage{enumerate}
\usepackage[OT1]{fontenc}
\usepackage{amsmath,amssymb}
\usepackage{natbib}
\usepackage[usenames,dvipsnames]{xcolor}
\usepackage{geometry}
\usepackage{dsfont}
\usepackage{pgfplots}
\usepackage{smile}
\usepackage{multirow}
\usepackage{rotating}
\usepackage{enumerate}
\usepackage{esvect}
\usepackage{tikz}
\usetikzlibrary{patterns}
\usetikzlibrary{arrows}
\usepackage[colorlinks,
            linkcolor=red,
            anchorcolor=blue,
            citecolor=blue
            ]{hyperref} 
\usepackage{algorithm}
\usepackage{algorithmic}
\usepackage[most,skins]{tcolorbox}
\definecolor{rliableolive}{HTML}{BBCC33}
\definecolor{rliableblue}{HTML}{77AADD}
\definecolor{rliablered}{HTML}{EE8866}
\definecolor{LightCyan}{rgb}{0.88,1,1}
\definecolor{darkblue}{HTML}{2878D9}
\definecolor{navyblue}{HTML}{0000FF}

\usepackage{enumitem}
\usepackage{acronym}
\acrodef{svd}[SVD]{Singular-Value Decomposition}
\acrodef{rgi}[RGI]{Relative Generalization Invariance}
\acrodef{argi}[ARGI]{Asymptotic RGI}
\acrodef{llm}[LLM]{Large Language Model}
\acrodef{ntk}[NTK]{Neural Tangent Kernel}
\acrodef{mf}[MF]{Mean-Field}
\acrodef{ccc}[CCC]{Concordance Correlation Coefficient}
\acrodef{ood}[OOD]{Out-Of-Distribution}
\acrodef{fa}[FA]{Full Attention}
\acrodef{swa}[SWA]{Sliding-Window Attention}
\acrodef{ffn}[FFN]{Feed-Forward Networks}

\newcommand{\gd}{\mathsf{GD}}
\newcommand{\muon}{\mathsf{Muon}}
\newcommand{\adam}{\mathsf{Adam}}

\usepackage{wrapfig}
\definecolor{sorange}{RGB}{252,91,90}
\definecolor{sblue}{RGB}{9,48,138}
\definecolor{syellow}{RGB}{253,179,51}
\definecolor{steal}{RGB}{19,119,116}

\newcommand{\horange}[1]{{\color{sorange}#1}}
\newcommand{\hblue}[1]{{\color{sblue}#1}}
\newcommand{\hyellow}[1]{{\color{syellow}#1}}
\newcommand{\hteal}[1]{{\color{teal}#1}}

\newcommand{\init}{{\text{init}}}
\newcommand{\spec}{\texttt{spec}}
\newcommand{\sgn}{\texttt{sgn}}
\newcommand{\val}{\mathrm{val}}
\newcommand{\var}{\mathrm{Var}}
\newcommand{\cov}{\mathrm{Cov}}
\newcommand{\differr}{r_{\mathrm{diff}}}
\newcommand{\cccerr}{r_{\mathrm{CCC}}}
\newcommand{\cossim}{\mathrm{cos\text{-}sim}}
\newcommand{\relerr}{r_{\ell_2}}
\newcommand{\frakT}{\mathfrak{T}}

\theoremstyle{plain}

\usepackage{mathrsfs}
\usepackage{fullpage}

\usepackage{hyperref}
\usepackage[protrusion=false,expansion=true]{microtype}

\title{\huge Relative Generalization Invariance of LLM Pretraining}
\author{Fengzhuo Zhang\textsuperscript{1,$\ast$,$\dagger$}\quad Shuche Wang\textsuperscript{2,$\ast$}\quad  Shenggui Li\textsuperscript{3,$\ast$}\quad Tianyu Ruan\textsuperscript{1}\quad Jianliang He\textsuperscript{1}\\
Ivor Tsang\textsuperscript{4}\quad Tianyu Pang\textsuperscript{5}\quad Chao Du\textsuperscript{5}\quad Tianwei Zhang\textsuperscript{3}\quad Zhuoran Yang\textsuperscript{1}\\
\textsuperscript{1}Yale University\quad\textsuperscript{2}National University of Singapore\quad \textsuperscript{3}Nanyang Technological University \\
\textsuperscript{4}A$^{*}$STAR \quad \textsuperscript{5}Sea AI Lab
}
\date{}
\begin{document}
\maketitle
\renewcommand\thefootnote{}\footnotetext{$\ast$ Equal contribution.}\footnotetext{$\dagger$ Project Lead.}\footnotetext{Correspondence to: \textless{}fengzhuo.zhang@yale.edu\textgreater{},\textless{}zhuoran.yang@yale.edu\textgreater{}}

\begin{abstract}
Large Language Model (LLM) pretraining performance is jointly shaped by three components of the training triplet: the optimizer, model architecture, and training data stream. However, how these components influence performance in distinct ways remains unclear. We take a first step toward isolating their effects by studying relative generalization. We introduce Relative Generalization Invariance (RGI), the invariance of the validation-loss difference between any two tokens across models. We show that RGI approximately holds across a wide range of optimizers and moderate architectural variations, suggesting that these choices induce an approximately uniform shift in token-wise losses. In contrast, changing the training data stream can substantially alter relative generalization. We further show that RGI cannot be explained by the neural tangent kernel or mean-field regimes alone and prove that it can emerge in an overparameterized quadratic model. Overall, our work identifies RGI as a new phenomenon in LLM pretraining that helps distinguish the effects of optimizers and architectures from those of training data.

\end{abstract}

\section{Introduction}
\begin{figure}[H]
    \centering
    \subfigure[Training losses of optimizers and architectures]{ \includegraphics[width=0.47\textwidth]{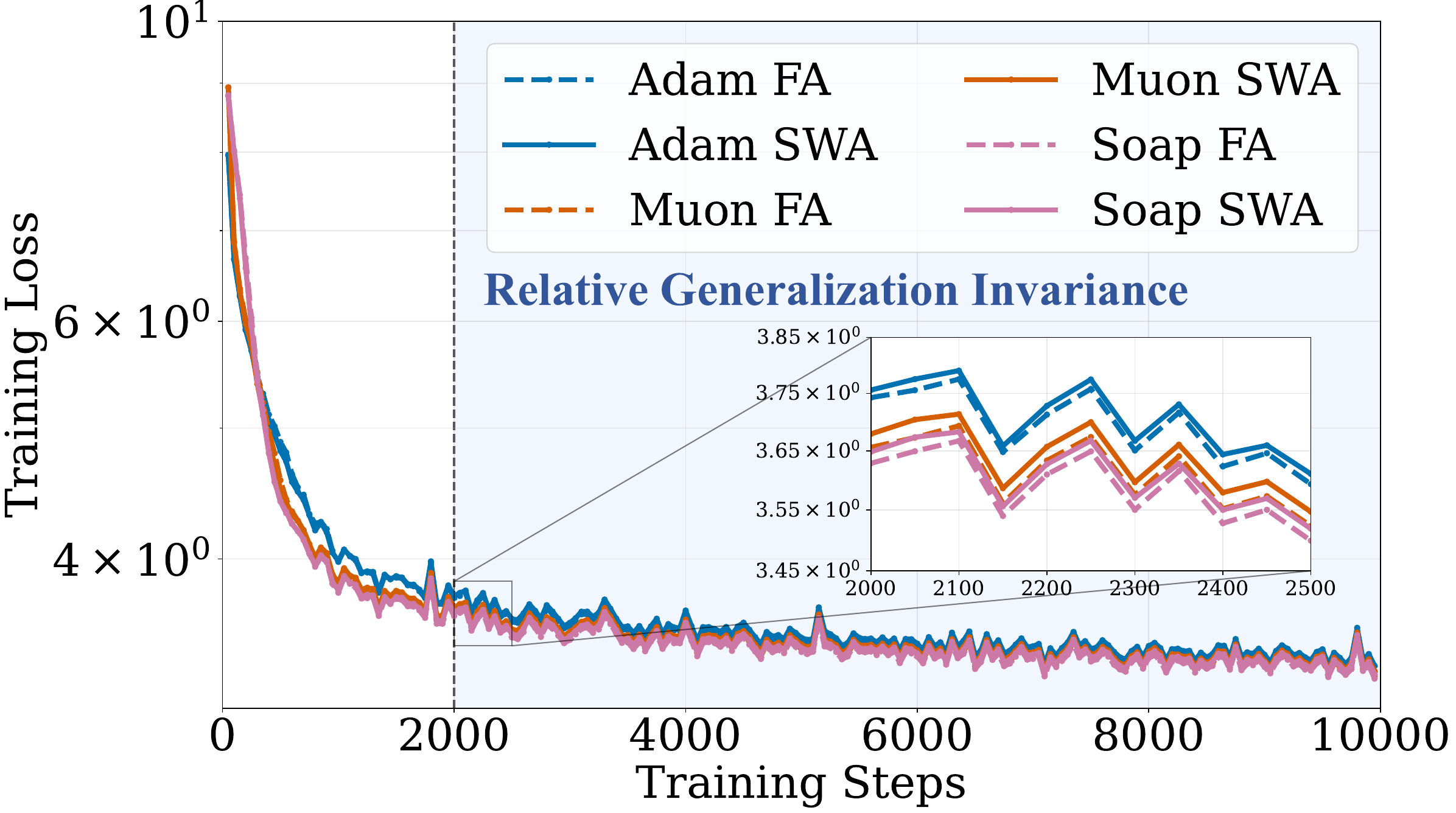}\label{fig:loss_cur}}
    \subfigure[Relative Generalization Invariance]{ \includegraphics[width=0.39\textwidth]{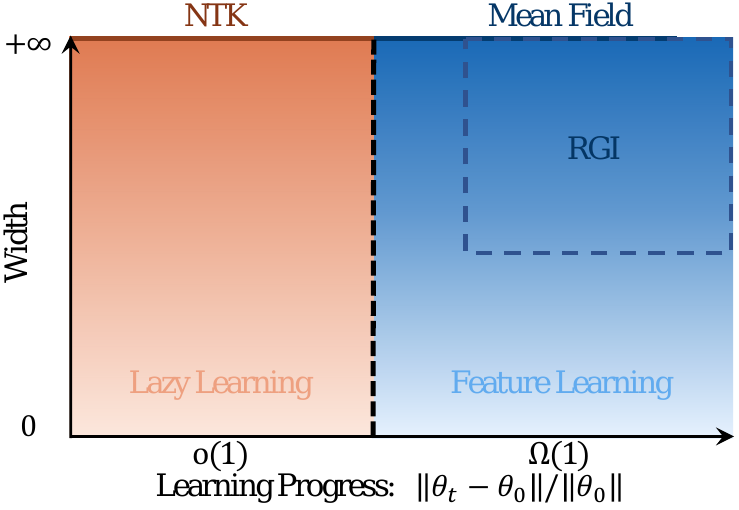}\label{fig:rgi}}
    % % \vspace{-0.5em}
    \caption{\ac{rgi} emerges across optimizers and architectures and is not explained by the \ac{ntk} or \ac{mf} regimes alone. Panel (a) shows that training-loss curves across optimizers and architectures exhibit similar shapes after $2{,}000$ steps. Panel (b) shows that \ac{rgi} persists under substantial feature learning and at reduced model widths, providing evidence beyond \ac{ntk}- and \ac{mf}-based explanations. }
    \label{fig:demo}
    % % \vspace{-1em}
\end{figure}

% \begin{wrapfigure}{r}{0.5\textwidth}
%     \centering
%     % % \vspace{-1.5em}
%     \includegraphics[width=0.49\textwidth]{figures/rgi_main.pdf}
%     \caption{\ac{rgi} cannot be explained by \ac{ntk} and \ac{mf} alone.}
%     \label{fig:rgi}
%     % % \vspace{-2em}
% \end{wrapfigure}

Pretraining equips \acp{llm} with broad knowledge and reasoning capabilities and provides the foundation for subsequent instruction tuning and reinforcement learning. Its performance is jointly shaped by the three components of the \emph{training triplet}: the optimizer, model architecture, and training data stream. Prior work has improved pretraining performance and efficiency by advancing each component, including new optimizers such as Soap~\citep{vyas2024soap} and Muon~\citep{jordan6muon}, architectural designs such as \ac{swa}~\citep{xiao2024efficient} and SiLU~\citep{elfwing2018sigmoid}, and data pipelines such as dynamic data selection~\citep{wang2026opus}.

Despite this progress, much less is known about how the optimizer, model architecture, and training data stream affect performance in distinct ways. We take a step toward isolating their effects by studying a fundamental structure of generalization, which we call \emph{relative generalization}: the relative loss structure across validation tokens. We show that optimizers and architectures affect token-wise losses in an approximately uniform way, preserving relative generalization across a wide range of optimizers and moderate architectural variations, a phenomenon we term \acf{rgi}. In contrast, changing the training data stream can substantially alter relative generalization. Our primary contributions are as follows:

First, we define \ac{rgi} between models as the invariance of the validation-loss difference between any two tokens across models. We show that \ac{rgi} is equivalent to the two models having a token-independent constant gap in validation loss, i.e., a uniform shift across tokens. We then extend \ac{rgi} to training triplets through the models they produce.

Second, we show that \ac{rgi} holds across a wide range of optimizers and moderate architectural variations, but depends strongly on the training data stream. Specifically, we consider coordinate-wise optimizers, such as Adam, matrix-based optimizers, such as Muon, and variations in attention and \ac{ffn} designs. We show that \ac{rgi} persists under joint variations in optimizer and architecture and is robust to optimizer hyperparameters and moderate shifts in the validation distribution. However, \ac{rgi} can break when models are trained on data streams from different sources. In other words, optimizer and architecture choices induce approximately uniform shifts in token-wise losses, whereas changes in the training data stream can alter the relative loss structure across tokens.

Third, we show that \ac{rgi} cannot be explained by the \acf{ntk} or \acf{mf} regimes alone. When \ac{rgi} emerges, models deviate substantially from initialization in both parameters and hidden-state representations, and different optimizers lead to low similarity in both quantities. These findings rule out an \ac{ntk}-style lazy-learning explanation. We further reduce the model width and show that \ac{rgi} persists even when the model departs from the \ac{mf} regime.

Finally, we study a stylized quadratic model and theoretically characterize the emergence of \ac{rgi}. Beyond the \ac{ntk} and \ac{mf} regimes, we identify a setting where \ac{rgi} holds across GD, Adam, and Muon, different initializations, and learning rates, despite low parameter similarity.

Together, these results take a first step toward distinguishing the effects of optimizers, architectures, and data through the definition, empirical demonstration, and theoretical characterization of \ac{rgi}.

\section{Related Work}\label{app:related_work}
\textbf{Optimizers of \acp{llm}} play a central role in pretraining and can substantially affect the resulting model capabilities, which are further shaped by post-training~\citep{yue2025does,liu2025not}. Over the past decade, a large body of work has developed and refined optimizers to improve the performance, stability, and efficiency of large-scale training. Starting from (Stochastic Gradient Descent) SGD~\citep{robbins1951stochastic}, adaptive coordinate-wise optimizers were introduced to better adapt updates to the optimization landscape, including AdaGrad~\citep{duchi2011adaptive}, RMSProp~\citep{tieleman2017rmsprop}, AdaDelta~\citep{zeiler2012adadelta}, and Adam~\citep{kingma2014adam}. Owing to its strong empirical performance in large-scale training, Adam became the default optimizer for many \ac{llm} pretraining pipelines. Subsequent works further improved Adam through techniques such as Nesterov momentum~\citep{dozat2016incorporating}, modified momentum schemes~\citep{ma2018quasi}, bias correction~\citep{reddi2019convergence}, adaptive step sizes~\citep{zhuang2020adabelief}, and decoupled weight decay~\citep{loshchilov2017decoupled}. To better exploit the matrix structure of \ac{llm} parameters, matrix-based optimizers were later proposed, including Muon~\citep{jordan6muon}, Shampoo~\citep{gupta2018shampoo}, Soap~\citep{vyas2024soap}, Dion~\citep{ahn2025dion}, and Scion~\citep{pethick2025training}. Following the introduction of Muon, several recent works further improved its performance through alternative normalization strategies and adaptive designs~\citep{li2025normuon,cheng2026trasmuon,shan2026muonrec,pang2026htmuon,deng2026rmnp,zhang2026mousse,southworth2026beyond,zhang2026muon+,liu2026muon,zhang2026adam}. Relatedly, another line of work seeks to optimize parameters by approximating Gauss--Newton updates~\citep{liu2023sophia,abreu2025potential}. Orthogonal to these directions, some works focus primarily on reducing the memory and communication cost of training, including Adafactor~\citep{shazeer2018adafactor}, Adam-mini~\citep{zhang2024adam}, GaLore~\citep{zhao2024galore}, Apollo~\citep{zhu2025apollo}, and MuonBP~\citep{khaled2025muonbp}. Comprehensive surveys of \ac{llm} optimizers are provided by \citet{he2021large,wan2023efficient,abdulkadirov2023survey,zhou2024training,luo2025survey,tian2025survey,xiao2025survey}. In contrast to these works, which propose new optimizers, we study the shared characteristics of existing optimizers. Specifically, we show that many commonly used optimizers exhibit the same relative generalization behavior.

\paragraph{Influential Factors in \ac{llm} Pretraining} \hspace{-0.8em} include model architecture, datasets, and optimizers. Following the seminal works on scaling laws~\citep{kaplan2020scaling,rosenfeld2019constructive}, a growing body of work has studied how these three factors separately and jointly affect pretraining performance. From the perspective of model design, \citet{hoffmann2022training,tao2024scaling} study the optimal allocation of compute and the role of vocabulary size as model size scales. \citet{bian2025scaling} further studies the trade-off between throughput and performance across different architectural configurations. From the perspective of optimization, \citet{zhao2024deconstructing,semenov2025benchmarking,wen2025fantastic} examine how different optimizers affect training efficiency and model performance across different data and model-size regimes. In addition, \citet{volkova2026towards} proposes a robust method for fitting scaling laws under different optimizers. A large body of work also studies how datasets influence pretraining performance. \citet{hernandez2021scaling,ye2024data,mizrahi2025language,liu2025not,shukor2025scaling,subramanyam2025scaling,mayilvahanan2025llms} investigate how pretraining data mixtures affect target-task and downstream performance through the lens of scaling laws. In contrast, \citet{kang2025demystifying} and \citet{zhao2024deciphering} study properties of data mixtures using human-controlled synthetic data and machine unlearning, respectively. Beyond these mechanism-oriented studies, another line of work investigates how to construct near-optimal data mixtures efficiently for large-scale pretraining, including DoReMi~\citep{xie2023doremi}, RegMix~\citep{liu2024regmix}, and Chameleon~\citep{xie2025chameleon}. See \citet{li2025mis} for a survey of scaling laws across different pretraining factors and \citet{luo2025survey} for a survey of data-centric \ac{llm} pretraining. Different from these works, we study the distinct roles of optimizers and data in shaping relative generalization, which holds almost surely with respect to the data stream.

\paragraph{Pretraining Dynamics} \hspace{-0.8em} of \acp{llm} have received considerable attention over the past decade, and these studies have shed light on the behavior of different optimizers throughout training. Systematic theoretical studies of pretraining dynamics typically focus on over-parameterized regimes, in which the number of trainable parameters is much larger than the number of training samples~\citep{du2018power,allen2019learning}. Among these settings, two regimes have been studied most extensively: the \ac{ntk} regime and the \ac{mf} regime. The \ac{ntk} regime is also known as the lazy-training regime~\citep{chizat2019lazy}, in which the kernel induced by the network remains effectively fixed throughout training~\citep{jacot2018neural,lee2019wide}. Under suitable scalings of model width, initialization, and learning rate, a line of work establishes convergence rates to global optima in the \ac{ntk} regime~\citep{arora2019exact,novak2019neural,geiger2020disentangling,liu2020linearity}. In addition, \citet{arora2019fine} explicitly derives the kernel on the training data and shows how convergence depends on its properties. See \citet{golikov2022neural} for a survey of the \ac{ntk} regime. To relax the restrictive assumption that the kernel remains fixed during training, the \ac{mf} regime enables theoretical analysis of an evolving kernel via mean-field approximation~\citep{mei2018mean,sirignano2020mean,takakura2024mean}. Rather than tracking each neuron individually, the mean-field approximation represents neurons collectively through a distribution. This framework was developed rigorously by \citet{mei2019mean} and \citet{sirignano2020mean} for two-layer networks and later extended to multi-layer networks~\citep{sirignano2022mean}. A line of work further extends mean-field analysis from fully connected networks to transformers~\citep{kim2024transformers,rigollet2025mean,poc2024dynamical}. Motivated by mean-field analysis, $\mu$P was proposed to preserve nontrivial parameter updates in the infinite-width limit~\citep{yang2021tensor,yang2021tuning}. Its analysis supports hyperparameter transfer across models of different sizes~\citep{lingle2024empirical,blake2024u,hayou2025proof,zheng2026spectral}. Different from these works, the relative-generalization invariance behavior of optimizers studied by our work is not explained by either the \ac{ntk} or the \ac{mf} perspective.
% \vspace{-1em}
\section{Preliminaries}\label{sec:prelim}

\textbf{Pretraining} of \acp{llm} optimizes model parameters over a stream of data sampled from a corpus. Let the initial parameters $\theta_0\in\Theta$ be drawn from an initialization distribution $P_{\init}\in\Delta(\Theta)$. Starting from $\theta_0$, the model is trained on a data stream $D_{1:\infty}=(D_t)_{t=1}^{\infty}$, where each batch $D_t\subseteq\calD$ is sampled from the corpus $\calD$ according to a possibly time-dependent distribution $P_t(\cdot\mid D_{1:t-1})$, independent of $\theta_0$. Each data point $d\in D_t$ consists of a prefix and a target token, e.g., $d=($``LLM pretraining is'', ``important''$)$. Given an optimizer $\sfA\in\calA$, we define the \emph{training triplet} as $\frakT=(\sfA,\theta_0,D_{1:\infty})$, where $\theta_0$ specifies the initialization and its associated model architecture. The parameters at step $t$ can then be written as
$
\theta_t^{\frakT}=F(\frakT,t),
$
where $F$ is deterministic given the training triplet.\footnote{Our work focuses on deterministic optimizers. Randomized optimizers can be incorporated by treating the algorithmic randomness as an additional input to $F$~\citep{mandt2017stochastic,jin2017escape}.} The map $F$ is induced by the optimizer's update rule. For a batch $D\subseteq\calD$, let
$
L(\theta,D)=|D|^{-1}\sum_{d\in D}L(\theta,d)
$
denote the average cross-entropy loss, and define the batch gradient as
$
\xi_t^\frakT=\nabla_{\theta}L(\theta_{t-1}^\frakT,D_t).
$
The optimizer $\sfA$ recursively updates its internal state $S_t^\frakT\in\calS$ and the model parameters according to
\begin{align*}
S_t^\frakT=G(S_{t-1}^\frakT,\xi_t^\frakT,\sfA),\quad
\theta_t^\frakT=H(S_t^\frakT,\theta_{t-1}^\frakT,\sfA),
\end{align*}
where $G$ and $H$ denote the state-update and parameter-update maps, respectively.

\textbf{Optimizers} instantiate the update maps $G$ and $H$, and hence the training map $F$, differently. Coordinate-wise optimizers, such as Adam~\citep{kingma2014adam}, Lion~\citep{chen2023symbolic}, and Adam-mini~\citep{zhang2024adam}, use entry-wise gradient statistics. For the training triplet $\frakT_{\adam}=(\adam,\theta_0,D_{1:\infty})$, Adam maintains first- and second-moment states $S_t^{\adam}=S_t^{\frakT_{\adam}}=(m_t,v_t)\in\calS_{\adam}$ and updates
\begin{align*}
    m_t = \gamma_1 m_{t-1} \!+\! (1-\gamma_1)\xi_t,
    v_t = \gamma_2 v_{t-1} \!+\! (1-\gamma_2)\xi_t \odot \xi_t,
    \theta_t^{\adam} = \theta_{t-1}^{\adam}\! -\!\eta\, m_t \odot (\sqrt{v_t}+\varepsilon)^{-1},
\end{align*}
where $\odot$ denotes the Hadamard product and $\gamma_1,\gamma_2,\eta,\varepsilon$ are hyperparameters; bias correction is omitted for simplicity. Here and below, we abbreviate $S_t^{\frakT_{\adam}}$ and $\theta_t^{\frakT_{\adam}}$ as $S_t^{\adam}$ and $\theta_t^{\adam}$. In contrast, matrix-based optimizers, such as Muon~\citep{jordan6muon}, Soap~\citep{vyas2024soap}, Scion~\citep{pethick2025training}, and Shampoo~\citep{gupta2018shampoo}, exploit the matrix structure of parameter blocks. For example, an idealized Muon update for each matrix-shaped block is
\begin{align*}
    m_t = \gamma m_{t-1} + (1-\gamma)\xi_t,\quad
    \theta_t^{\muon} = \theta_{t-1}^{\muon} - \eta U_tV_t^\top,
\end{align*}
where $S_t^{\muon}=m_t$, $U_t\Sigma_tV_t^\top$ is the \ac{svd} of $m_t$, and $\gamma,\eta$ are hyperparameters. In practice, Muon approximates $U_tV_t^\top$ using Newton--Schulz iterations. See Section~\ref{app:related_work} for further discussion of these optimizers.

\section{Relative Generalization Invariance}\label{sec:rgi}
In this section, we define a new relationship between the validation losses of \acp{llm} pretrained with different optimizers, model architectures, and data streams. Specifically, consider two training triplets $\frakT=(\sfA,\theta_0,D_{1:\infty})$ and $\frakT'=(\sfA',\theta_0',D_{1:\infty}')$, which produce parameters $\theta_t^{\frakT}$ and $\theta_t^{\frakT'}$ at step $t$. We allow $\theta_0$ and $\theta_0'$ to correspond to different model architectures and, in particular, models of different sizes. In our experiments, $(\theta_0,D_{1:\infty})$ and $(\theta_0',D_{1:\infty}')$ are either identical or independently sampled. We compare the resulting models on subsets $D_{\val}\subseteq\calD_{\val}$, with
$L(\theta_t^{\frakT},D_{\val})=|D_{\val}|^{-1}\sum_{d\in D_{\val}}L(\theta_t^{\frakT},d),$
and analogously for $L(\theta_t^{\frakT'},D_{\val})$. For generality, we define \ac{rgi} under an arbitrary joint distribution over the two initializations and training data streams.

\begin{definition}[Relative Generalization Invariance]\label{def:rgi}
    \acf{rgi} holds between two parameters $\theta,\theta^{\prime}\in\Theta$  over $\calD_{\val}$ if
    \begin{align}
        L(\theta,d)-L(\theta,d^{\prime}) = L(\theta^{\prime},d)-L(\theta^{\prime},d^{\prime})\text{ for any }d,d^{\prime}\in\calD_\val.\label{eq:rgi}
    \end{align}
    For two training triplets $\frakT=(\sfA, \theta_0, D_{1:\infty})$ and $\frakT^{\prime}=(\sfA', \theta_0^{\prime}, D_{1:\infty}^{\prime})$ and their parameters at step $t$, $
    \theta_t^{\frakT}=F(\frakT,t)$ and $\theta_t^{\frakT'}=F(\frakT^{\prime},t)$,
     we say that \ac{rgi} holds across $\frakT$ and $\frakT^{\prime}$ over $\calD_\val$ at time step $t$ if
    \begin{align}
        L\big(\theta_t^{\frakT},d\big)-L\big(\theta_t^{\frakT},d^{\prime}\big) = L\big(\theta_t^{\frakT^{\prime}},d\big)-L\big(\theta_t^{\frakT^{\prime}},d^{\prime}\big)\label{eq:rgi_o}
    \end{align}
    for any $d,d^{\prime}\in\calD_\val$, almost surely with respect to the joint distribution of two initializations and training data streams. We say that \acf{argi} holds between training triplets  $\frakT$ and $\frakT^{\prime}$ if there exists $t^*>0$ such that \ac{rgi} holds between them for any $t\geq t^*$.
\end{definition}
The notion of \ac{rgi} between parameters requires that the validation-loss difference between any two tokens $d$ and $d'$ is the same under $\theta$ and $\theta'$. Thus, $\theta$ and $\theta'$ preserve the same relative token-wise loss structure; in particular, if token $d'$ incurs a higher loss than token $d$ under $\theta$, then the same holds under $\theta'$. At the level of training triplets, \ac{rgi} between $\frakT$ and $\frakT'$ at step $t$ requires this property to hold between the resulting parameters almost surely with respect to the joint randomness in their initializations and training data streams. Equivalently, their token-wise validation losses differ only by a token-independent constant. This characterization is formalized in the following result.

\begin{proposition}[Equivalent Statements of \ac{rgi}]\label{prop:equivalence}
    The following statements are equivalent.
    \begin{itemize}[leftmargin=2em]
        \item (\ac{rgi})  \ac{rgi} holds between training triplets $\frakT$ and $\frakT^{\prime}$ at time step $t$.
        \item (Constant Generalization Gap) The training triplets $\frakT$ and $\frakT^{\prime}$ have a constant generalization gap at time $t$, i.e., $
            L(\theta_t^{\frakT},d)-L(\theta_t^{\frakT^{\prime}},d) =C(\frakT,\frakT^{\prime},t) \text{ for any }d\in\calD_{\val}  \text{ almost surely.}$ 
    \end{itemize}
\end{proposition}
This result shows that \ac{rgi} is equivalent to a token-independent constant gap in validation loss, motivating two complementary metrics for evaluating \ac{argi}. First, motivated by the constant-gap characterization, we evenly partition $\calD_{\val}$ into $K$ non-overlapping subsets $(D_{\val,k})_{k=1}^{K}$ and compute the \ac{ccc} between the centered values of $L(\theta_t^{\frakT},D_{\val,k})$ and $L(\theta_t^{\frakT'},D_{\val,k})$ across $k$~\citep{lawrence1989concordance}, denoted by $\cccerr(\frakT,\frakT')\in[-1,1]$.  Larger $\cccerr$ indicates closer agreement with \ac{rgi}, with $\cccerr=1$ implying that the two sets of validation losses differ only by a constant shift.  Second, we define the loss difference $\Delta(\frakT,D_{\val},D_{\val}^{\prime})=L(\theta_t^{\frakT},D_{\val})-L(\theta_t^{\frakT},D_{\val}^{\prime})$.
When $|D_{\val}|=|D_{\val}^{\prime}|=1$, \ac{rgi} is equivalent to
$\Delta(\frakT,D_{\val},D_{\val}^{\prime})=\Delta(\frakT',D_{\val},D_{\val}^{\prime})$.
We define $\differr(\frakT,\frakT')$ as the normalized discrepancy between the corresponding loss differences under $\frakT$ and $\frakT'$. Smaller $\differr$ indicates closer alignment with \ac{rgi}; when each subset contains a single token, $\differr=0$ implies exact \ac{rgi}. These metrics are normalized to facilitate the evaluation of \ac{rgi}; see Appendix~\ref{app:alter} for further discussion. Their explicit definitions are provided in Appendix~\ref{app:exp_detail}.

% \begin{figure}[t]
%     \centering
%     \subfigure[Training losses of optimizers.]{ \includegraphics[width=0.45\textwidth]{figures/loss_vs_steps.pdf}\label{fig:train_loss}}
%     % \hspace{2em}
%     \subfigure[Average values of $|\alpha_t|$ and $|\beta_t|$.]{ \includegraphics[width=0.49\textwidth]{figures/ckpt_demo_alpha_beta_bar.pdf}\label{fig:bar}}
%     \caption{(a) Training losses of different optimizers with identical initial parameters and the same training data stream exhibit similar shapes after sufficient training steps. (b) Average values of $|\alpha_t|$ and $|\beta_t|$ for $2000\leq t<10000$, showing that $|\beta_t|$ is more than $40$ times larger than $|\alpha_t|$.} 
% \end{figure}

\section{Empirical Studies of Relative Generalization Invariance}\label{sec:emp_studies}
In this section, we investigate when \ac{rgi} holds between two training triplets $\frakT$ and $\frakT^\prime$. Specifically, we examine four sources of variation in $L(\theta_t^{\frakT},\hteal{D_{\val}})$, where $\frakT=(\hblue{\sfA},\hyellow{\theta_0},\horange{D_{1:t}})$: the optimizer $\hblue{\sfA}$, the model architecture and initialization associated with $\hyellow{\theta_0}$, the training data stream $\horange{D_{1:t}}$, and the validation data $\hteal{D_{\val}}$.

\begin{figure}[t]
    \centering
    \subfigure[Validation losses of Adam and Muon]{ \includegraphics[width=0.25\textwidth]{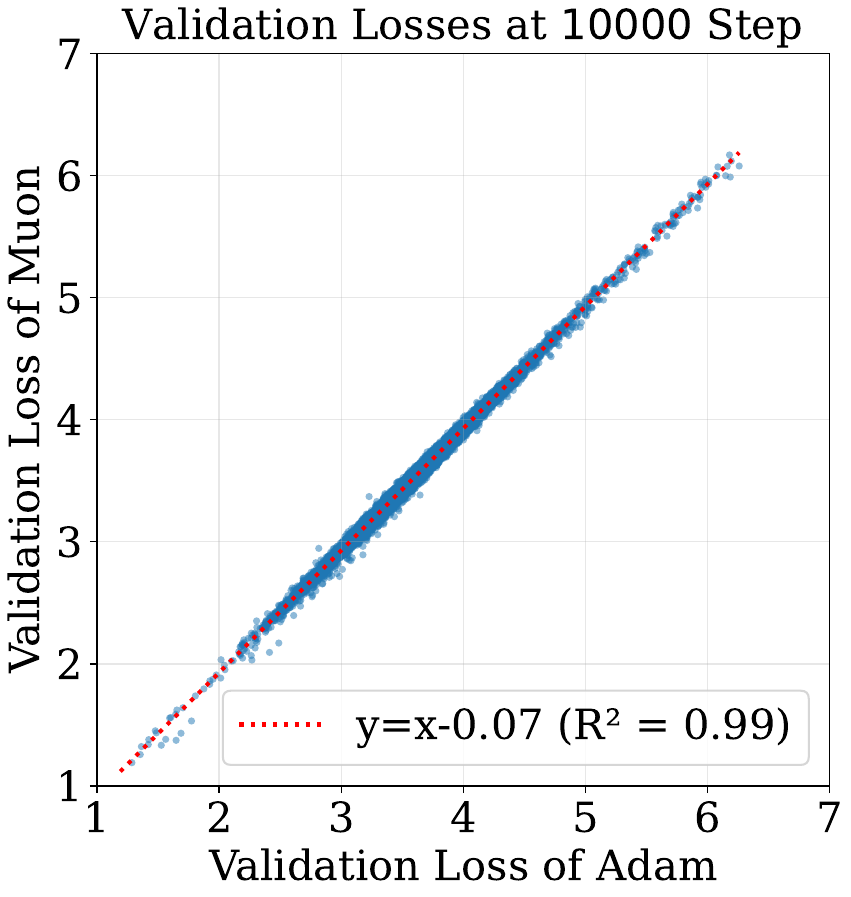}\label{fig:val_muon}}
    \subfigure[Values of $\cccerr$ across training]{ \includegraphics[width=0.35\textwidth]{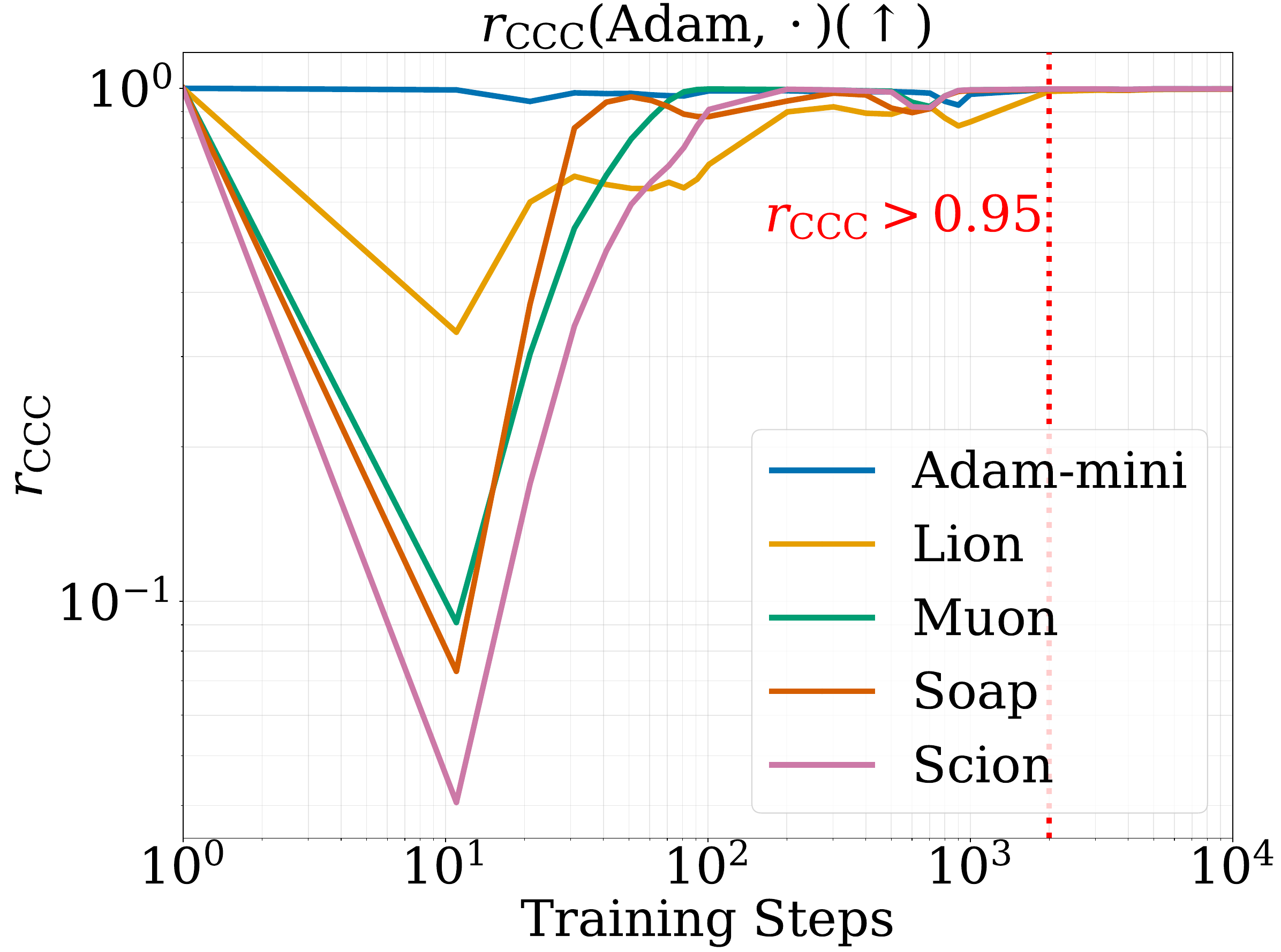}\label{fig:r_ccc_adam}}
    \subfigure[Values of $\differr$ across training]{ \includegraphics[width=0.35\textwidth]{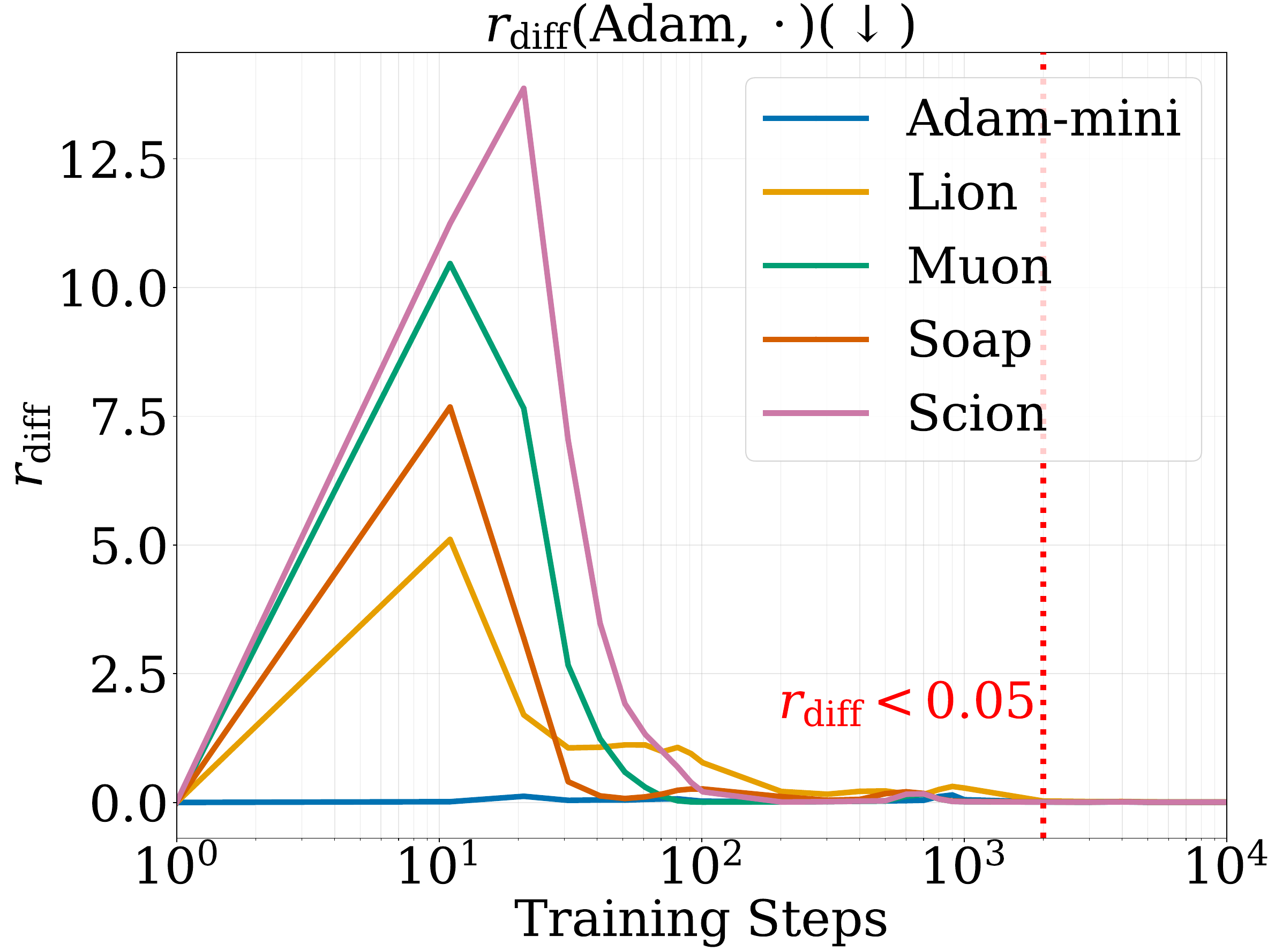}\label{fig:r_diff_adam}}
    \caption{Empirical evidence of \ac{rgi} among optimizers. Panel (a) plots the validation losses of the $124$M GPT models trained with Adam and Muon at the $10{,}000$-step checkpoint, evaluated on $20{,}480$ validation batches. Panel (b) reports $\cccerr(\text{Adam},\cdot)$ $(\uparrow)$ for the other five optimizers over training, all of which rise above $0.95$ after $2{,}000$ steps. Panel (c) reports $\differr(\text{Adam},\cdot)$ $(\downarrow)$ for the other five optimizers over training, all of which drop below $0.05$ after $2{,}000$ steps. These results show that \ac{rgi} holds across all considered optimizers after sufficient training when they are trained from the same initialization and on the same training data stream.} 
    % \vspace{-1.0em}
    \label{fig:base}
\end{figure}

\textbf{Experimental setup.} We first describe the baseline experimental setting, from which subsequent experiments vary the optimizer, model architecture, and training data. We consider three coordinate-wise optimizers, Adam~\citep{kingma2014adam}, Lion~\citep{chen2023symbolic}, and Adam-mini~\citep{zhang2024adam}, and three matrix-based optimizers, Soap~\citep{vyas2024soap}, Muon~\citep{jordan6muon}, and Scion~\citep{pethick2025training}. We use them to train $124$M, $0.7$B, and $4.5$B Nano-GPT models on FineWeb~\citep{penedo2024fineweb}, with RoPE~\citep{su2024roformer} in attention and GELU~\citep{hendrycks2016gaussian} in the \ac{ffn}. The validation set $\calD_{\val}$ is drawn from FineWeb but disjoint from the training set, so these experiments evaluate in-distribution generalization; the \ac{ood} setting is studied in Section~\ref{sec:val}. Within each model setting, all optimizers share the same initialization and training data stream. Following \citet{hu2024minicpm,wen2024understanding}, we use the Warmup-Stable-Decay learning-rate scheduler and tune the learning rate separately for each optimizer. We report the $124$M results in the main text and defer the $0.7$B and $4.5$B results, which lead to the same conclusions, to Appendix~\ref{app:add_result}. Across $10$ random seeds, we report the worst-case maximum $\differr$ and minimum $\cccerr$, providing evidence that \ac{rgi} is robust to training randomness and consistent with the almost-sure requirement in Definition~\ref{def:rgi}. Additional experimental details are provided in Appendix~\ref{app:exp_detail}.

\subsection{\ac{rgi} Holds Across \hblue{Optimizers}}\label{sec:opt}
In this section, we vary the optimizer and its hyperparameters while fixing the model initialization and training data stream, and show that \ac{argi} approximately holds across optimizers. For simplicity, we abbreviate $\differr(\frakT_{\adam},\cdot)$ and $\cccerr(\frakT_{\adam},\cdot)$ as $\differr(\adam,\cdot)$ and $\cccerr(\adam,\cdot)$, respectively.

Figure~\ref{fig:base} provides empirical evidence for approximate \ac{rgi} across optimizers with identical initial parameters and the same training data stream. Figure~\ref{fig:val_muon} shows that, after $10{,}000$ training steps, the validation losses of models trained with Muon and Adam across validation batches are well fitted by $y=x-0.07$, indicating a nearly constant generalization gap. Results for the other optimizers are deferred to Appendix~\ref{app:add_result} and yield the same conclusion. Taking Adam as the baseline, Figures~\ref{fig:r_diff_adam} and~\ref{fig:r_ccc_adam} further show that approximate \ac{rgi} emerges across optimizers after $2{,}000$ training steps, with $\differr<0.05$ and $\cccerr>0.95$, as marked by the red dotted lines. Results using other optimizers as the baseline are reported in Appendix~\ref{app:add_result} and are consistent with this finding. Overall, these results suggest that \ac{argi} emerges across optimizers on the in-distribution validation set after the early stage of training when each optimizer uses its best learning rate, while the remaining hyperparameters are set to their defaults and the initialization and training data stream are shared.

We next examine whether \ac{rgi} continues to hold under variations in optimizer hyperparameters, including learning rates, momentum coefficients, learning-rate schedulers, and weight decay. To avoid confounding changes in both models being compared, we keep Adam fixed under the setting in Figure~\ref{fig:base} and vary only the hyperparameters of the other optimizers. We first vary the learning rate of each optimizer $\sfA$ over $0.25\eta^{\sfA}$, $0.5\eta^{\sfA}$, $0.75\eta^{\sfA}$, and $1.25\eta^{\sfA}$, where $\eta^{\sfA}$ denotes its best learning rate. We exclude larger multipliers, such as $1.5\eta^{\sfA}$, because some optimizers diverge under these settings. Figures~\ref{fig:025eta} and~\ref{fig:125eta} report the $\cccerr$ results for $0.25\eta^{\sfA}$ and $1.25\eta^{\sfA}$, while the remaining results are deferred to Appendix~\ref{app:add_result}. The different times at which $\cccerr$ crosses the threshold $0.95$ show that the learning rate affects when approximate \ac{rgi} emerges. Nevertheless, $\cccerr$ eventually grows above $0.95$ for all tested learning rates. We next vary the momentum coefficient $\gamma$ over $0.85$, $0.9$, and $0.95$. Figures~\ref{fig:85mom} and~\ref{fig:95mom} show that \ac{argi} persists across these momentum values, with the $\gamma=0.9$ and corresponding $\differr$ results deferred to Appendix~\ref{app:add_result}. We further ablate the learning-rate scheduler and weight decay in Appendix~\ref{app:add_result}, obtaining the same conclusion. Overall, \ac{argi} is robust across the tested optimizer hyperparameters, although these hyperparameters can affect when it emerges.

\begin{figure}[t]
    \centering
    \subfigure[$\cccerr$ for $0.25\eta^{\sfA}$.]{ \includegraphics[width=0.23\textwidth]{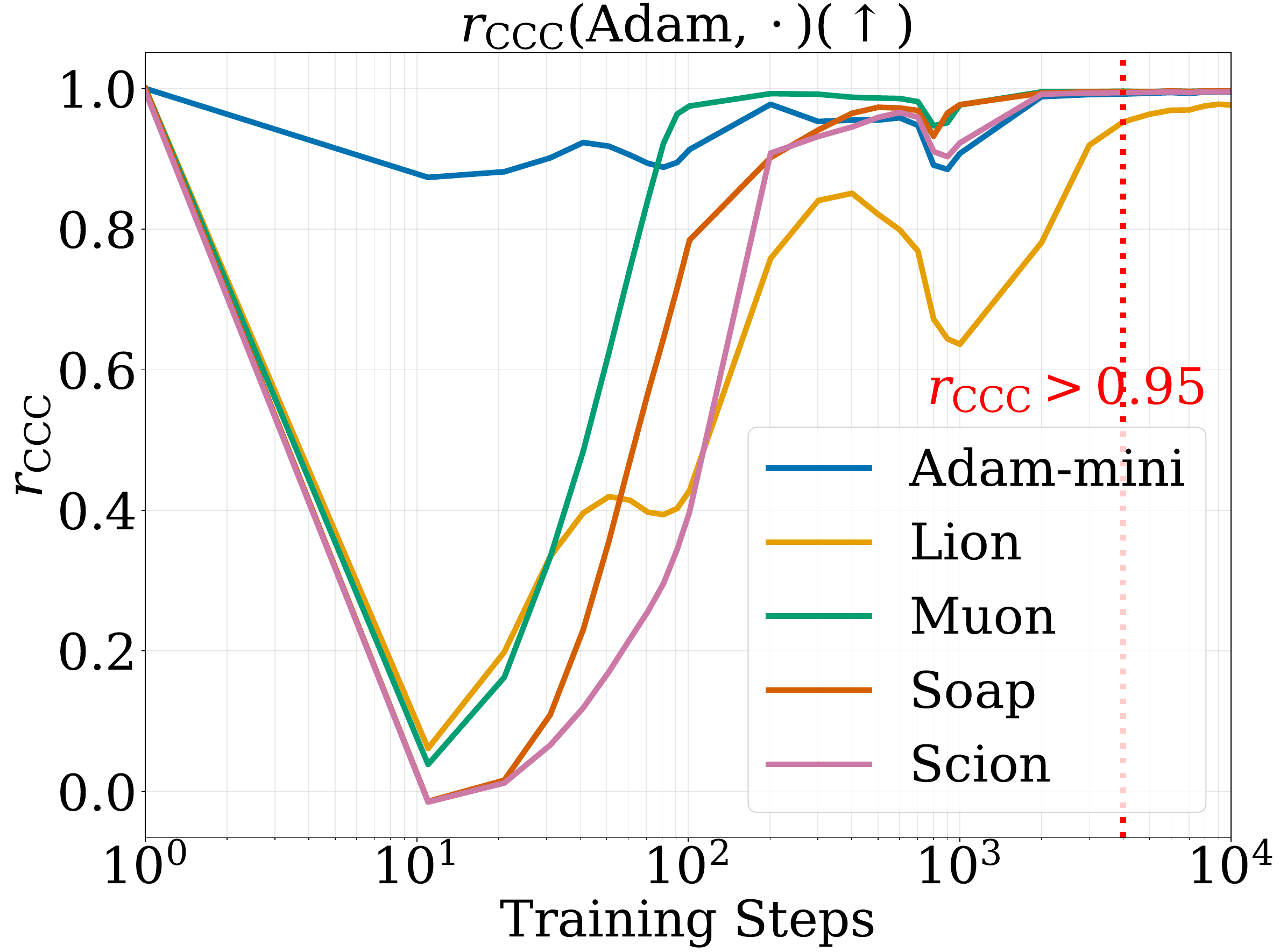}\label{fig:025eta}}
    \subfigure[$\cccerr$  for $1.25\eta^{\sfA}$.]{ \includegraphics[width=0.23\textwidth]{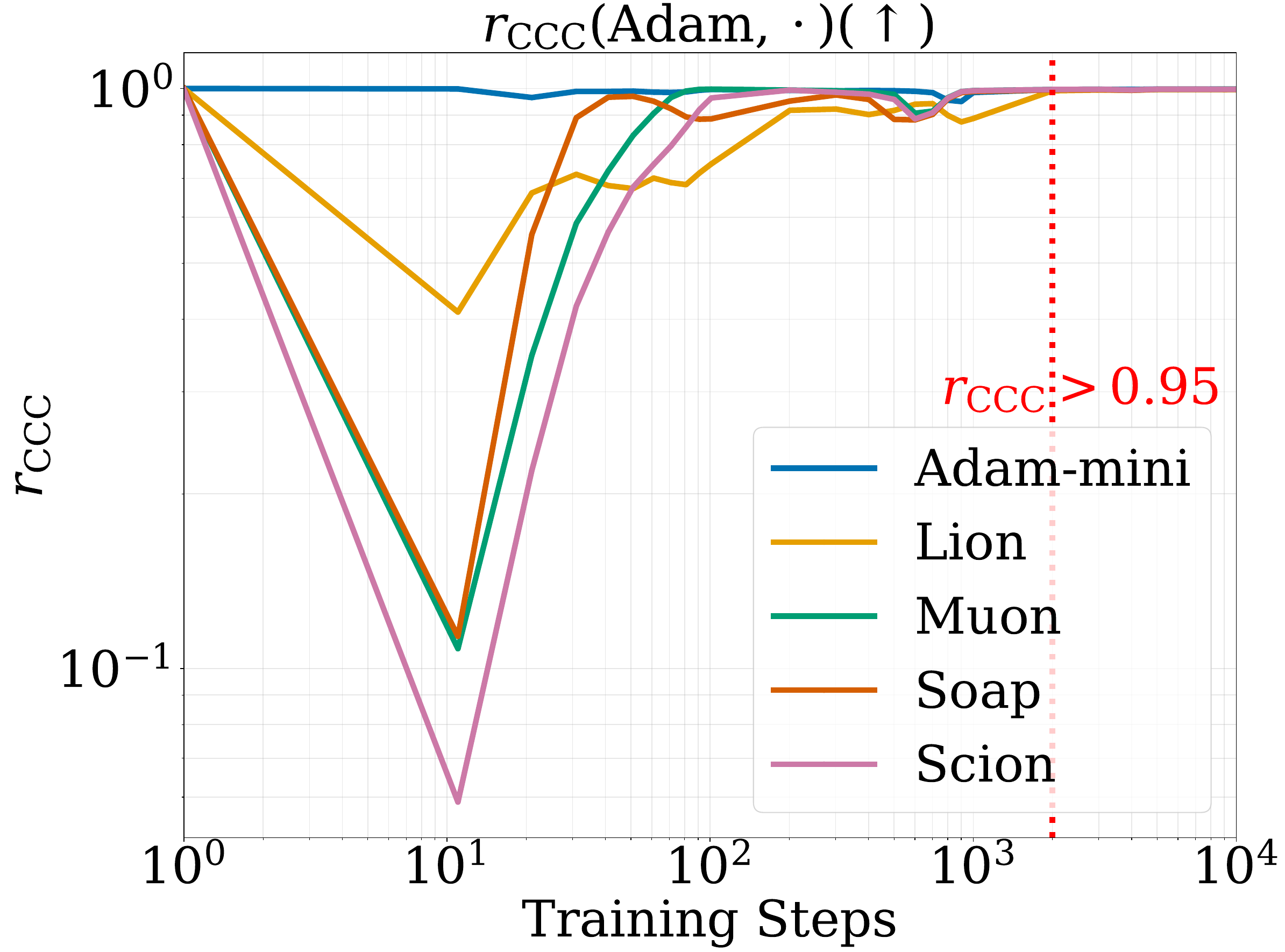}\label{fig:125eta}}
    \subfigure[$\cccerr$  for $\gamma=0.85$.]{ \includegraphics[width=0.23\textwidth]{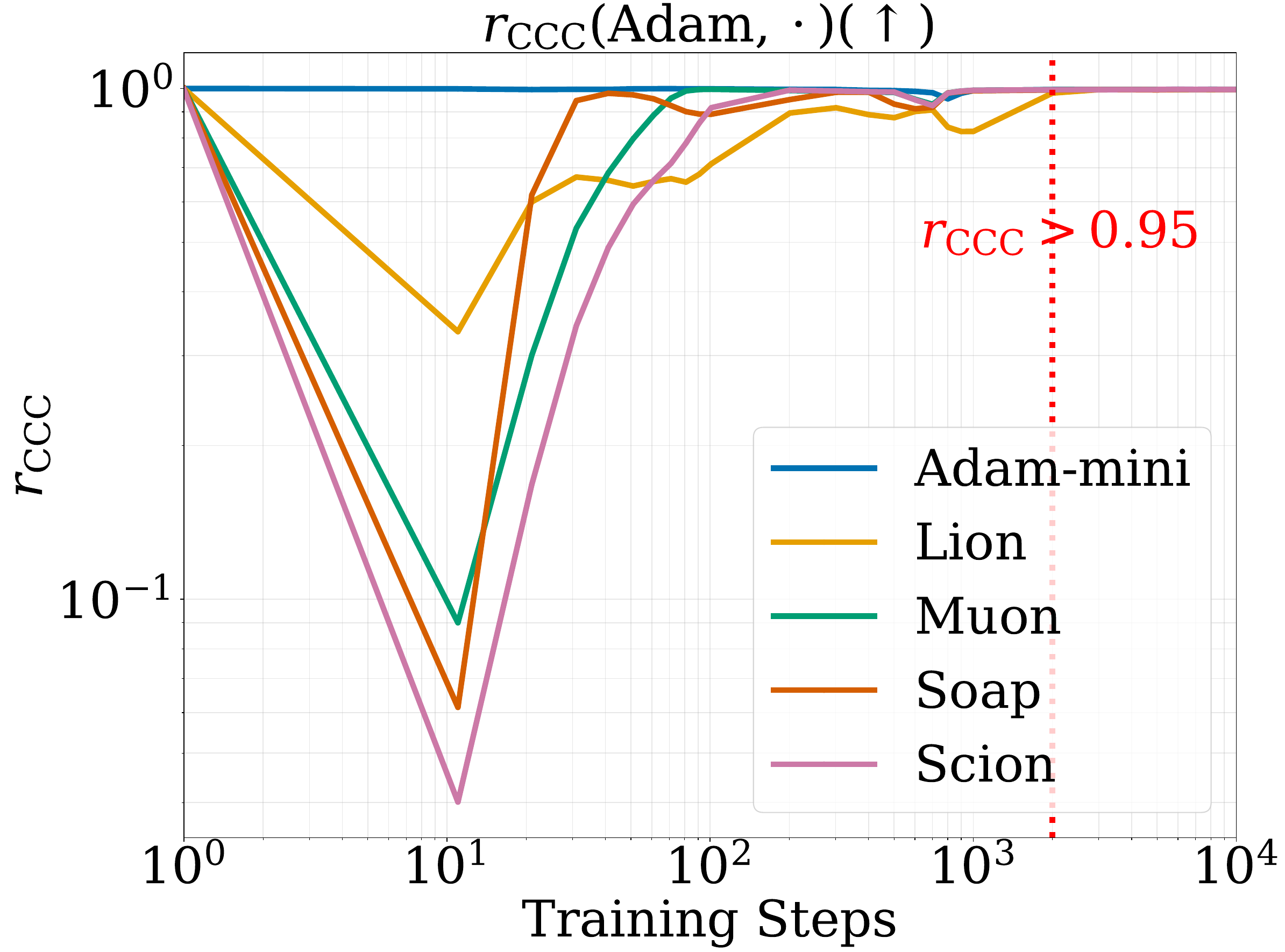}\label{fig:85mom}}
    \subfigure[$\cccerr$ for $\gamma=0.95$.]{ \includegraphics[width=0.23\textwidth]{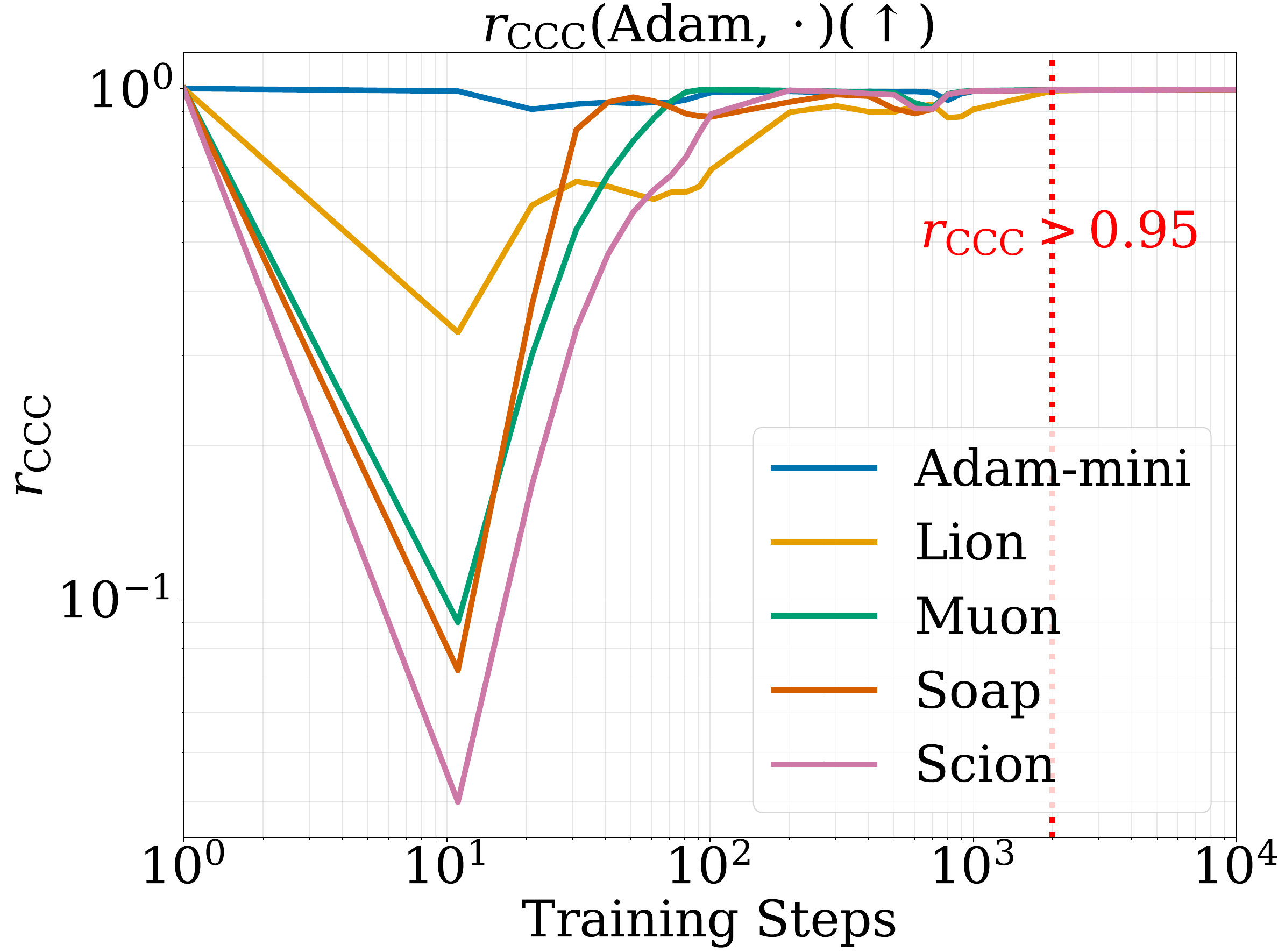}\label{fig:95mom}}
    % % \vspace{-0.5em}
    \caption{\ac{rgi} across optimizers under different hyperparameter settings. Panels (a) and (b) plot $\cccerr$ under different learning-rate multipliers, while panels (c) and (d) plot $\cccerr$ under different momentum values. \ac{rgi} persists across all tested learning-rate and momentum settings.
} 
    \label{fig:lr}
    % \vspace{-1.5em}
\end{figure}

\subsection{\ac{rgi} Holds Across Moderately Different \hyellow{Architectures}} \label{sec:arch}

In this section, we study whether \ac{argi} holds across model architectures. Consistent with Section~\ref{sec:opt}, we train models with different architectures using different optimizers, thereby examining \ac{argi} under joint variations in optimizer and architecture. Specifically, we consider two settings: (i) models with the same architectural components but different initializations and model sizes, and (ii) models with different architectural components.

We now examine whether \ac{rgi} persists across different initializations and model sizes while keeping the architectural components fixed. First, we train \acp{llm} from independent initializations $\theta_0,\theta_0^{\prime}\in\Theta$, drawn from the same initialization distribution, using $\sfA=\mathrm{Adam}$ and $\sfA^{\prime}$ chosen from Adam-mini, Lion, Muon, Soap, and Scion. Figure~\ref{fig:ccc_init} shows that, unlike the shared-initialization setting in Figure~\ref{fig:base}, $\cccerr$ initially deviates substantially from \ac{rgi}; nevertheless, approximate \ac{rgi} emerges after about $2{,}000$ training steps. Second, we vary the model size from the baseline hidden-state dimension $768$ to $384$ and $64$. Figures~\ref{fig:ccc_cross_384} and \ref{fig:ccc_cross_64} show that \ac{rgi} continues to hold between models with dimensions $768$ and $384$, but fails when the dimension is reduced to $64$. The corresponding $\differr$ results in Appendix~\ref{app:add_result} lead to the same conclusion.

\begin{figure}[t]
    \centering
    % \vspace{-1.5em}
    \subfigure[$\cccerr$ under different initial parameters.]{
        \includegraphics[width=0.31\textwidth]{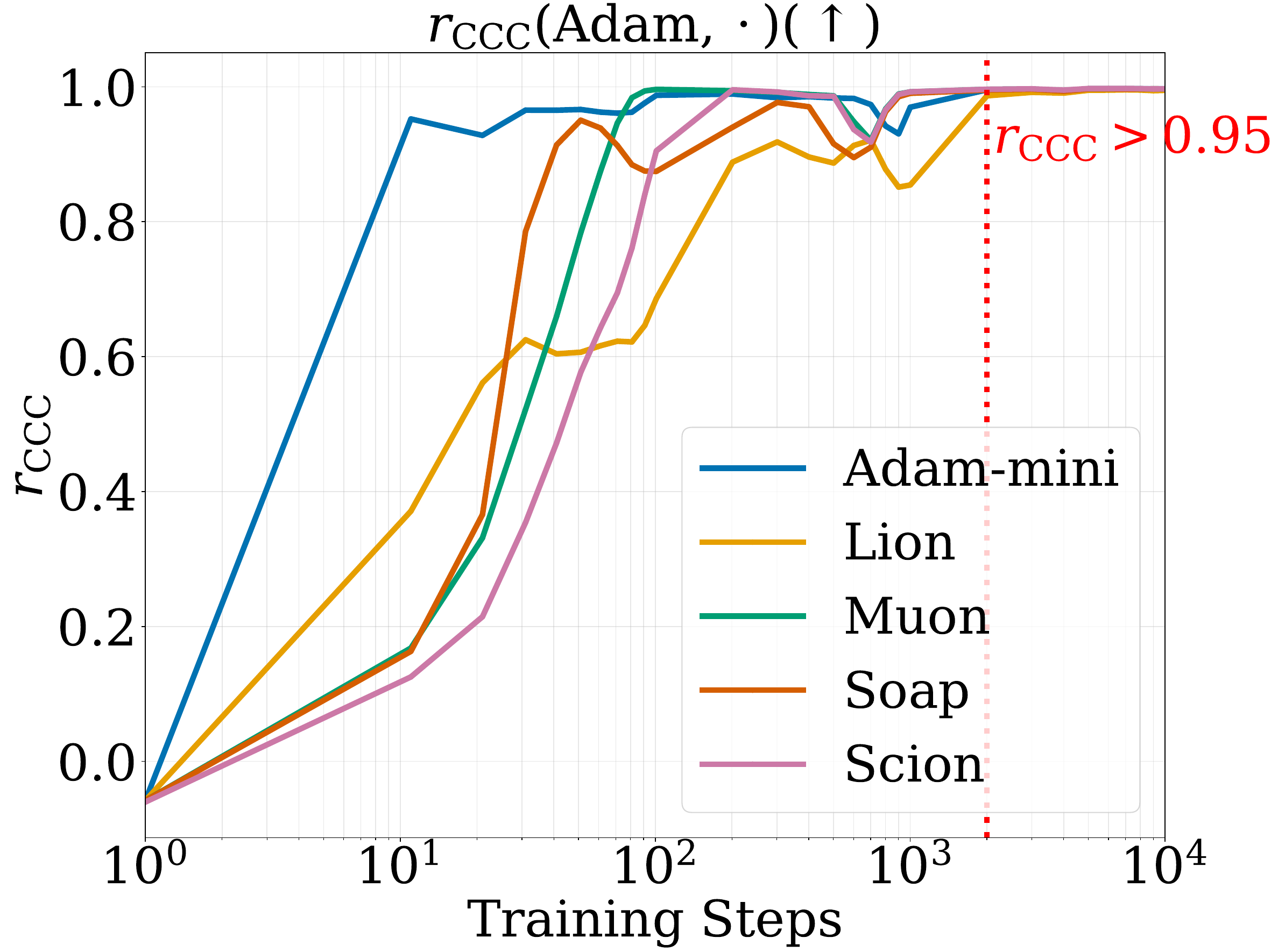}
        \label{fig:ccc_init}
    }
    \subfigure[$\cccerr$ for hidden state dimensions $768$ and $384$.]{
        \includegraphics[width=0.31\textwidth]{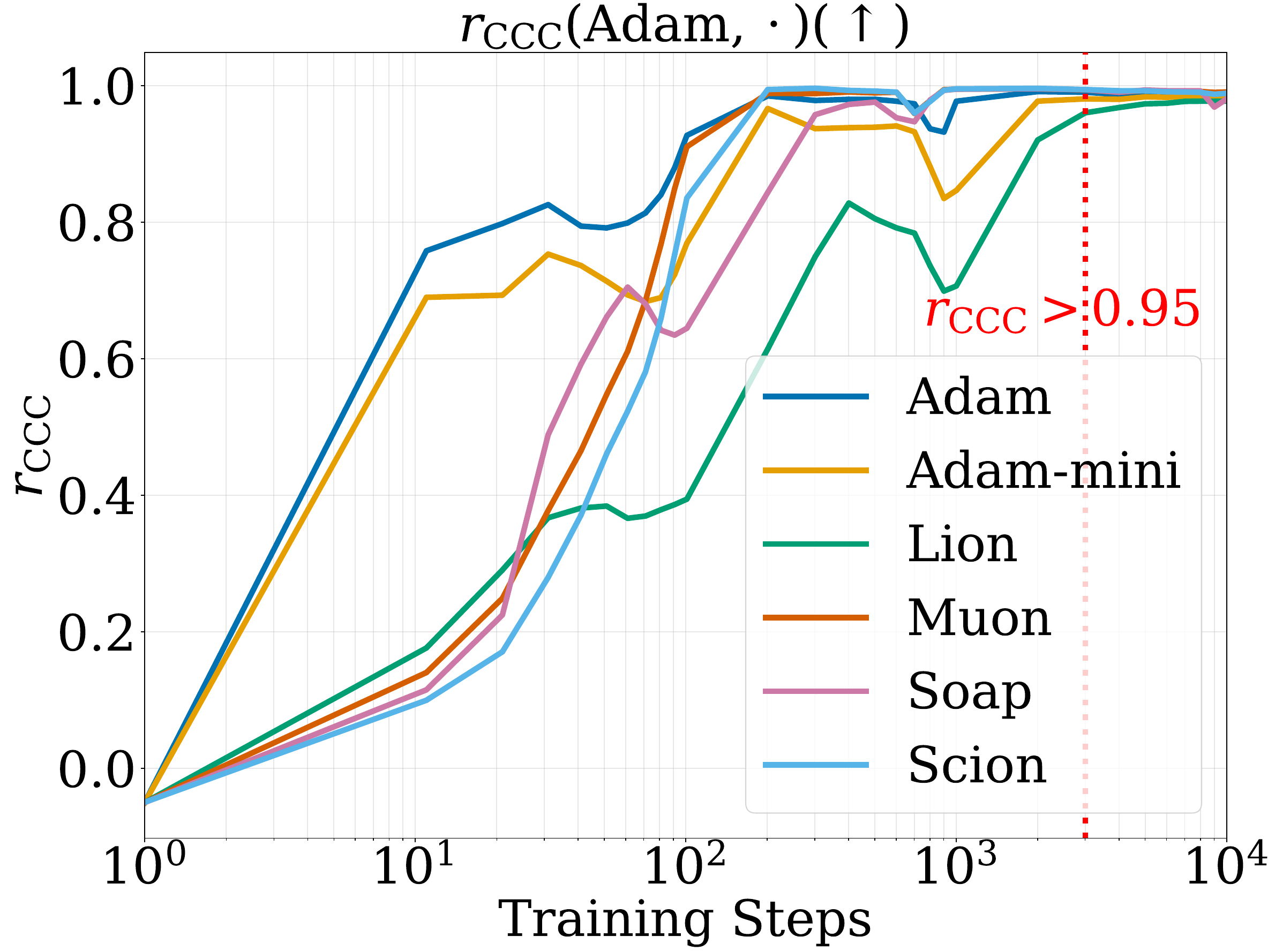}
        \label{fig:ccc_cross_384}
    }
    \subfigure[$\cccerr$ for hidden state dimensions $768$ and $64$.]{
        \includegraphics[width=0.31\textwidth]{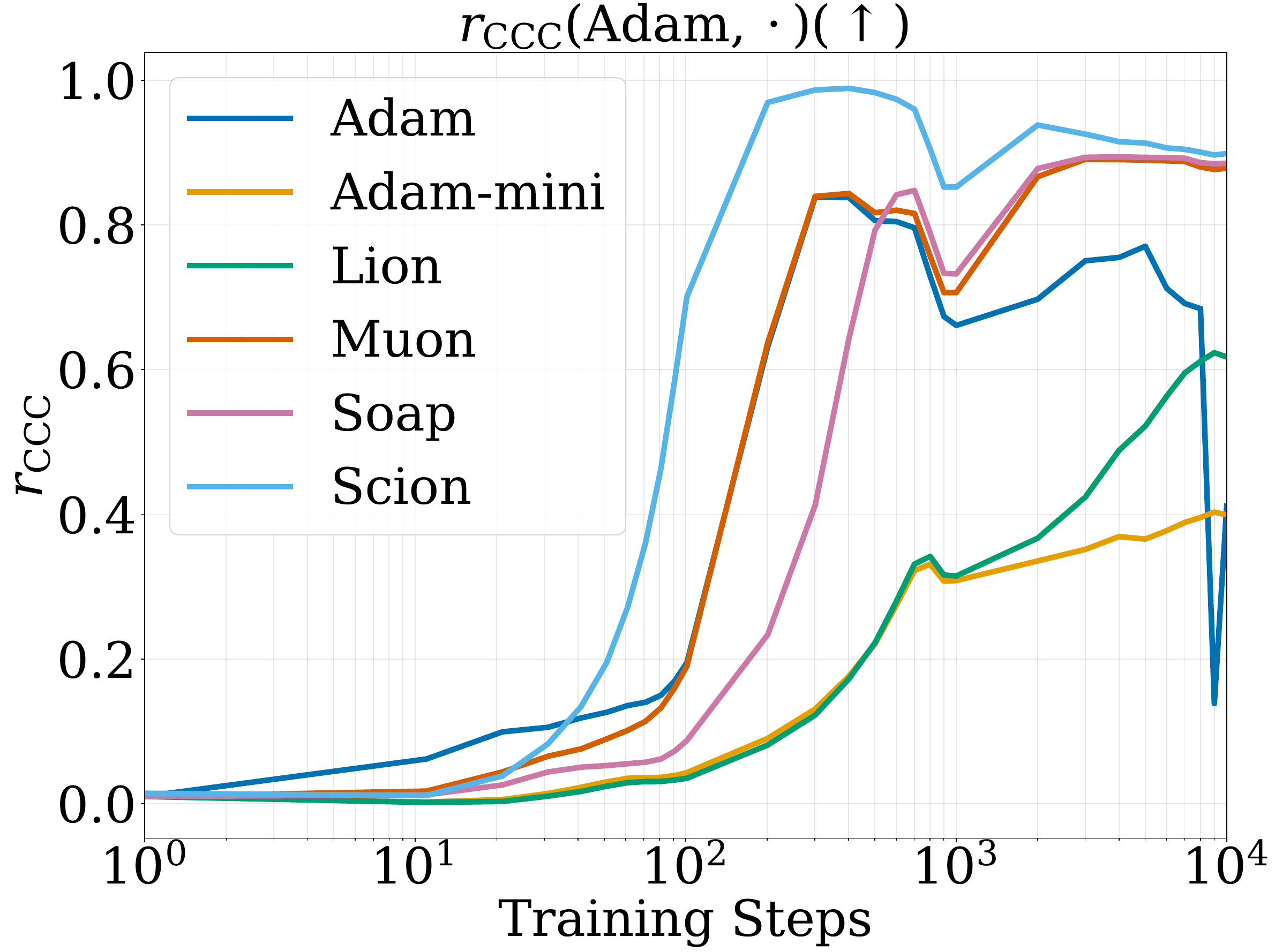}
        \label{fig:ccc_cross_64}
    }
    % \vspace{-0.5em}
    \caption{\ac{rgi} across optimizers, initializations, and model sizes. Panel (a) plots $\cccerr$ across optimizers with different initializations. Panels (b) and (c) plot $\cccerr$ across models with hidden-state dimensions $768$, $384$, and $64$. \ac{rgi} continues to hold across different initializations and moderate changes in model size.}
    \label{fig:same_arch}
    % \vspace{-2.0em}
\end{figure}
We next examine whether \ac{rgi} holds between models with different architectural components. First, we ablate the attention design. While our previous experiments use \ac{fa} with RoPE~\citep{su2024roformer}, we additionally consider \ac{fa} with NoPE~\citep{kazemnejad2023impact}, ALiBi~\citep{press2021train}, and \ac{swa} with RoPE. We train each variant with different optimizers and compare it with the baseline RoPE \ac{fa} model trained with $\adam$. Figures~\ref{fig:ccc_nope}--\ref{fig:ccc_swa} report the corresponding $\cccerr$ values and show that \ac{rgi} continues to hold across different attention designs and optimizers. Second, we ablate the \ac{ffn} design. While the baseline uses GELU~\citep{hendrycks2016gaussian}, we additionally consider Squared-ReLU (SReLU)~\citep{klusowski2016approximation} and SiLU~\citep{elfwing2018sigmoid}. We train both variants with all optimizers and compare them with the baseline GELU model trained with $\adam$. Figure~\ref{fig:ccc_srelu} reports the SReLU results, while the SiLU results and additional comparisons between gated and non-gated \acp{ffn} are deferred to Appendix~\ref{app:add_result}; all exhibit \ac{argi}. Together, these results show that \ac{rgi} persists across different attention and \ac{ffn} designs as well as optimizers.

\begin{figure}[t]
    \centering
    \subfigure[Values of $\cccerr$ between \ac{fa} models with RoPE and NoPE.]{ \includegraphics[width=0.23\textwidth]{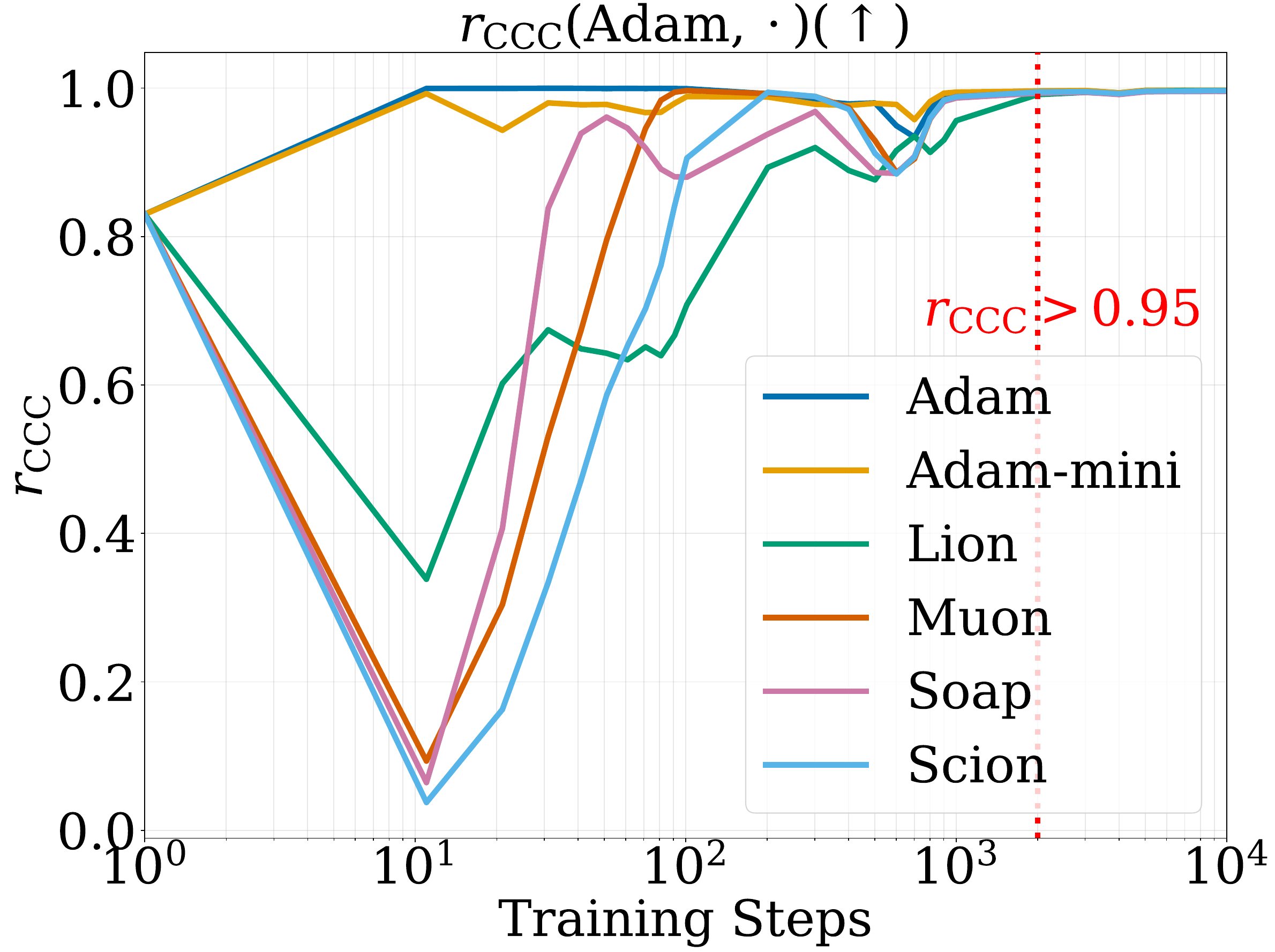}
    \label{fig:ccc_nope}
    }
    \subfigure[Values of $\cccerr$ between \ac{fa} models with RoPE and ALiBi.]{ \includegraphics[width=0.23\textwidth]{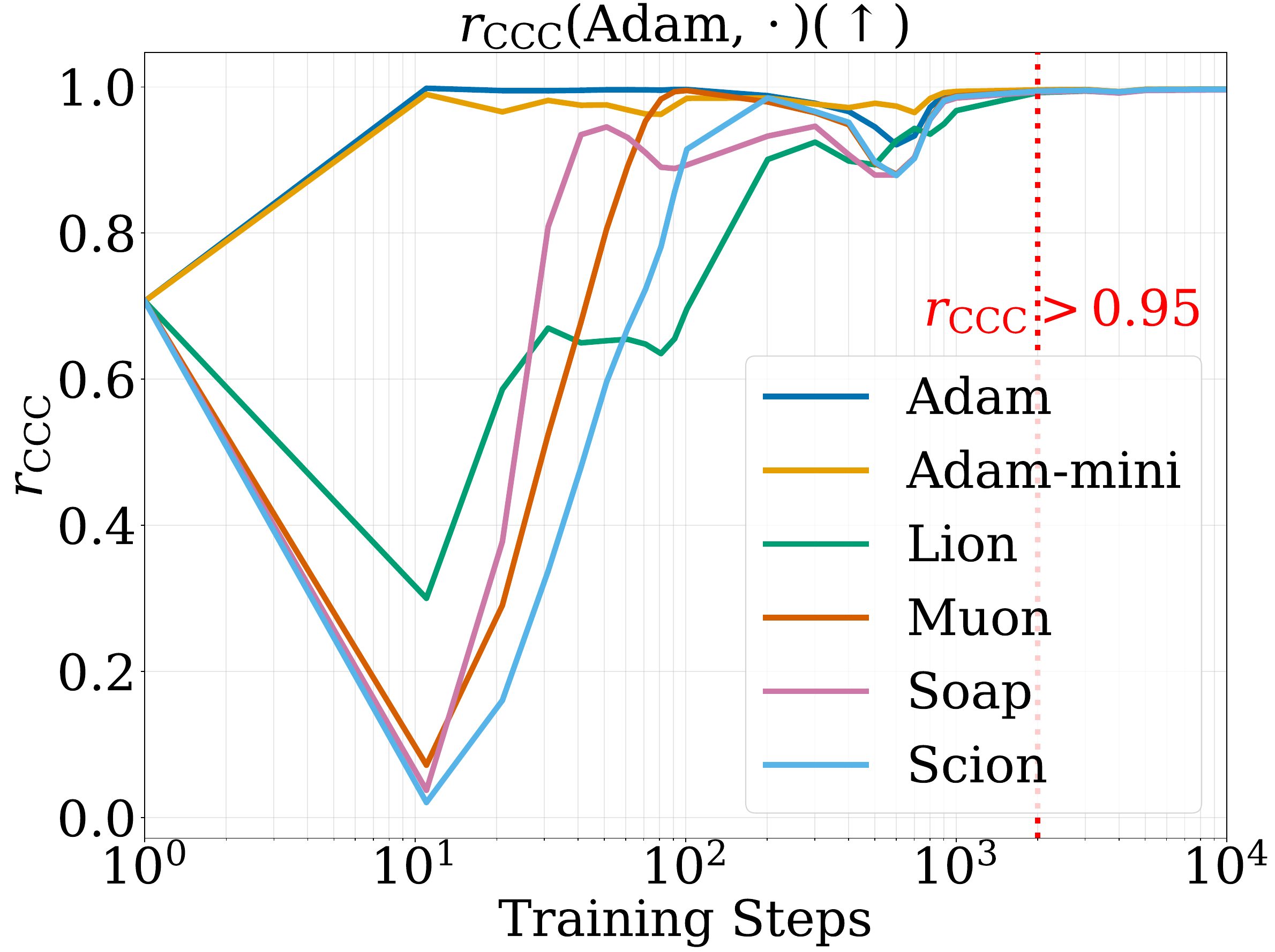}
    \label{fig:ccc_alibi}
    }
    \subfigure[Values of $\cccerr$ between \ac{fa} and \ac{swa} models.]{ \includegraphics[width=0.23\textwidth]{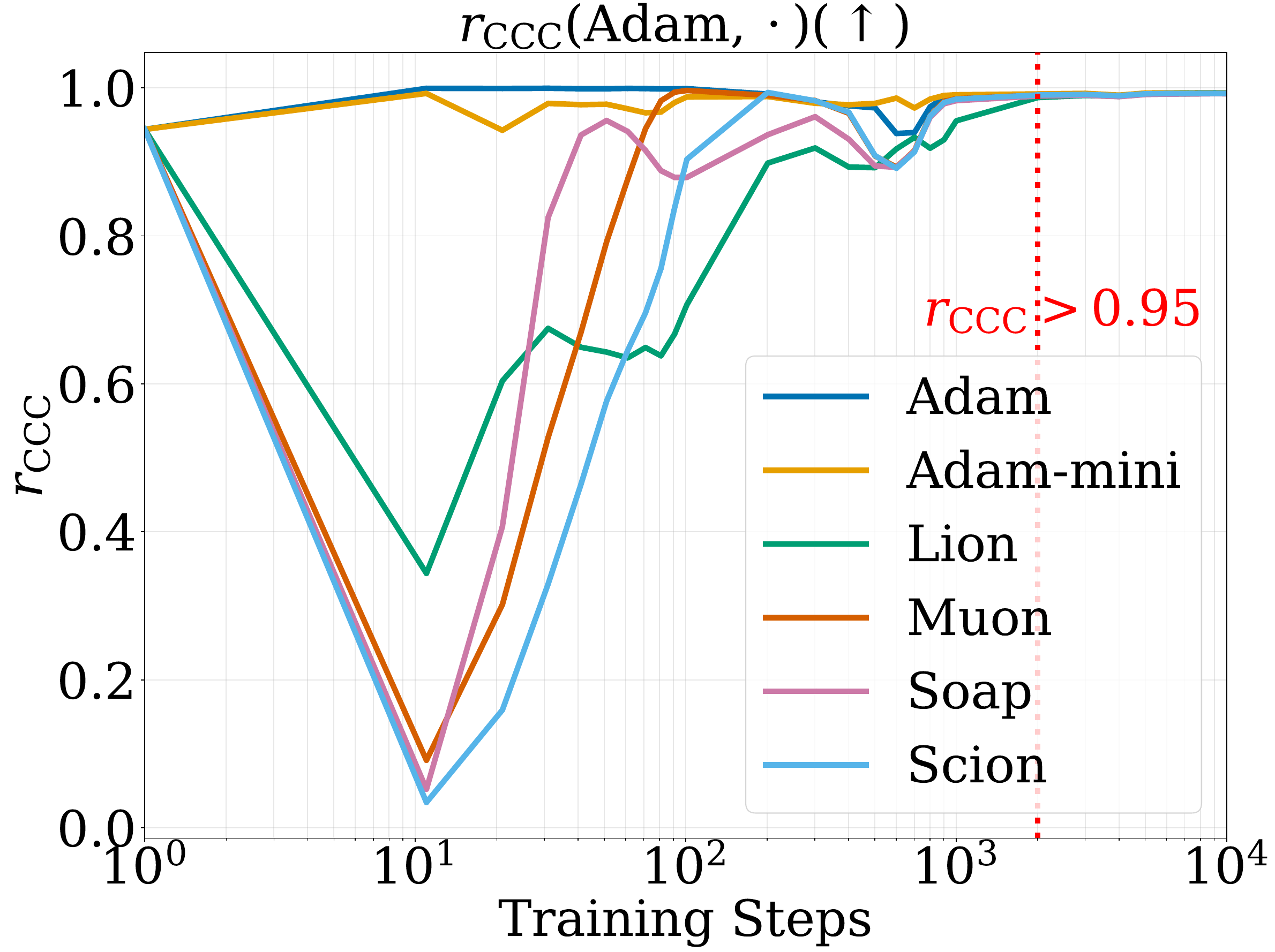}
    \label{fig:ccc_swa}
    }
    \subfigure[Values of $\cccerr$ between models with GELU and SReLU.]{ \includegraphics[width=0.23\textwidth]{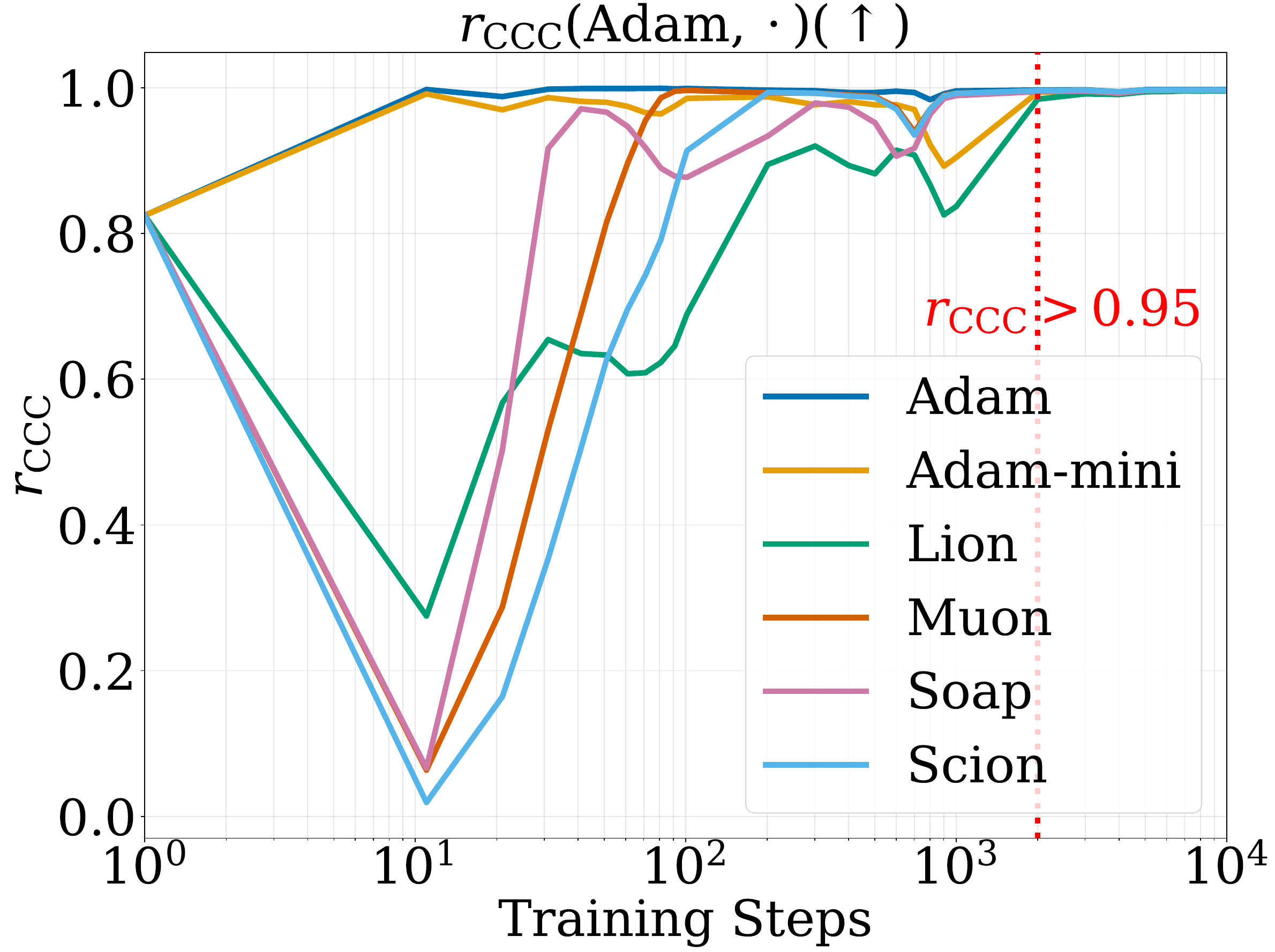}
    \label{fig:ccc_srelu}
    }
    % \vspace{-1.0em}
    \caption{\ac{rgi} across optimizers and models with different attention and \ac{ffn} modules. Panels (a)--(c) plot $\cccerr$ between models using RoPE \ac{fa} trained with $\adam$ and models using NoPE \ac{fa}, ALiBi \ac{fa}, and RoPE \ac{swa}, each trained with all optimizers, respectively. Panel (d) plots $\cccerr$ between models using GELU and SReLU activations. \ac{rgi} continues to hold across different attention and \ac{ffn} designs.} 
    % \vspace{-1.0em}
    \label{fig:diff_arch}
\end{figure}

In conclusion, \ac{argi} persists under joint variations in optimizers and moderate architectural changes, though these choices can affect its emergence time.

\paragraph{Shared Training-Loss Shapes across Optimizers and Architectures Arise from \ac{rgi}.} Training \acp{llm} with different optimizers and architectures on the same data stream often yields training-loss curves with similar shapes. As shown in Figure~\ref{fig:loss_cur}, after $2{,}000$ steps, all optimizers exhibit nearly the same training-loss profile. We show that this phenomenon can be explained by approximate \ac{rgi}. The training loss 
\begin{wrapfigure}{r}{0.5\textwidth}
    \centering
    % \vspace{-1.5em}
   \includegraphics[width=0.5\textwidth]{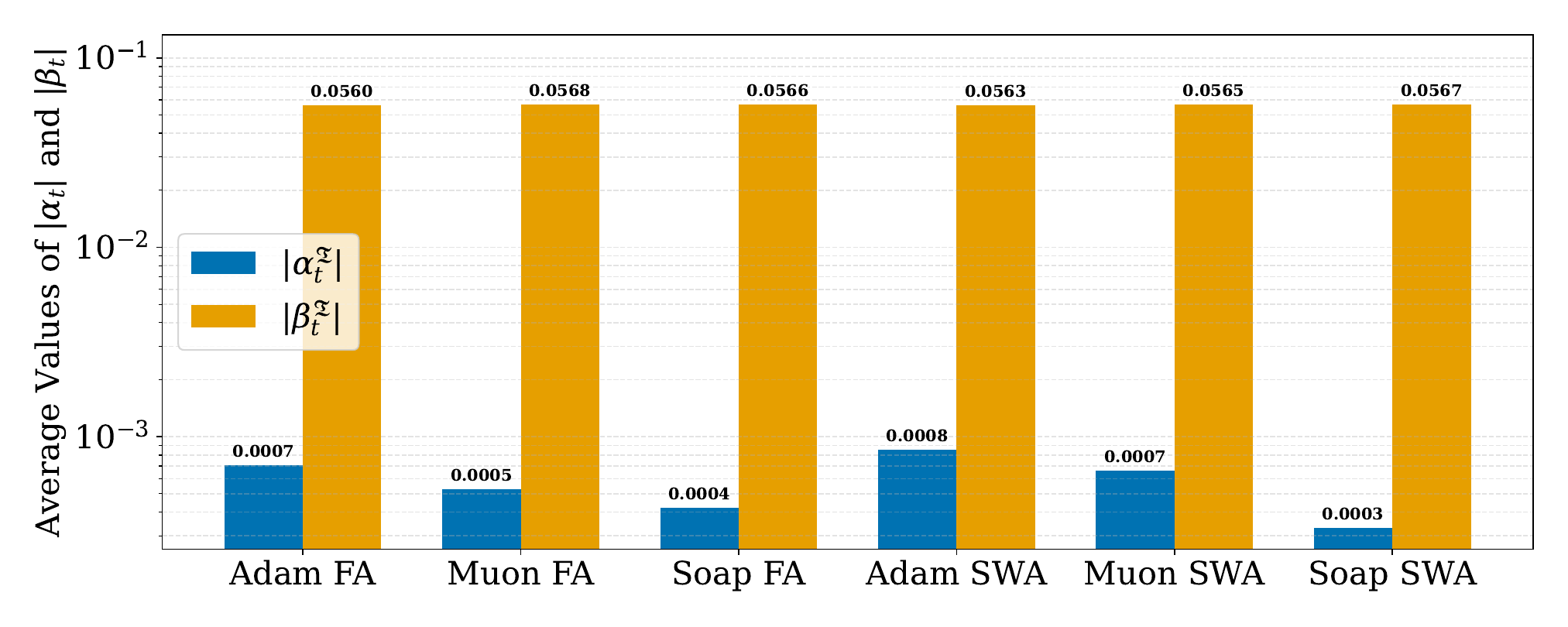}
    \caption{Values of $|\alpha_t^{\frakT}|$ and $|\beta_t^\frakT|$. It shows that $|\beta_t^\frakT|$ is much larger than $|\alpha_t^{\frakT}|$}
    \label{fig:alpha_beta}
    % % \vspace{-2em}
\end{wrapfigure}
at step $t$ is $L(\theta_{t-1}^{\frakT},D_t)$, namely the loss of the parameters trained on $D_{1:t-1}$ and evaluated on the current batch $D_t$. Its change between adjacent steps can be decomposed as
\begin{align*}
L\big(\theta_t^{\frakT},D_{t+1}\big)
-L\big(\theta_{t-1}^{\frakT},D_t\big)
=
\alpha_t^{\frakT}+\beta_t^{\frakT},
\end{align*}
where
$\alpha_t^{\frakT}=L(\theta_t^{\frakT},D_{t+1})-L(\theta_{t-1}^{\frakT},D_{t+1})$
captures the loss change induced by one optimizer update, while
$\beta_t^{\frakT}=L(\theta_{t-1}^{\frakT},D_{t+1})-L(\theta_{t-1}^{\frakT},D_t)$
captures the loss difference of the same parameters across adjacent batches. Figure~\ref{fig:alpha_beta} shows that $|\beta_t|$ is approximately $50$ times larger than $|\alpha_t|$, so the local shape of the training-loss curve is dominated by $\beta_t$. Since $\theta_{t-1}^{\frakT}$ has not been trained on either $D_t$ or $D_{t+1}$, these batches act as held-out data for $\theta_{t-1}^{\frakT}$. Approximate \ac{rgi} therefore implies that $\beta_t^{\frakT}$ is approximately invariant across optimizers and architectures. Consequently, the adjacent-step changes in training loss are nearly the same across optimizers and architectures, explaining their shared training-loss shapes.

\subsection{\ac{rgi} is Strongly Influenced by \horange{Pretraining Data Streams}}\label{sec:train_data}
\begin{figure}[H]
    \centering
    % \vspace{-1.0em}
    \subfigure[Values of $\cccerr$ with $D_l^{\prime}$ from C4.]{ \includegraphics[width=0.31\textwidth]{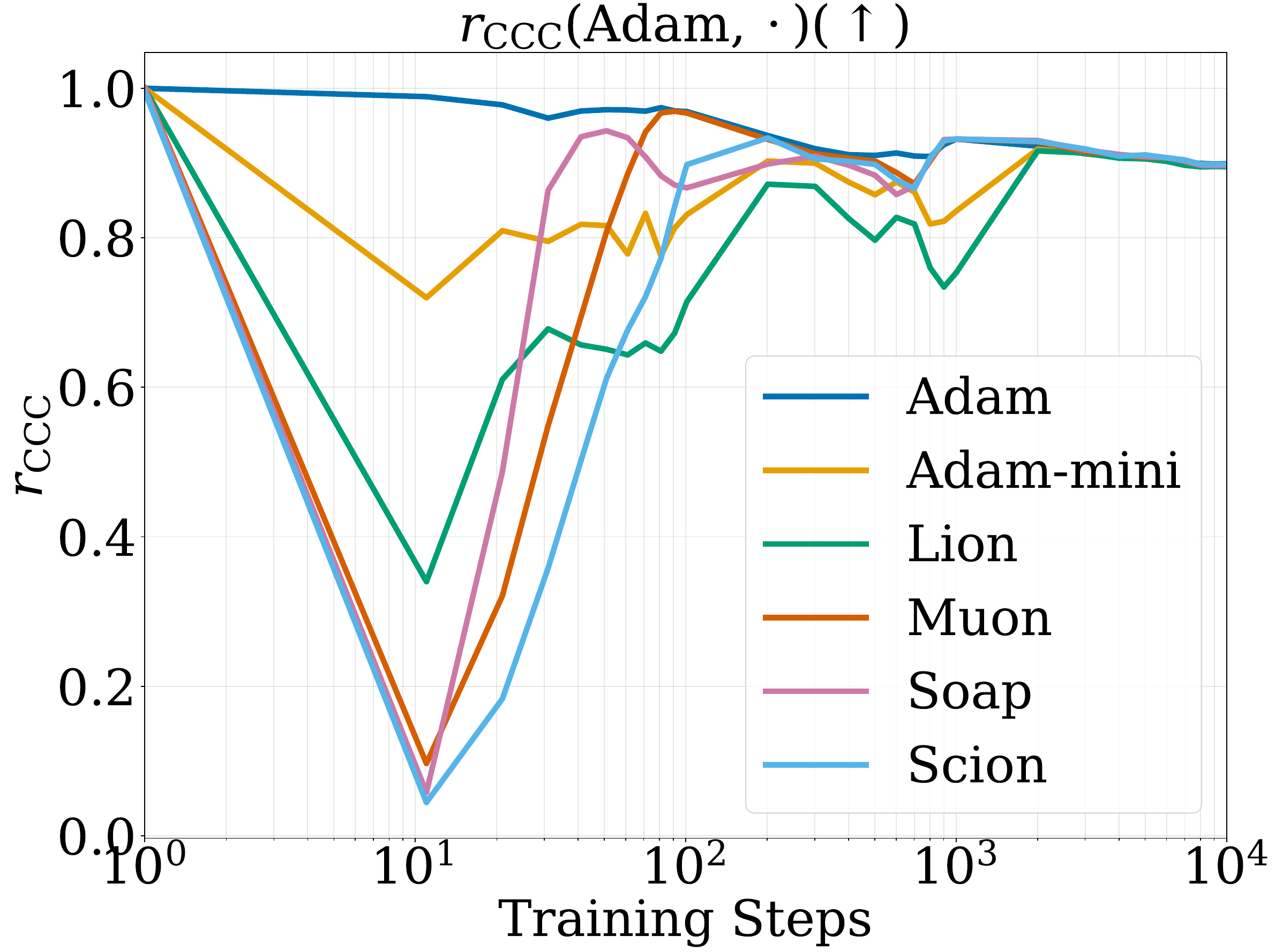}\label{fig:ccc_c4}}
    \subfigure[Values of $\cccerr$  with $D_l^{\prime}$ from Arxiv.]{ \includegraphics[width=0.31\textwidth]{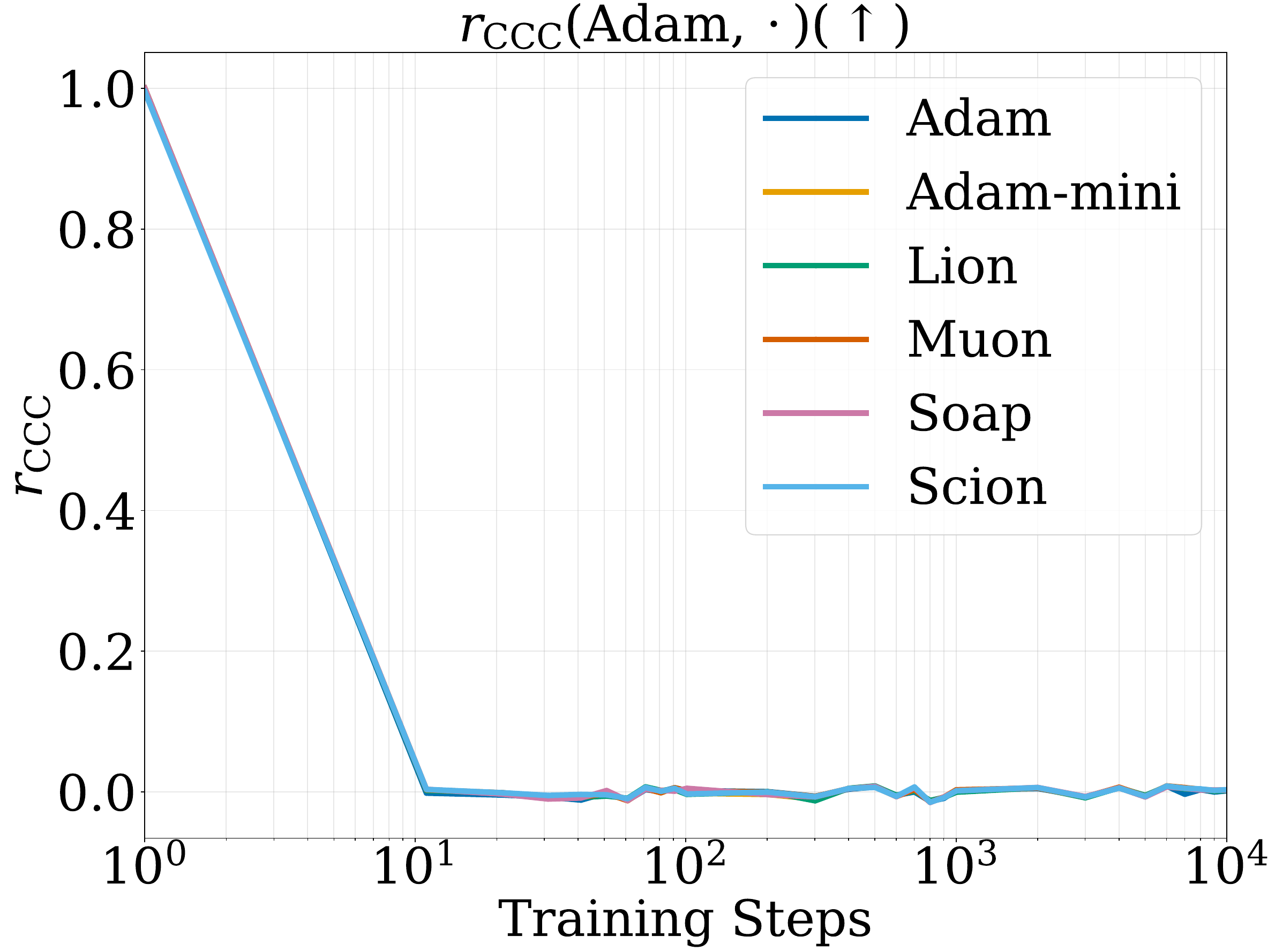}\label{fig:ccc_arxiv}}
    \subfigure[Values of $\cccerr$  with $D_l^{\prime}$ from CodeParrot.]{ \includegraphics[width=0.31\textwidth]{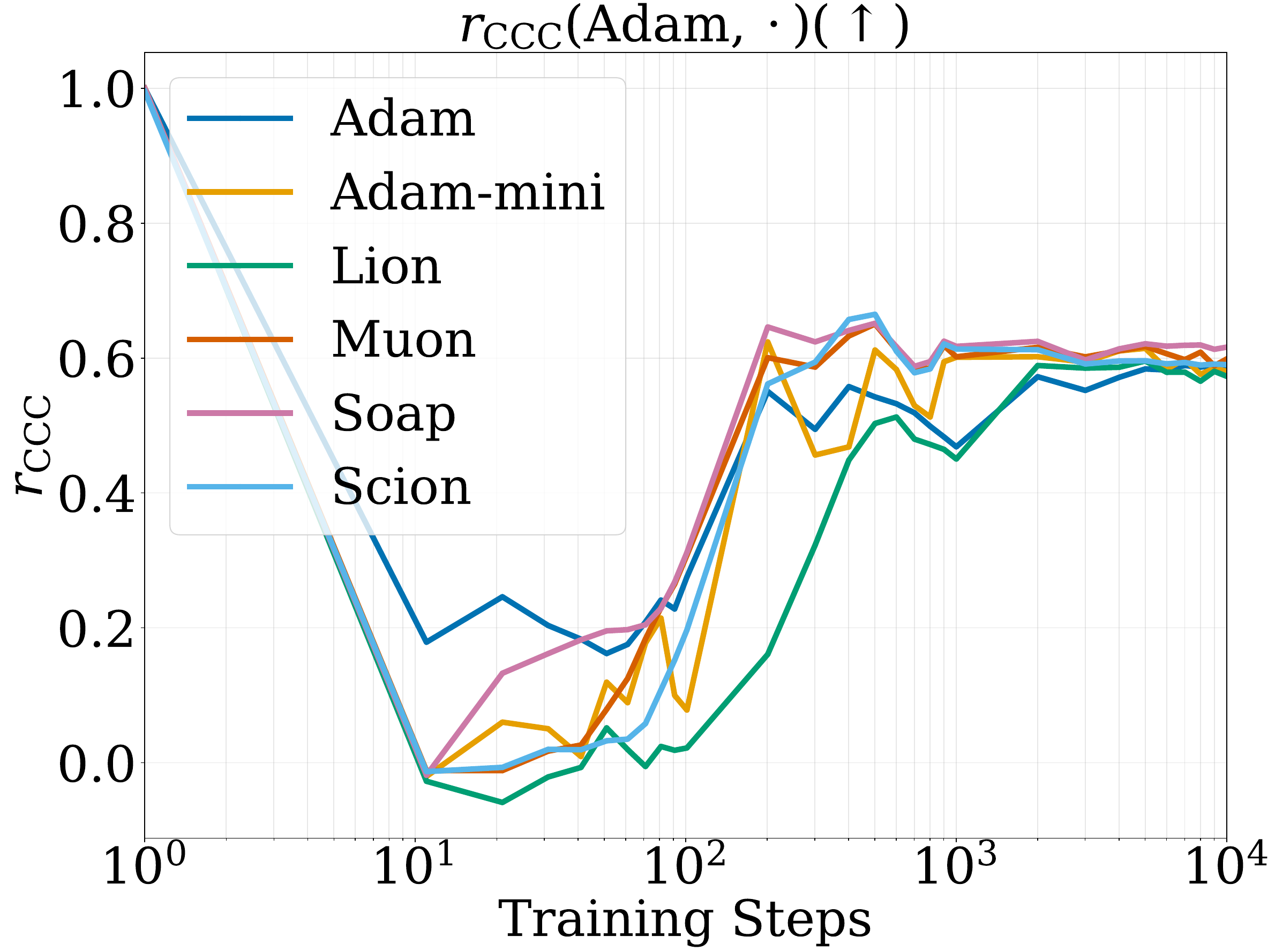}\label{fig:ccc_code}}
    % \vspace{-1.0em}
    \caption{\ac{rgi} across optimizers trained on different data streams. The panels plot $\cccerr$ when Adam is trained on FineWeb and the compared optimizer is trained on C4 (Panel (a)), ArXiv (Panel (b)), or CodeParrot (Panel (c)). The results show that \ac{rgi} is strongly affected by the training data stream: changing the data stream can substantially weaken, or even eliminate, \ac{rgi} across optimizers.} 
    % \vspace{-1.0em}
    \label{fig:dataset}
\end{figure}

We next explore whether \ac{argi} emerges when the training data stream varies, jointly or independently of the optimizer, and show that it can fail even when the optimizer and architecture are held fixed. Specifically, we examine whether \ac{rgi} emerges between
$\theta_t^{\frakT_{\adam}}=F(\frakT_{\adam},t)$, where
$\frakT_{\adam}=(\adam,\theta_0,D_{1})$,
and
$\theta_t^{\frakT^{\prime}}=F(\frakT^{\prime},t)$, where
$\frakT^{\prime}=(\sfA^{\prime},\theta_0,\horange{D_{1}^{\prime}})$.
Here, $D_{1}$ is drawn from FineWeb, while $\horange{D_{1}^{\prime}}$ is drawn from one of three alternative datasets: C4~\citep{raffel2020exploring}, the ArXiv subset of RedPajama~\citep{weber2024redpajama}, and CodeParrot~\citep{codeparrot2022}. FineWeb and C4 are both large-scale web-text datasets covering diverse topics and are derived, at least in part, from Common Crawl~\citep{baack2024critical}, suggesting relatively high distributional similarity. By contrast, ArXiv consists of academic papers from arXiv, while CodeParrot consists of source code collected from GitHub, making their distributions more distinct from FineWeb.

Figure~\ref{fig:dataset} shows that the training data stream strongly affects the emergence of \ac{rgi}. For C4, Figure~\ref{fig:ccc_c4} shows that $\cccerr$ first decreases, then rises to a relatively large value, and finally declines again; the intermediate increase may reflect the greater distributional similarity between FineWeb and C4. In contrast, for ArXiv and CodeParrot, Figures~\ref{fig:ccc_arxiv} and~\ref{fig:ccc_code} show that $\cccerr$ remains below $0.95$ throughout training even when the same optimizer $\adam$ is used, indicating that \ac{rgi} does not emerge across these training data streams. The corresponding $\differr$ results in Appendix~\ref{app:add_result} show the same trend. Overall, changes in the training data distribution can substantially weaken or even prevent the emergence of \ac{rgi}.

\subsection{ \ac{rgi} Persists under Moderate \hteal{Validation} Shifts}\label{sec:val}
\begin{figure}[H]
    \centering
    \subfigure[Values of $\cccerr$ with $D_{\val}$ from C4.]{ \includegraphics[width=0.23\textwidth]{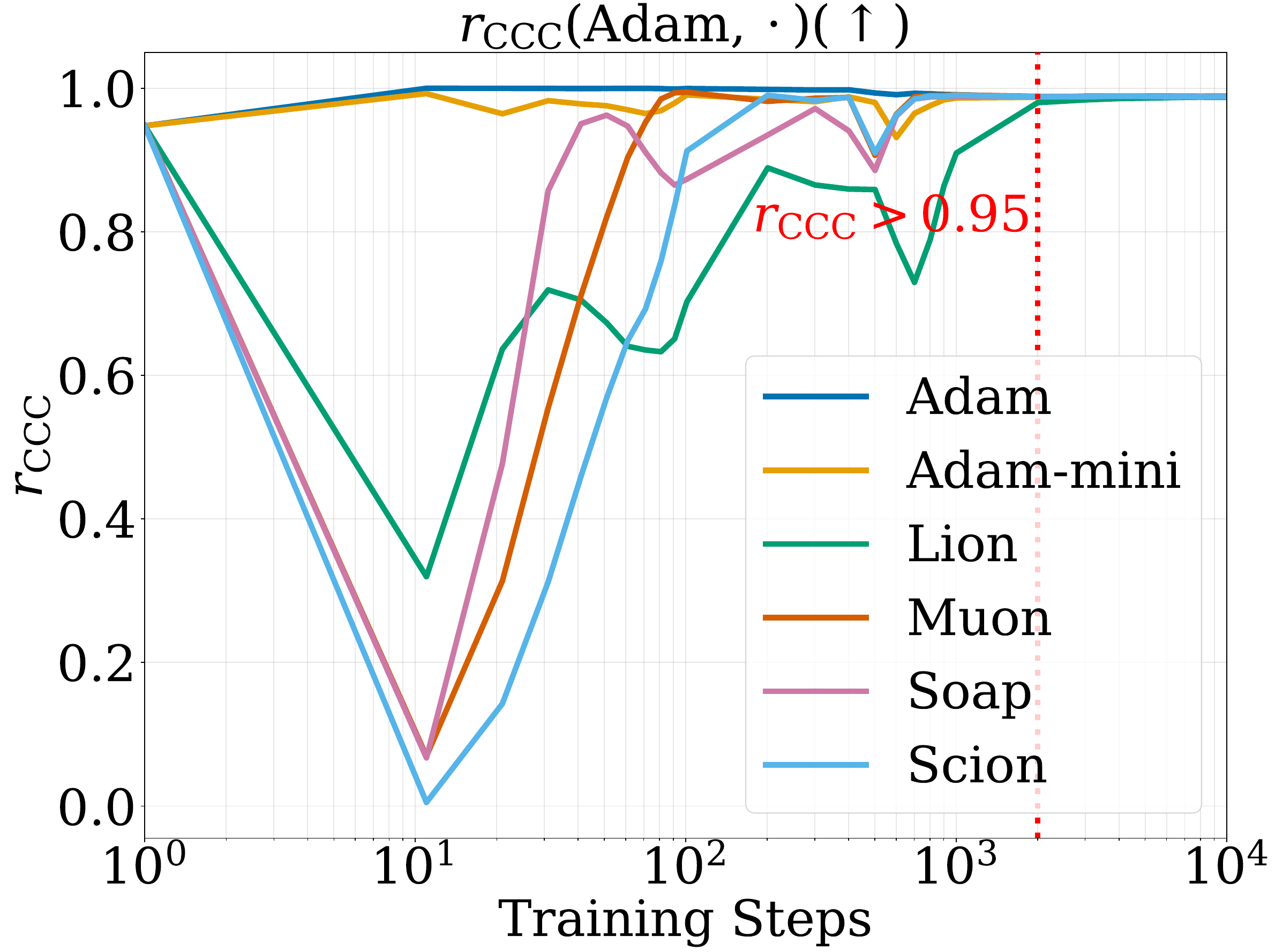}}
    \subfigure[Values of $\cccerr$ with $D_{\val}$ from ArXiv.]{ \includegraphics[width=0.23\textwidth]{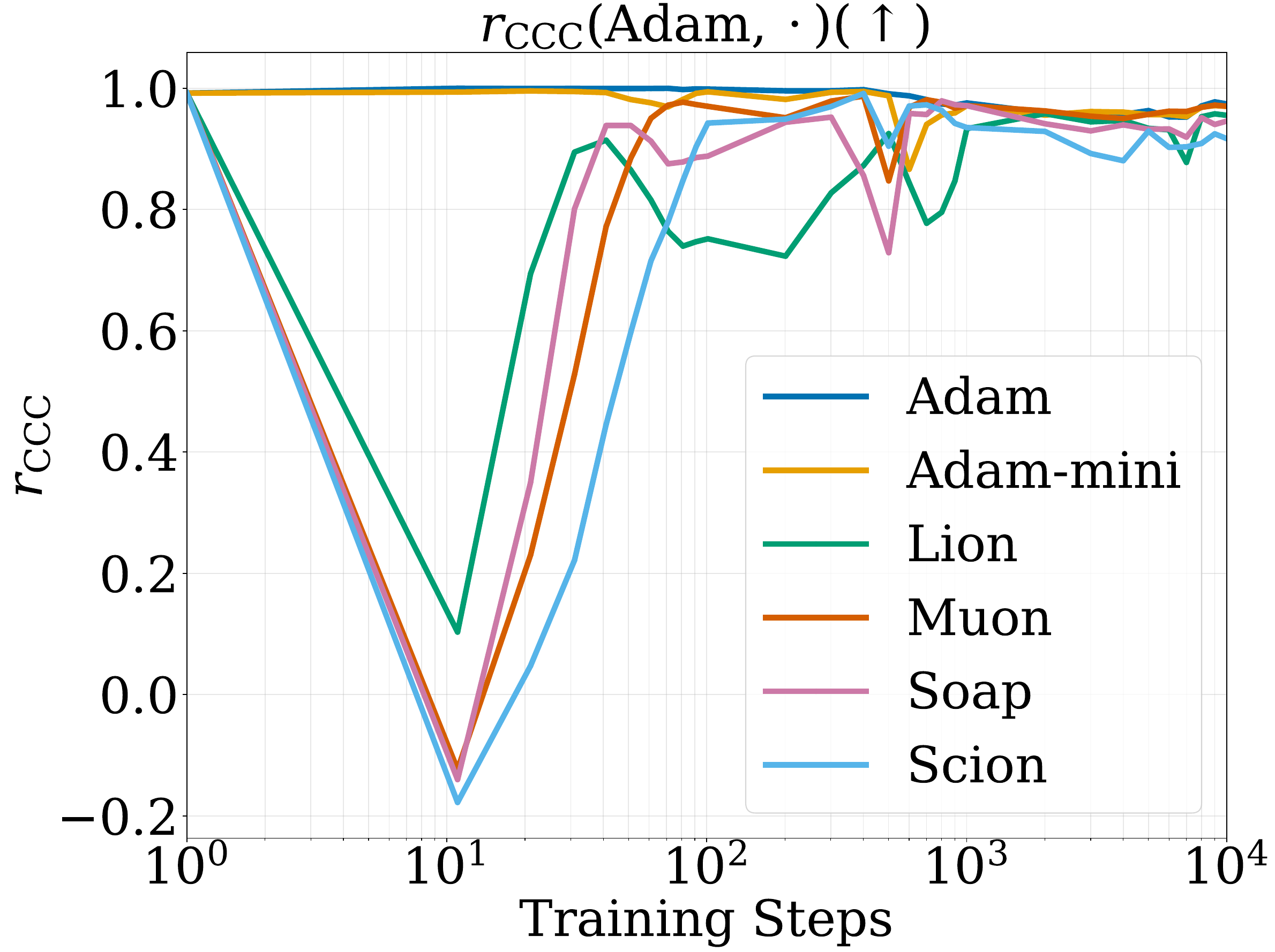}}
    \subfigure[Values of $\cccerr$ with $D_{\val}$ from CodeParrot.]{ \includegraphics[width=0.23\textwidth]{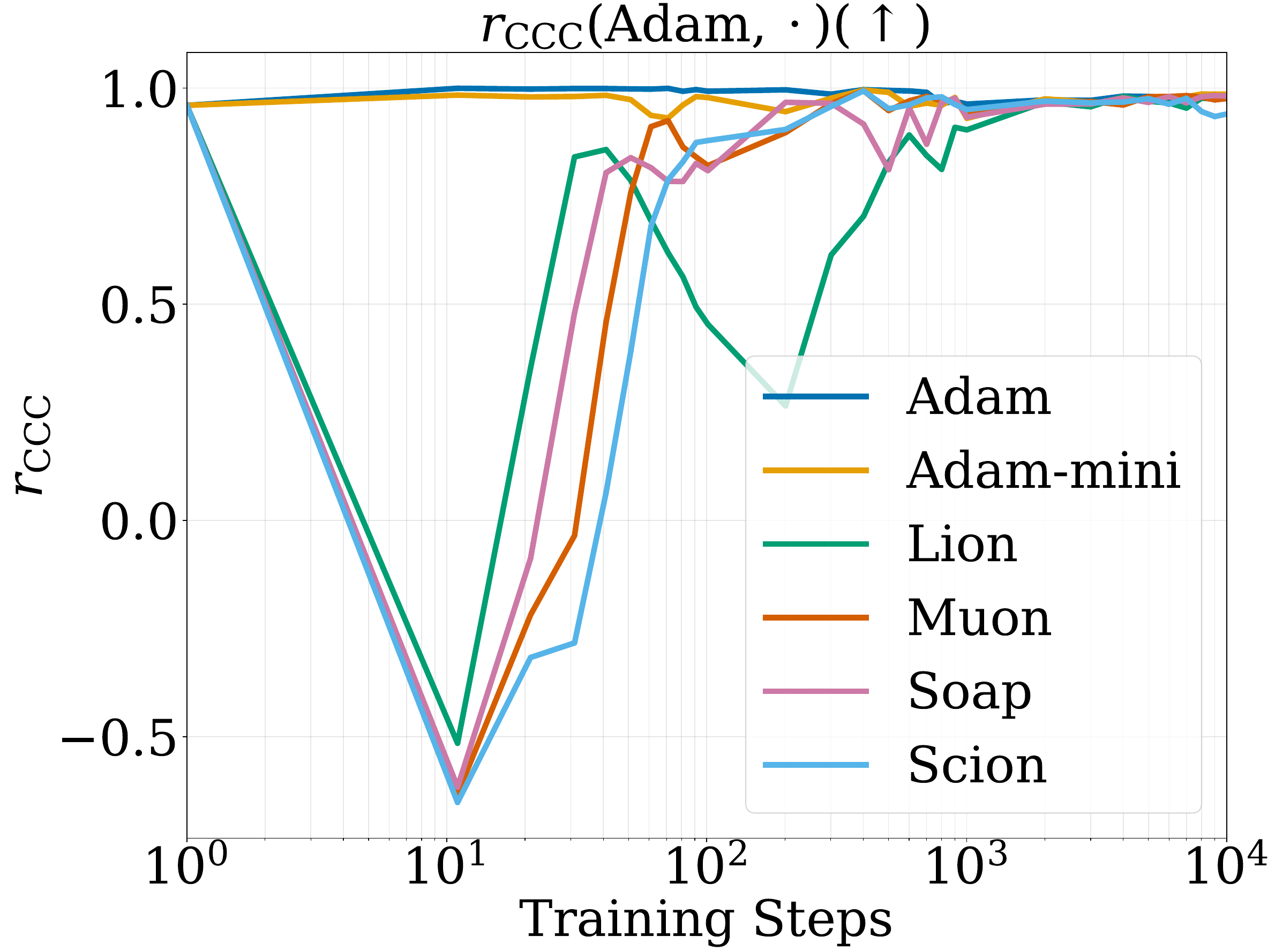}}
    \subfigure[Values of $\cccerr$ with $D_{\val}$ from Random.]{ \includegraphics[width=0.23\textwidth]{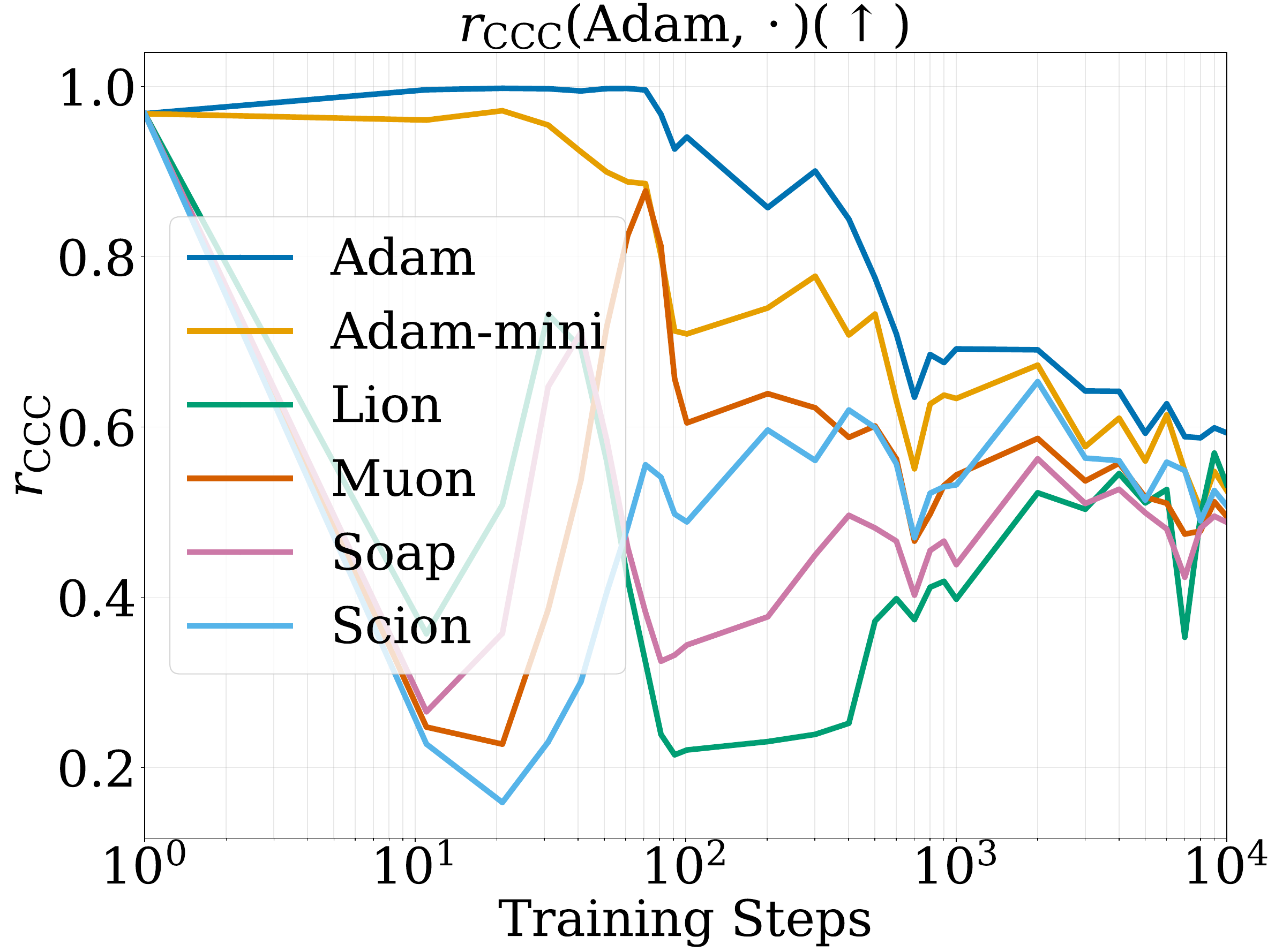}}
    \caption{ \ac{rgi} across optimizers and architectures evaluated on different validation datasets. The panels plot $\cccerr$ evaluated on C4, ArXiv, CodeParrot, and Random. These results show that \ac{rgi} can transfer to related validation distributions, but becomes weaker on more distributionally distinct validation datasets.} 
    % \vspace{-1.0em}
    \label{fig:val}
\end{figure}

% \begin{figure}[t]
%     \centering
%     \subfigure[Values of $\differr$ with $D_{\val}$ from C4.]{ \includegraphics[width=0.23\textwidth]{figures/r_diff_c4_Adam_ood.pdf}}
%     \subfigure[Values of $\differr$ with $D_{\val}$ from ArXiv.]{ \includegraphics[width=0.23\textwidth]{figures/r_diff_arxiv_Adam_ood.pdf}}
%     \subfigure[Values of $\differr$ with $D_{\val}$ from CodeParrot.]{ \includegraphics[width=0.23\textwidth]{figures/r_diff_codeparrot_Adam_ood.pdf}}
%     \subfigure[Values of $\differr$ with $D_{\val}$ from Random.]{ \includegraphics[width=0.23\textwidth]{figures/r_diff_random_Adam_ood.pdf}}
%     \caption{ \ac{rgi} across optimizers evaluated on different validation datasets. The panels plot $\differr$ evaluated on C4, ArXiv, CodeParrot, and Random. These results show that \ac{rgi} can transfer to related validation distributions, but becomes weaker on more distributionally distinct validation datasets.} 
%     % \vspace{-1.0em}
%     \label{fig:val}
% \end{figure}
In this section, we ablate the effect of the validation data $\hteal{D_{\val}}$ on \ac{rgi} across optimizers and architectures. For clarity, we focus on \ac{fa} and \ac{swa} as representative architectural variants; additional results are provided in Appendix~\ref{app:add_result}. All models are trained on FineWeb, while $\differr$ and $\cccerr$ are evaluated on validation data from other sources. As in Section~\ref{sec:train_data}, we consider C4, ArXiv, and CodeParrot, and additionally construct a Random dataset by sampling each token i.i.d. uniformly from the vocabulary, yielding an extreme \ac{ood} setting with no meaningful semantic structure. Figure~\ref{fig:val} reports the corresponding $\cccerr$ values, comparing \ac{swa} models trained with various optimizers against the \ac{fa} baseline trained with Adam. $\cccerr$ remains large on C4, while on ArXiv and CodeParrot it is smaller but generally increases as training progresses. In contrast, the Random dataset shows no clear emergence of \ac{rgi}. These results suggest that validation-distribution shifts weaken \ac{argi}: moderate shifts preserve the trend, whereas sufficiently extreme shifts can prevent its emergence. The corresponding $\differr$ results in Appendix~\ref{app:add_result} support the same conclusion.

\section{Relationships to the NTK and MF Regimes}\label{sec:relation}
Section~\ref{sec:emp_studies} shows that approximate \ac{rgi} holds across a wide range of optimizers and architectures. We now examine whether this phenomenon can be explained by two well-known theoretical regimes of training dynamics: the \ac{ntk} and \ac{mf} regimes. We focus on \ac{rgi} across optimizers, since different architectures have parameterizations and intermediate representations that are not directly comparable.

\subsection{RGI Emerges Beyond the \ac{ntk} Regime}\label{sec:ntk}

\begin{figure}[t]
    \centering
    % % \vspace{-0.9em}
    \subfigure[$\cossim$ between parameters within optimizers.]{ \includegraphics[width=0.23\textwidth]{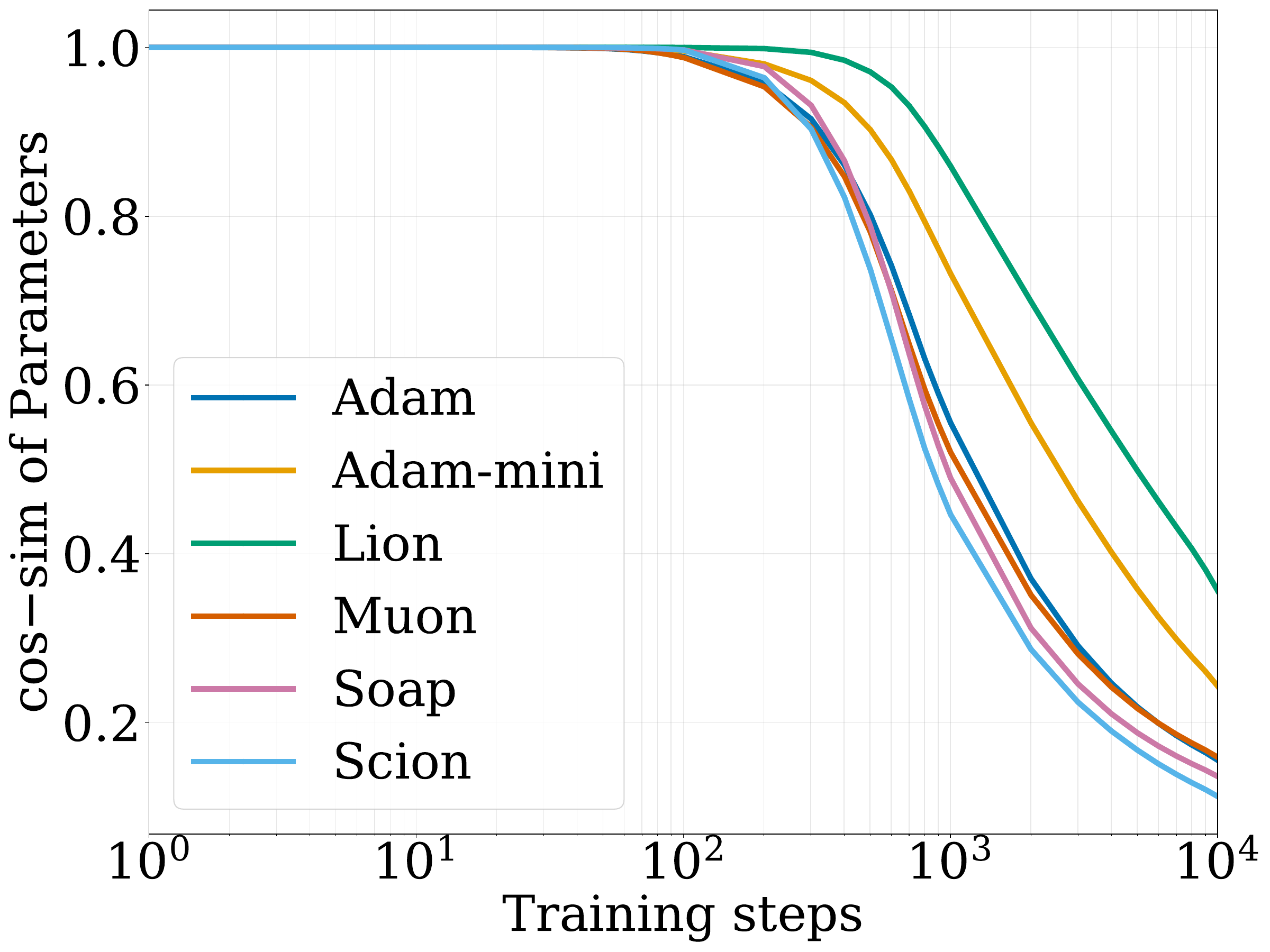}\label{fig:cs_self_p}}
    \subfigure[$\cossim$ between hidden states within optimizers.]{ \includegraphics[width=0.23\textwidth]{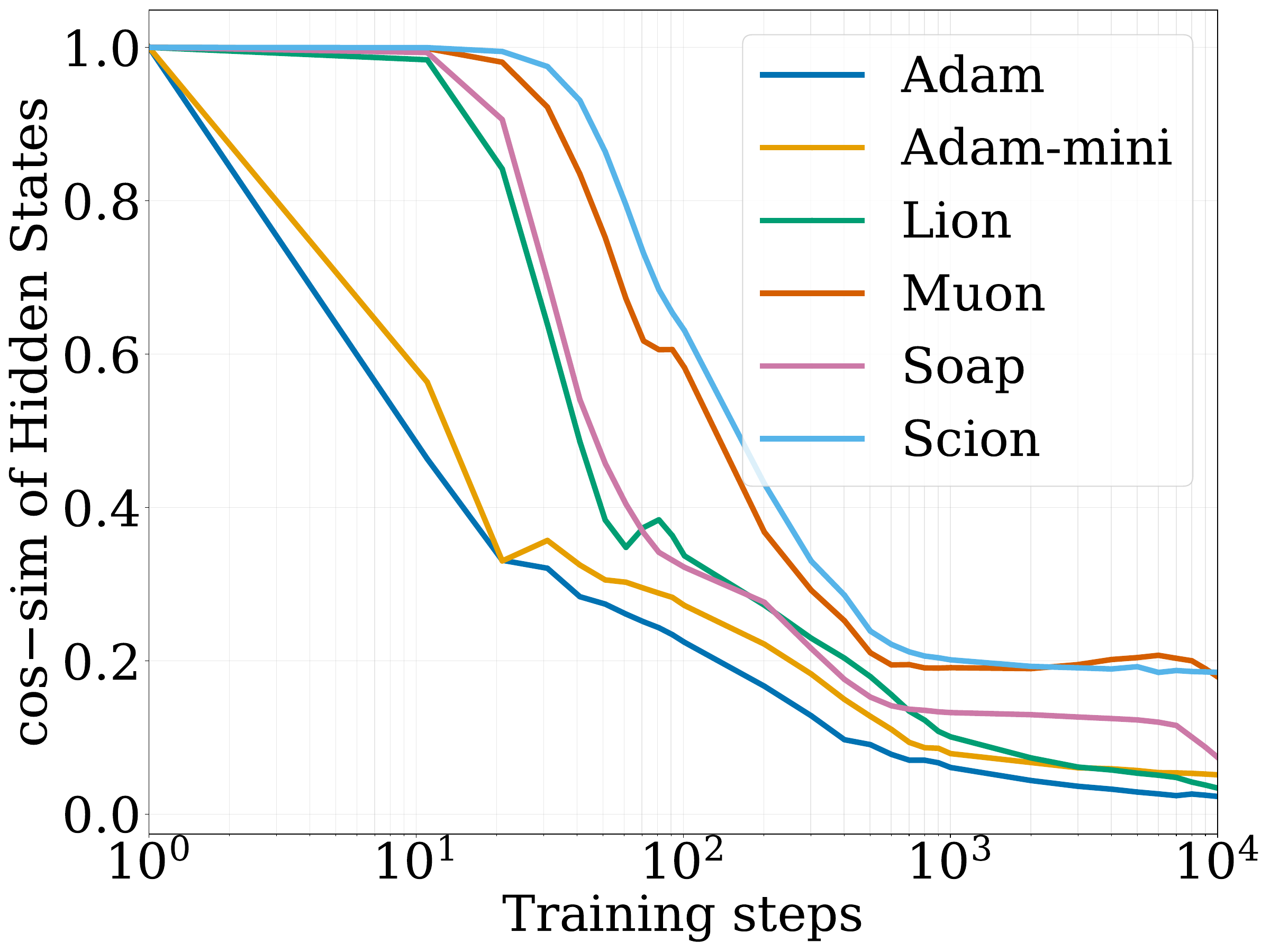}\label{fig:cs_self_h}}
    \subfigure[$\cossim$ between parameters cross optimizers.]{ \includegraphics[width=0.23\textwidth]{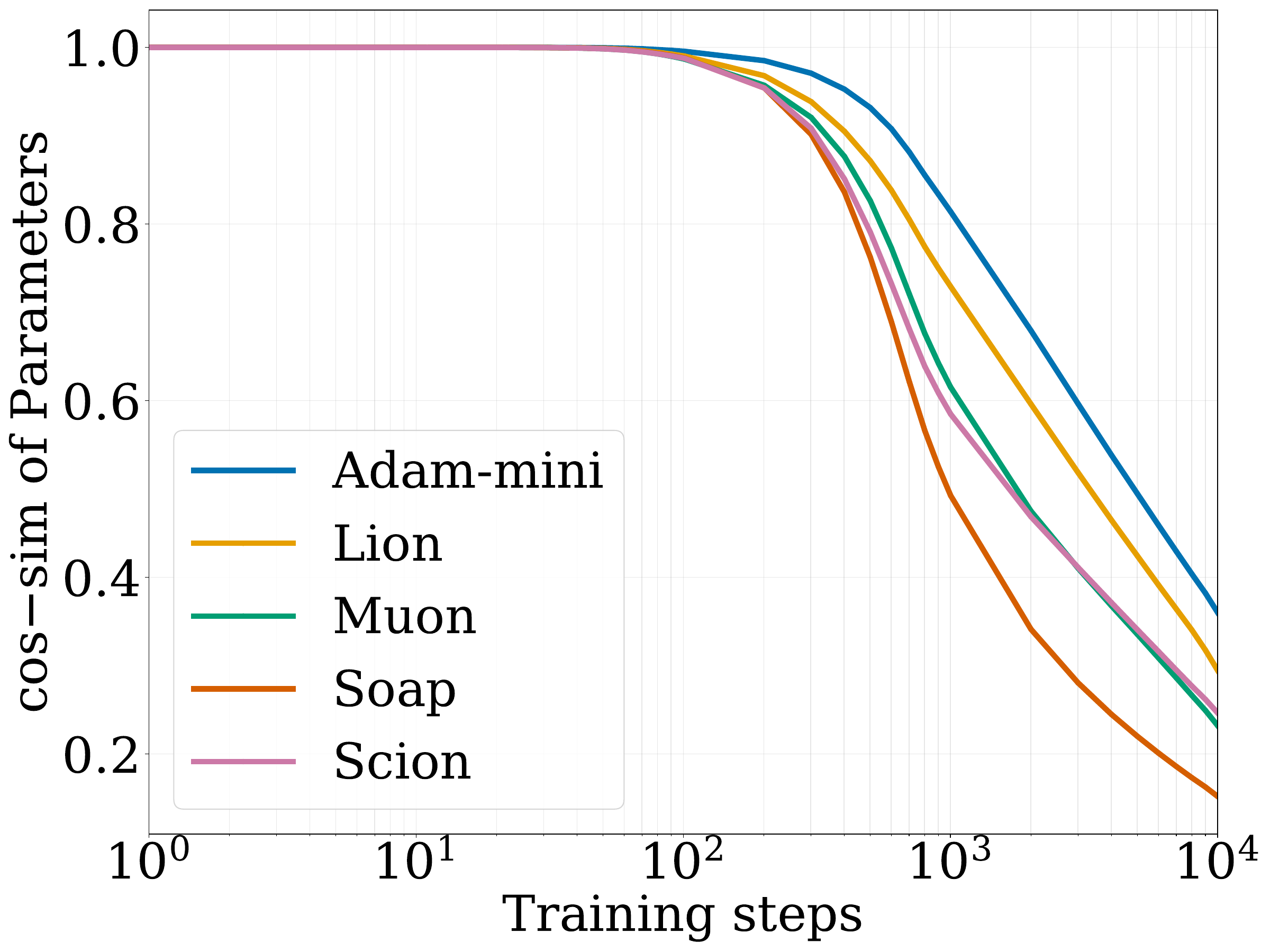}\label{fig:cs_cross_p}}
    \subfigure[$\cossim$ between hidden states cross optimizers.]{ \includegraphics[width=0.23\textwidth]{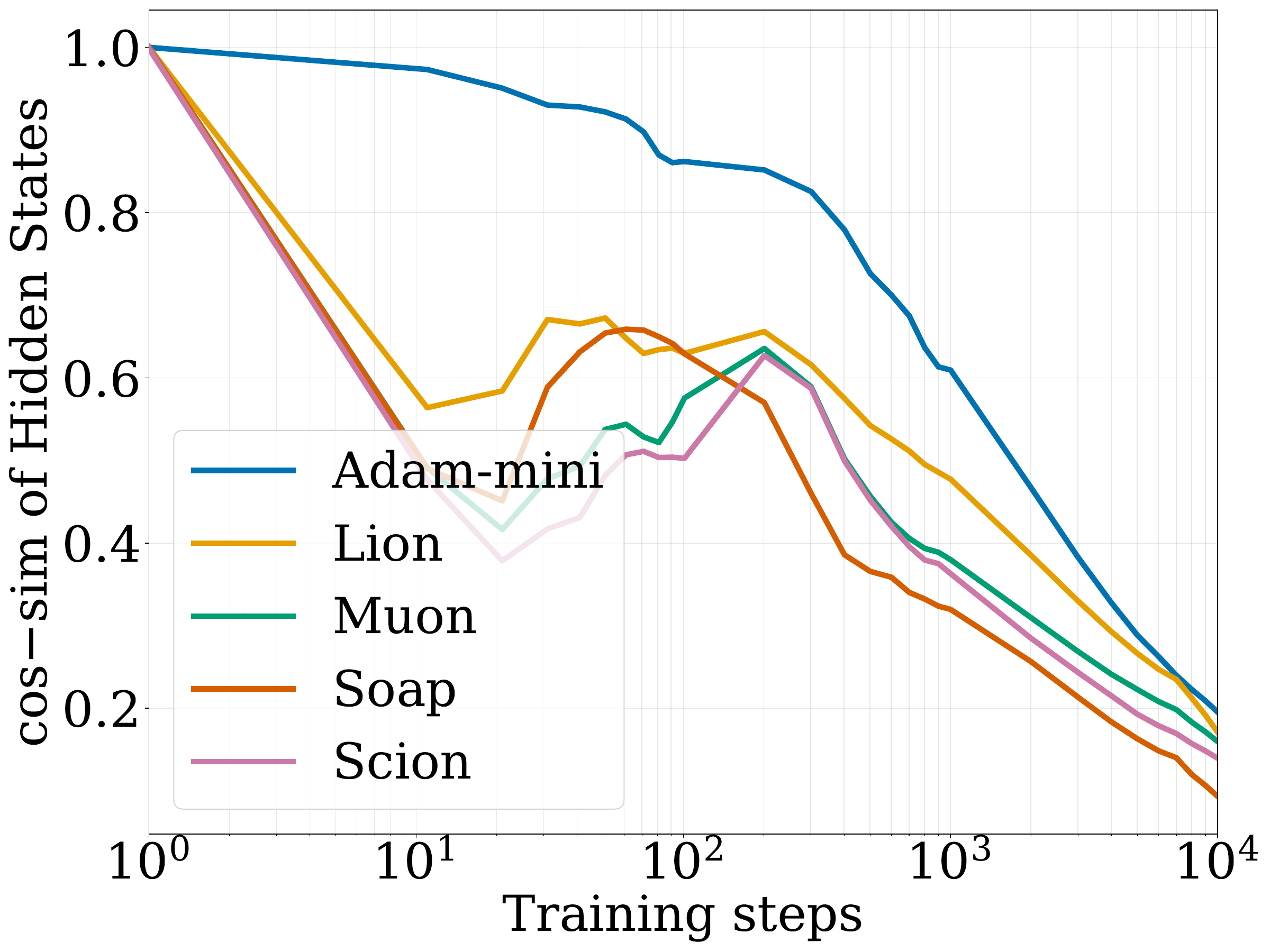}\label{fig:cs_cross_h}}
    % \vspace{-1.0em}
    \caption{Within- and cross-optimizer $\cossim(\uparrow)$ for parameters and hidden states. Panels (a)--(b) compare $\theta_0^{\frakT}$ with $\theta_t^{\frakT}$, while Panels (c)--(d) compare $\theta_t^{\frakT}$ with $\theta_t^{\frakT^{\prime}}$. Hidden-state similarities are averaged over layers $l\in[L]$. Within-optimizer comparisons show that all optimizers learn nontrivial features throughout training, whereas cross-optimizer comparisons show that different optimizers learn distinct features.} 
    \label{fig:ntk}
    % \vspace{-1.0em}
\end{figure}
In this section, we examine whether approximate \ac{rgi} across optimizers sharing the same initialization and training data stream can be explained by the \ac{ntk} regime. In this regime, sufficiently wide models under suitable Gaussian initialization are well approximated by their linearization around initialization, with the induced kernel remaining nearly constant~\citep{jacot2018neural,golikov2022neural}. This is also known as the lazy-learning regime~\citep{geiger2020disentangling}, whose characteristic signature is that parameters and representations remain close to initialization:
$\|\theta_t-\theta_0\|_2 \ll \|\theta_0\|_2$ and
$\|h(l,\theta_t,d)-h(l,\theta_0,d)\|_2 \ll \|h(l,\theta_0,d)\|_2$
for relevant $t$, $l\in[L]$, and $d\in\calD_{\val}$~\citep{du2018gradient}. We test these signatures using cosine similarity and relative $\ell_2$ difference,
$\cossim(x,x')=x^\top x'/(\|x\|_2\|x'\|_2)$ and
$\relerr(x,x')=2|x-x'|_2/(\|x\|_2+\|x'\|_2)$,
for both parameters and hidden states, comparing them across training steps within each optimizer and across optimizers at the same step. We report cosine-similarity results in the main text and defer the relative-$\ell_2$ results, which show consistent trends, to Appendix~\ref{app:add_result}.

Figures~\ref{fig:cs_self_p}--\ref{fig:cs_self_h} report within-optimizer cosine similarities between initialization and step $t$ for both parameters, $\cossim(\theta_0^{\frakT},\theta_t^{\frakT})$, and hidden states, $\cossim(h(l,\theta_0^{\frakT},d),h(l,\theta_t^{\frakT},d))$, averaged over layers $l\in[L]$, validation examples $d\in\calD_{\val}$, and random seeds. After $5{,}000$ training steps, both parameters and hidden states deviate substantially from initialization, suggesting that the observed \ac{rgi} cannot be explained by the \ac{ntk} lazy-learning regime alone. We further compare models trained by different optimizers at the same step $t$. Figures~\ref{fig:cs_cross_p}--\ref{fig:cs_cross_h} report the cosine similarities of parameters and hidden states between Adam and the other optimizers, both of which decrease as training proceeds. Thus, approximate \ac{rgi} cannot be attributed either to \ac{ntk}-style fixed representations or to different optimizers converging to shared parameters or hidden representations.

\subsection{RGI Emerges Beyond the \ac{mf} Regime}\label{sec:mf}

\begin{figure}[H]
    \centering
    \subfigure[Values of $\cccerr$ at width $576$.]{ \includegraphics[width=0.23\textwidth]{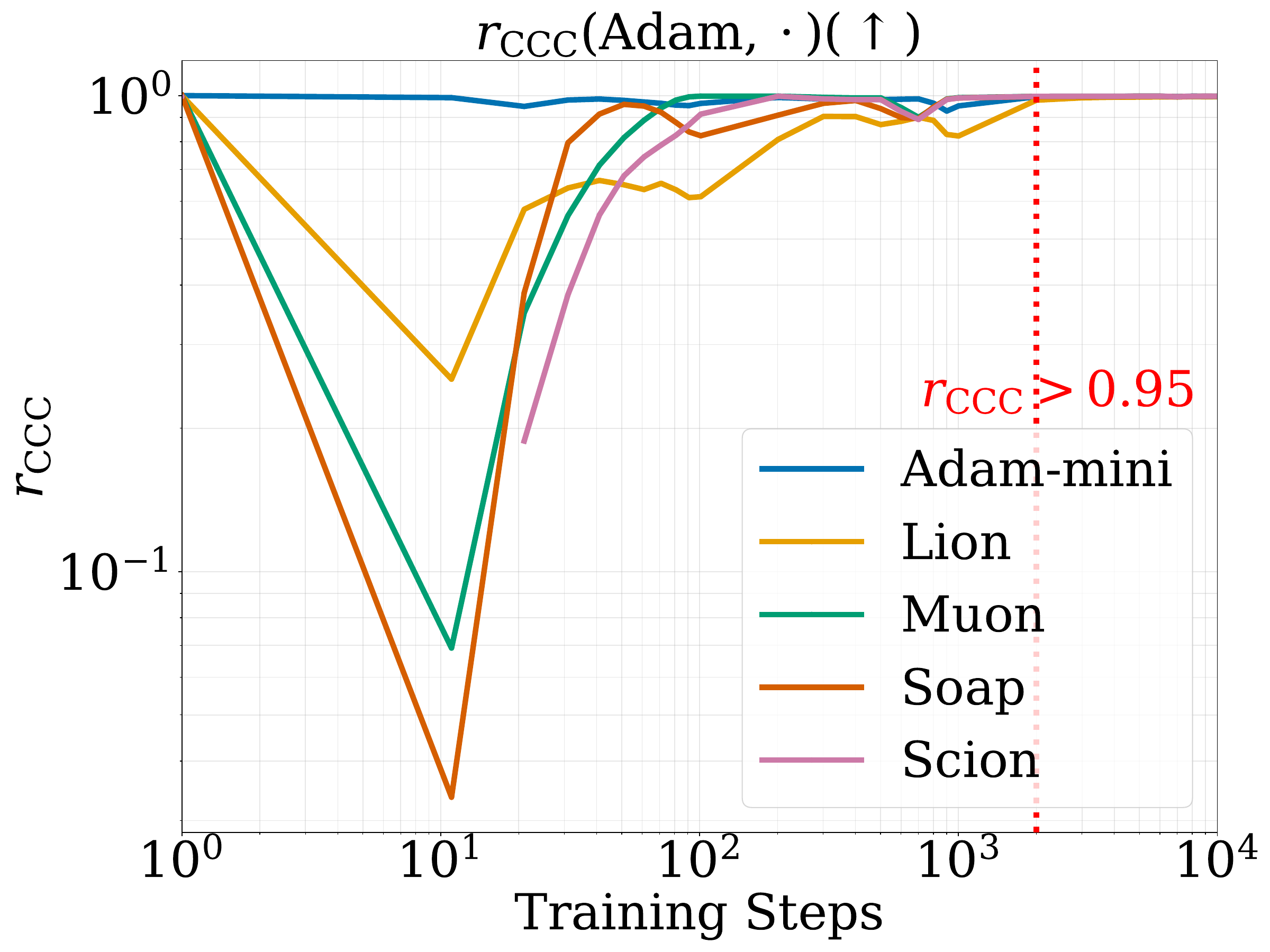}\label{fig:ccc_576}}
    \subfigure[Values of $\cccerr$ at width $384$.]{ \includegraphics[width=0.23\textwidth]{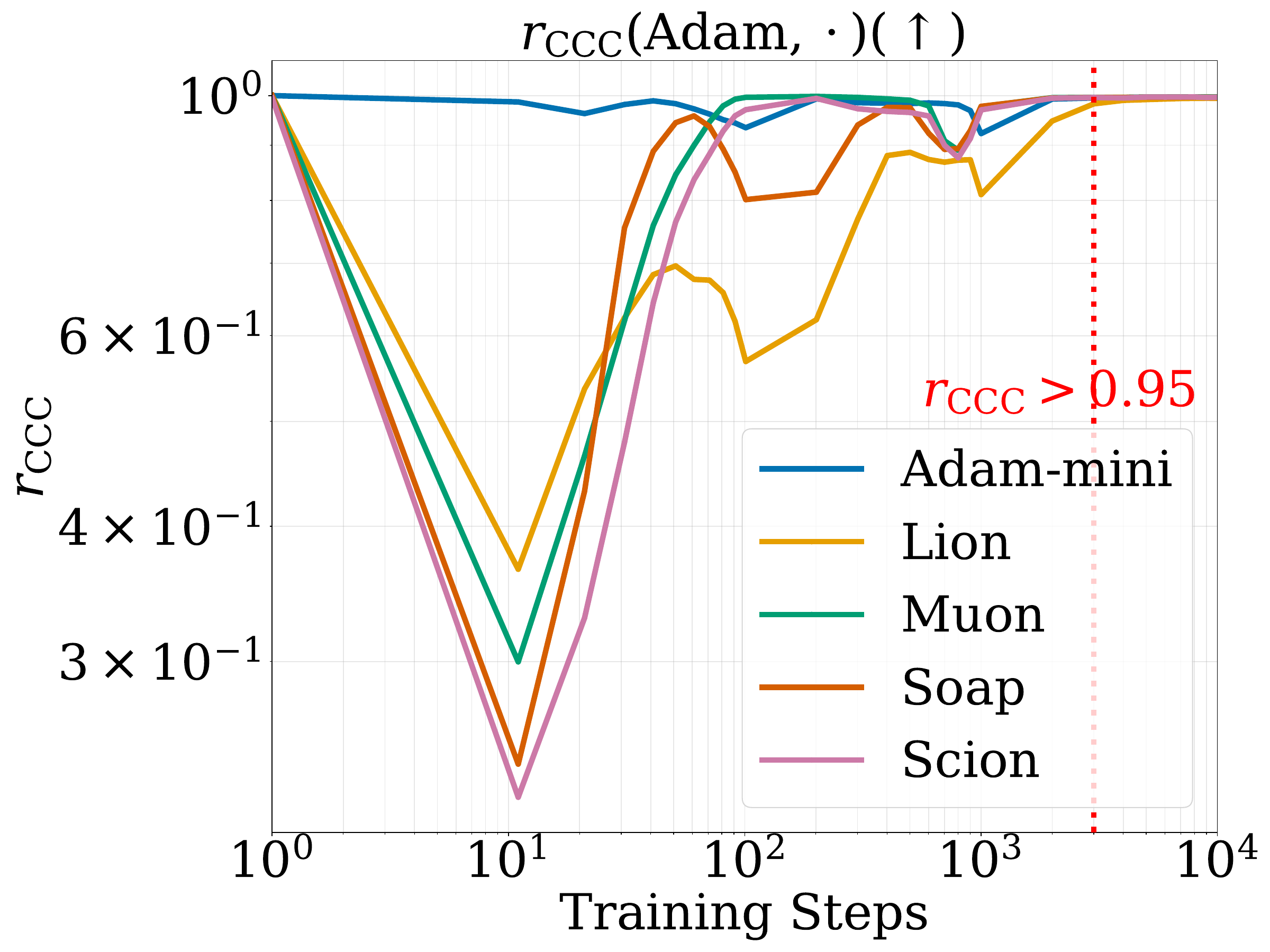}\label{fig:ccc_384}}
    \subfigure[Values of $\cccerr$ at width $192$.]{ \includegraphics[width=0.23\textwidth]{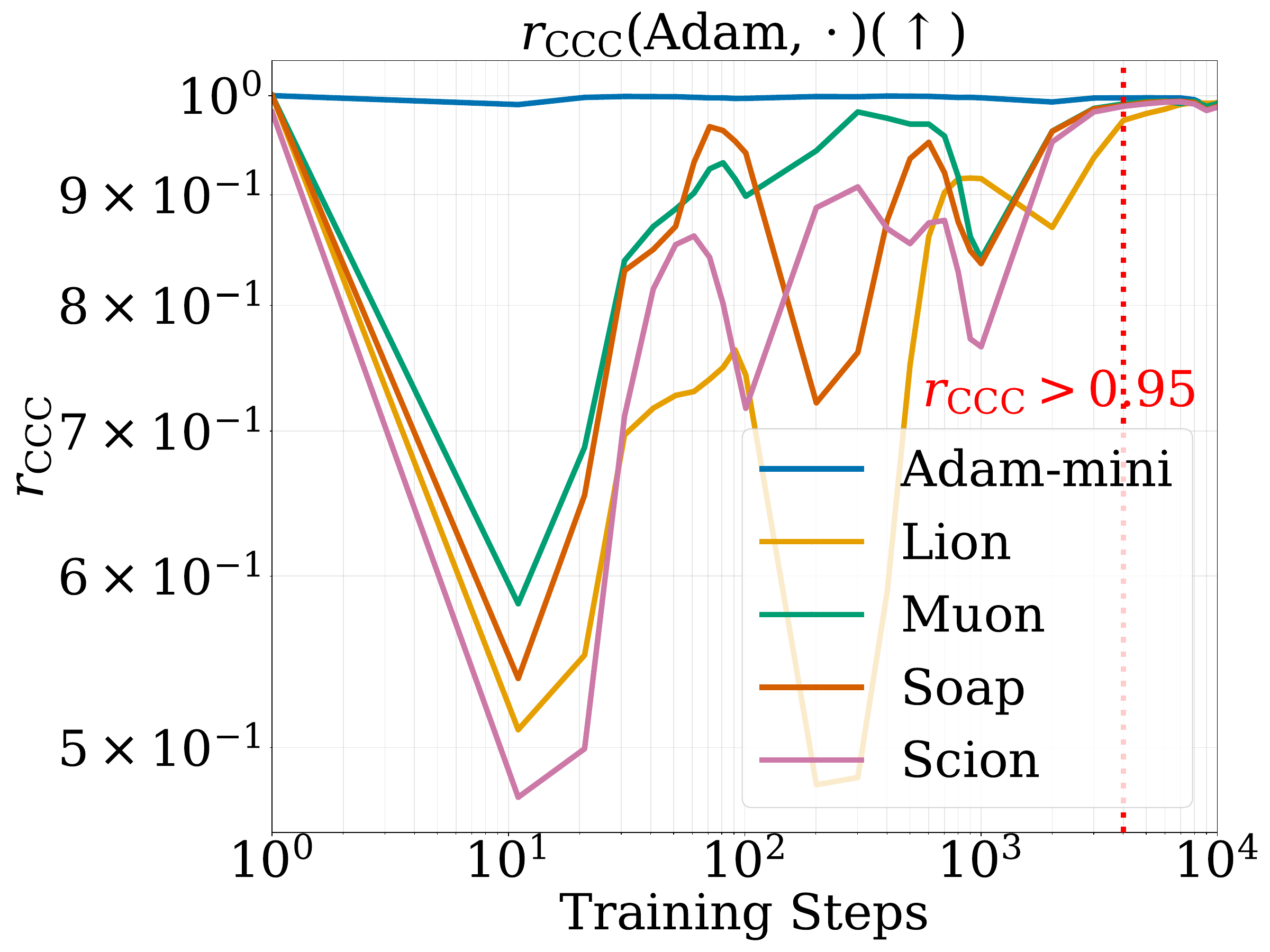}\label{fig:ccc_192}}
    \subfigure[Validation losses of different learning rates of Adam with various widths.]{ \includegraphics[width=0.23\textwidth]{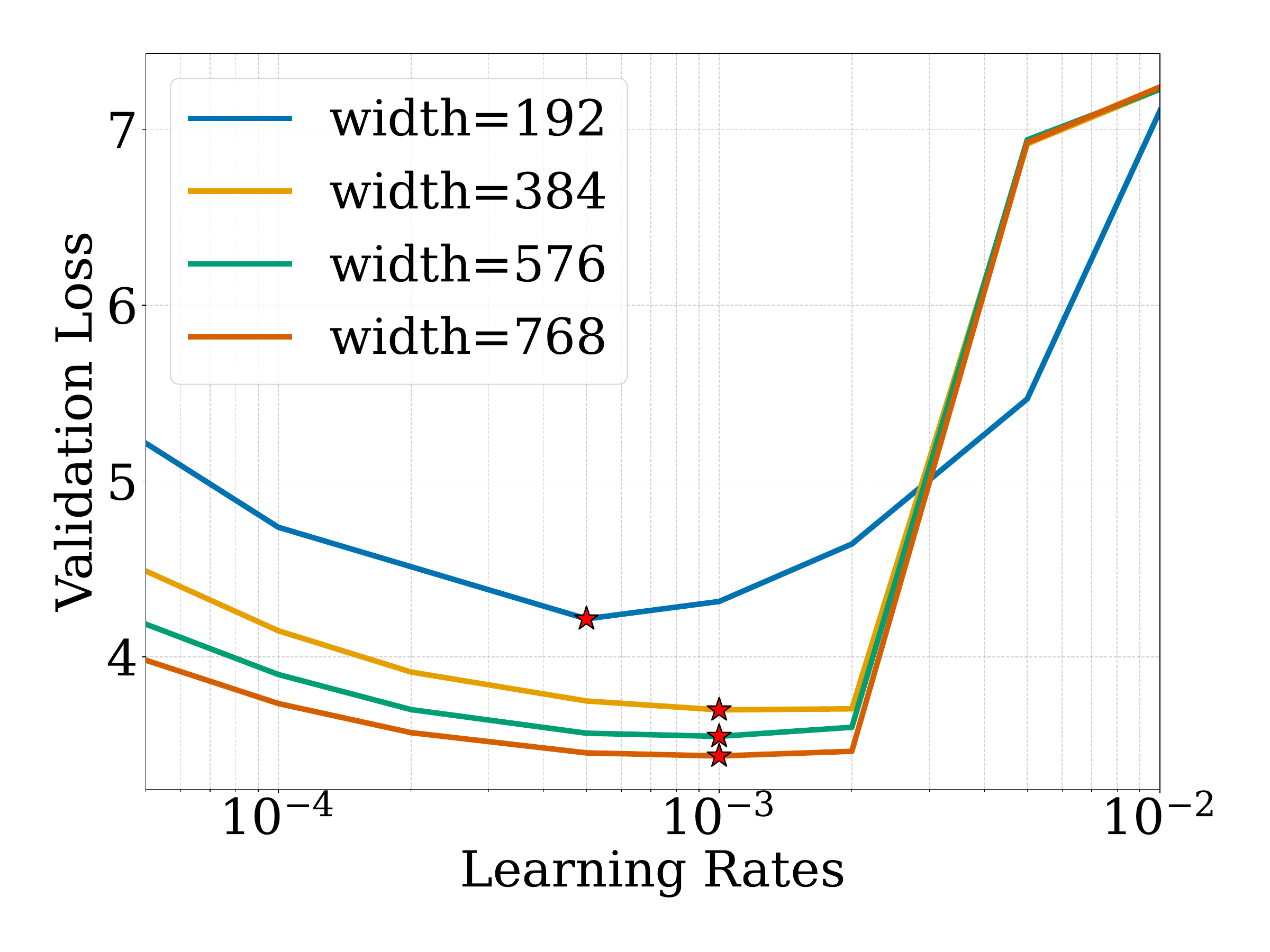}\label{fig:grid_search}}
    \caption{\ac{rgi} across models of different widths. Panels (a)--(c) plot $\cccerr$ for models with different widths, showing that \ac{rgi} holds across all tested widths. Panel (d) plots the validation loss of Adam at $10{,}000$ steps across learning rates and widths, showing that the width-$192$ model exhibits behavior inconsistent with the \ac{mf} regime.}
    \label{fig:mean_field}
    % \vspace{-1.0em}
\end{figure}
In this section, we examine whether approximate \ac{rgi} across optimizers with identical initializations and training data streams can be explained by the \ac{mf} regime. In this regime, the parameter distribution evolves nontrivially in the infinite-width limit, allowing both parameters and features to change substantially during training~\citep{mei2018mean,mei2019mean,sirignano2020mean}. Starting from the $124$M GPT model with hidden dimension $768$ used in Section~\ref{sec:emp_studies}, we progressively reduce the width to $576$, $384$, and $192$, where $192$ is one quarter of the baseline width, and test whether approximate \ac{rgi} persists. We also grid search the Adam learning rate at each width, using the width dependence of the optimal learning rate as an empirical diagnostic of consistency with the \ac{mf} regime.

Figures~\ref{fig:ccc_576}--\ref{fig:ccc_192} evaluate \ac{rgi} across optimizers at different fixed widths, unlike Figures~\ref{fig:ccc_cross_384} and \ref{fig:ccc_cross_64}, which compare models across widths. Approximate \ac{rgi} persists even at width $192$, one quarter of the standard GPT-2 width; the corresponding $\differr$ results in Appendix~\ref{app:add_result} lead to the same conclusion. To assess whether this narrow model remains consistent with the \ac{mf} regime, Figure~\ref{fig:grid_search} plots Adam's validation loss after $10{,}000$ steps across widths and learning rates, with red stars marking the optimal learning rates. The width-$192$ model exhibits a qualitatively different validation-loss--learning-rate profile from the width-$768$, $576$, and $384$ models, suggesting that it is no longer well described by the same \ac{mf}-style approximation. Thus, approximate \ac{rgi} can persist even beyond the regime where the \ac{mf} approximation appears adequate.

\section{Case Study of Over-Parameterized Quadratic Models}\label{sec:theory}

Section~\ref{sec:relation} shows that \ac{rgi} cannot be explained by \ac{ntk} or \ac{mf} regimes alone. We fill this theoretical gap in this section by studying the behavior of GD, Adam, and Muon on a quadratic problem. Given a linear target response $y(x)=Hx$ with $H\in\bbR^{q\times  r}, x\in\bbR^{r}$ and $q\geq 3r$, we adopt a two-layer overparameterized model $f_{A,B}(x)=ABx$ to learn it, where $A\in\bbR^{q\times s}, B\in\bbR^{s\times r}$ with $s\geq 3r$. With the dataset $\{(y_i,x_i)\}_{i=1}^{n}$, we train the parameters $A,B$ via the loss  $L(A,B)=1/2n\cdot\sum_{i=1}^n \|y_i - ABx_i\|_2^2.$ We require the target matrix $H$ to have the following structure.

\begin{assumption}\label{ass:H}
The target matrix $H\in\mathbb{R}^{q\times r}$ satisfies $H^\top H=\mu^2 I_r, \operatorname{sign}(H)=\kappa H$
for some $\mu>0$ and $\kappa>0$, where the sign map is applied entrywise with
$\operatorname{sign}(0)=0$.
\end{assumption}
This assumption imposes two requirements. First, the columns of $H$ are orthogonal and have the same norm. Second, all nonzero entries of $H$ must have the same magnitude. This assumption is satisfied by a broad class of matrices; for example, the columns of $H$ may form either a one-hot basis or a Hadamard basis. 

\begin{assumption}\label{ass:x}
    The covariates $\{x_i\}_{i=1}^{n}$ are isotropic, i.e., $\sum_{i=1}^n x_i x_i^\top/n=\sigma_x^2 I_r$ for some $\sigma_x>0$.
\end{assumption}
This assumption requires the training covariates to cover the feature space approximately isotropically. We emphasize that Assumptions~\ref{ass:H} and~\ref{ass:x} define a stylized setting under which \ac{rgi} can hold across optimizers. In Appendix~\ref{app:verify}, we provide empirical evidence that these assumptions are relevant to practical \ac{llm} training; relaxing them is left for future work.

We then optimize the parameters $A$ and $B$ of our model with three optimizers, $\gd$, $\adam$, and $\muon$. Given any parameter $W\in\{A,B\}$, these optimizers update the parameter $W_t$ at step $t$ as follows.
\begin{itemize}[leftmargin=1em]
\item  $\gd$ updates the parameter according to the gradient, i.e., $W_{t+1}^{\gd} = W_{t}^{\gd}-\eta_t\nabla_W L(W_{t}^{\gd})$. Here, the notation $L(W_{t}^{\gd})$ suppresses its dependency on other parameters for conciseness.
\item With $\gamma_1=\gamma_2=0$, Adam updates the parameter with the coordinate-wise normalized gradient:
$W_{t+1}^{\adam}=W_{t}^{\adam}-\gamma_t\sgn(\nabla_W L(W_{t}^{\adam}))$,
where $\sgn(\cdot)$ denotes the entrywise sign operator.
\item With momentum set to $0$, Muon updates the parameters using the spectrally normalized gradient:
$W_{t+1}^{\muon}=W_{t}^{\muon}-\tau_t\spec(\nabla_W L(W_{t}^{\muon}))$,
where $\spec(G)=UV^\top$ for $G=U\Sigma V^\top$, which normalizes all nonzero singular values $\Sigma$ of $G$ to $1$.
\end{itemize}

\iffalse
\begin{assumption}
\label{ass:auxiliary_frames}
Let $G,J\in\mathbb{R}^{q\times r}$ be two auxiliary matrices with orthonormal columns chosen in the orthogonal complement of the target column
space $\operatorname{col}(H)$ such that $G^\top G=J^\top J=I_r, H^\top G=H^\top J=G^\top J=0$.
\end{assumption}

To decouple the end-to-end predictor from parameter-space similarity, we use an
over-parameterized hidden layer of width $d\ge 3r$. We decompose the hidden
space as $\bbR^{d}=\bbR^r\oplus\bbR^r\oplus\bbR^{d-2r}$.
For a hidden vector $z\in\bbR^{d}$, write $z=(z_1,z_2,z_3)$, $z_1,z_2\in\bbR^r$ and $z_3\in\bbR^{d-2r}$.
For the existence construction, we choose initializations, step sizes,
and admissible update selections so that the input representation remains
supported on the first hidden block. Concretely, $B$ maps
inputs only into the first block, so the hidden representation has the form
$(B_1x,0,0)$. If $A=[A_1,A_2,A_3]$ is decomposed into matching column blocks, then
$ABx=A_1B_1x$. Thus the end-to-end predictor depends only on the active pair $(A_1,B_1)$. The remaining blocks $A_{2}$ and $A_{3}$ are
predictor-invisible. The auxiliary matrices $G$ and $J$ in Assumption~\ref{ass:auxiliary_frames} are inserted into these invisible blocks. They
do not change $AB$, but they allow different optimizer trajectories to
accumulate large parameter components in mutually orthogonal directions. This
is the mechanism that makes the parameter cosine small while preserving the
same end-to-end predictor.
\fi

\begin{theorem}[A stylized realization of exact RGI with low parameter cosine]
\label{thm:quadratic_case_study}
For each optimizer $\sfA\in\{\gd,\adam,\muon\}$, set the initial parameters as $A_0^{\sfA}=0,
B_0^{\sfA}=[b_0^{\sfA} I_r,0]^\top$ with  any $b_0^{\sfA}>0$ and set the step size to constants as
\begin{equation*}
    \eta_t=\eta>0,\,
    \gamma_t=\gamma\!<\!\min\{b_0^{\adam},(\kappa b_0^{\adam}+\sqrt{\kappa})^{-1}\},\,
    \tau_t=\tau\!<\!\min\{b_0^{\muon},( b_0^{\muon}/\mu+\mu^{-1/2})^{-1}\}.
\end{equation*}
Under Assumptions~\ref{ass:H} and \ref{ass:x},  there exist tie-breaking rules on zero-gradient, for any two optimizers $\sfA,\sfA'\in\{\gd,\adam,\muon\}$, any two inputs $x,x'$ with
$\|x\|_2=\|x'\|_2=\rho$, and any $t\geq 0$, we have
\begin{align*}
    L_t^\sfA(x)- L_t^\sfA(x')
    =
     L_t^{\sfA'}(x)- L_t^{\sfA'}(x').
\end{align*}

Moreover, let $\theta_t^\sfA=(\operatorname{vec}(A_t^\sfA),\operatorname{vec}(B_t^\sfA))$. At the same time, with cosine similarity defined to be zero whenever either vector is zero, the following holds as $t$ increases.
\begin{align*}
    |\cossim(\theta_t^{\gd}\!\!,\theta_t^{\adam})|\!=\!O(1/t),
    |\cossim(\theta_t^{\gd}\!\!,\theta_t^{\muon})|\!=\!O(1/t),
    |\cossim(\theta_t^{\adam}\!\!,\theta_t^{\muon})|\!=\!O(1/t^2).
\end{align*}
\end{theorem}
This theorem shows that \ac{rgi} holds across $\gd$, $\muon$, $\adam$, and different initializations under Assumptions~\ref{ass:H} and~\ref{ass:x}, with the covariate $x$ playing the same role as the token $d$ in Definition~\ref{def:rgi}. These results support the empirical findings in Sections~\ref{sec:opt} and~\ref{sec:arch}. The theorem also permits a broad range of learning rates, imposing only upper-bound constraints on those of $\muon$ and $\adam$, consistent with the robustness of \ac{rgi} to optimizer hyperparameters observed in Section~\ref{sec:opt}. Moreover, we show that the cosine similarity between parameters decreases during training, matching the low parameter similarity reported in Section~\ref{sec:ntk}. Overall, our analysis identifies a class of quadratic problems in which \ac{rgi} is robust across optimizers, initializations, and learning rates.

\section{Conclusion}\label{sec:conclusion}
We study how optimizers, architectures, and training data streams affect model performance through the lens of relative generalization. We show that \ac{rgi} persists under joint variations in optimizers and architectures, which largely preserve relative generalization and induce approximately uniform shifts in token-wise losses, whereas changes in the training data stream can substantially alter it. Neither the \ac{ntk} nor the \ac{mf} regime alone explains \ac{rgi}; instead, we identify a class of quadratic problems in which \ac{rgi} emerges robustly across initializations and learning rates. Our analysis focuses on autoregressive \acp{llm}, leaving other model classes, such as diffusion models, for future work.

\bibliographystyle{ims}
\bibliography{reference}

\newpage 

\appendix
\section{Experimental Details}\label{app:exp_detail}

\subsection{Concordance Correlation Coefficient}
We present the detailed definition of \ac{ccc} among centered losses $L(\theta_t^{\frakT},D_{\val,k})$ and $L(\theta_t^{\frakT^{\prime}},D_{\val,k})$ as follows. The centered losses are defined as
\begin{align*}
    \barL\big(\theta_t^{\frakT},D_{\val,k}\big)&= L\big(\theta_t^{\frakT},D_{\val,k}\big)-K^{-1}\sum_{k^{\prime}=1}^{K}L\big(\theta_t^{\frakT},D_{\val,k^{\prime}}\big).
    % \barL\big(\theta_t^{\frakT^{\prime}},D_{\val,k}\big)&= L\big(\theta_t^{\frakT^{\prime}},D_{\val,k}\big)-K^{-1}\sum_{k^{\prime}=1}^{K}L\big(\theta_t^{\frakT^{\prime}},D_{\val,k^{\prime}}\big).
\end{align*}
Then we have that
\begin{align*}
    \cccerr(\frakT,\frakT^{\prime}) = \bbE\Bigg[\frac{2\,\cov_K\Big(\barL\big(\theta_t^{\frakT},D_{\val,k}\big),\barL\big(\theta_t^{\frakT^{\prime}},D_{\val,k}\big)\Big)}{\var_K\Big(\barL\big(\theta_t^{\frakT},D_{\val,k}\big)\Big)+\var_K\Big(\barL\big(\theta_t^{\frakT^{\prime}},D_{\val,k}\big)\Big)}\Bigg],
\end{align*}
where the expectation is taken with respect to the joint distribution of $\theta_0$ and $(D_l)_{l=1}^{t}$, and $\cov_K$ and $\var_K$ are covariance and variance with respect to $k$. In other words, for any functions $f(D_{\val,k})$ and $g(D_{\val,k})$, we have
\begin{align*}
    \var_K\big(f(D_{\val,k})\big)
    &=\frac{1}{K}\sum_{k=1}^{K}
    \Big(
    f(D_{\val,k})
    -\frac{1}{K}\sum_{k^{\prime}=1}^{K}f(D_{\val,k^{\prime}})
    \Big)^2 \\
    \cov_K\big(f(D_{\val,k}),g(D_{\val,k})\big)&= \frac{1}{K}\sum_{k=1}^{K}f(D_{\val,k})g(D_{\val,k})-\frac{1}{K^2}\sum_{k^{\prime}=1}^{K}f(D_{\val,k^{\prime}})\!\!\cdot\!\! \sum_{k^{\prime}=1}^{K}g(D_{\val,k^{\prime}}).
\end{align*}

\subsection{Definition of $\differr$}
\begin{align*}
    s(\frakT,\frakT^{\prime})&= 
    \sum_{k,k^{\prime}\in[K]}
    \Big(
    \Delta(\frakT,D_{\val,k},D_{\val,k^{\prime}})
    -
    \Delta(\frakT^{\prime},D_{\val,k},D_{\val,k^{\prime}})
    \Big)^2\\
    \differr(\frakT,\frakT^{\prime})
    &=
    \frac{1}{2}\cdot\bbE\Bigg[
    \frac{
    s(\frakT,\frakT^{\prime})
    }{
    \sum_{k,k^{\prime}\in[K]}
    \Delta(\frakT,D_{\val,k},D_{\val,k^{\prime}})^2
    }+\frac{
    s(\frakT,\frakT^{\prime})
    }{
    \sum_{k,k^{\prime}\in[K]}
    \Delta(\frakT^{\prime},D_{\val,k},D_{\val,k^{\prime}})^2
    }
    \Bigg],
\end{align*}
where the expectation is taken with respect to the joint distribution of $\theta_0$ and $D_{1:t}$.

\subsection{Experimental Setups}

In all experiments, we train \acp{llm} using the GPT-2 tokenizer with sequence length $1024$ and batch size $524{,}288$. For each optimizer, we disable weight decay and search the learning rate over $\{1,2,5\}\times\{10^{-1},10^{-2},10^{-3},10^{-4}\}$. For Adam, Adam-mini, and Lion, we set $\beta_1=0.9$ and $\beta_2=0.95$; for SOAP, we set $\beta_1=\beta_2=0.95$; for Muon and Scion, we use momentum values $0.95$ and $0.9$, respectively. We consider two GPT-style model sizes. The $124$M model has $12$ layers, $12$ attention heads, and hidden dimension $768$, while the $0.7$B model has $36$ layers, $20$ attention heads, and hidden dimension $1280$. For evaluations, we partition 
$\calD_{\val}$ into validation batches $(D_{\val,k})_{k=1}^K$ and evaluate two 
batch granularities: $|D_{\val,k}|=1{,}024$ and $|D_{\val,k}|=1$. In the main 
text, we report results with $|D_{\val,k}|=1{,}024$ and $K=20{,}480$, 
corresponding to approximately $21$M validation tokens. The single-token results, 
i.e., $|D_{\val,k}|=1$, are computed on $5$M tokens due to computational cost and 
deferred to Appendix~\ref{app:add_result}; they yield conclusions consistent 
with the $|D_{\val,k}|=1{,}024$ setting. 

\paragraph{Compute Resources.} All experiments are conducted on NVIDIA A40 GPUs. Each 124M model run uses 4 A40 GPUs and takes approximately 10 hours, while each 0.7B run uses 4 A40 GPUs and takes approximately 50 hours. Each 4.5B run uses 8 A40 GPUs and takes approximately three days.

\section{Additional Experimental Results}\label{app:add_result}
\subsection{Additional Results of Validation Losses Between Optimizers}

\begin{figure}[H]
    \centering
    \subfigure[Validation Losses of Adam and Adam-mini]{ \includegraphics[width=0.23\textwidth]{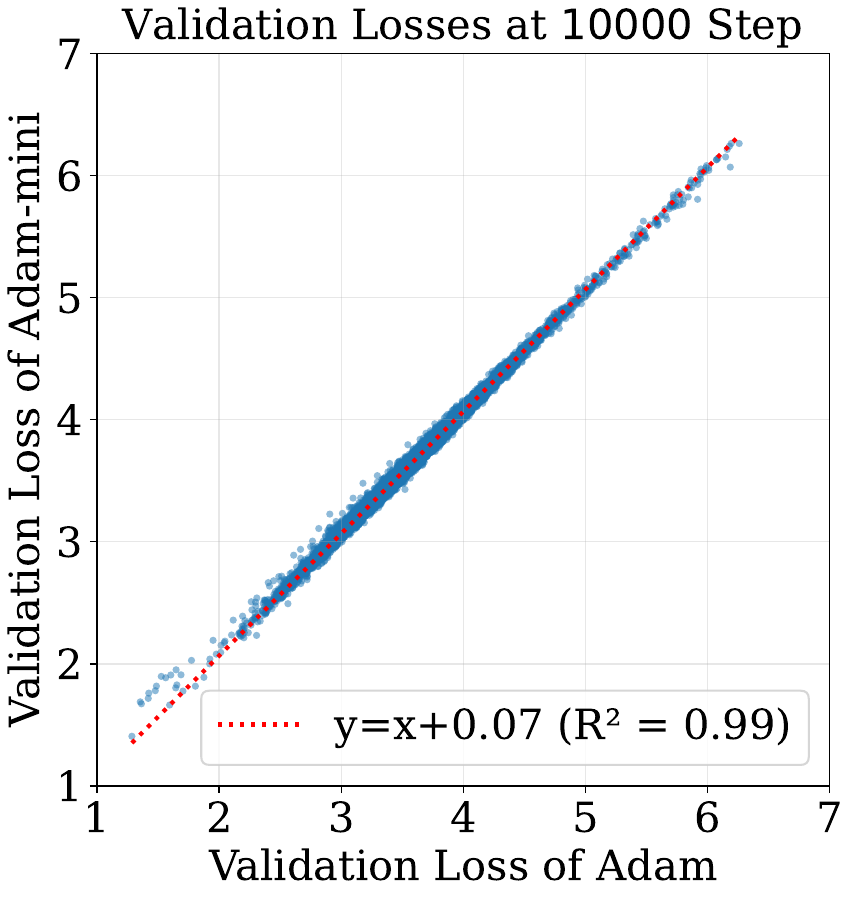}}
    \subfigure[Validation Losses of Adam and Lion]{ \includegraphics[width=0.23\textwidth]{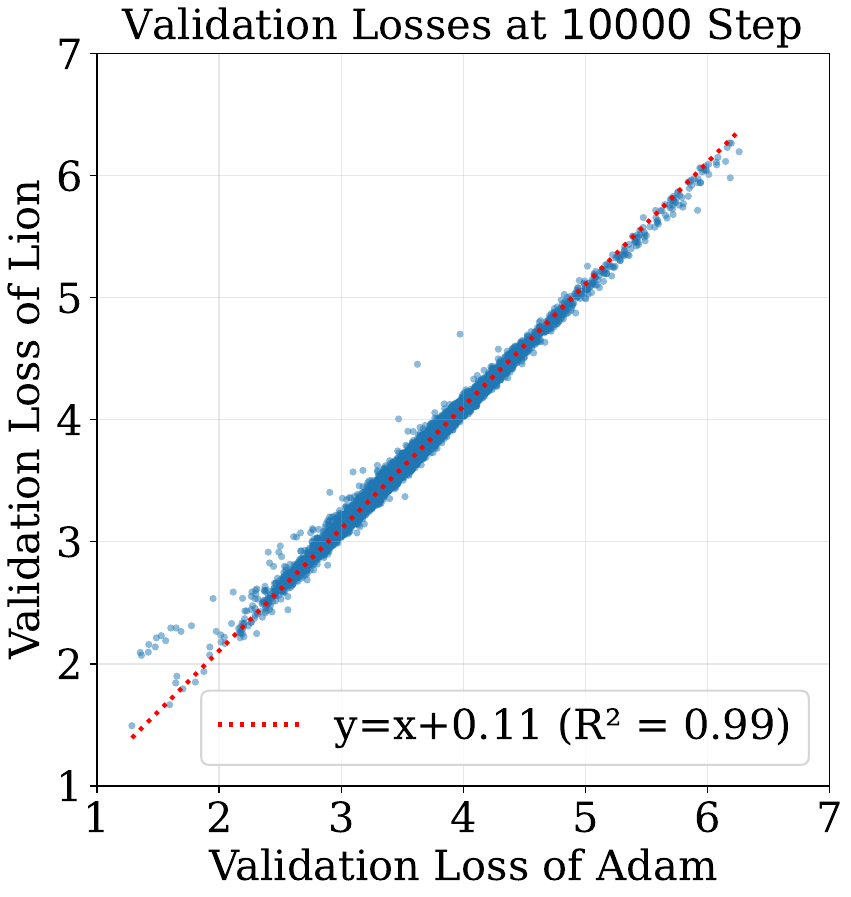}}
    \subfigure[Validation Losses of Adam and Soap]{ \includegraphics[width=0.23\textwidth]{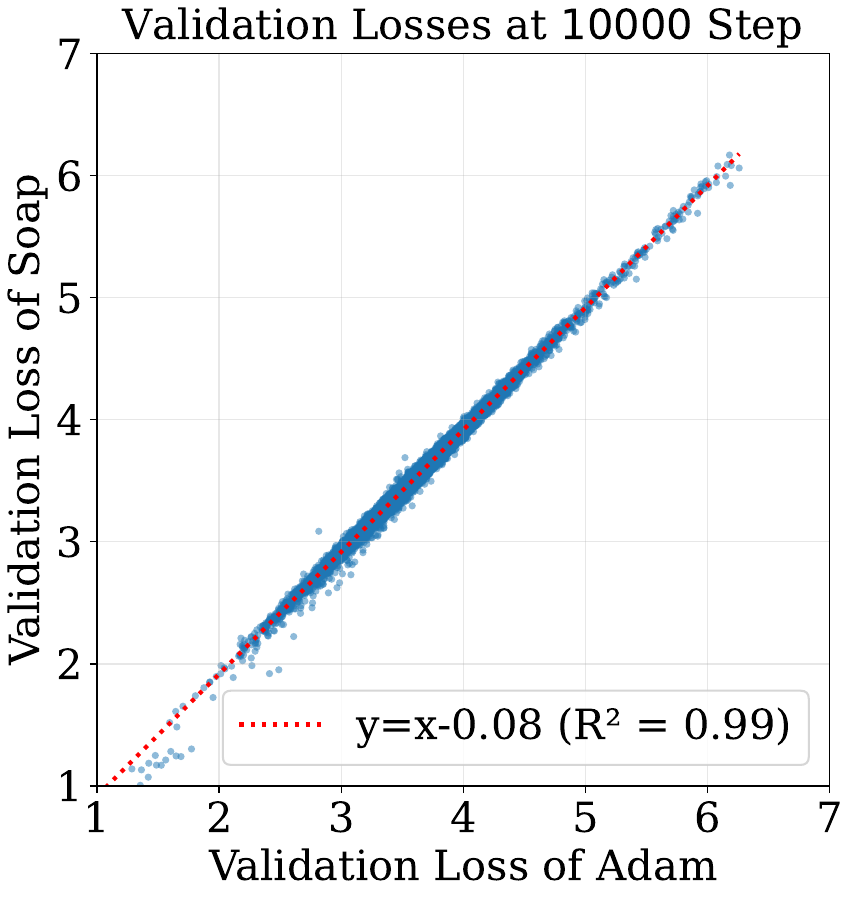}}
    \subfigure[Validation Losses of Adam and Scion]{ \includegraphics[width=0.23\textwidth]{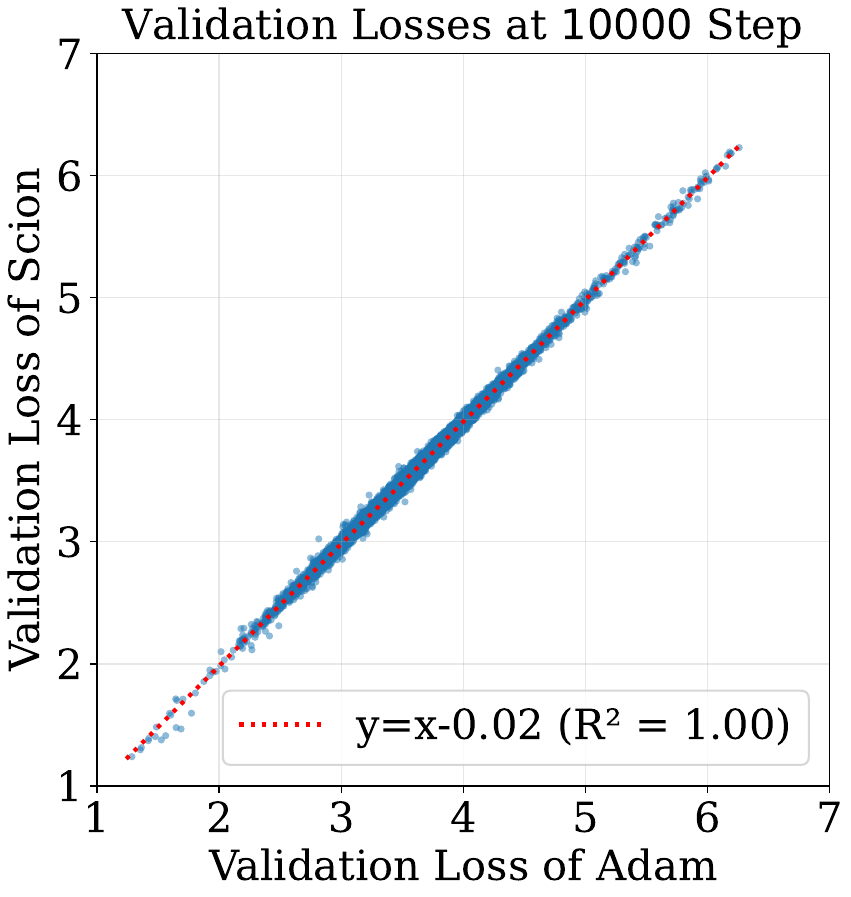}}
   
    \caption{Validation-loss linearity across Adam and other optimizers. Panels (a)--(d) plot the validation-loss relationship between Adam and Adam-mini, Lion, Soap, and Scion, respectively, at the $10{,}000$-step checkpoint. The near-perfect fits, with $R^2\geq 0.99$, support approximate \ac{rgi} across optimizers.} 
\end{figure}
These figures show that the validation losses of models trained by different optimizers are highly linearly correlated at step $10{,}000$. This supports the emergence of approximate \ac{rgi} among optimizers.

\subsection{Additional Results of $\differr$ and $\cccerr$ with Other Optimizers as Baselines }
In the following, we report $\differr$ and $\cccerr$ using optimizers other than Adam as the baseline, while keeping the initial parameters and training data stream identical across all optimizers. The results show that changing the baseline optimizer does not qualitatively affect approximate \ac{rgi}. This is intuitive: if exact \ac{rgi} holds between $\theta$ and $\theta^{\prime}$, and also between $\theta^{\prime}$ and $\theta^{\prime\prime}$, then it also holds between $\theta$ and $\theta^{\prime\prime}$. In the approximate case, the same transitivity holds up to accumulated error.
\begin{figure}[H]
    \centering
    \subfigure[Values of $\differr$ with Adam-mini as the baseline.]{ \includegraphics[width=0.27\textwidth]{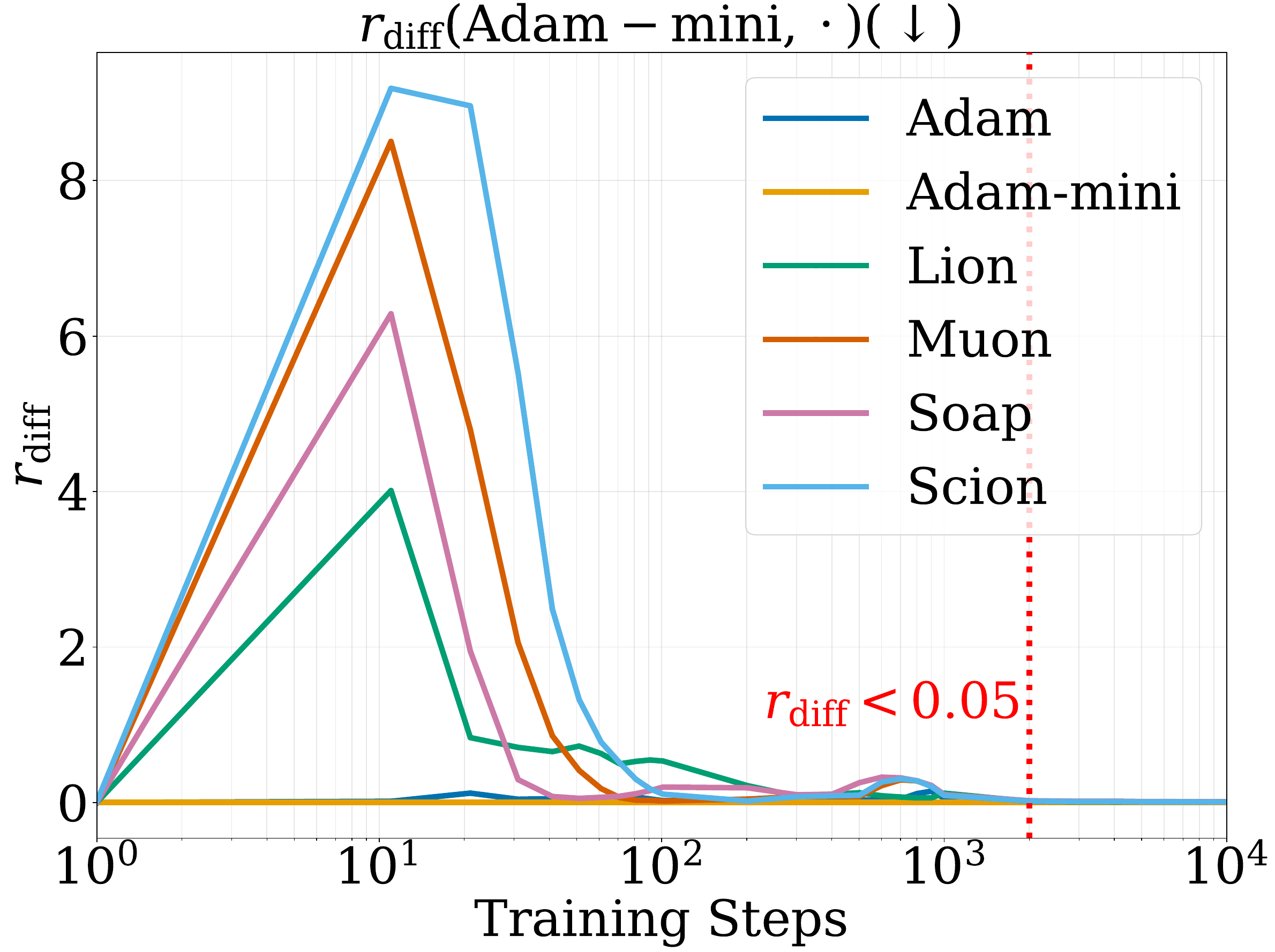}}
    \subfigure[Values of $\differr$ with Lion as the baseline.]{ \includegraphics[width=0.27\textwidth]{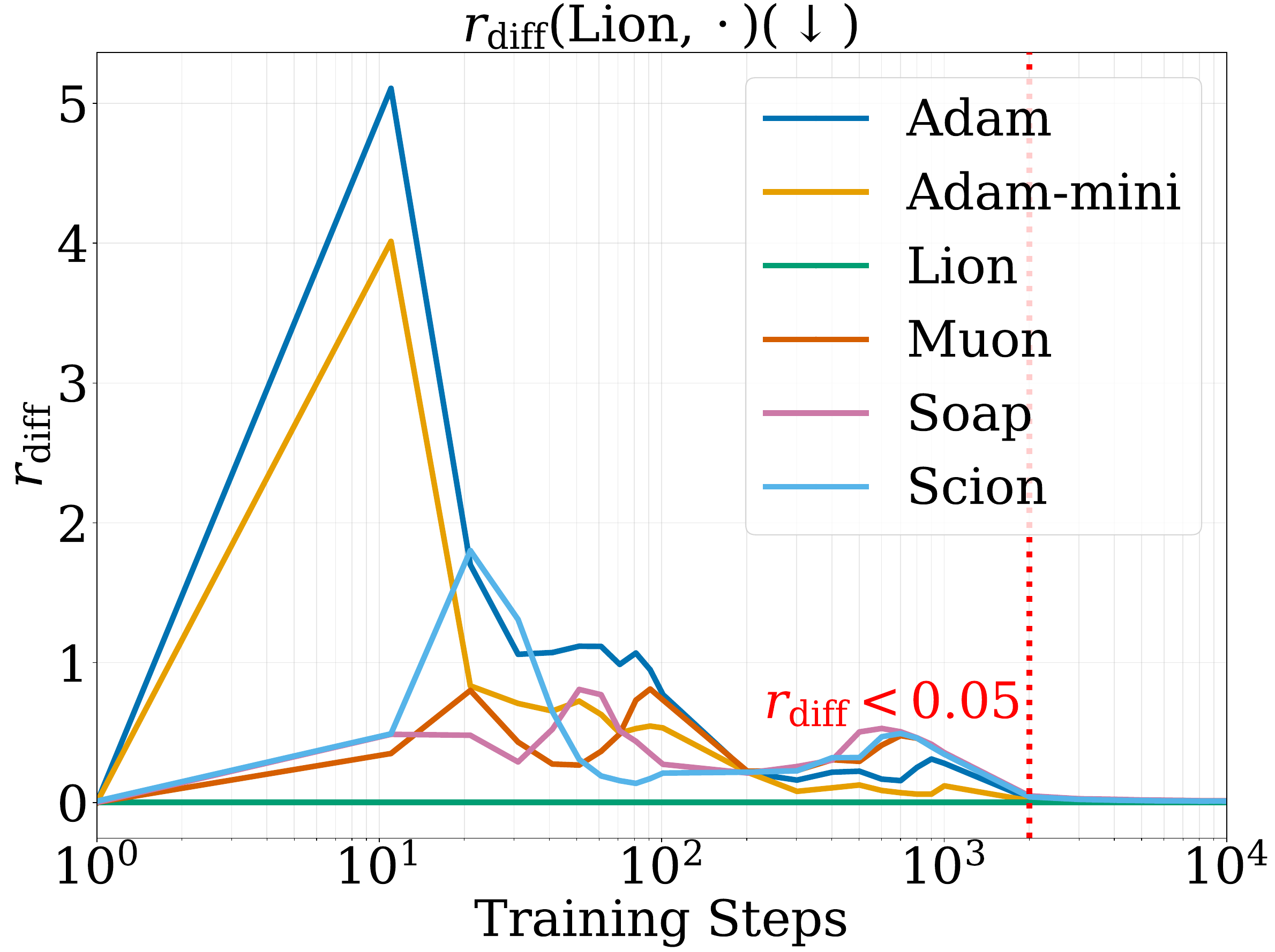}}
    \subfigure[Values of $\differr$ with Muon as the baseline.]{ \includegraphics[width=0.27\textwidth]{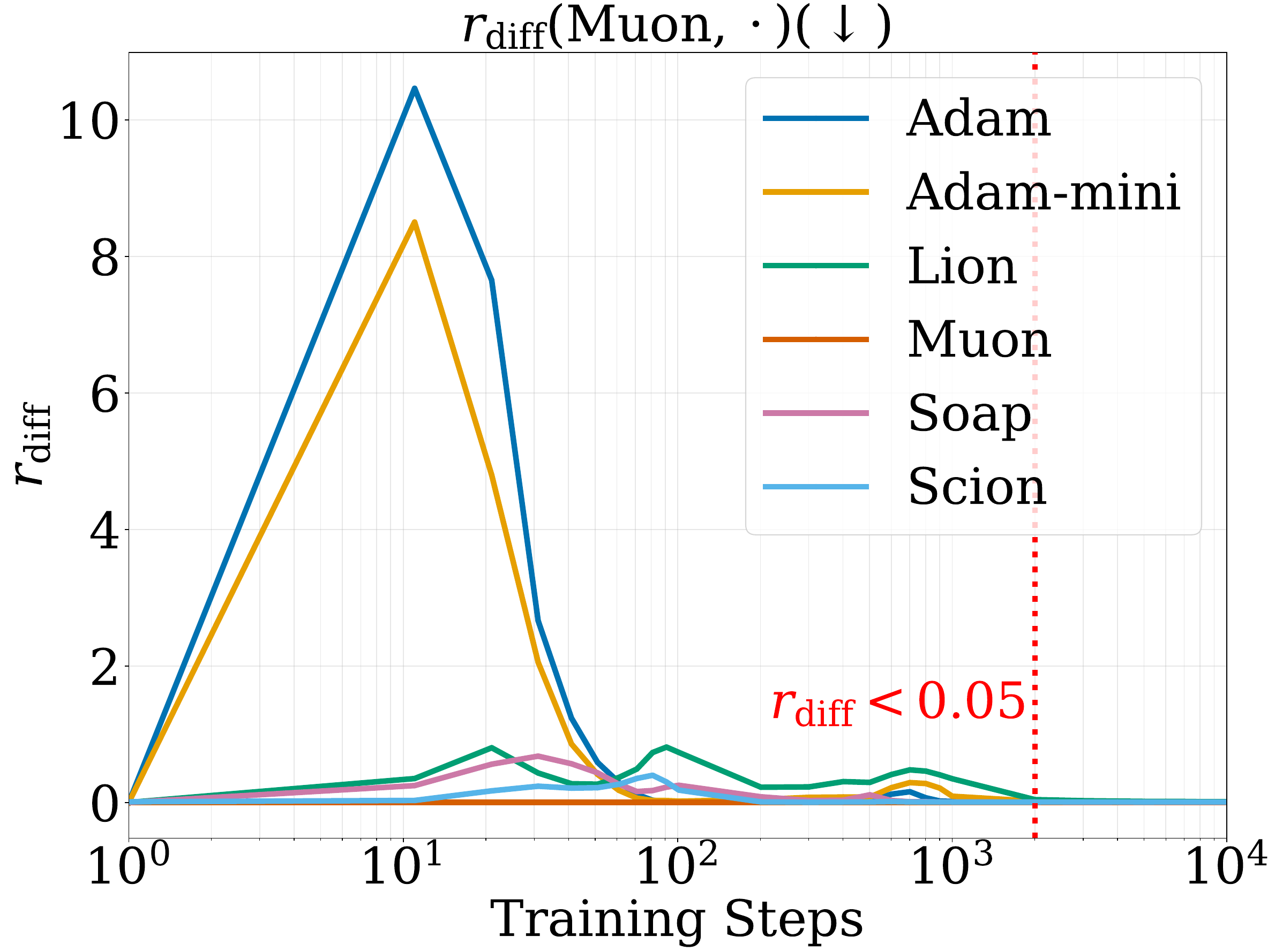}}

    \subfigure[Values of $\differr$ with Soap as the baseline.]{ \includegraphics[width=0.27\textwidth]{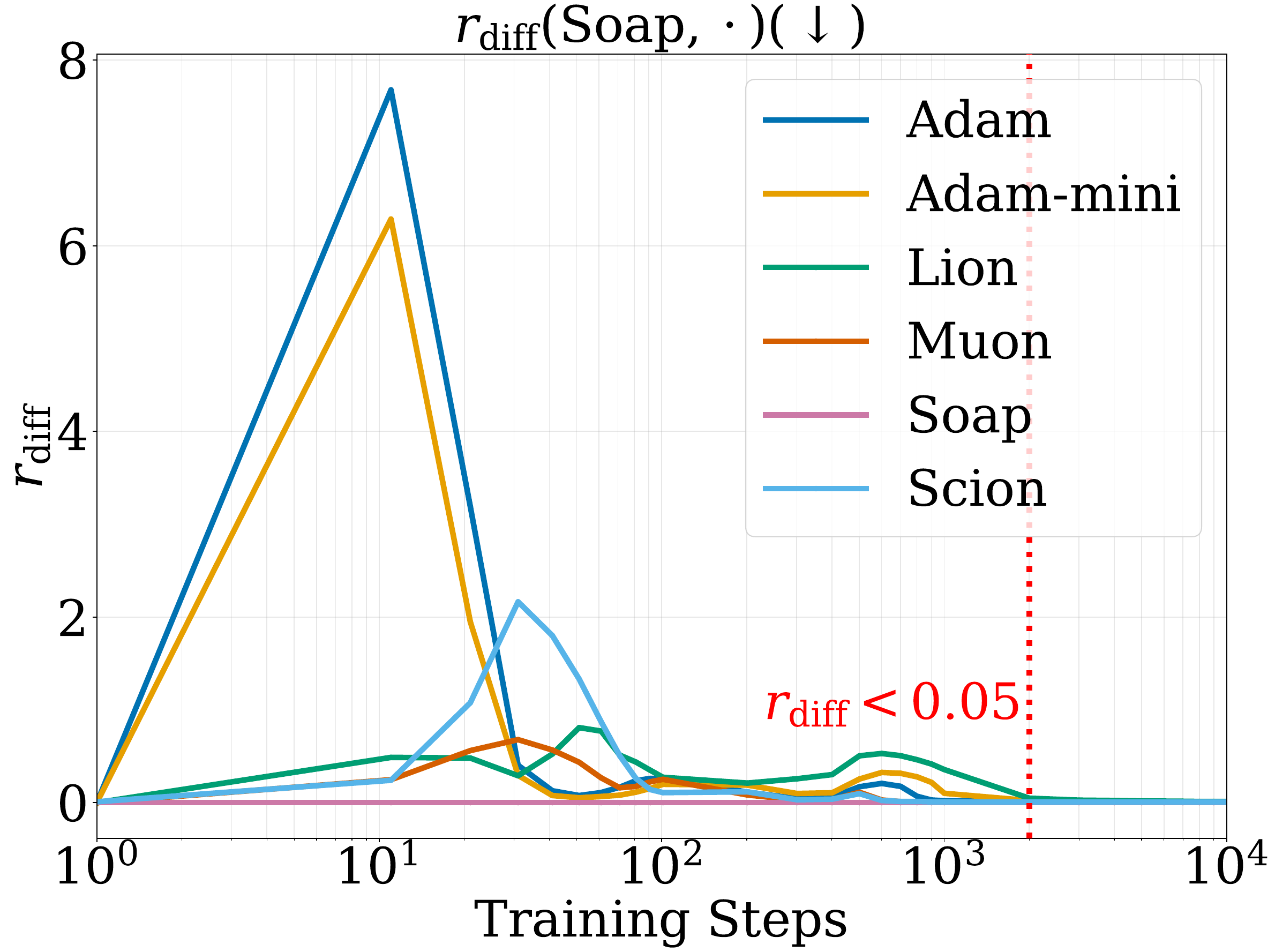}}
    \subfigure[Values of $\differr$ with Scion as the baseline.]{ \includegraphics[width=0.27\textwidth]{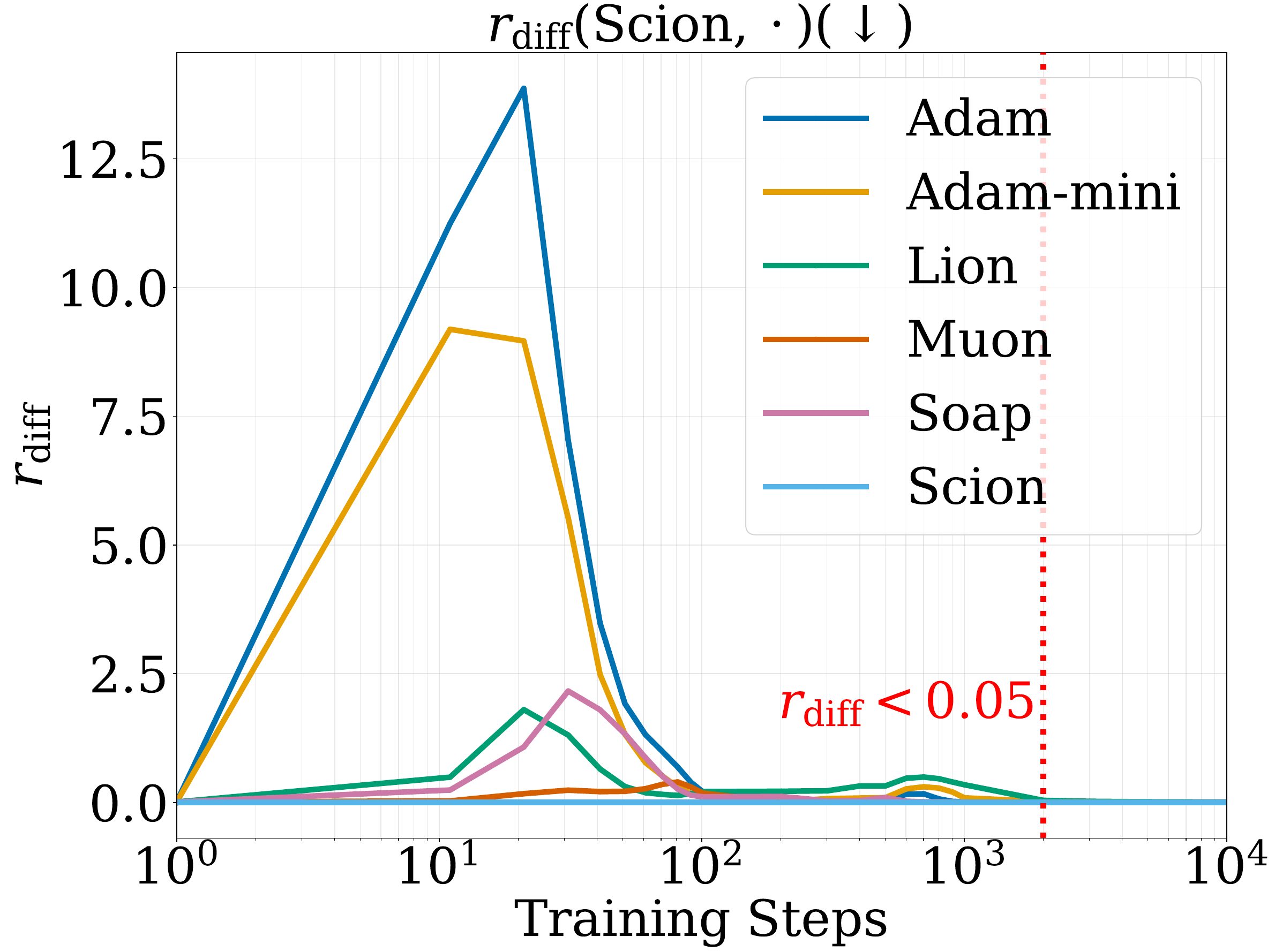}}
    
    \caption{Baseline-optimizer ablation for $\differr$. Panels (a)--(e) plot $\differr$ when the baseline optimizer is Adam-mini, Lion, Muon, Soap, and Scion, respectively. Approximate \ac{rgi} is insensitive to which optimizer is used as the evaluation baseline.} 
\end{figure}

\begin{figure}[H]
    \centering
    \subfigure[Values of $\cccerr$ with Adam-mini as the baseline.]{ \includegraphics[width=0.27\textwidth]{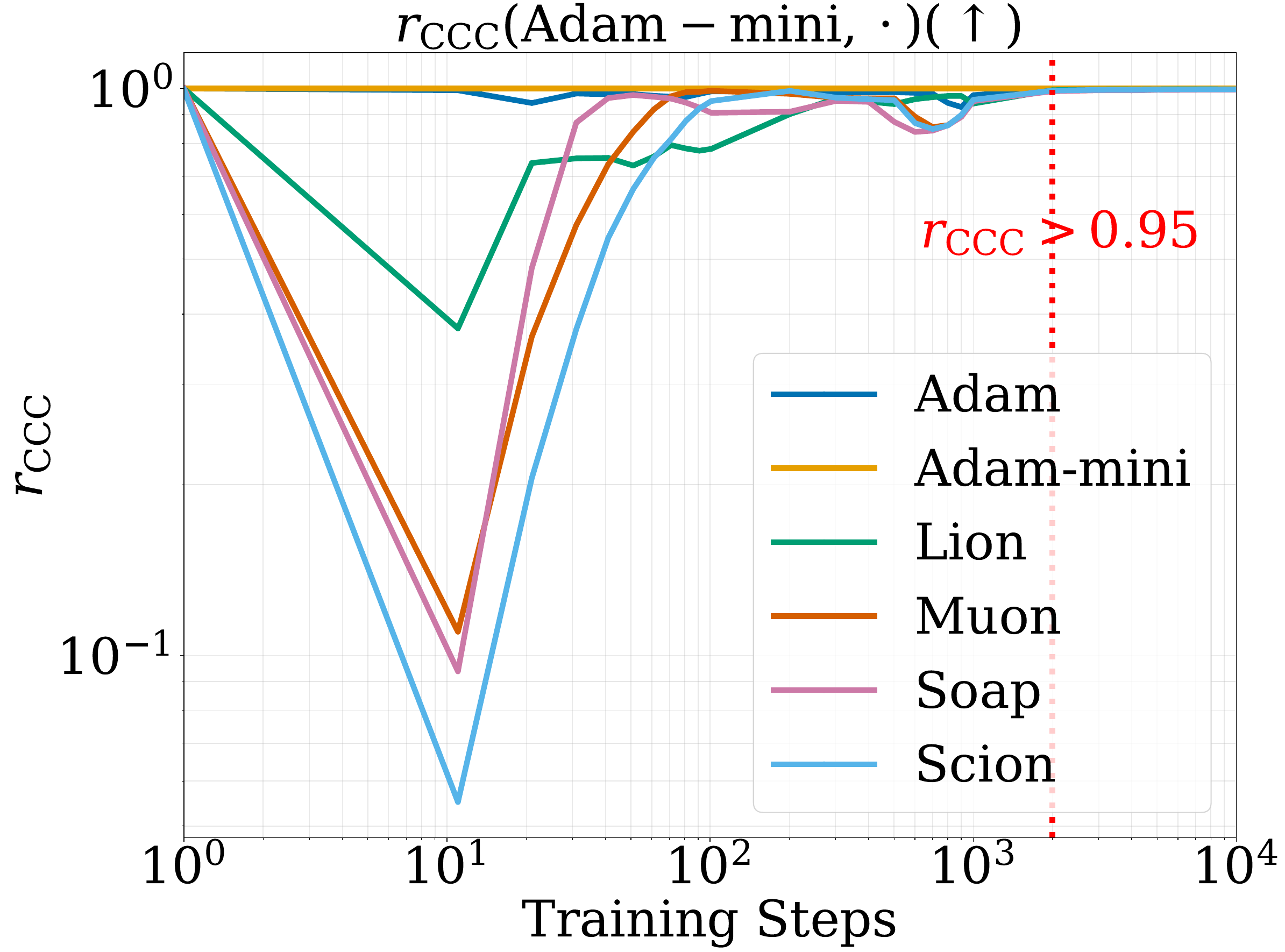}}
    \subfigure[Values of $\cccerr$ with Lion as the baseline.]{ \includegraphics[width=0.27\textwidth]{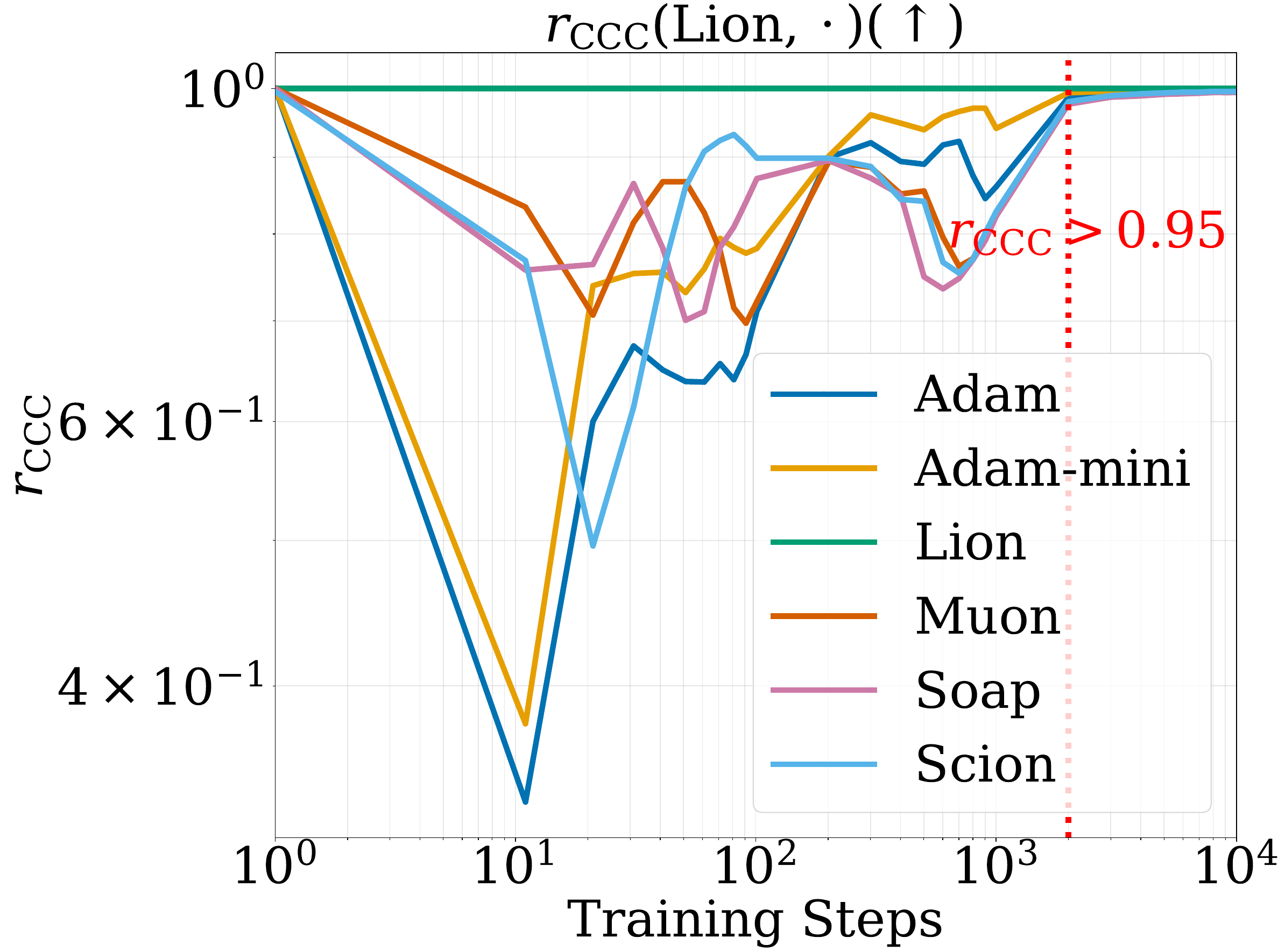}}
    \subfigure[Values of $\cccerr$ with Muon as the baseline.]{ \includegraphics[width=0.27\textwidth]{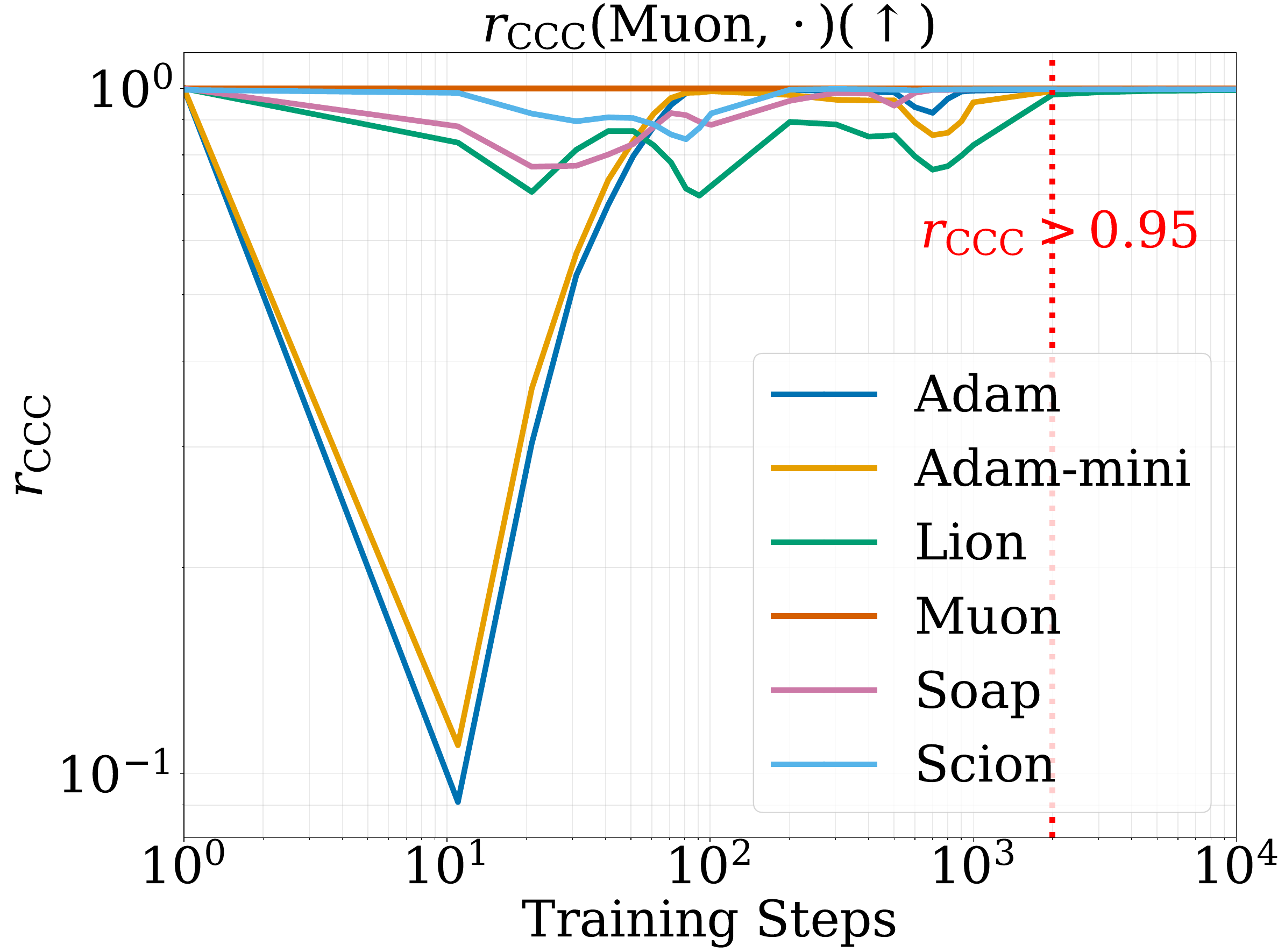}}

    \subfigure[Values of $\cccerr$ with Soap as the baseline.]{ \includegraphics[width=0.27\textwidth]{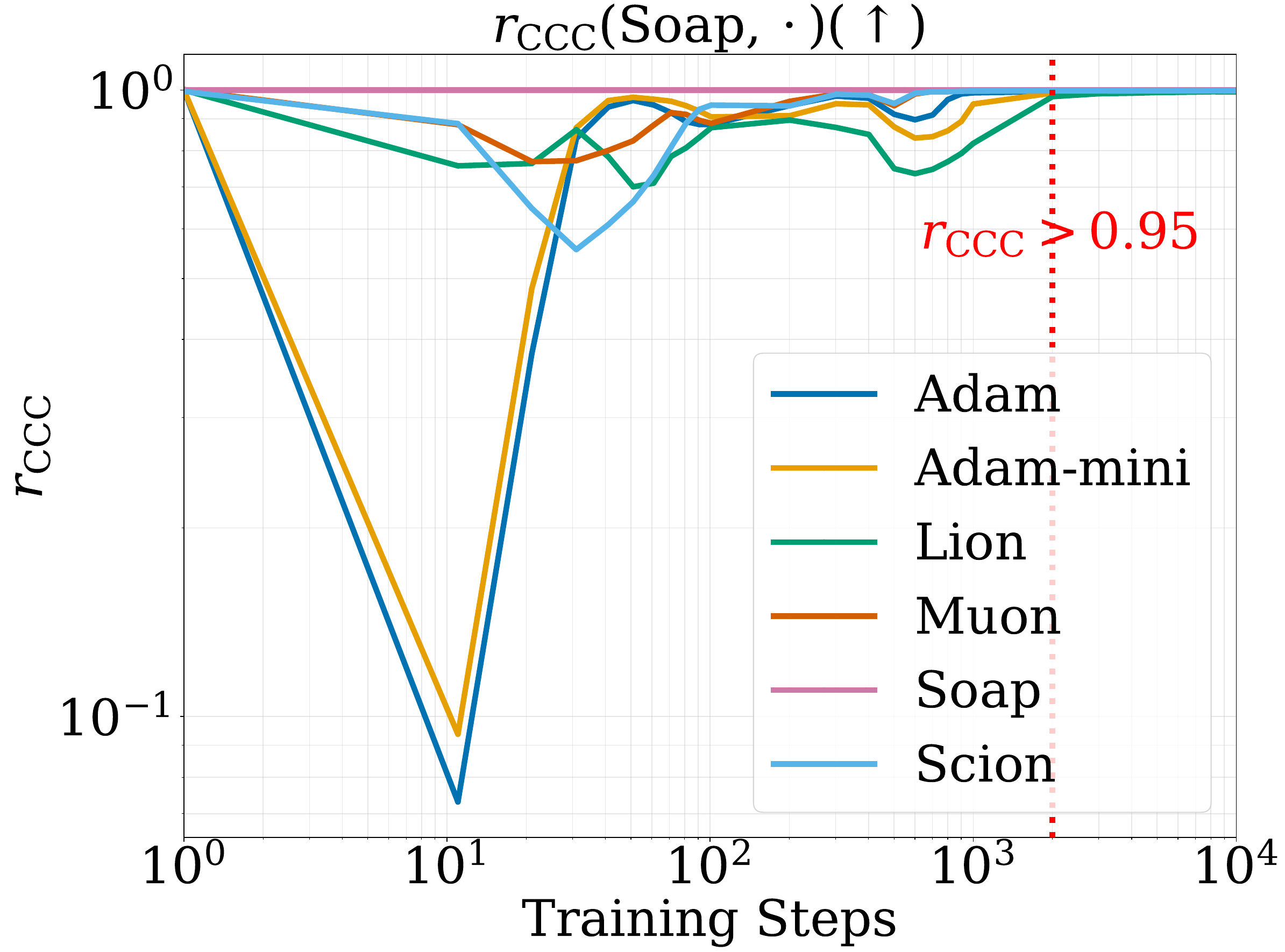}}
    \subfigure[Values of $\cccerr$ with Scion as the baseline.]{ \includegraphics[width=0.27\textwidth]{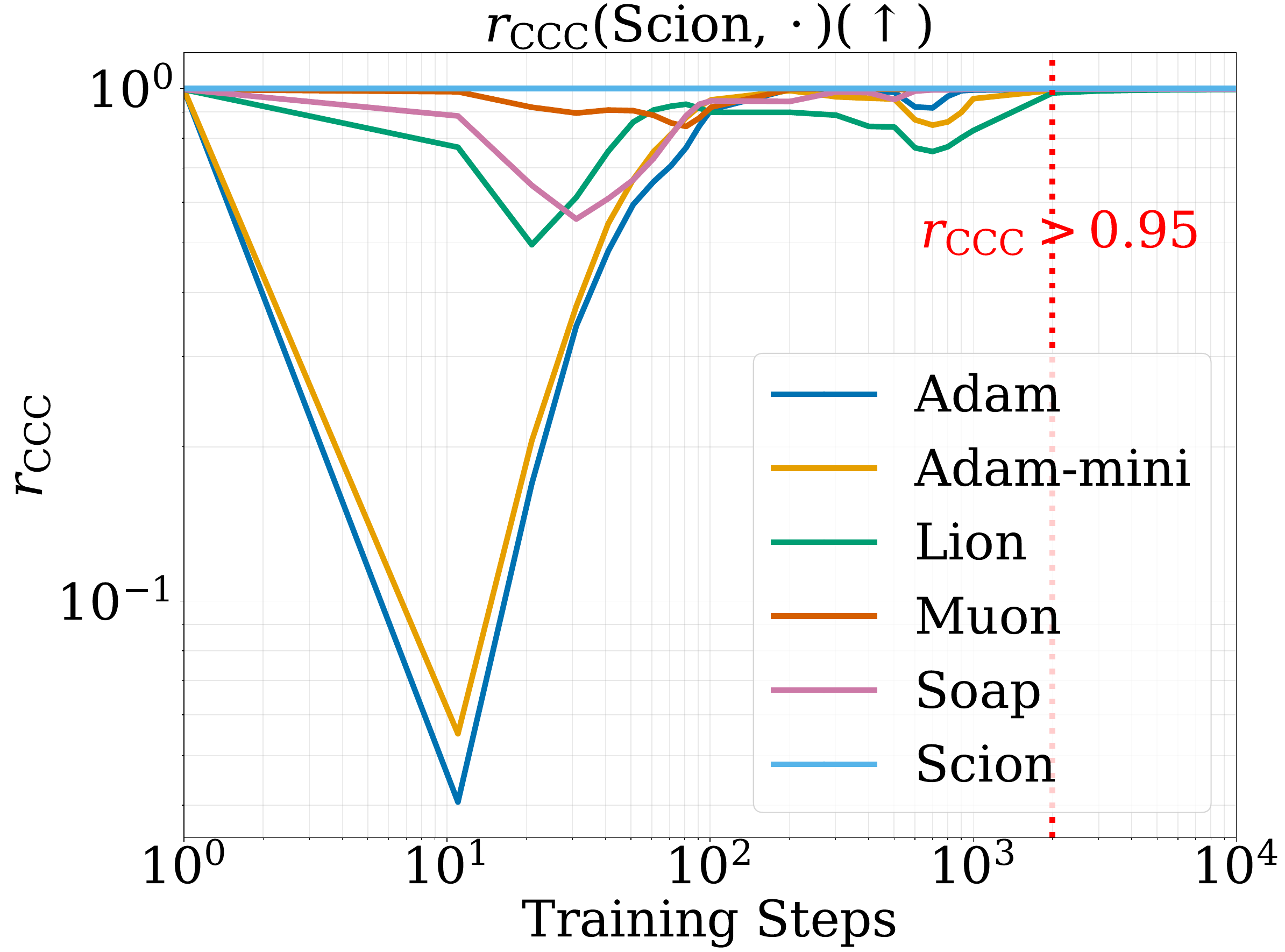}}
    
    \caption{Baseline-optimizer ablation for $\cccerr$. Panels (a)--(e) plot $\cccerr$ when the baseline optimizer is Adam-mini, Lion, Muon, Soap, and Scion, respectively. Approximate \ac{rgi} is insensitive to which optimizer is used as the evaluation baseline.} 
\end{figure}
\subsection{Per-Token Results of Relative Generalization Invariance}\label{app:per_token}
In the following, we present the results of $\differr$ and $\cccerr$ calculated with $|D_{\val,k}|=1$ when the initial parameters and the training data stream are the same for all the optimizers. The results show that \ac{rgi} among optimizers holds when we verify it in a token-wise manner.
\begin{figure}[H]
    \centering
    
    \subfigure[Values of $\differr$ calculated per token.]{ \includegraphics[width=0.35\textwidth]{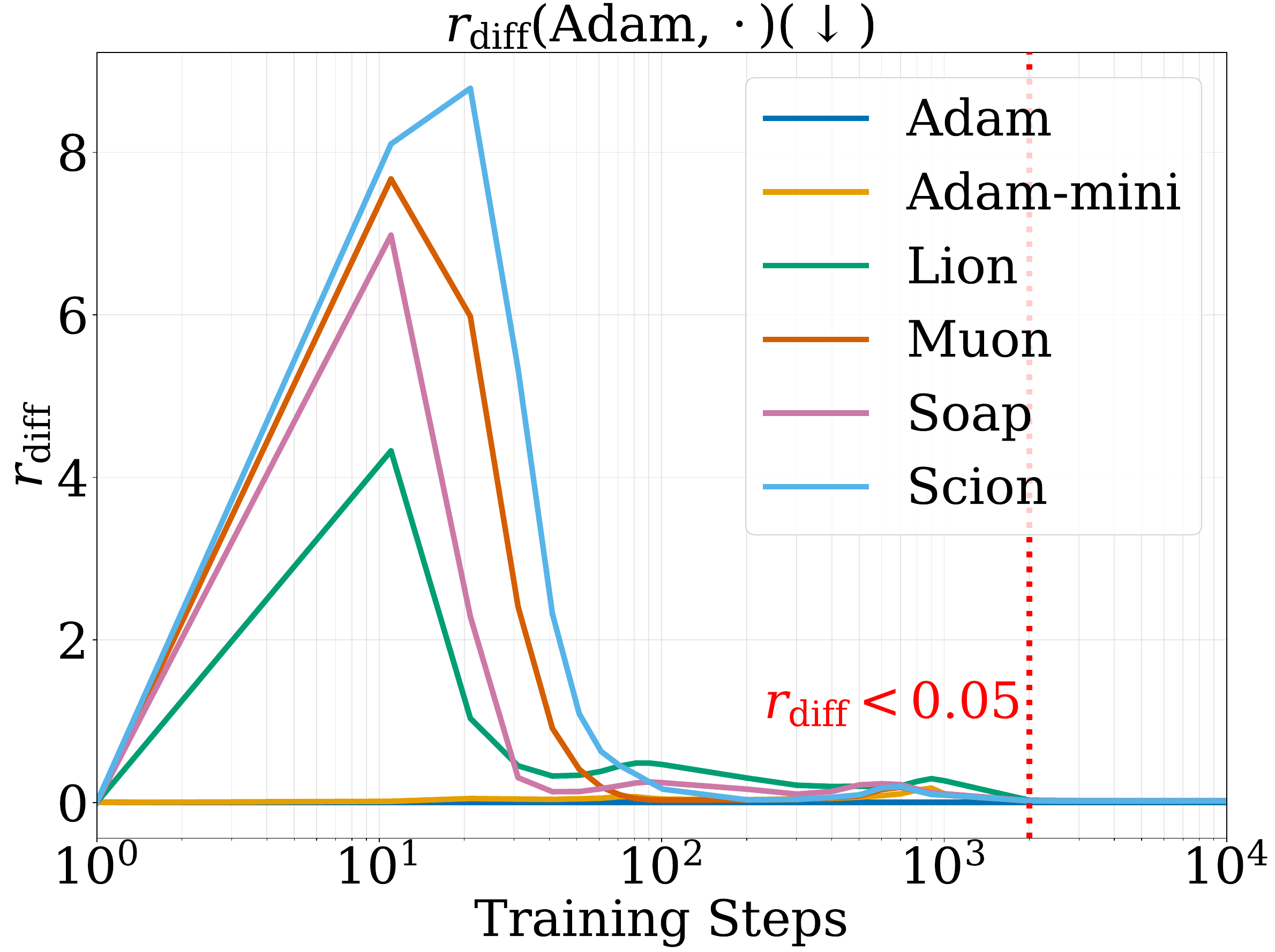}}
    \subfigure[Values of $\cccerr$ calculated per token.]{ \includegraphics[width=0.35\textwidth]{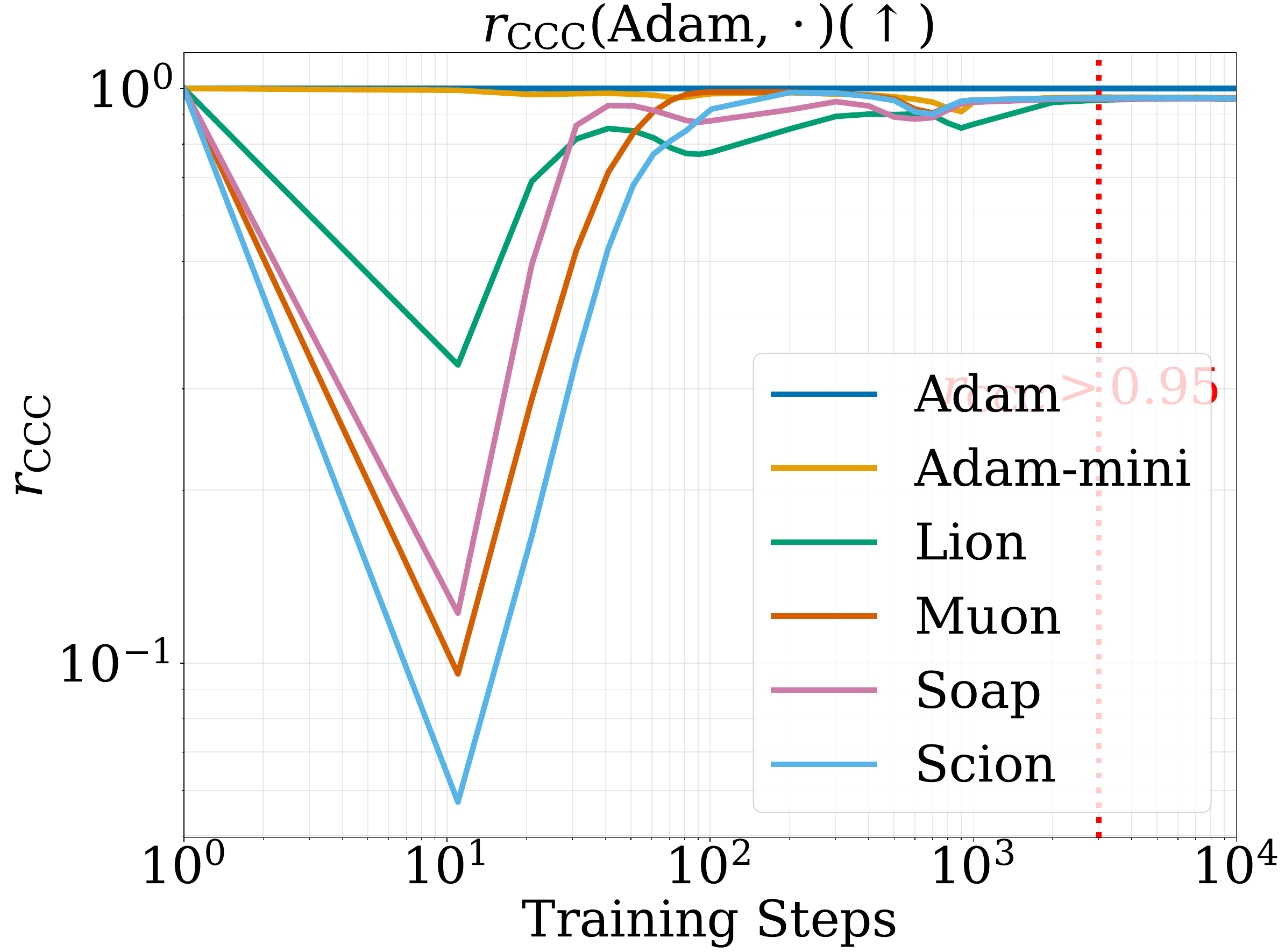}}
    
    \caption{Per-token evaluation of \ac{rgi}. Panels (a) and (b) plot $\differr$ and $\cccerr$, respectively, when each validation batch contains a single token, i.e., $|D_{\val,k}|=1$. The small $\differr$ and large $\cccerr$ values show that \ac{rgi} also holds at token-level granularity.} 
\end{figure}
\subsection{Results for $0.7$B and $4.5$B Models}\label{app:large}
Due to the high computational cost of training $0.7$B and $4.5$B models, we focus on Adam and Muon at this scale. We provide results for $0.7$B and $4.5$B models when Adam and Muon share the same initial parameters and training data stream. These results mirror those in Section~\ref{sec:opt}, showing that \ac{rgi} among optimizers also holds at larger scales.
\begin{figure}[H]
    \centering

    \subfigure[Values of $\differr$ for $0.7$B models.]{ \includegraphics[width=0.35\textwidth]{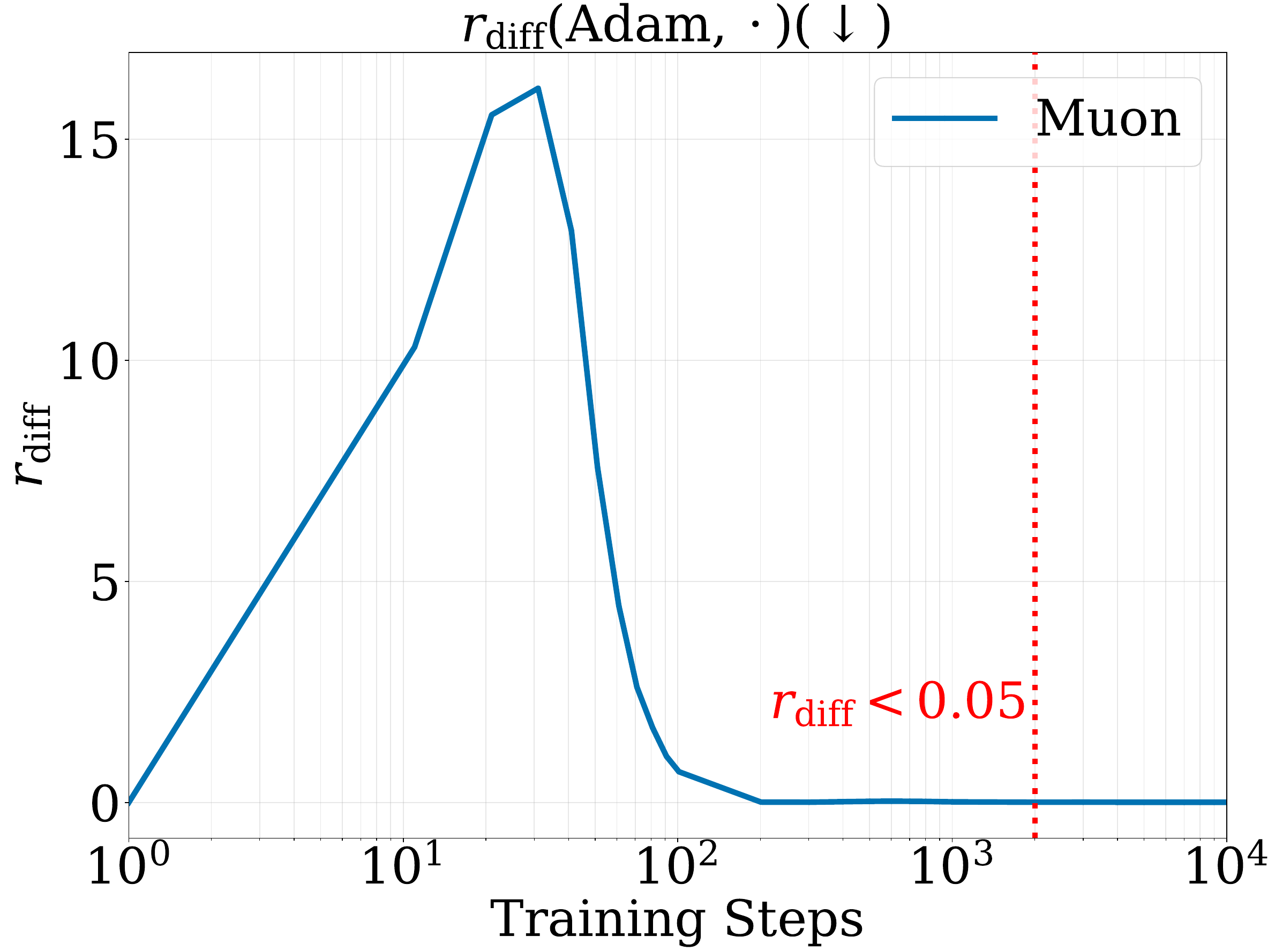}}
    \subfigure[Values of $\cccerr$ for $0.7$B models.]{ \includegraphics[width=0.35\textwidth]{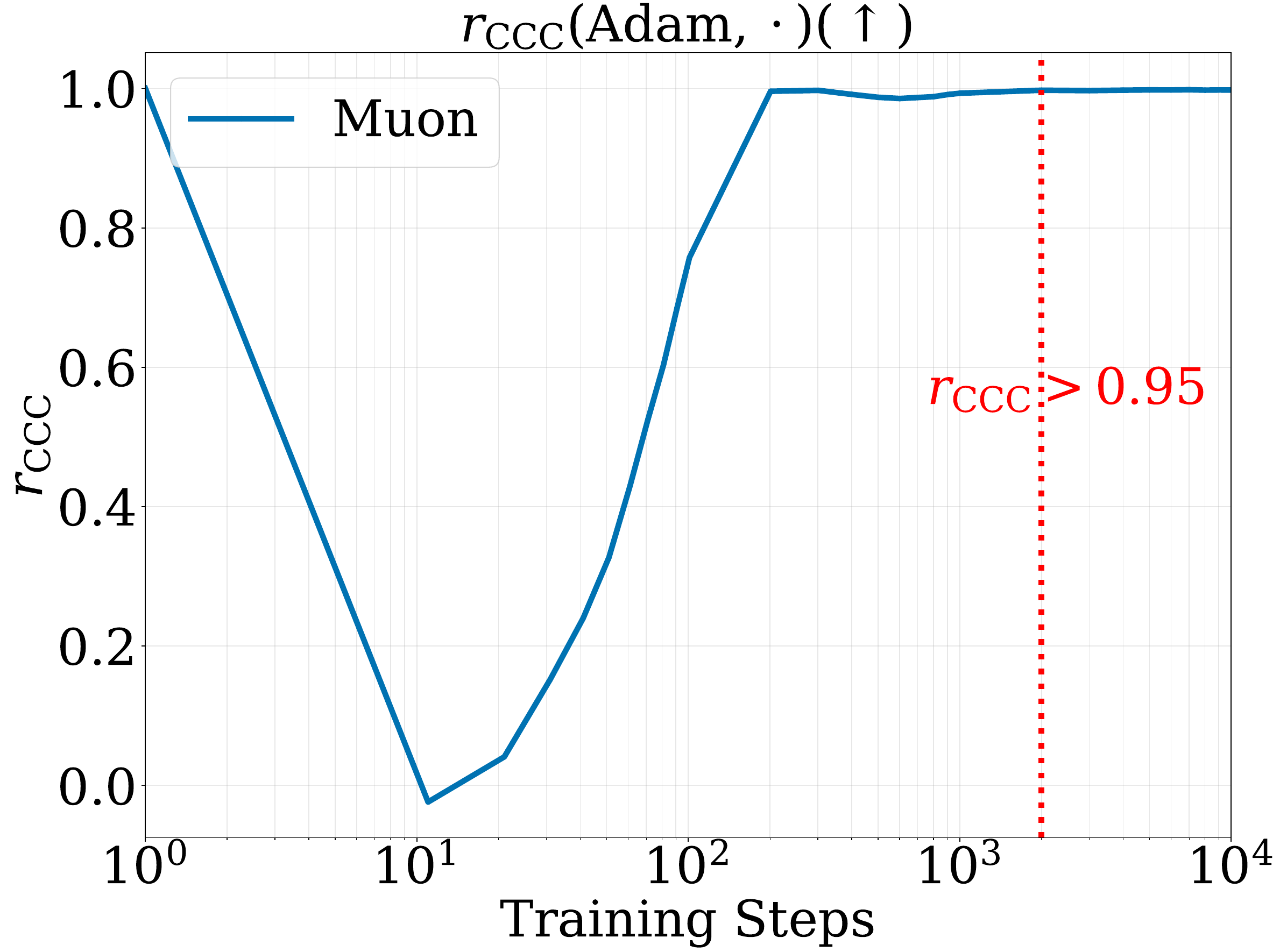}}
    
    \caption{Baseline \ac{rgi} for $0.7$B models. Panels (a) and (b) plot $\differr$ and $\cccerr$, respectively, between Adam and Muon when they share the same initial parameters and training data stream. The results show that approximate \ac{rgi} persists at the $0.7$B scale.} 
    \label{fig:0.7base}
\end{figure}

\begin{figure}[H]
    \centering

    \subfigure[Values of $\differr$ for $4.5$B models.]{ \includegraphics[width=0.35\textwidth]{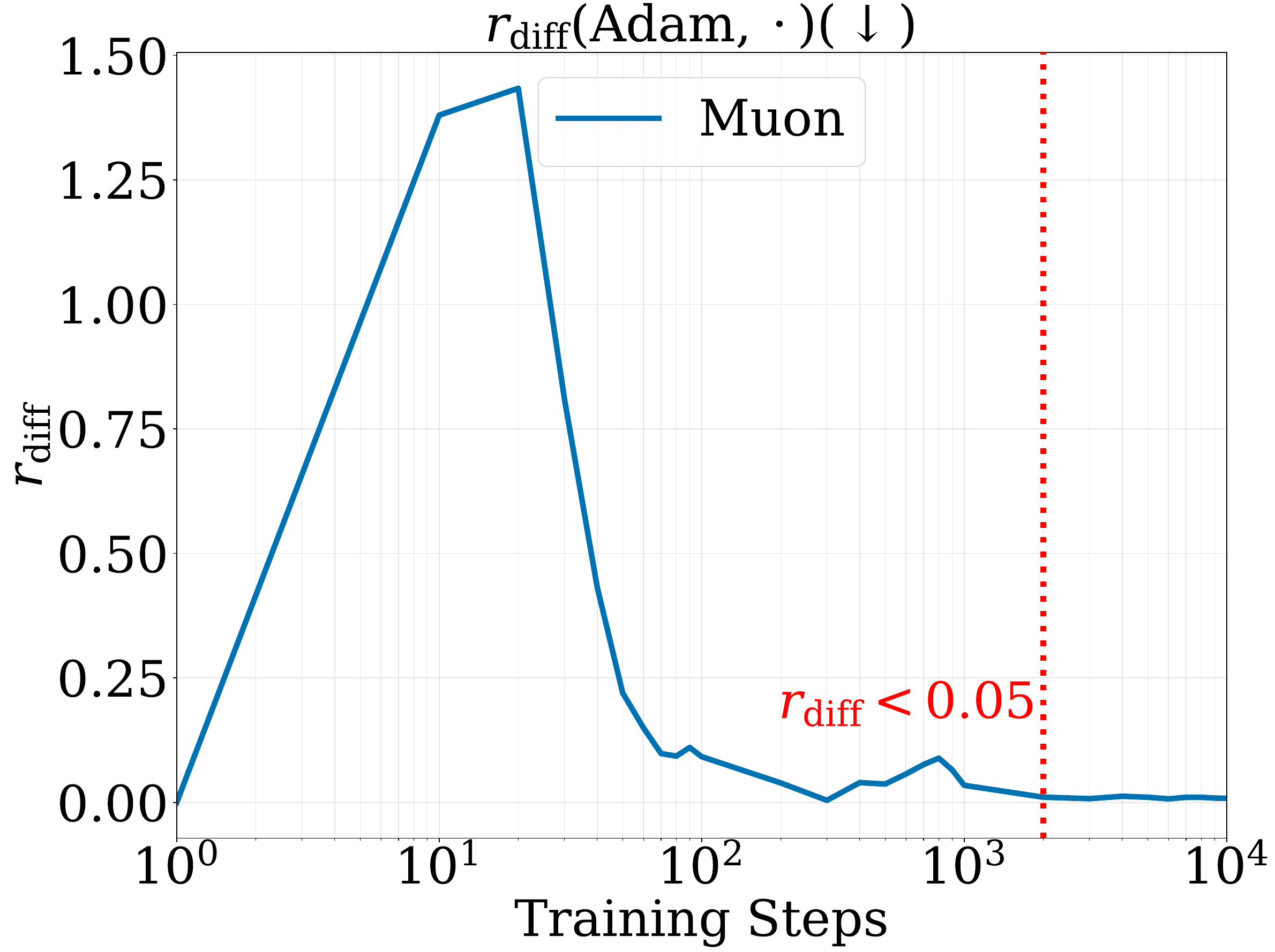}}
    \subfigure[Values of $\cccerr$ for $4.5$B models.]{ \includegraphics[width=0.35\textwidth]{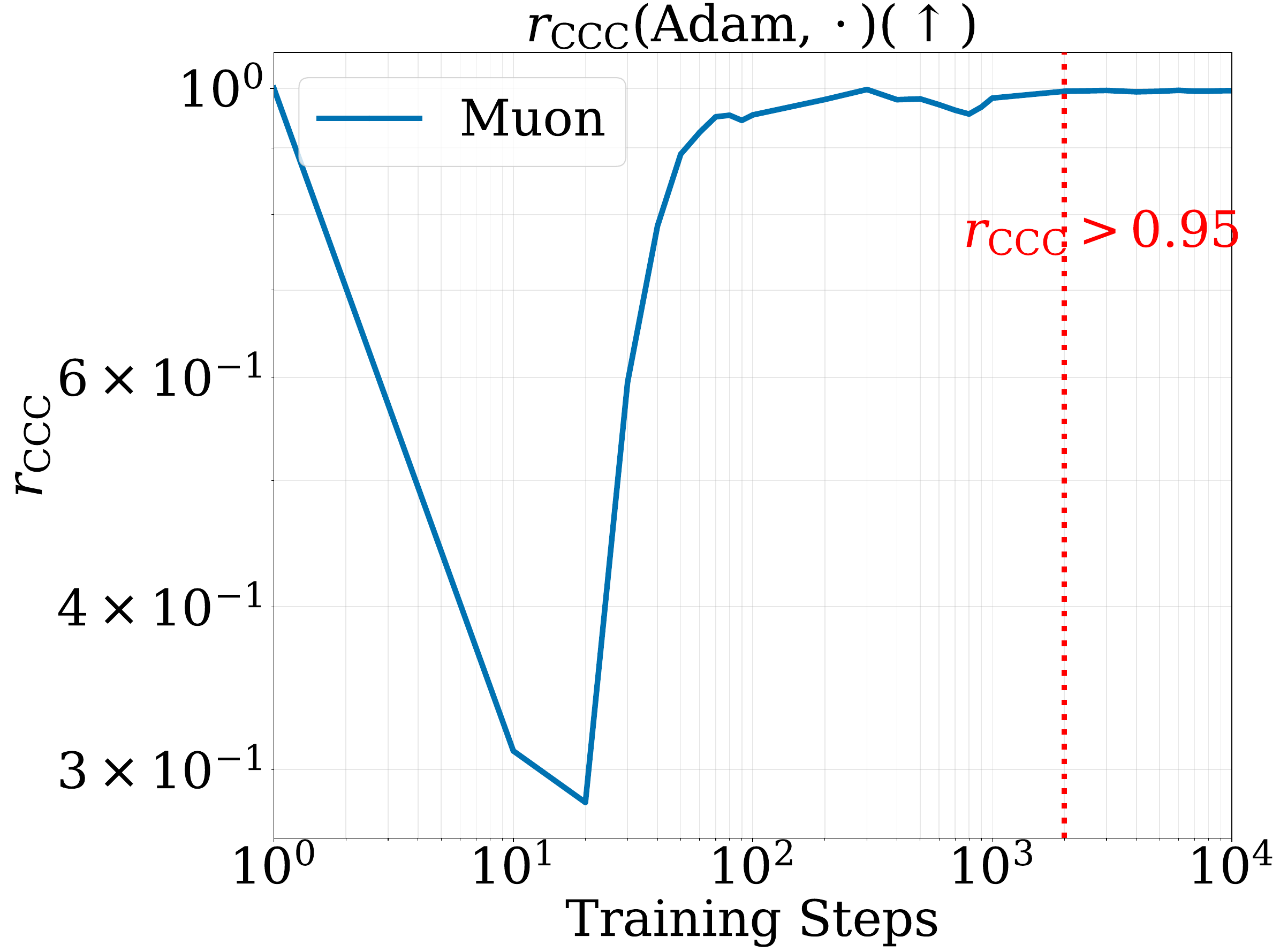}}
    
    \caption{Baseline \ac{rgi} for $4.5$B models. Panels (a) and (b) plot $\differr$ and $\cccerr$, respectively, between Adam and Muon when they share the same initial parameters and training data stream. The results show that approximate \ac{rgi} persists at the $4.5$B scale.} 
    \label{fig:4.5base}
\end{figure}
In the following experiments, we ablate how the hyperparameters influence \ac{rgi} of $0.7$B models. All results are compared against Adam under the baseline in Figure~\ref{fig:0.7base}. Therefore, we also measure \ac{rgi} between Adam runs with different settings. 
\begin{figure}[H]
    \centering

    \subfigure[Values of $\differr$ for $0.7$B models with different initial parameters.]{ \includegraphics[width=0.35\textwidth]{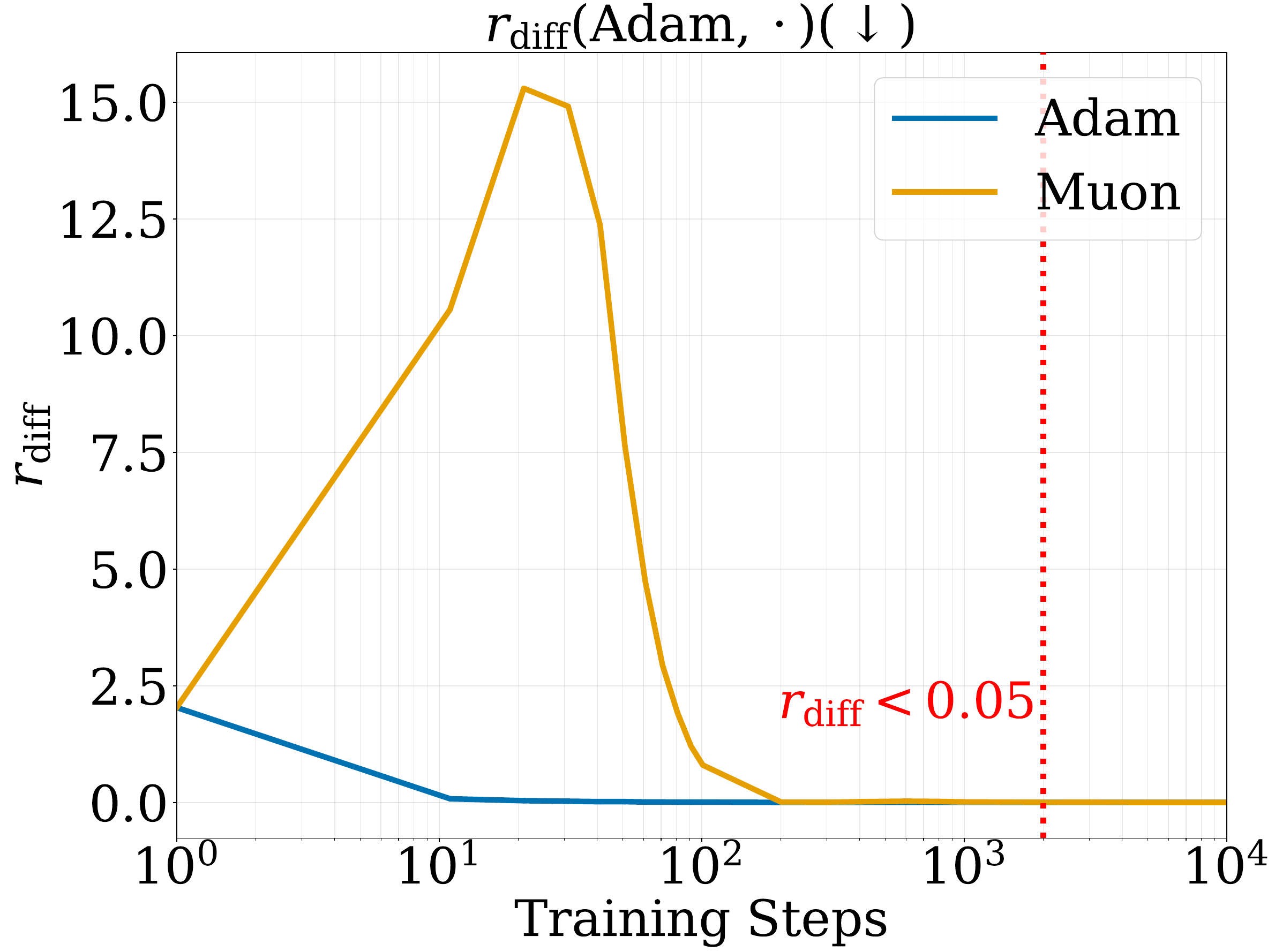}}
    \subfigure[Values of $\cccerr$ for $0.7$B with different initial parameters.]{ \includegraphics[width=0.35\textwidth]{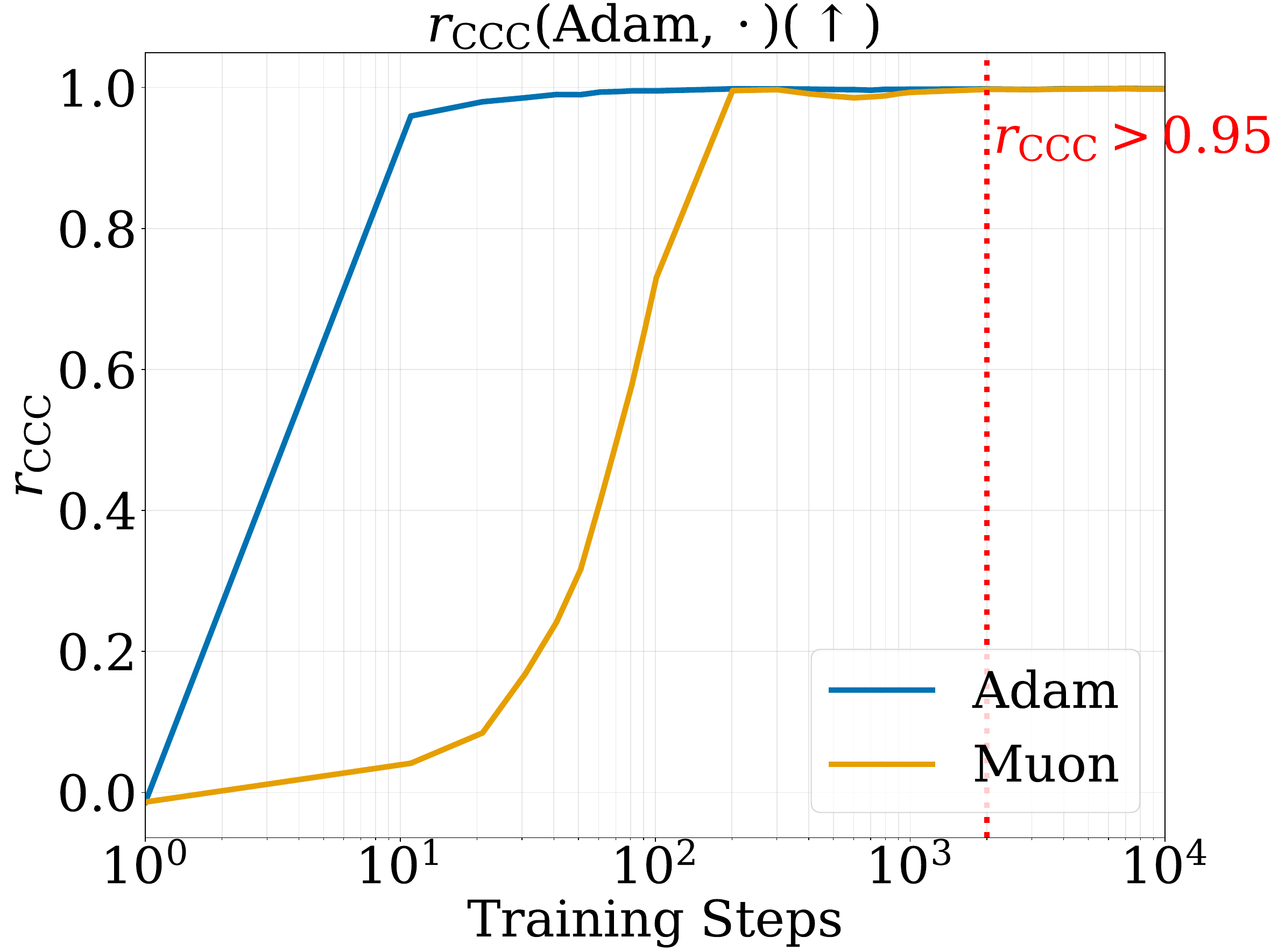}}
    
    \caption{Initialization ablation for $0.7$B models. Panels (a) and (b) plot $\differr$ and $\cccerr$, respectively, for Adam and Muon runs started from different initial parameters. Approximate \ac{rgi} still emerges at the $0.7$B scale, showing that initialization differences do not prevent the phenomenon.} 
\end{figure}

\begin{figure}[H]
    \centering

    \subfigure[Values of $\differr$ for $0.7$B models with $0.25\eta^\sfA$.]{ \includegraphics[width=0.23\textwidth]{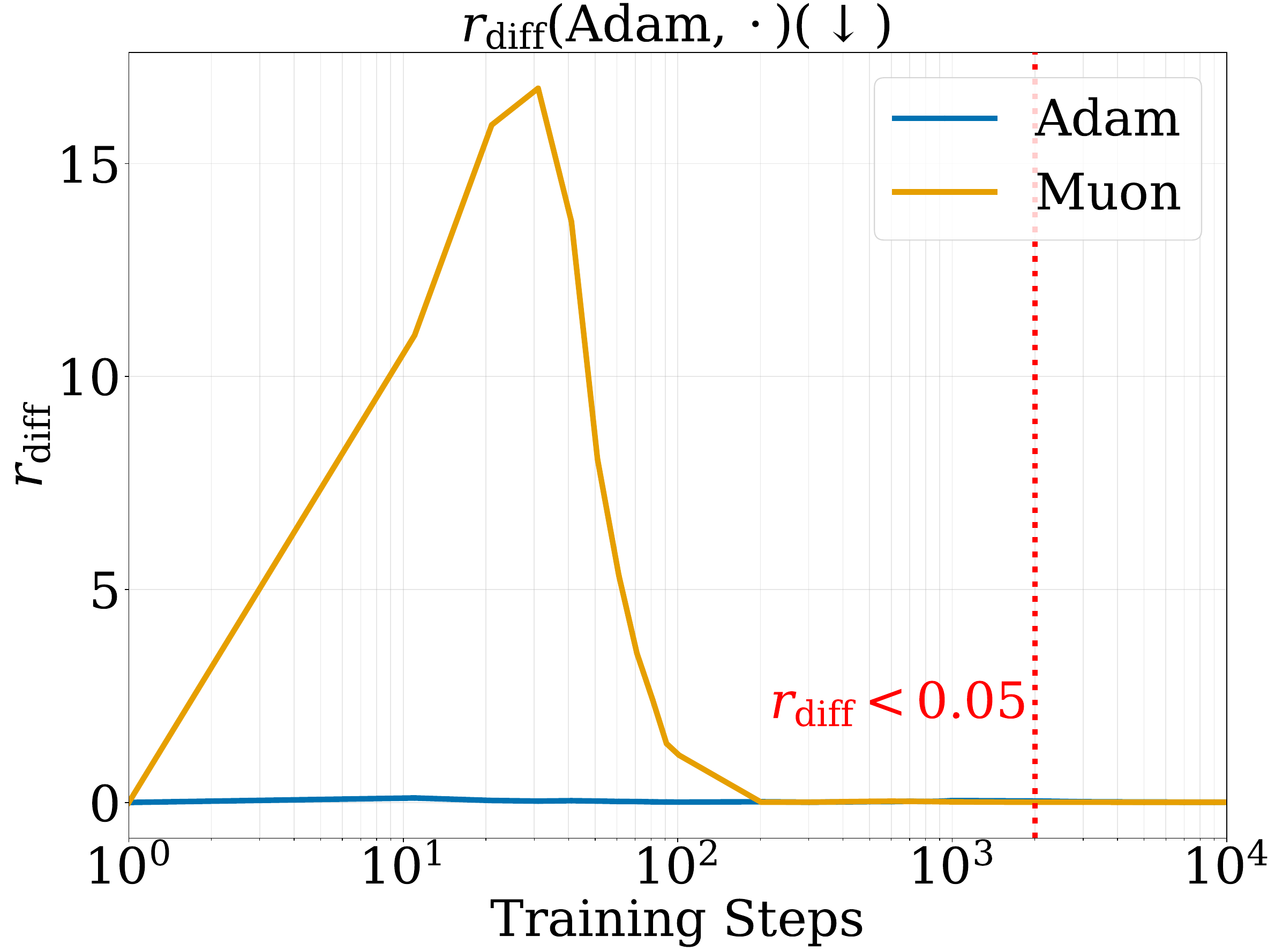}}
    \subfigure[Values of $\differr$ for $0.7$B models with $1.25\eta^\sfA$.]{ \includegraphics[width=0.23\textwidth]{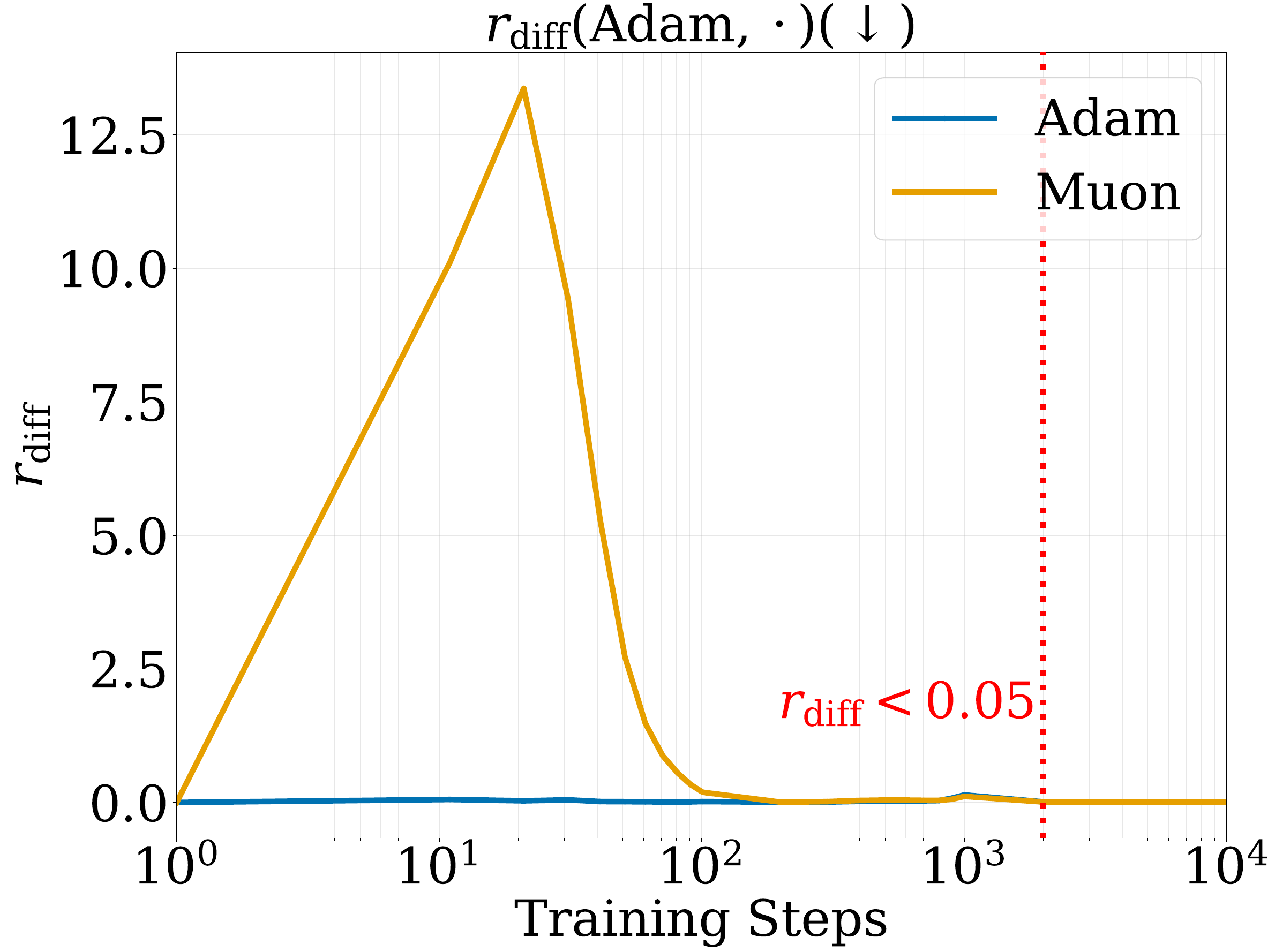}}
    \subfigure[Values of $\cccerr$ for $0.7$B with $0.25\eta^\sfA$.]{ \includegraphics[width=0.23\textwidth]{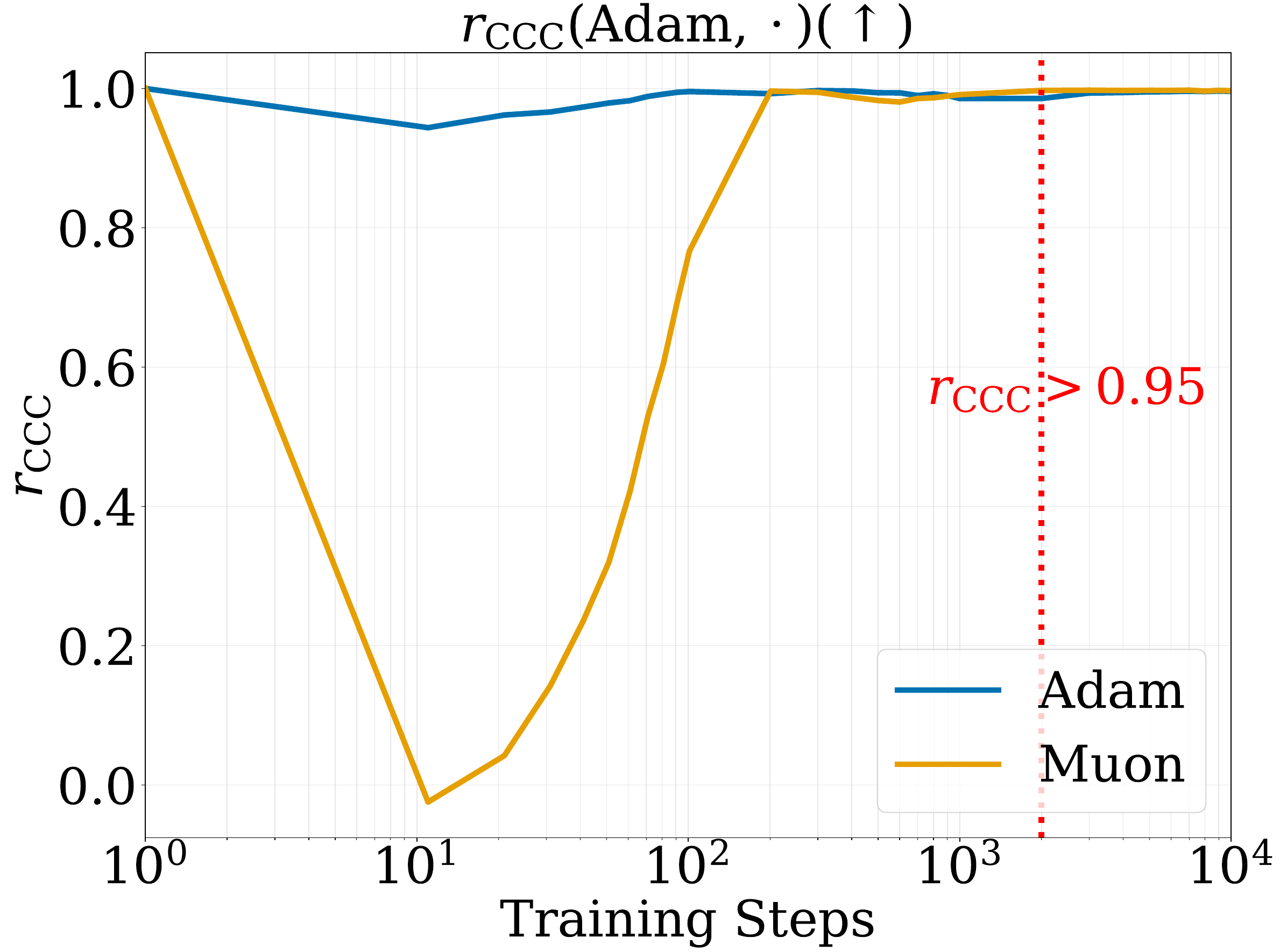}}
    \subfigure[Values of $\cccerr$ for $0.7$B with $1.25\eta^\sfA$.]{ \includegraphics[width=0.23\textwidth]{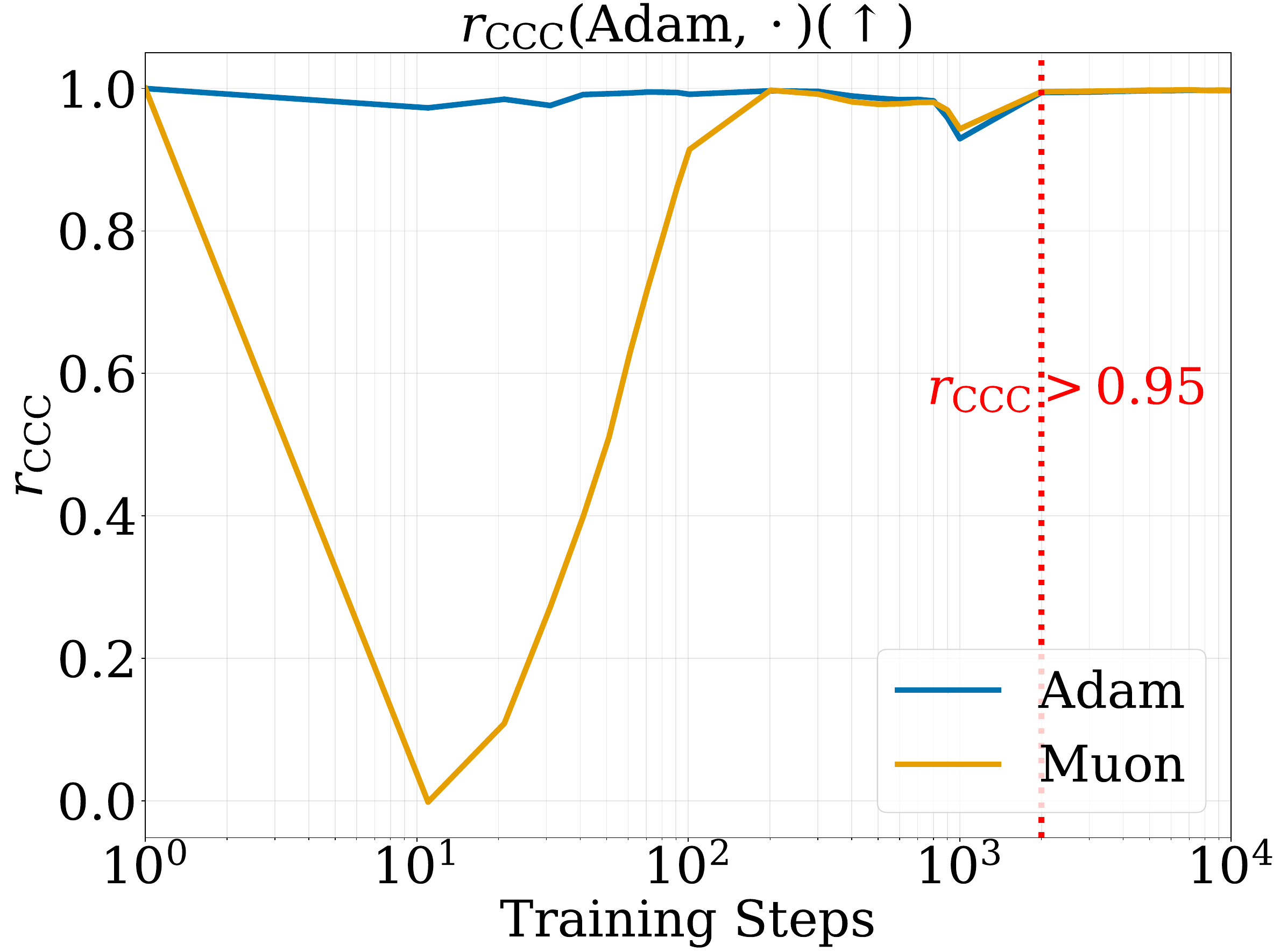}}
    
    \caption{Learning-rate ablation for $0.7$B models. Panels (a) and (b) plot $\differr$ for $0.25\eta^{\sfA}$ and $1.25\eta^{\sfA}$, while Panels (c) and (d) plot $\cccerr$ for the same two learning-rate multipliers. Approximate \ac{rgi} remains robust to these learning-rate changes at the $0.7$B scale.} 
\end{figure}

\begin{figure}[H]
    \centering

    \subfigure[Values of $\differr$ for $0.7$B models with $\gamma=0.85$.]{ \includegraphics[width=0.27\textwidth]{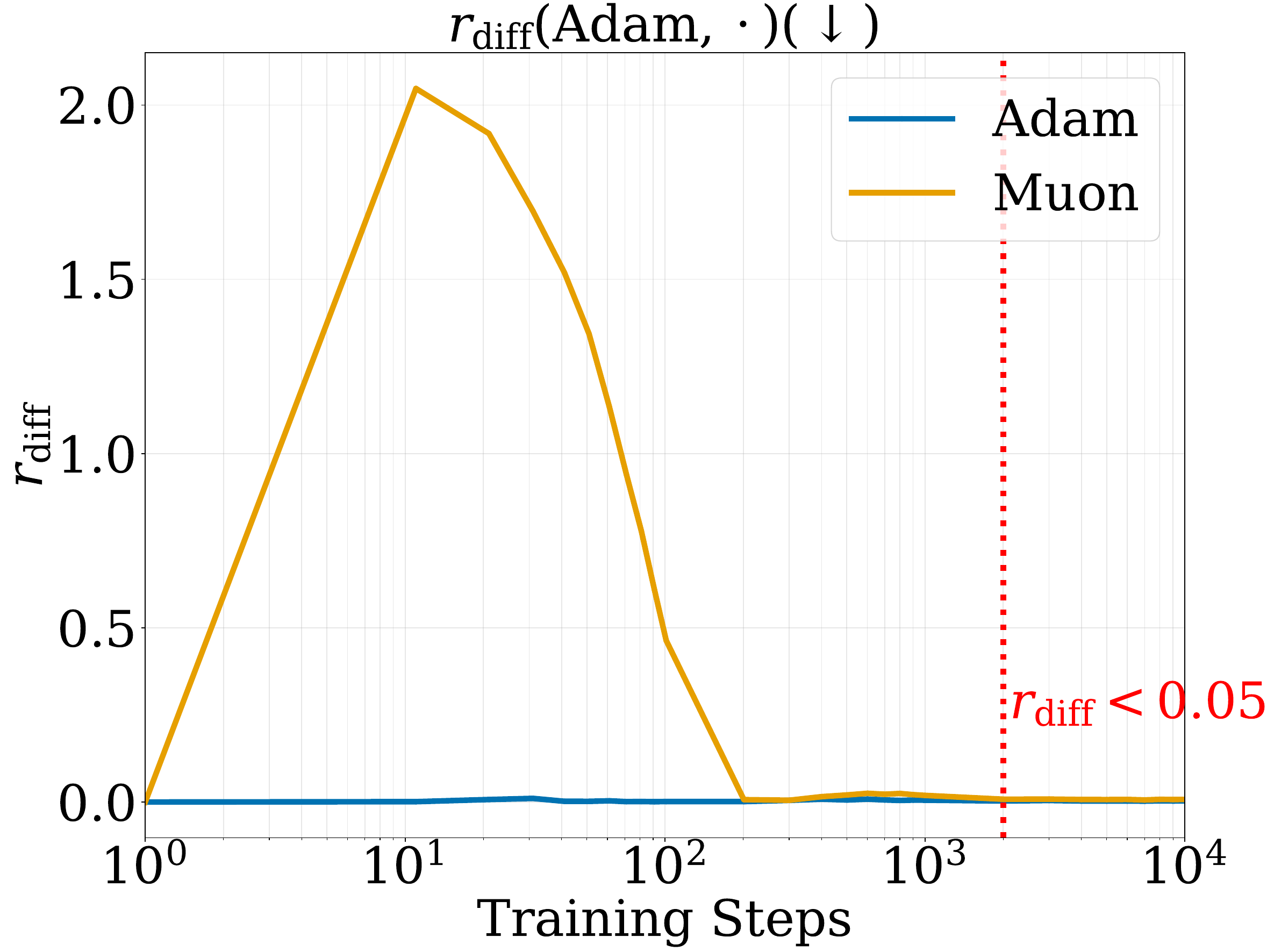}}
    \subfigure[Values of $\differr$ for $0.7$B models with $\gamma=0.9$.]{ \includegraphics[width=0.27\textwidth]{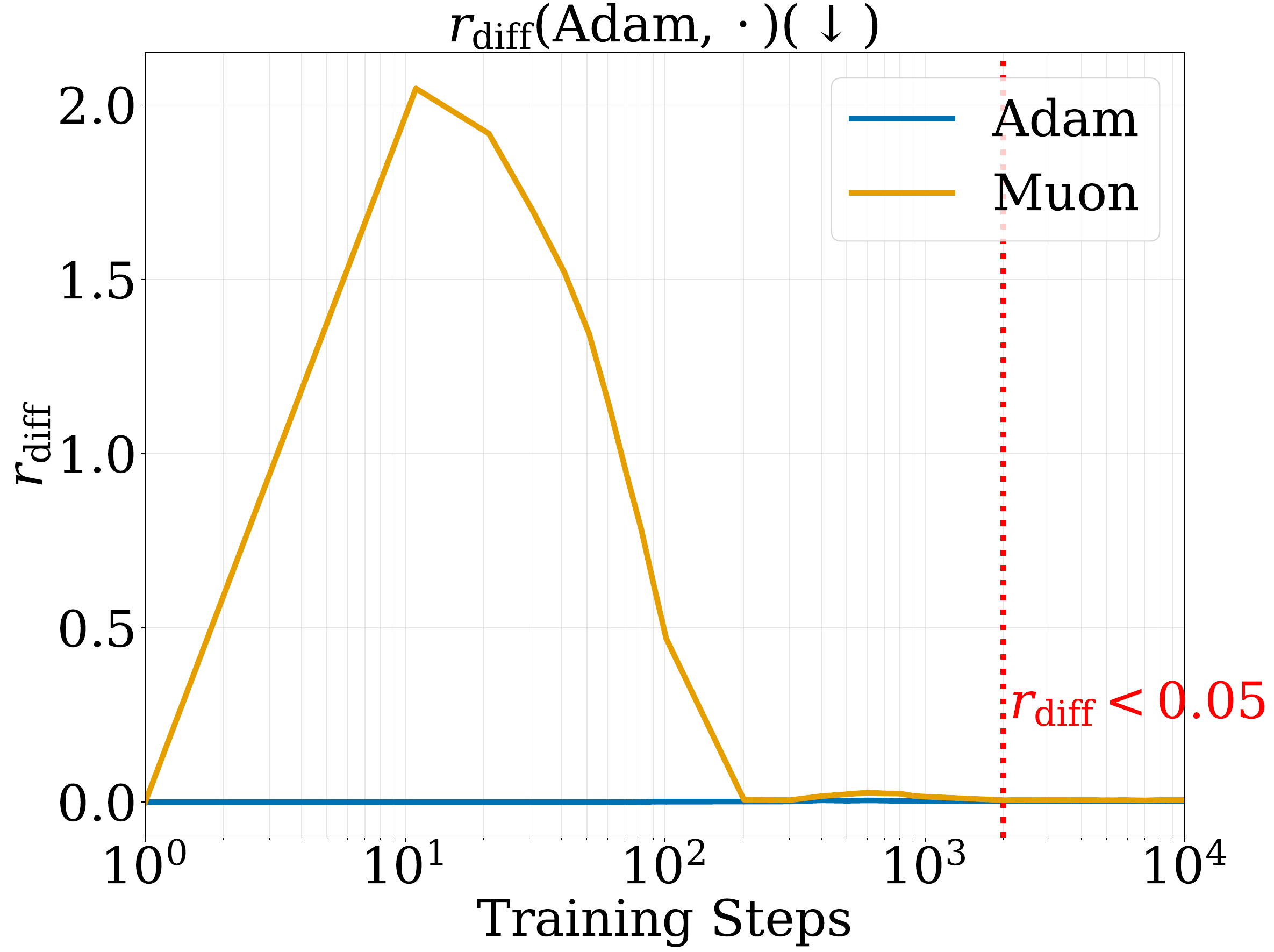}}
    \subfigure[Values of $\differr$ for $0.7$B models with $\gamma=0.95$.]{ \includegraphics[width=0.27\textwidth]{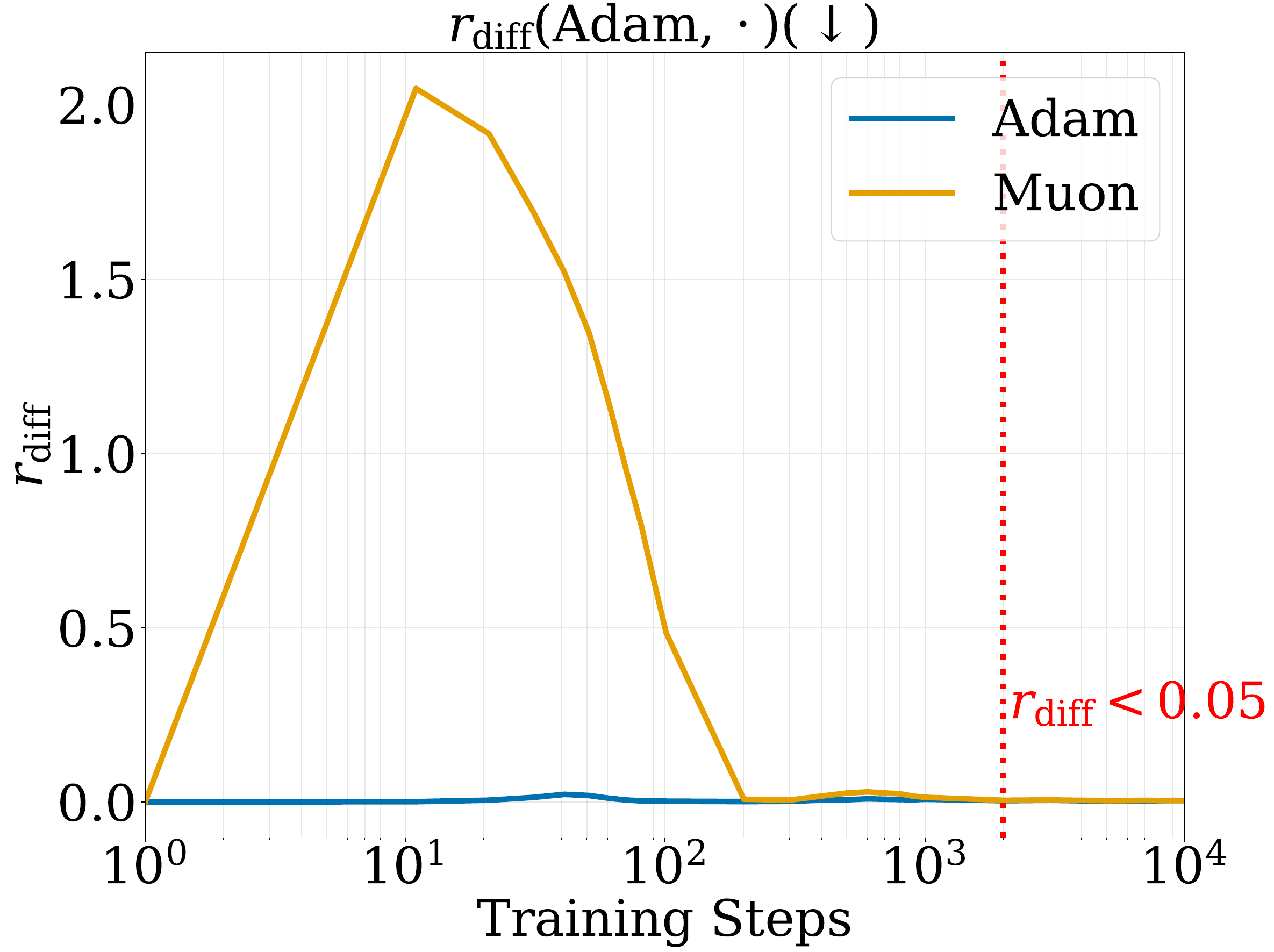}}

    \subfigure[Values of $\cccerr$ for $0.7$B with $\gamma=0.85$.]{ \includegraphics[width=0.27\textwidth]{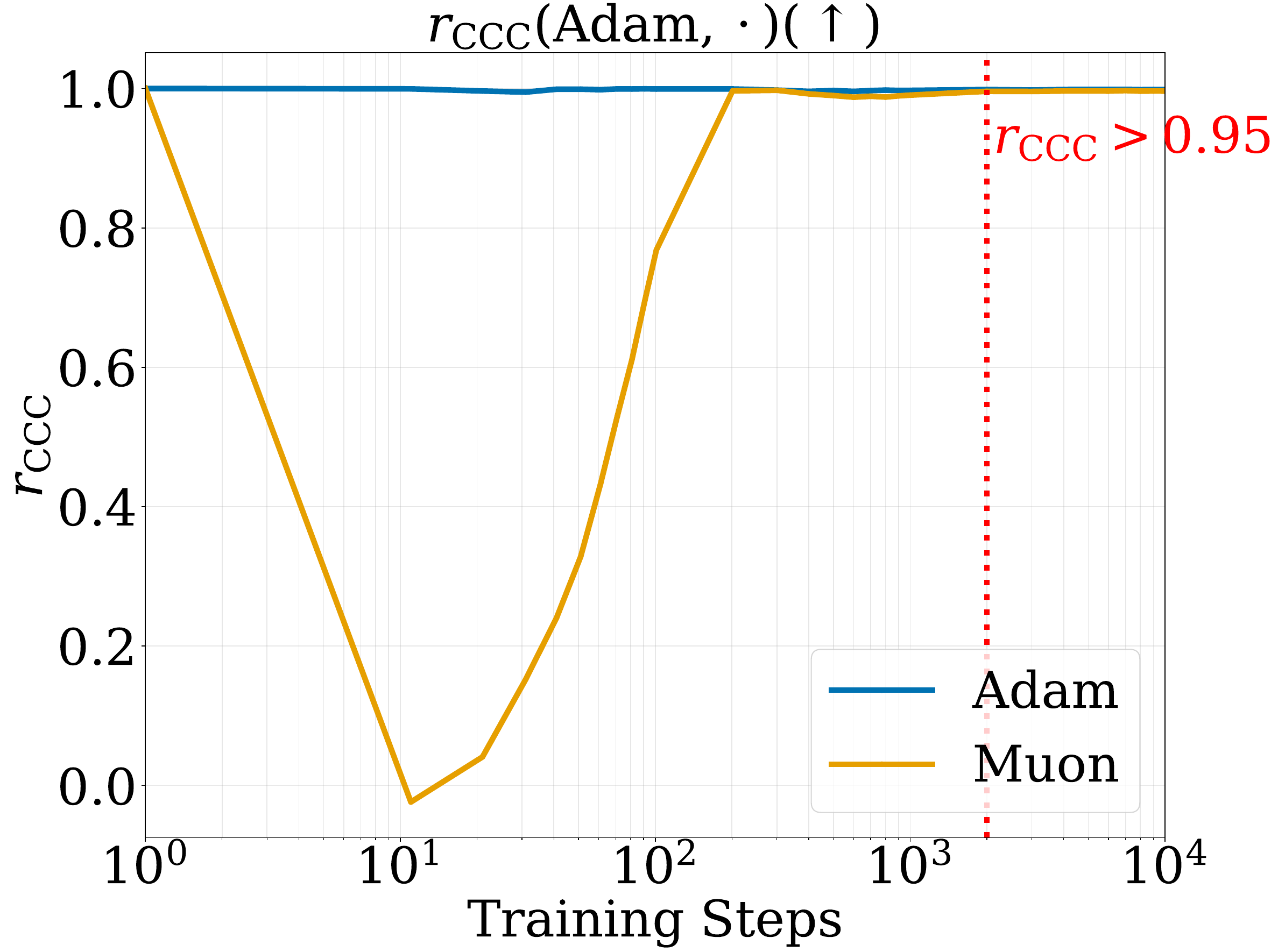}}
    \subfigure[Values of $\cccerr$ for $0.7$B with $\gamma=0.9$.]{ \includegraphics[width=0.27\textwidth]{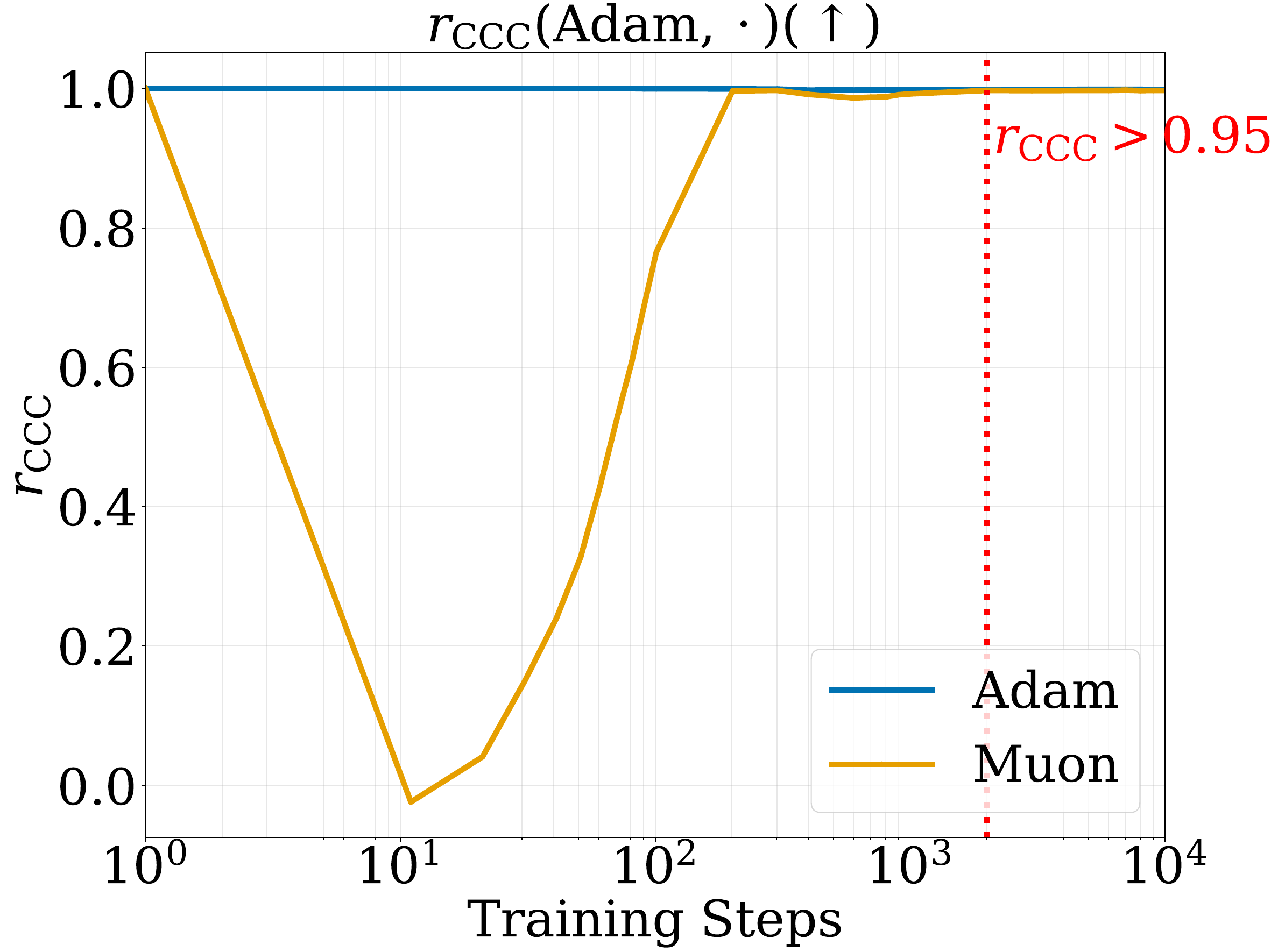}}
    \subfigure[Values of $\cccerr$ for $0.7$B with $\gamma=0.95$.]{ \includegraphics[width=0.27\textwidth]{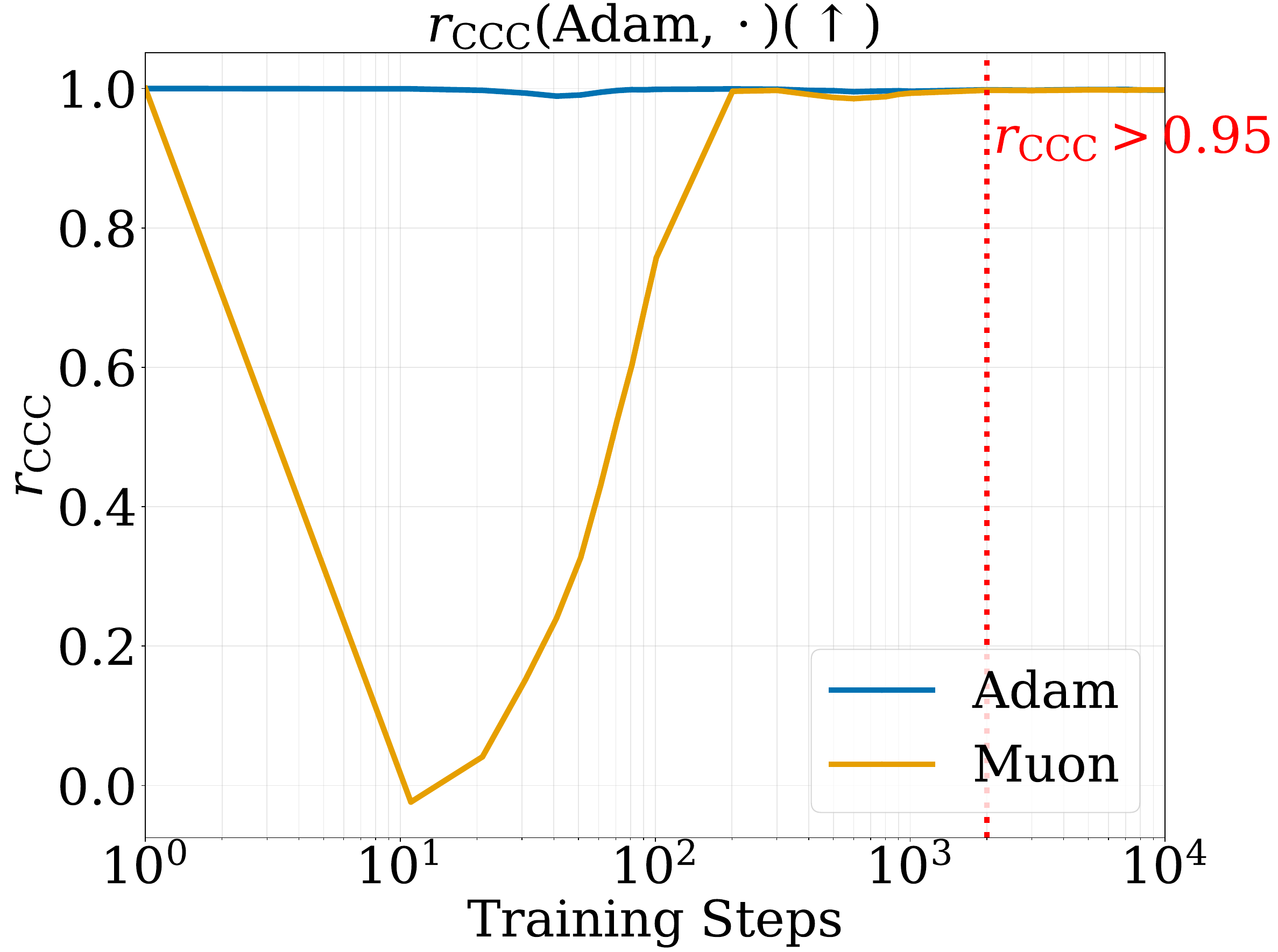}}
    
    \caption{Momentum ablation for $0.7$B models. Panels (a)--(c) plot $\differr$ for $\gamma=0.85$, $0.9$, and $0.95$, while Panels (d)--(f) plot $\cccerr$ for the same momentum values. Approximate \ac{rgi} remains robust to the tested momentum settings at the $0.7$B scale.} 
\end{figure}

\begin{figure}[H]
    \centering

    \subfigure[Values of $\differr$ for $0.7$B models with $100$ warmup steps.]{ \includegraphics[width=0.23\textwidth]{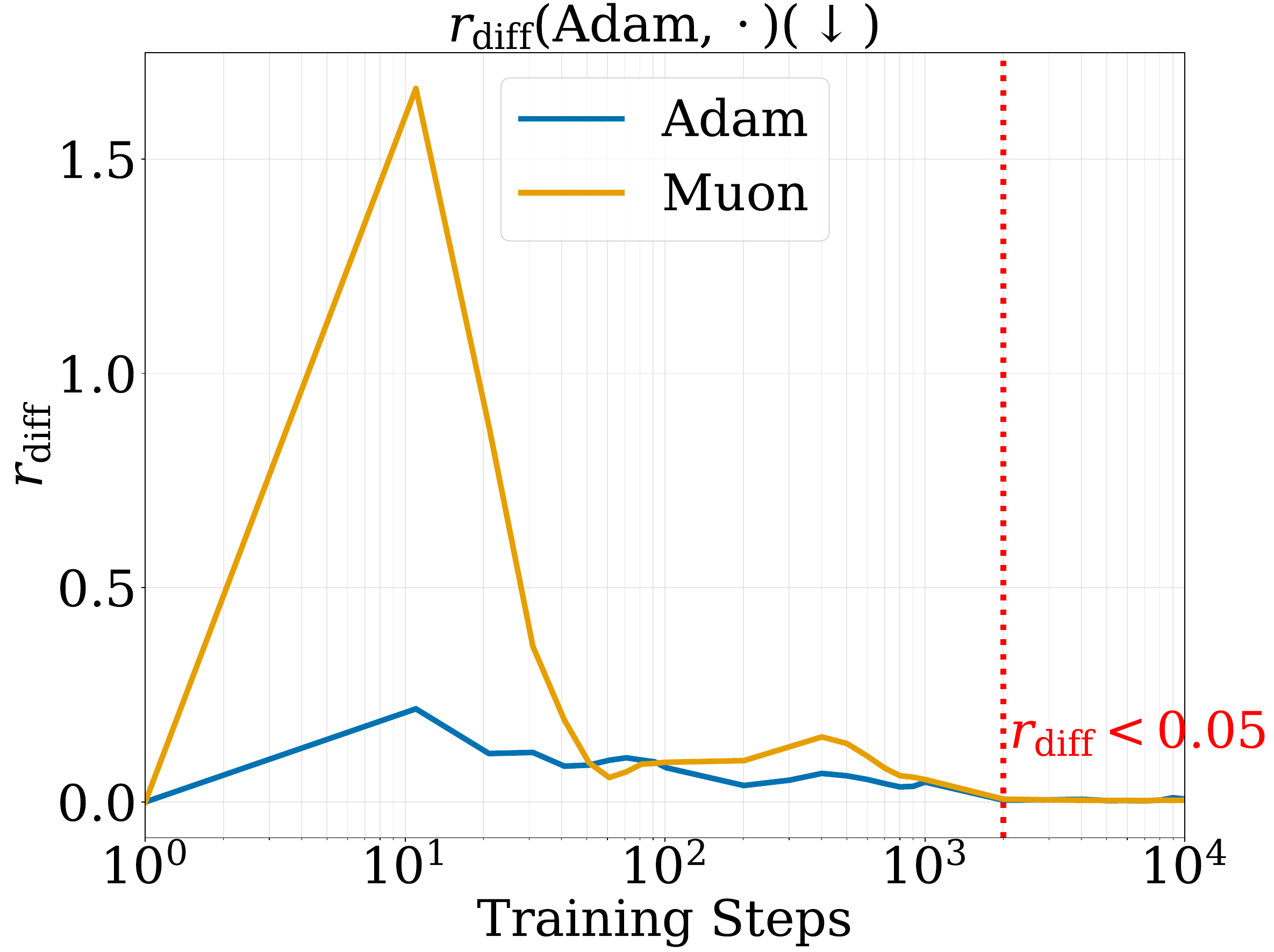}}
    \subfigure[Values of $\cccerr$ for $0.7$B models with $100$ warmup steps.]{ \includegraphics[width=0.23\textwidth]{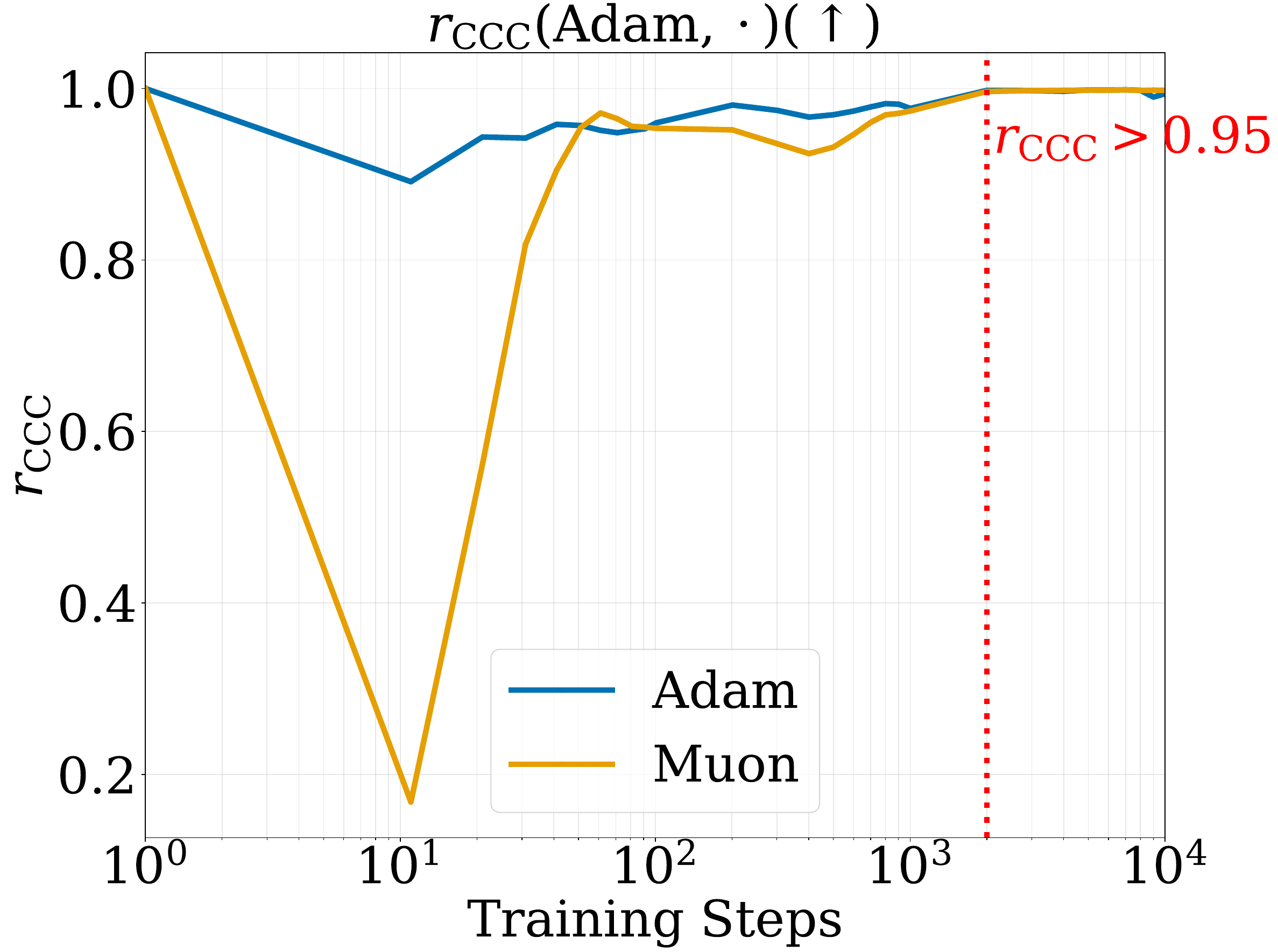}}
    \subfigure[Values of $\differr$ for $0.7$B models with $1400$ warmup steps.]{ \includegraphics[width=0.23\textwidth]{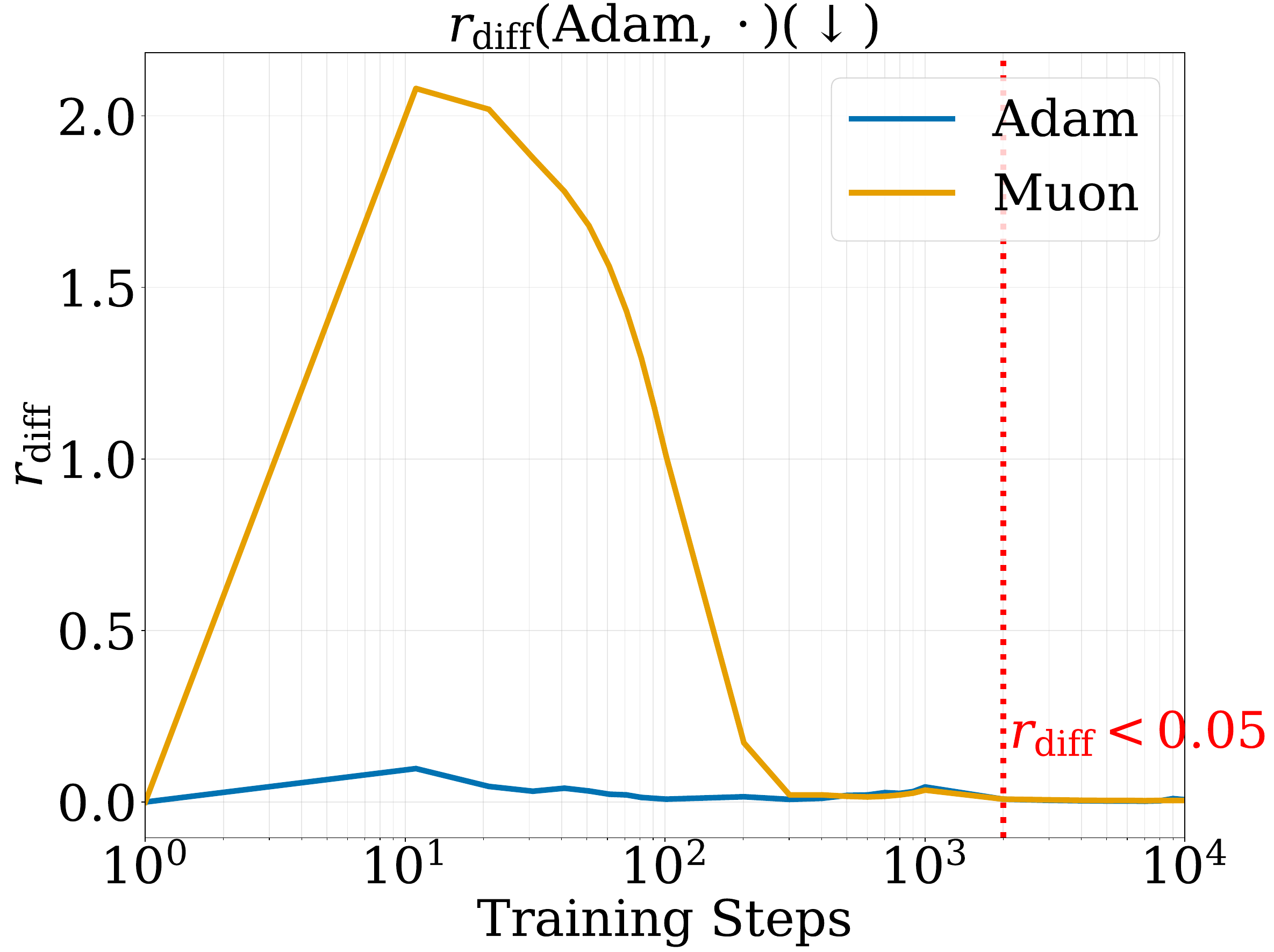}}
    \subfigure[Values of $\cccerr$ for $0.7$B models with $1400$ warmup steps.]{ \includegraphics[width=0.23\textwidth]{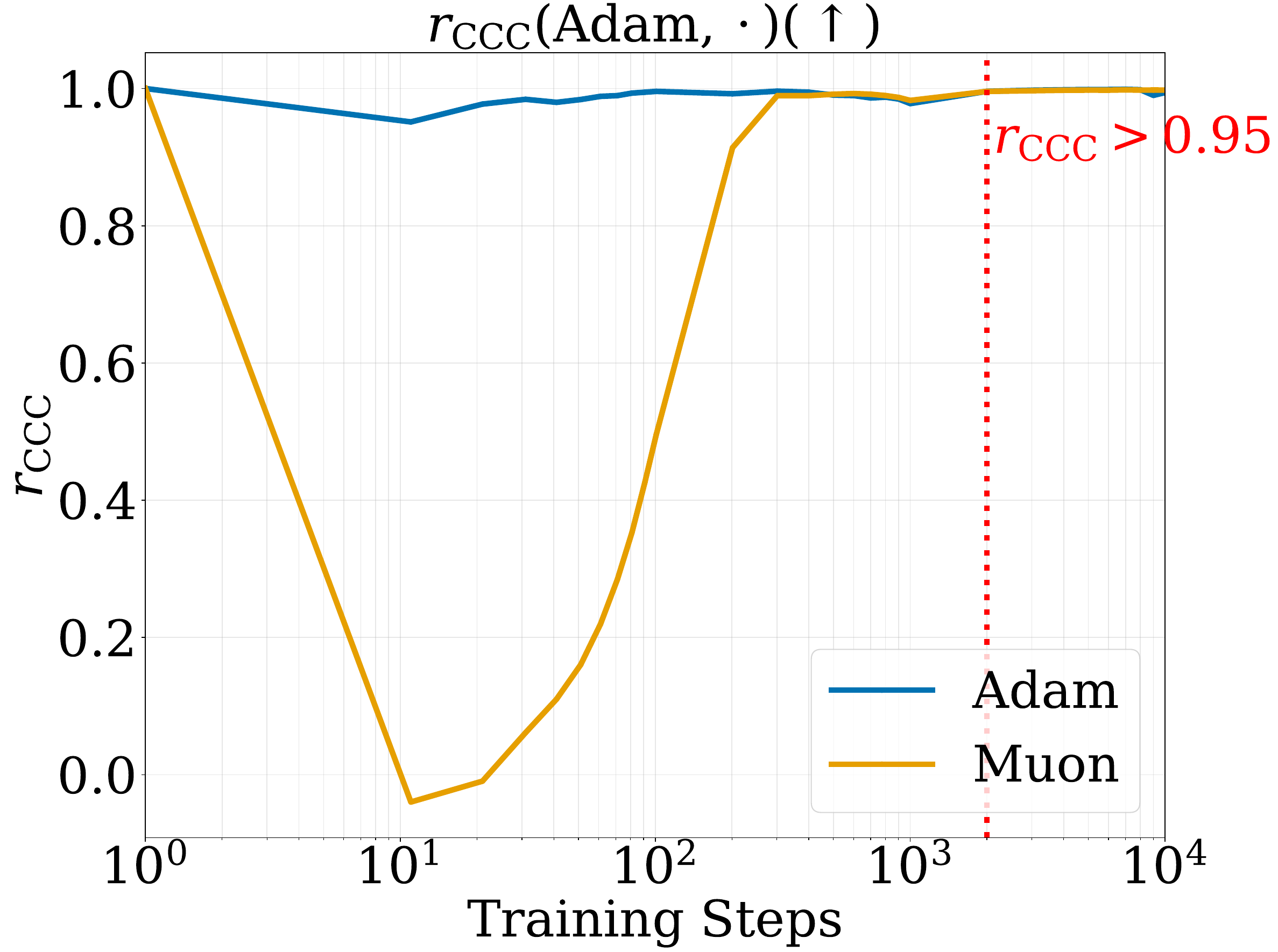}}
    
    \caption{Warmup-step ablation for $0.7$B models. Panels (a) and (c) plot $\differr$ for $100$ and $1{,}400$ warmup steps, while Panels (b) and (d) plot $\cccerr$ for the same warmup settings. Approximate \ac{rgi} remains robust to the tested scheduler warmup lengths at the $0.7$B scale.} 
\end{figure}

\begin{figure}[H]
    \centering
    \subfigure[Values of $\differr$ for $0.7$B models when Adam is trained on FineWeb and Muon is trained on CodeParrot.]{ \includegraphics[width=0.35\textwidth]{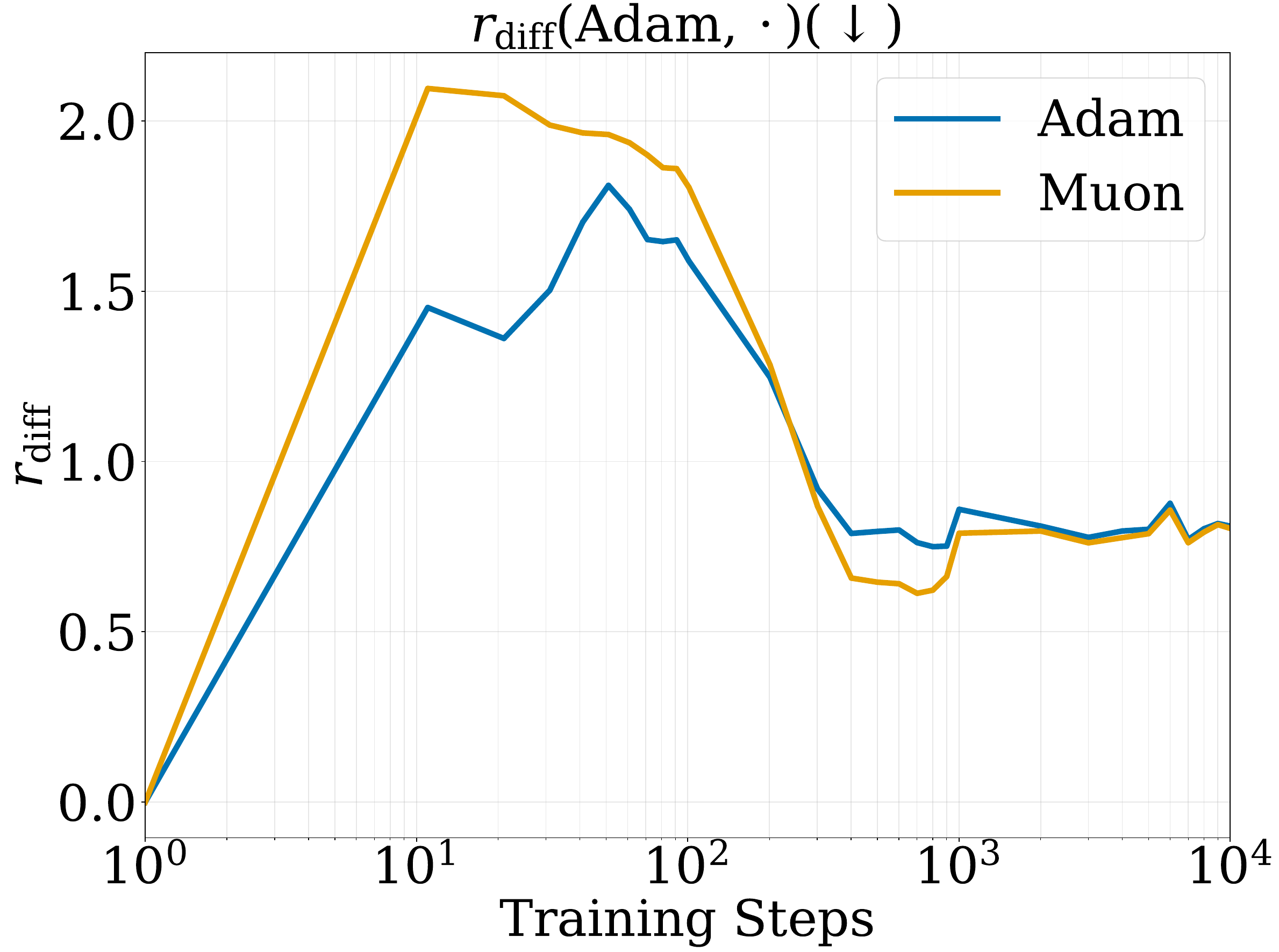}}
    \hspace{2em}
    \subfigure[Values of $\cccerr$ for $0.7$B models when Adam is trained on FineWeb and Muon is trained on CodeParrot.]{ \includegraphics[width=0.35\textwidth]{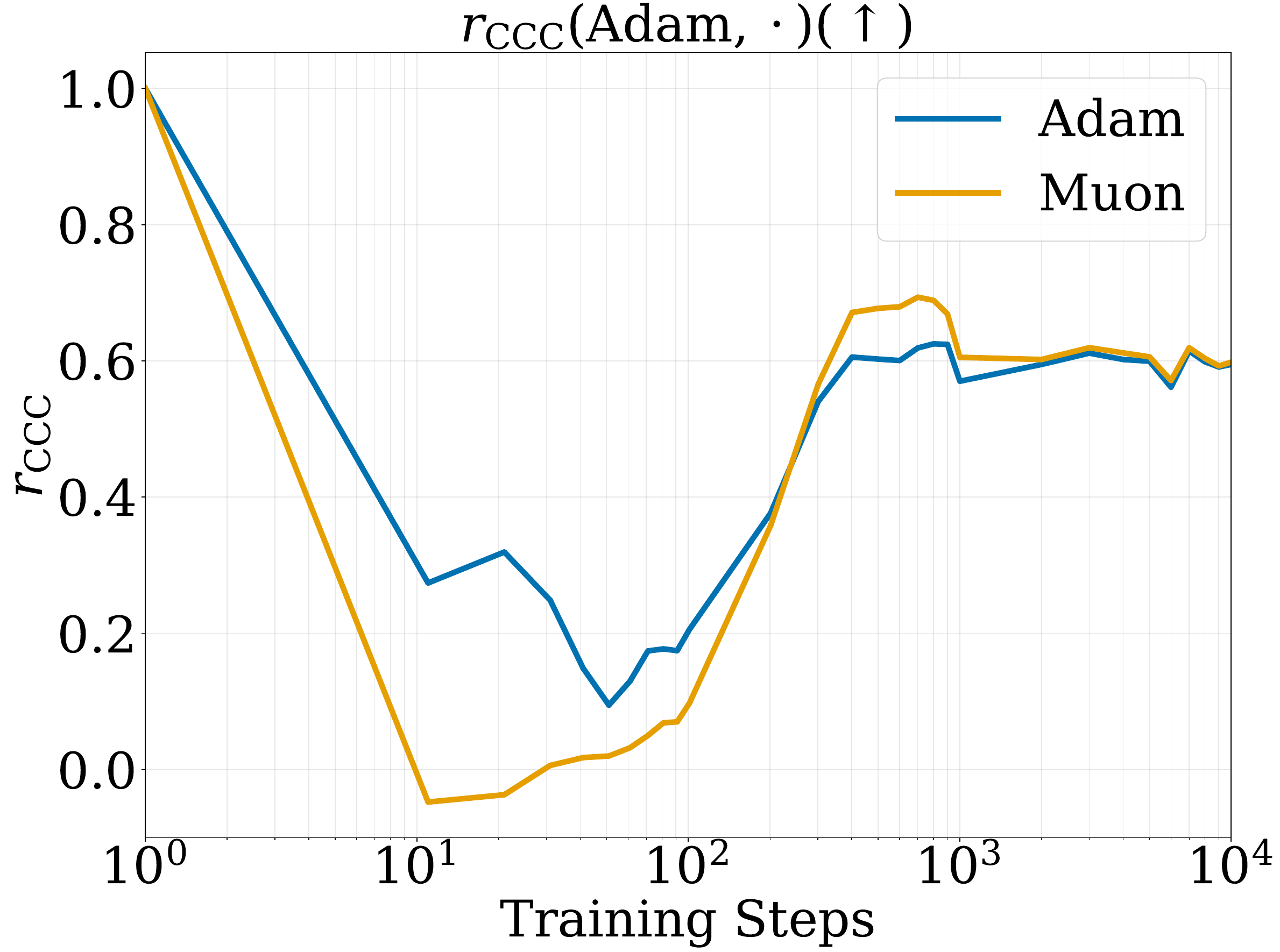}}
    
    \caption{Training data-stream ablation for $0.7$B models. Panels (a) and (b) plot $\differr$ and $\cccerr$, respectively, when Adam is trained on FineWeb and Muon is trained on CodeParrot. The weak agreement shows that changing the training data stream strongly disrupts \ac{rgi} at the $0.7$B scale.} 
\end{figure}

\begin{figure}[H]
    \centering
    \subfigure[Values of $\differr$ for $0.7$B models with $D_{\val}$ from C4.]{ \includegraphics[width=0.27\textwidth]{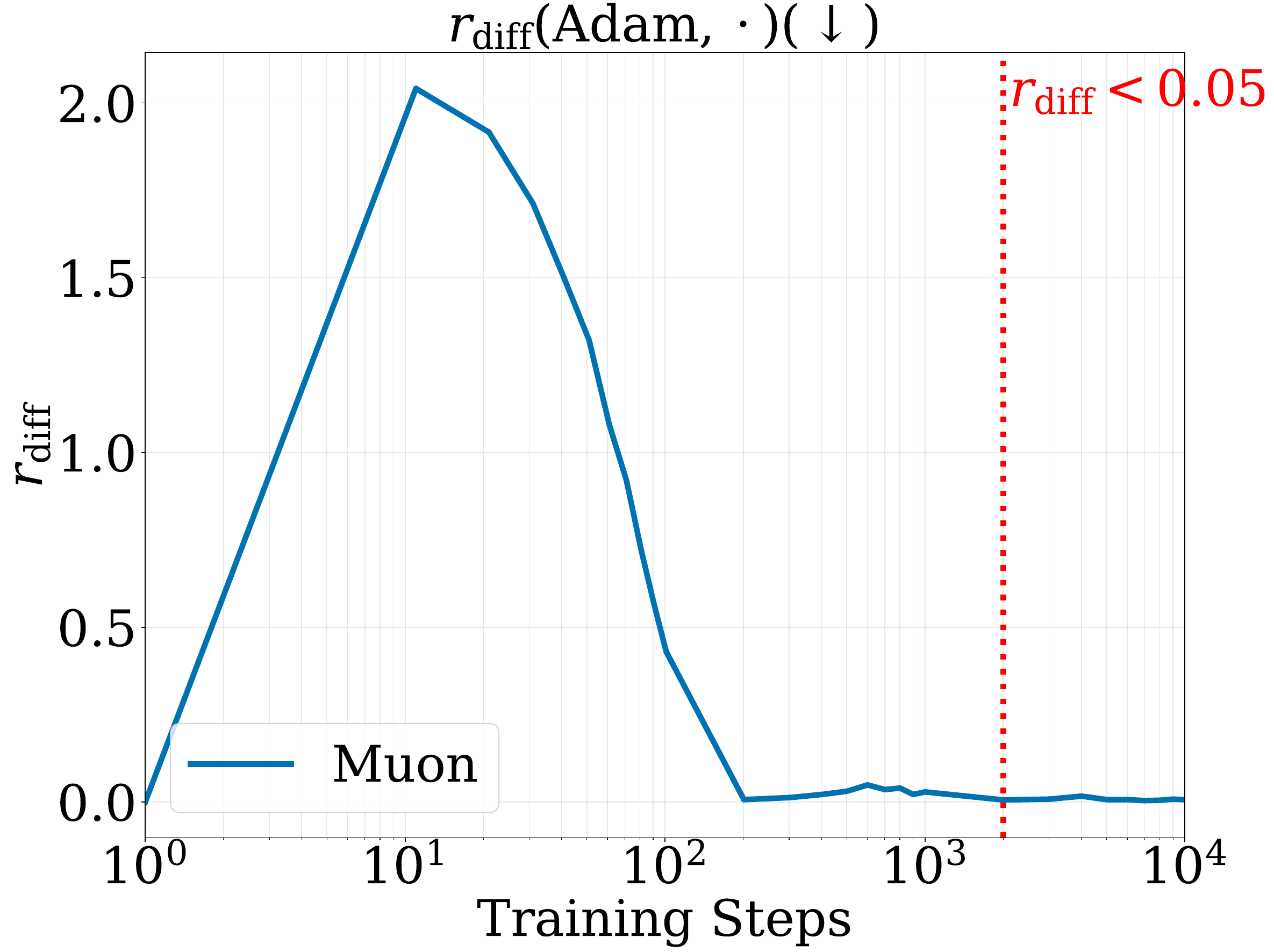}}
    \subfigure[Values of $\differr$ for $0.7$B models with $D_{\val}$ from ArXiv.]{ \includegraphics[width=0.27\textwidth]{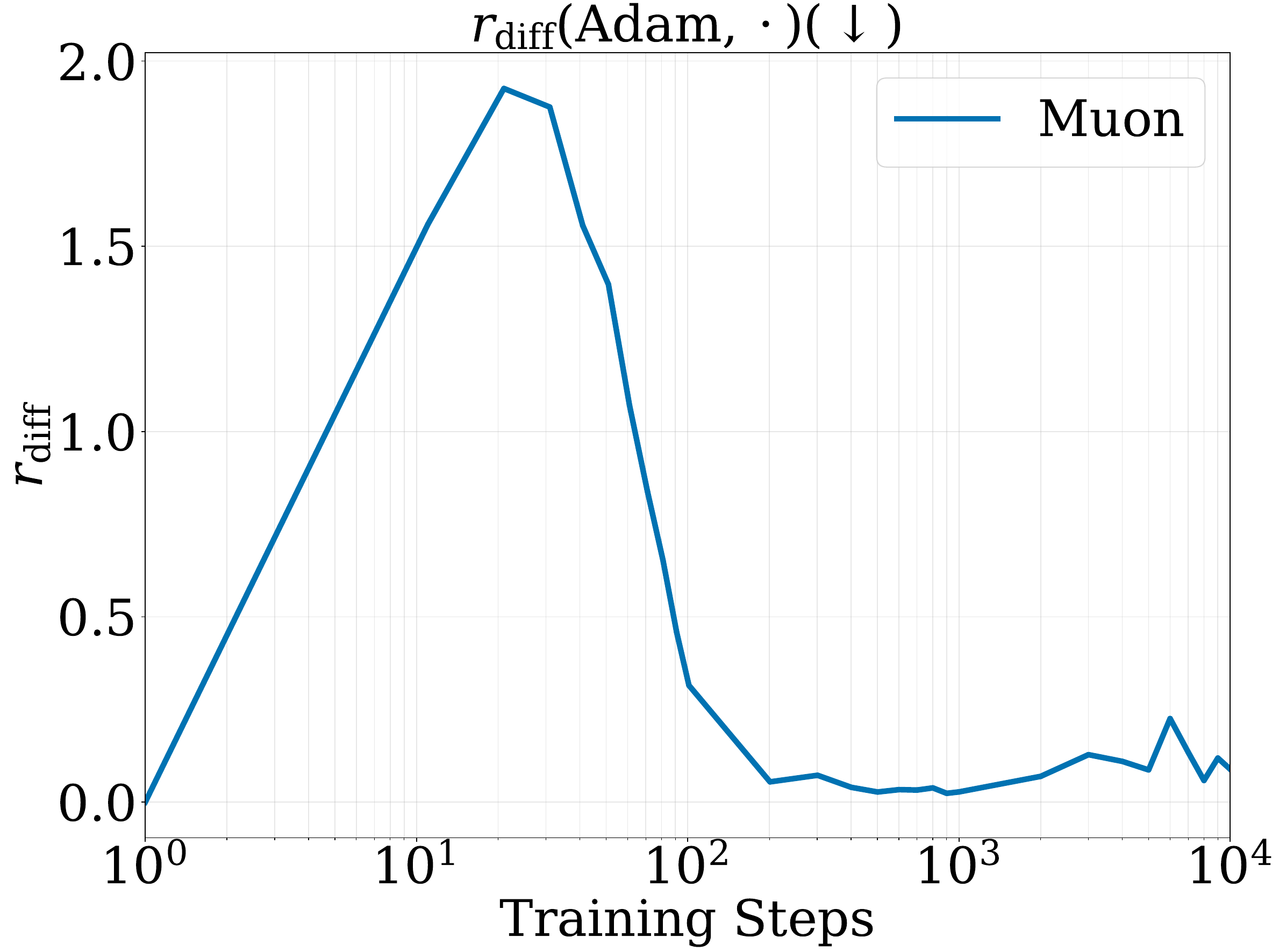}}
    \subfigure[Values of $\differr$ for $0.7$B models with $D_{\val}$ from CodeParrot.]{ \includegraphics[width=0.27\textwidth]{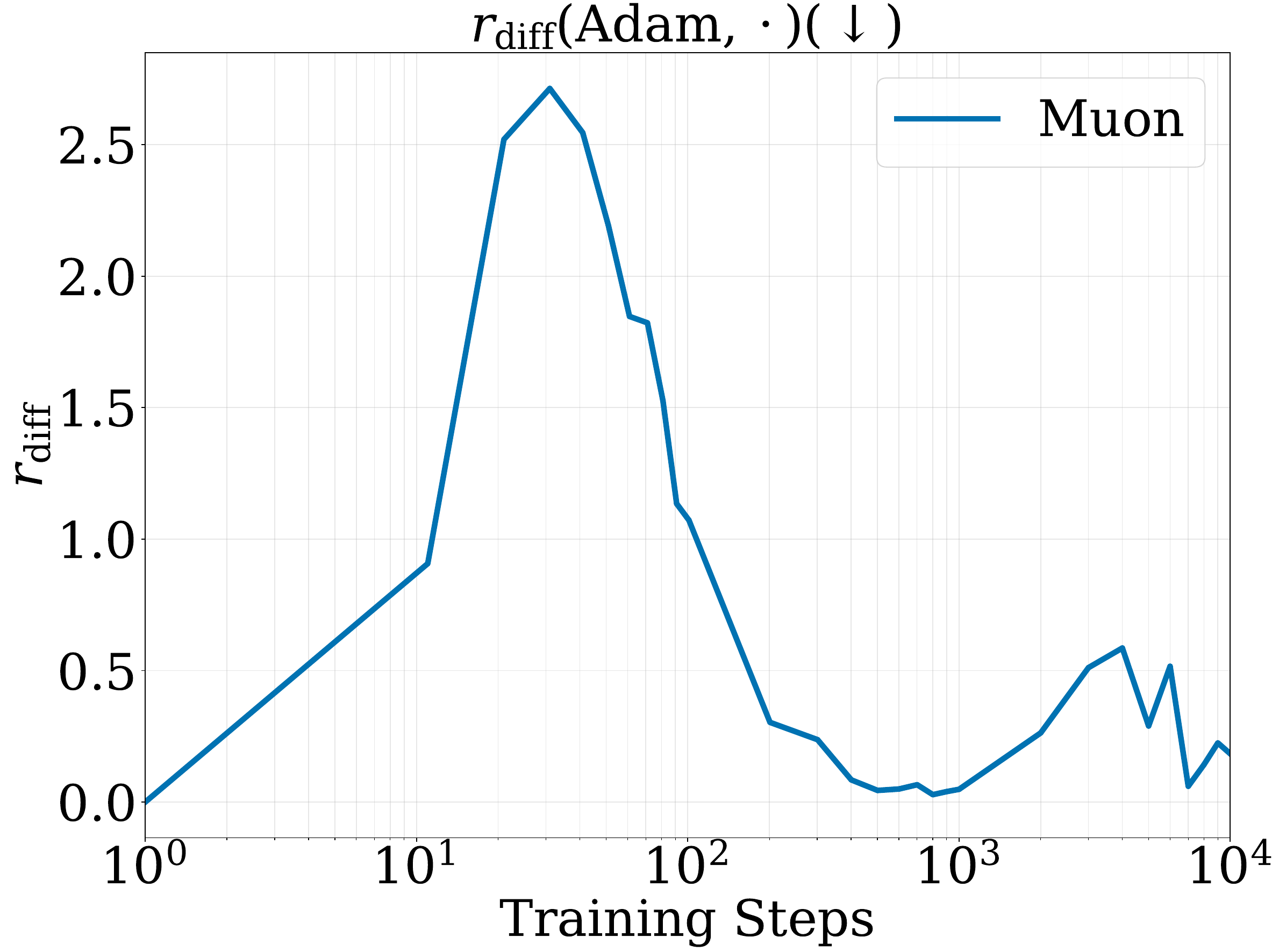}}

    \subfigure[Values of $\cccerr$ for $0.7$B models with $D_{\val}$ from C4.]{ \includegraphics[width=0.27\textwidth]{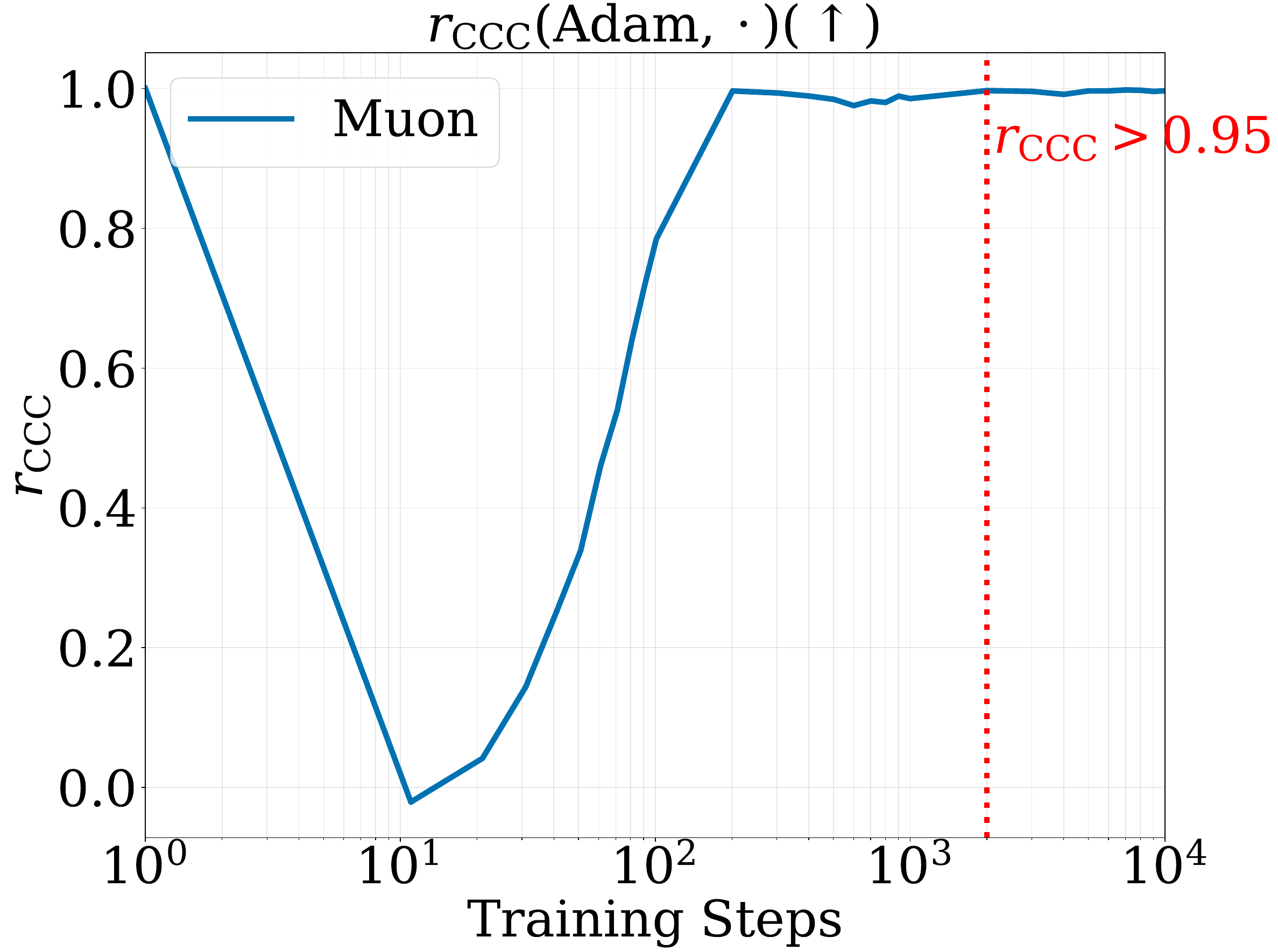}}
    \subfigure[Values of $\cccerr$ for $0.7$B models with $D_{\val}$ from ArXiv.]{ \includegraphics[width=0.27\textwidth]{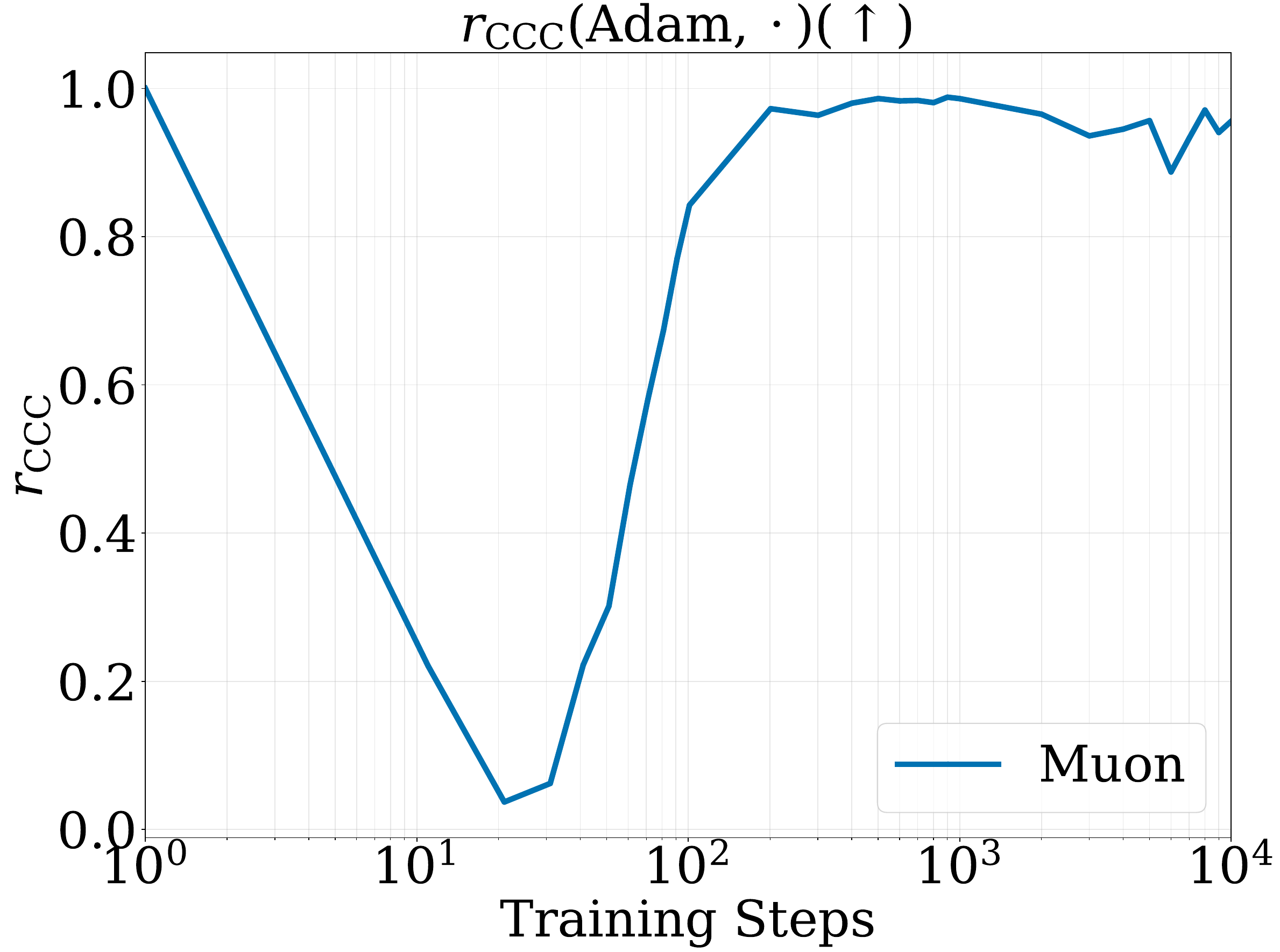}}
    \subfigure[Values of $\cccerr$ for $0.7$B models with $D_{\val}$ from CodeParrot.]{ \includegraphics[width=0.27\textwidth]{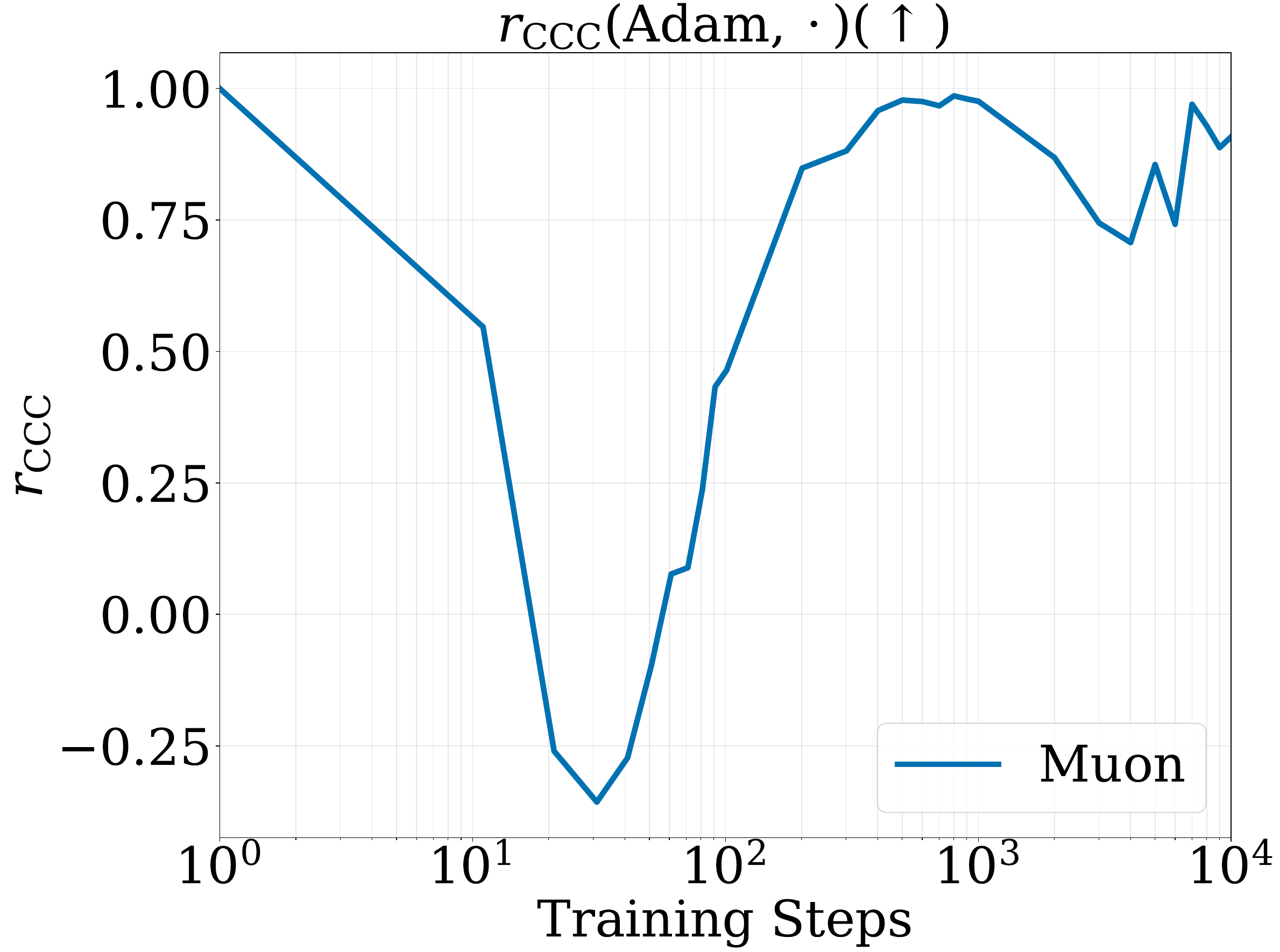}}
    \caption{Validation-distribution ablation for $0.7$B models. Panels (a)--(c) plot $\differr$ evaluated on C4, ArXiv, and CodeParrot, while Panels (d)--(f) plot $\cccerr$ on the same validation datasets. \ac{rgi} remains clearest on C4 and weakens on the more distributionally distinct ArXiv and CodeParrot validation sets.} 
\end{figure}
\subsection{Additional Results for Section~\ref{sec:emp_studies}}

\subsubsection{Additional Results for Section~\ref{sec:opt}}

\begin{figure}[H]
    \centering
    \subfigure[$\differr$ for $0.25\eta^{\sfA}$.]{ \includegraphics[width=0.23\textwidth]{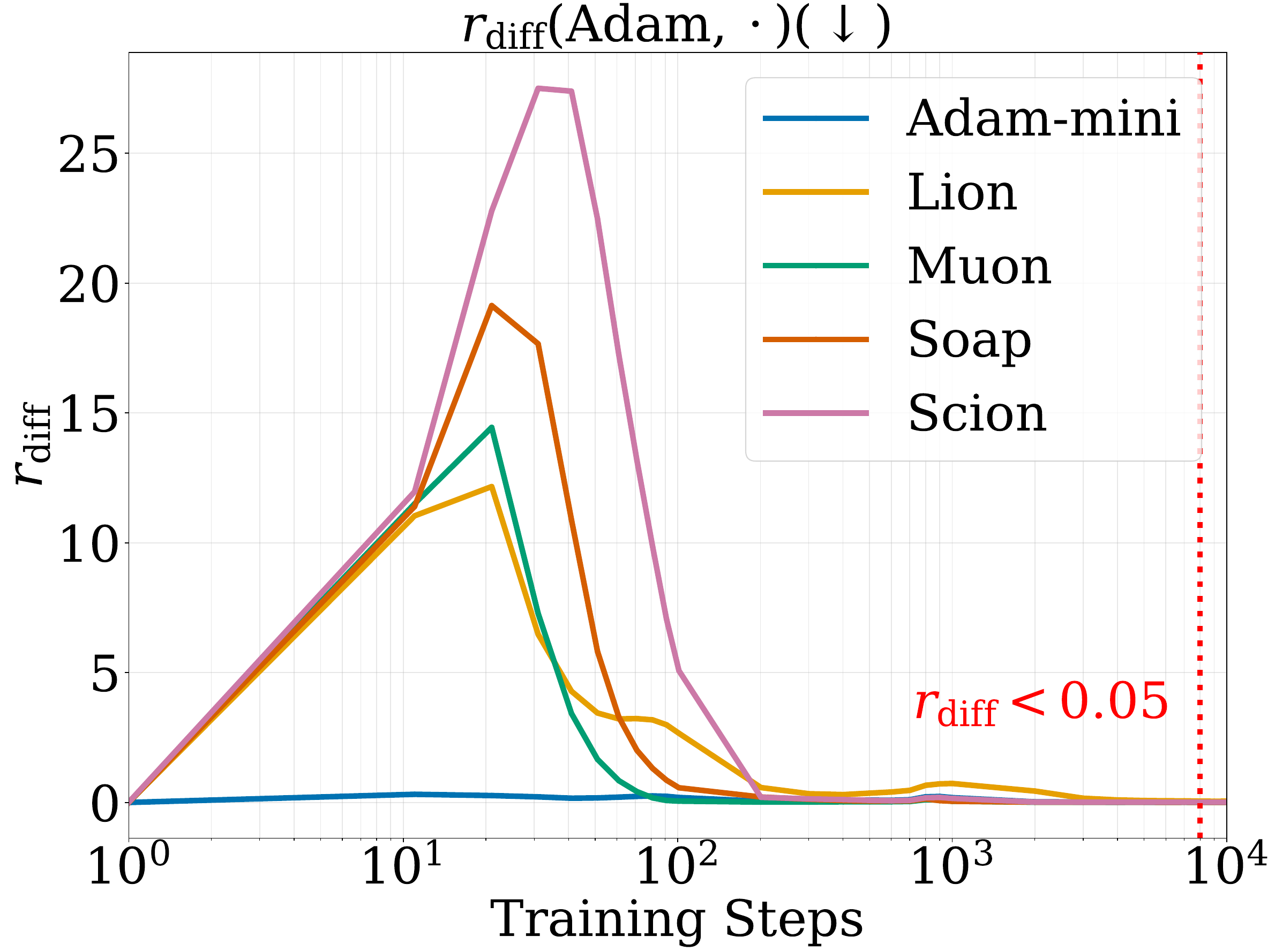}}
    \subfigure[$\differr$ for $0.5\eta^{\sfA}$.]{ \includegraphics[width=0.23\textwidth]{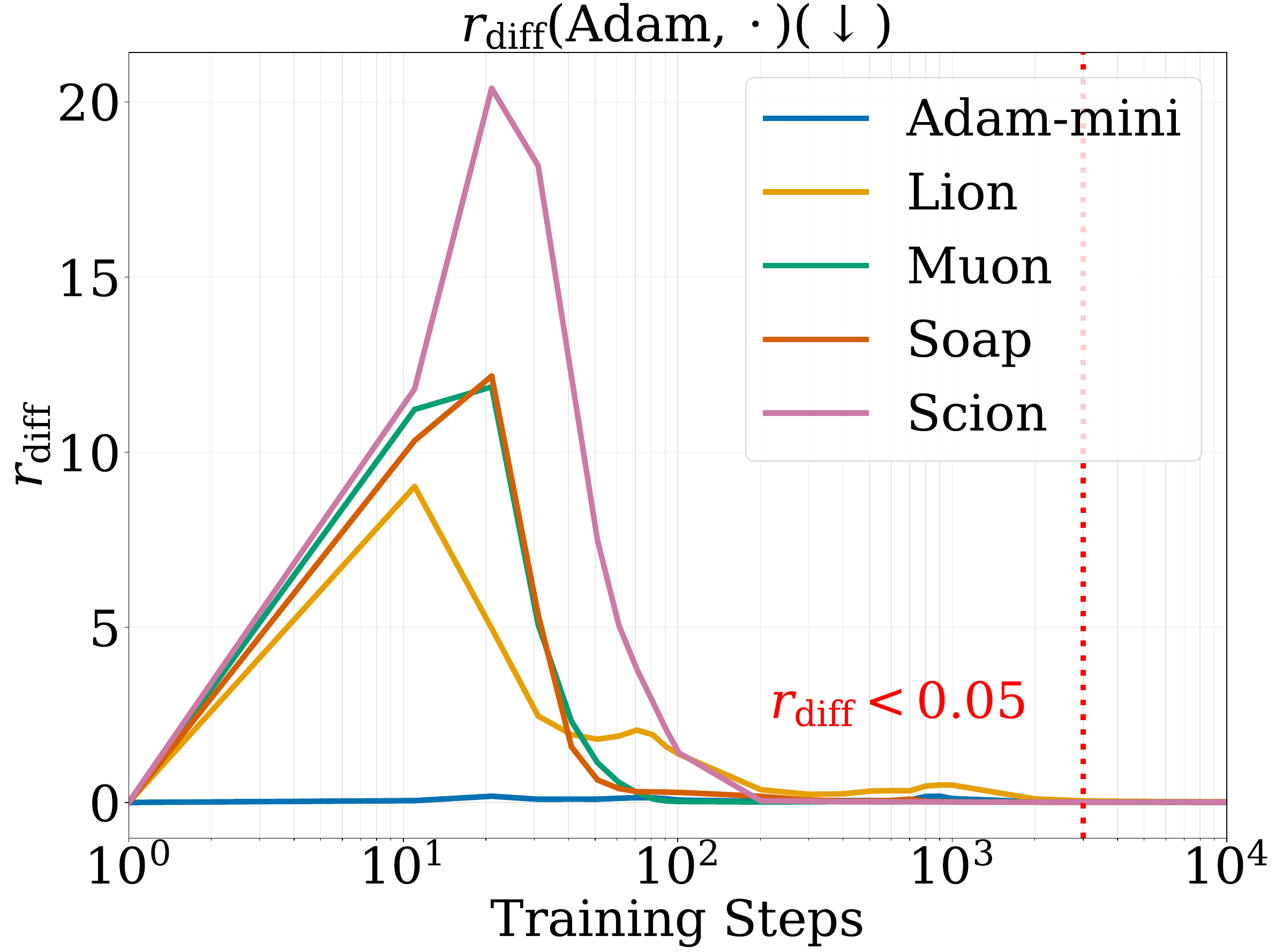}}
    \subfigure[$\differr$ for $0.75\eta^{\sfA}$.]{ \includegraphics[width=0.23\textwidth]{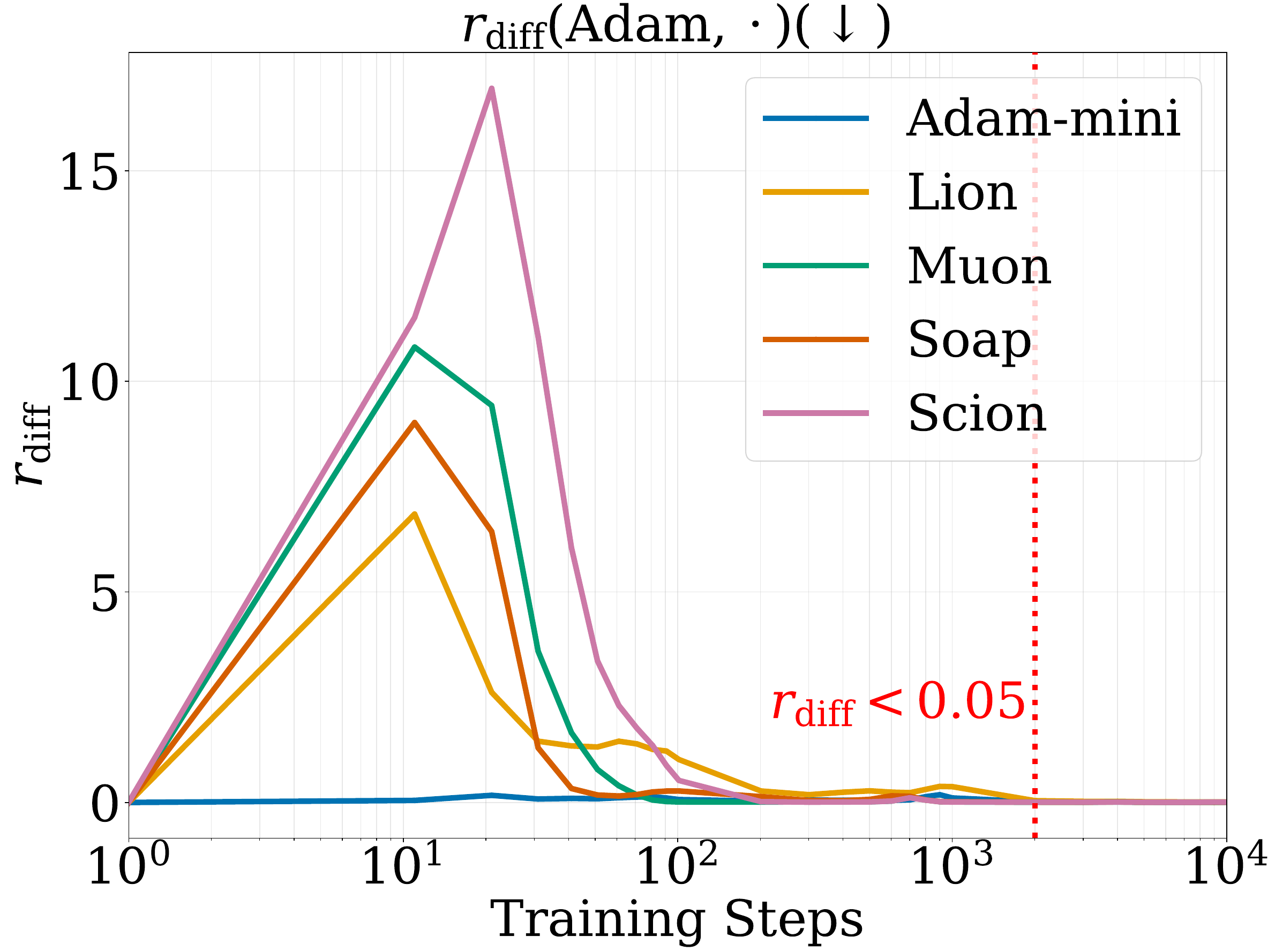}}
    \subfigure[$\differr$ for $1.25\eta^{\sfA}$.]{ \includegraphics[width=0.23\textwidth]{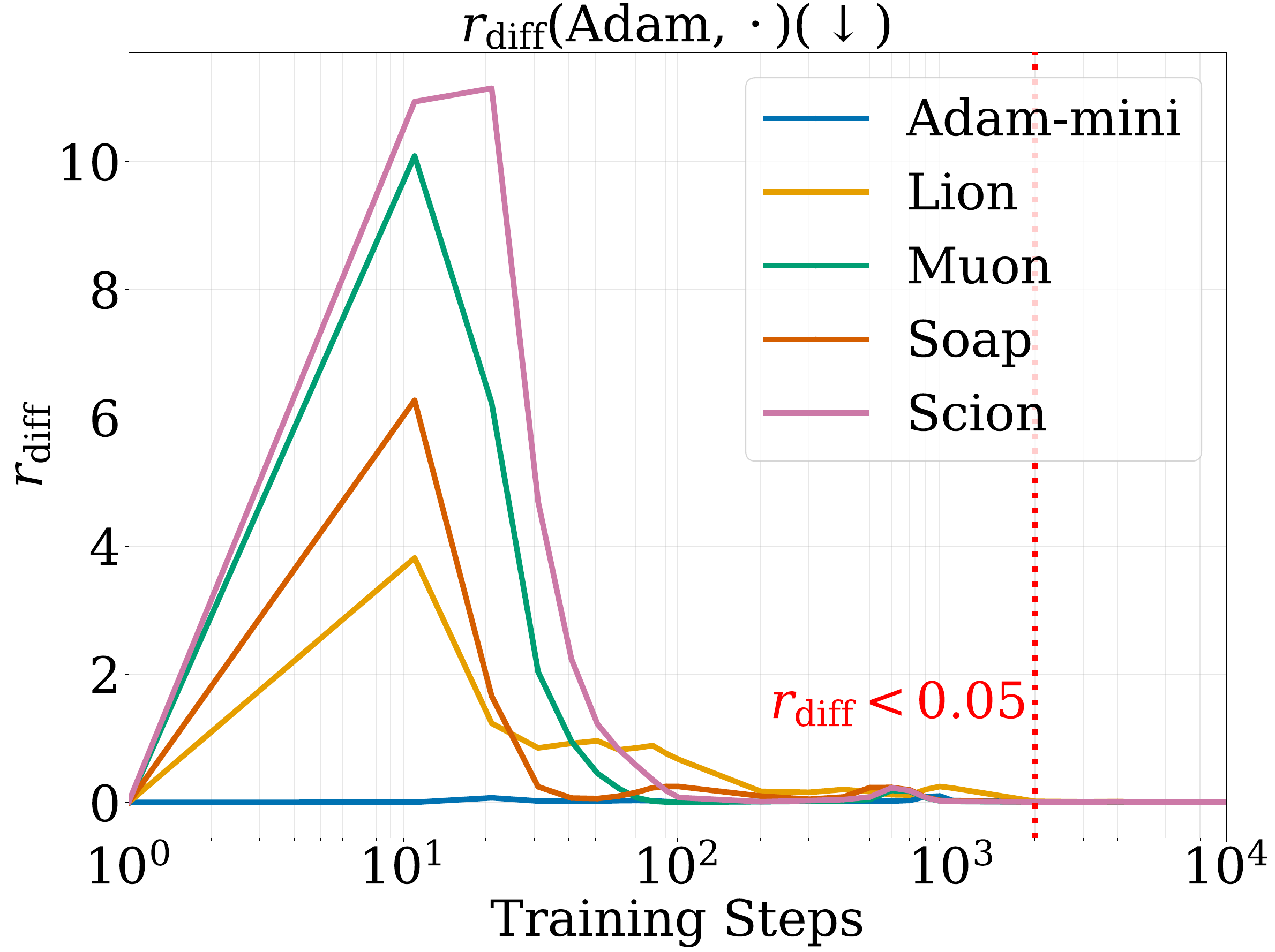}}

    \subfigure[$\cccerr$ for $0.25\eta^{\sfA}$.]{ \includegraphics[width=0.23\textwidth]{figures/CCC_centered_fineweb_Adam_lr_0.25.pdf}}
    \subfigure[$\cccerr$ for $0.5\eta^{\sfA}$.]{ \includegraphics[width=0.23\textwidth]{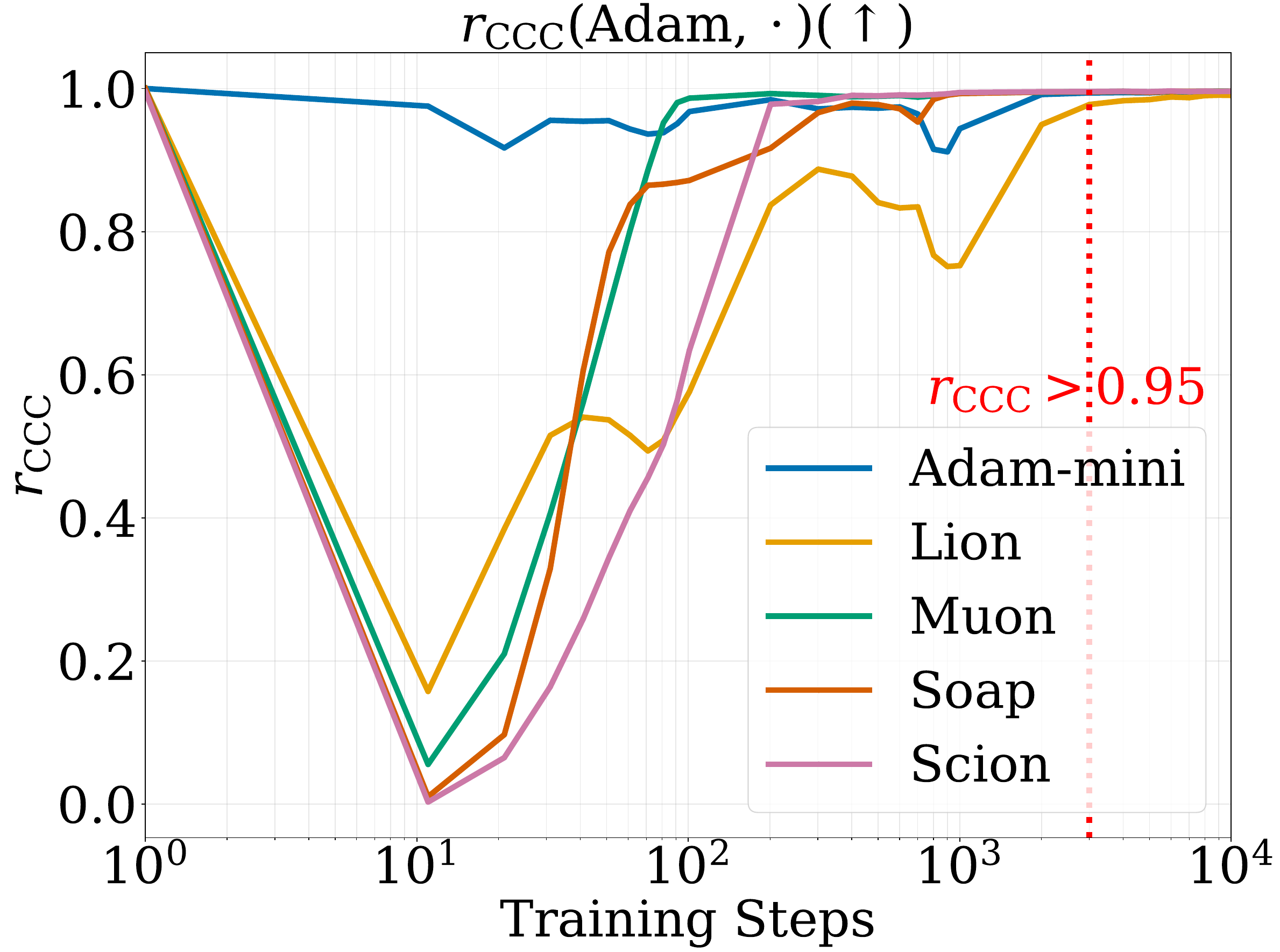}}
    \subfigure[$\cccerr$ for $0.75\eta^{\sfA}$.]{ \includegraphics[width=0.23\textwidth]{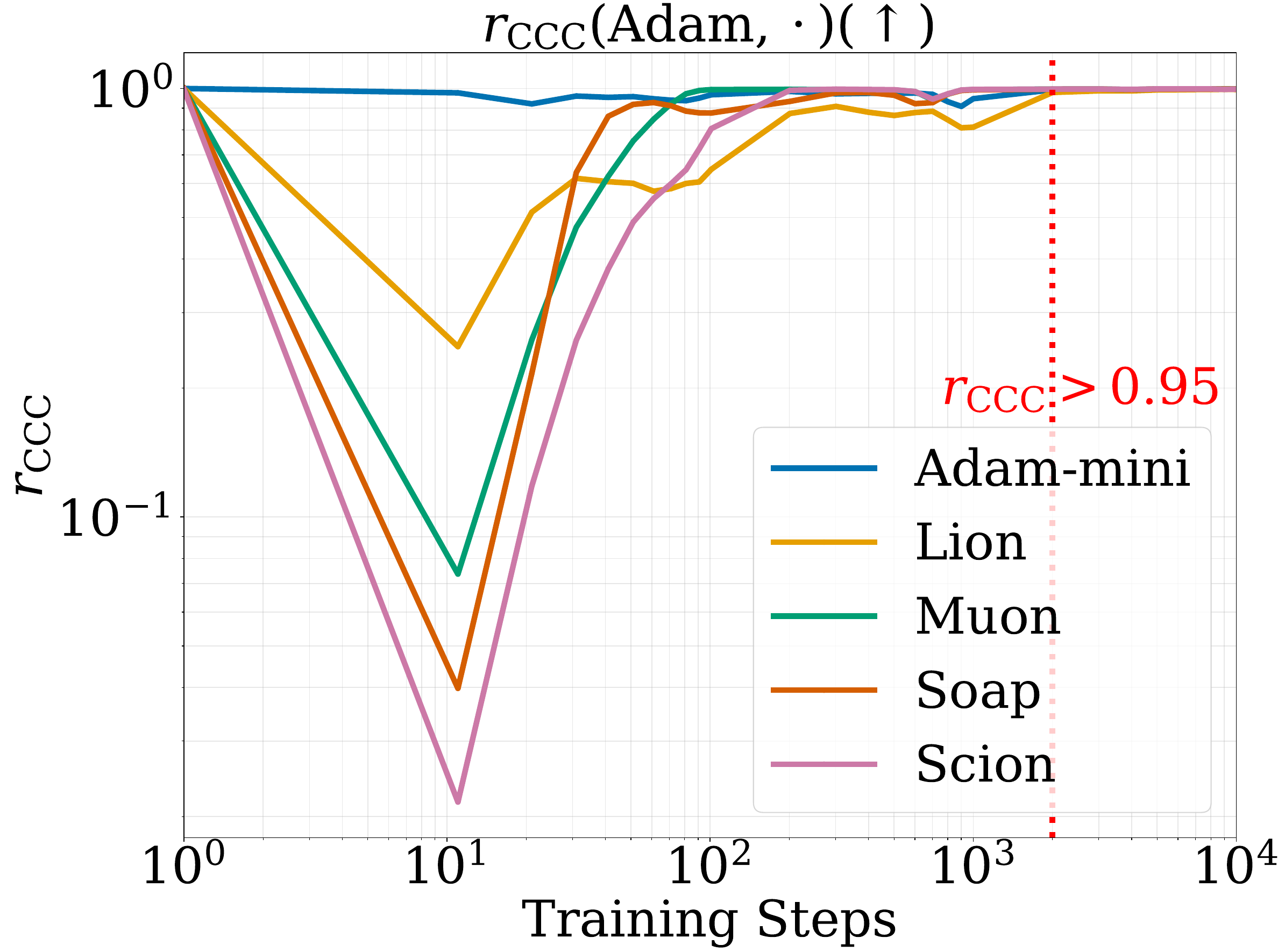}}
    \subfigure[$\cccerr$ for $1.25\eta^{\sfA}$.]{ \includegraphics[width=0.23\textwidth]{figures/CCC_centered_fineweb_Adam_lr_1.25.pdf}}
    \caption{Learning-rate ablations for $\differr$ and $\cccerr$. Panels (a)--(d) plot $\differr$ for learning-rate multipliers $0.25\eta^{\sfA}$, $0.5\eta^{\sfA}$, $0.75\eta^{\sfA}$, and $1.25\eta^{\sfA}$, respectively. Panels (e)--(h) plot $\cccerr$ for them. The results show that approximate \ac{rgi} holds robustly across the tested learning rates.} 
\end{figure}

\begin{figure}[H]
    \centering
    \subfigure[$\differr$ for $\gamma=0.85$.]{ \includegraphics[width=0.23\textwidth]{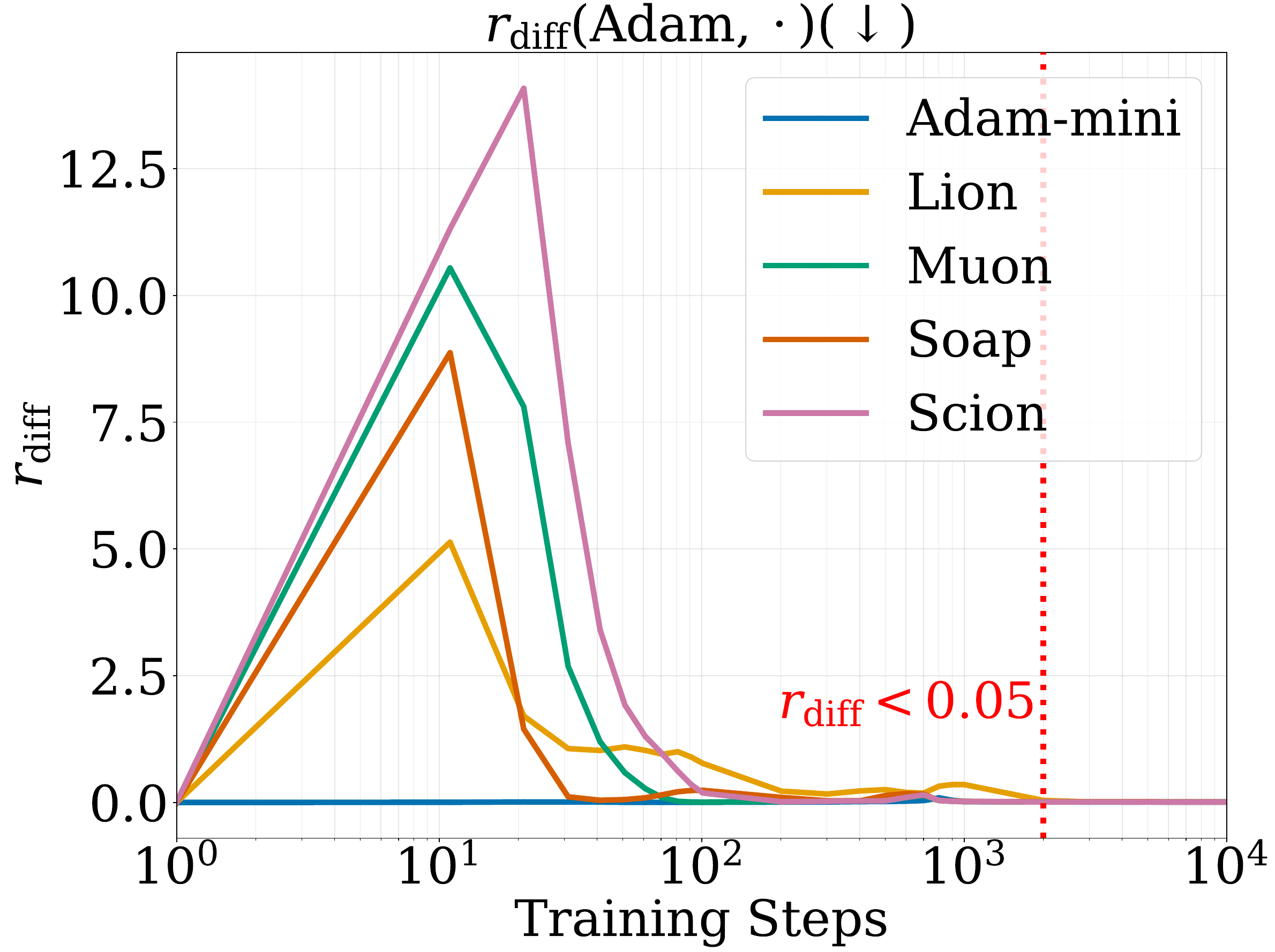}}
    \subfigure[$\differr$ for $\gamma=0.9$.]{ \includegraphics[width=0.23\textwidth]{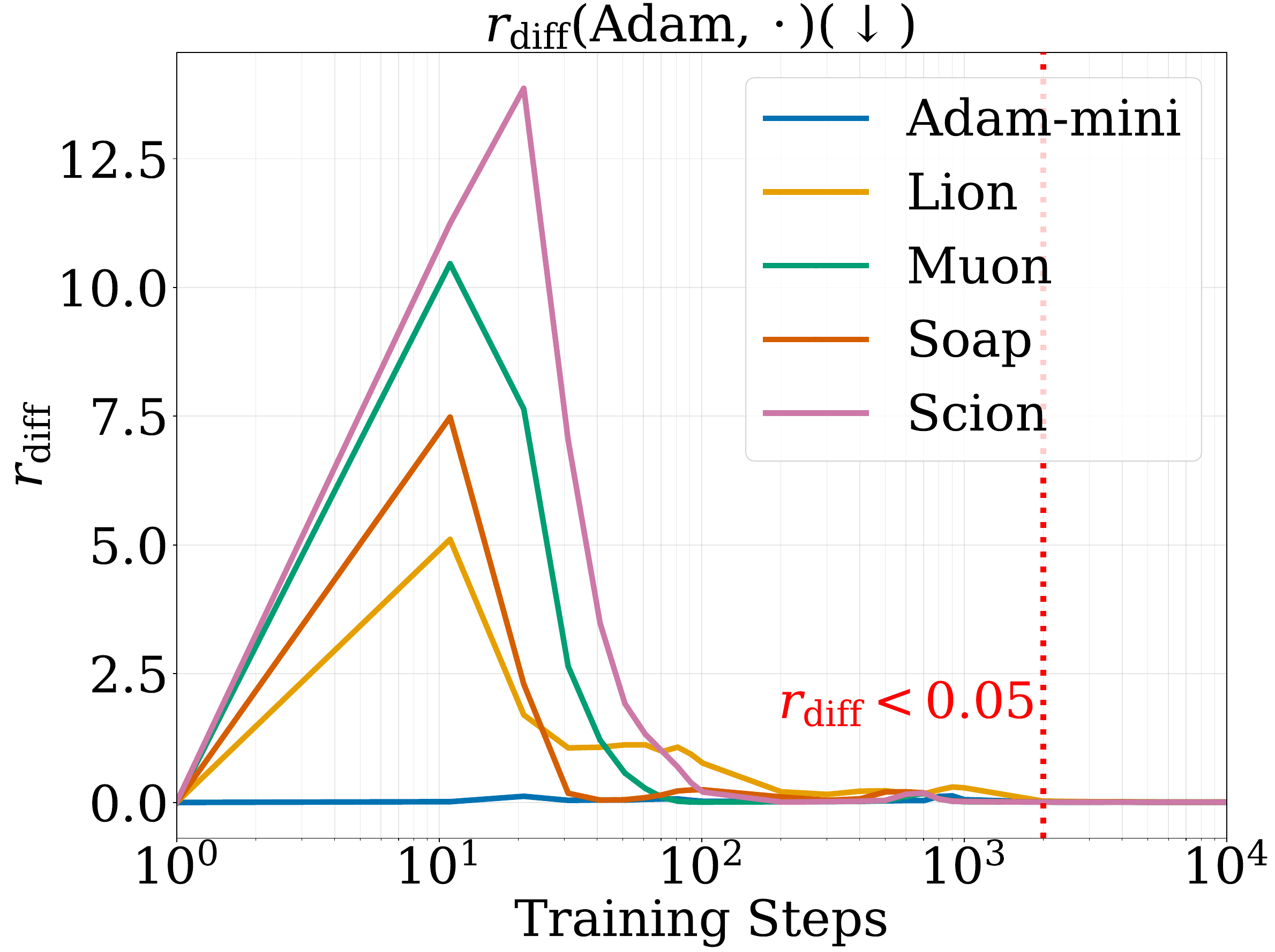}}
    \subfigure[$\differr$ for $\gamma=0.95$.]{ \includegraphics[width=0.23\textwidth]{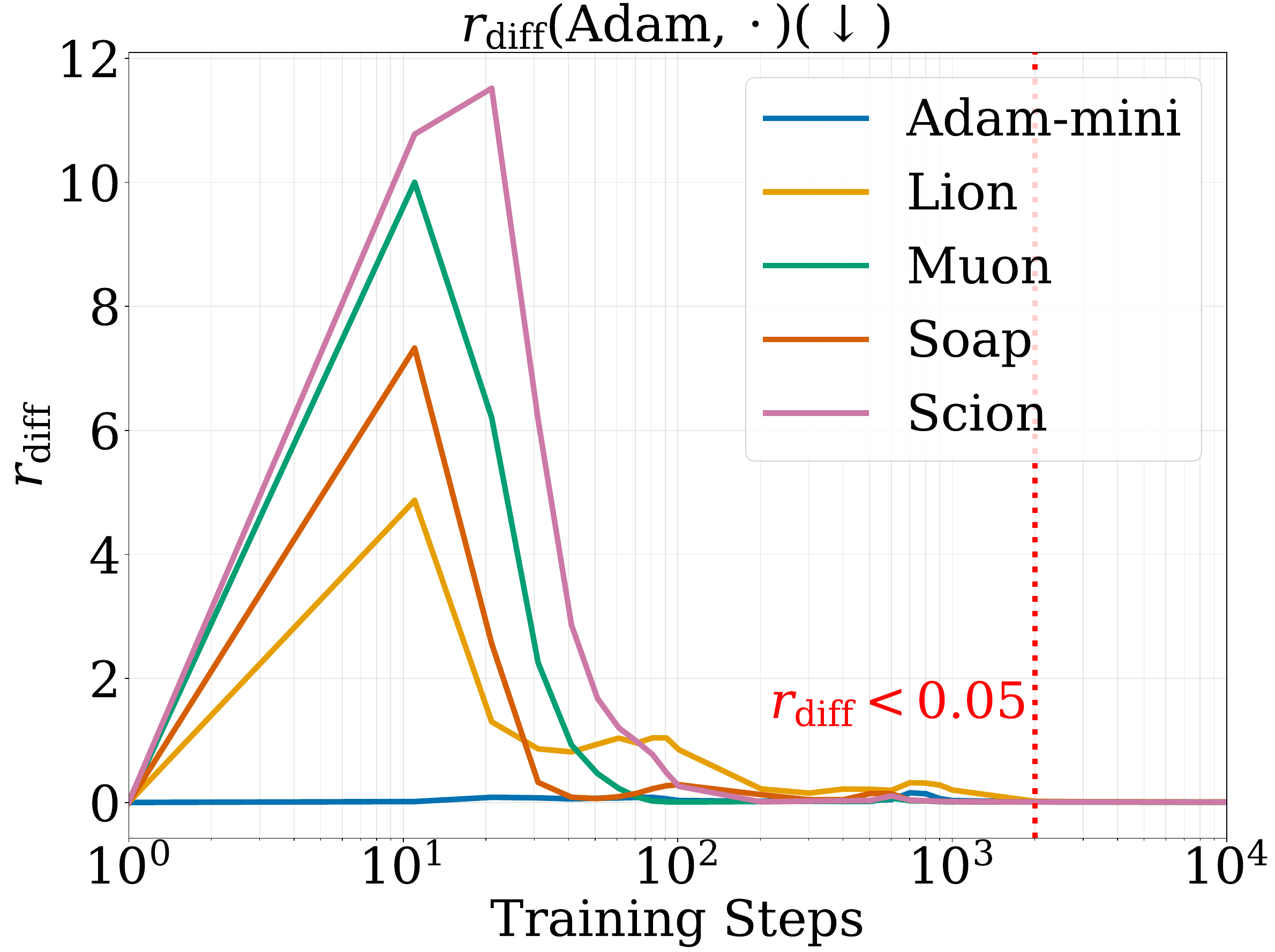}}

    \subfigure[$\cccerr$ for $\gamma=0.85$.]{ \includegraphics[width=0.23\textwidth]{figures/CCC_centered_fineweb_Adam_mom_0.85.pdf}}
    \subfigure[$\cccerr$ for $\gamma=0.9$.]{ \includegraphics[width=0.23\textwidth]{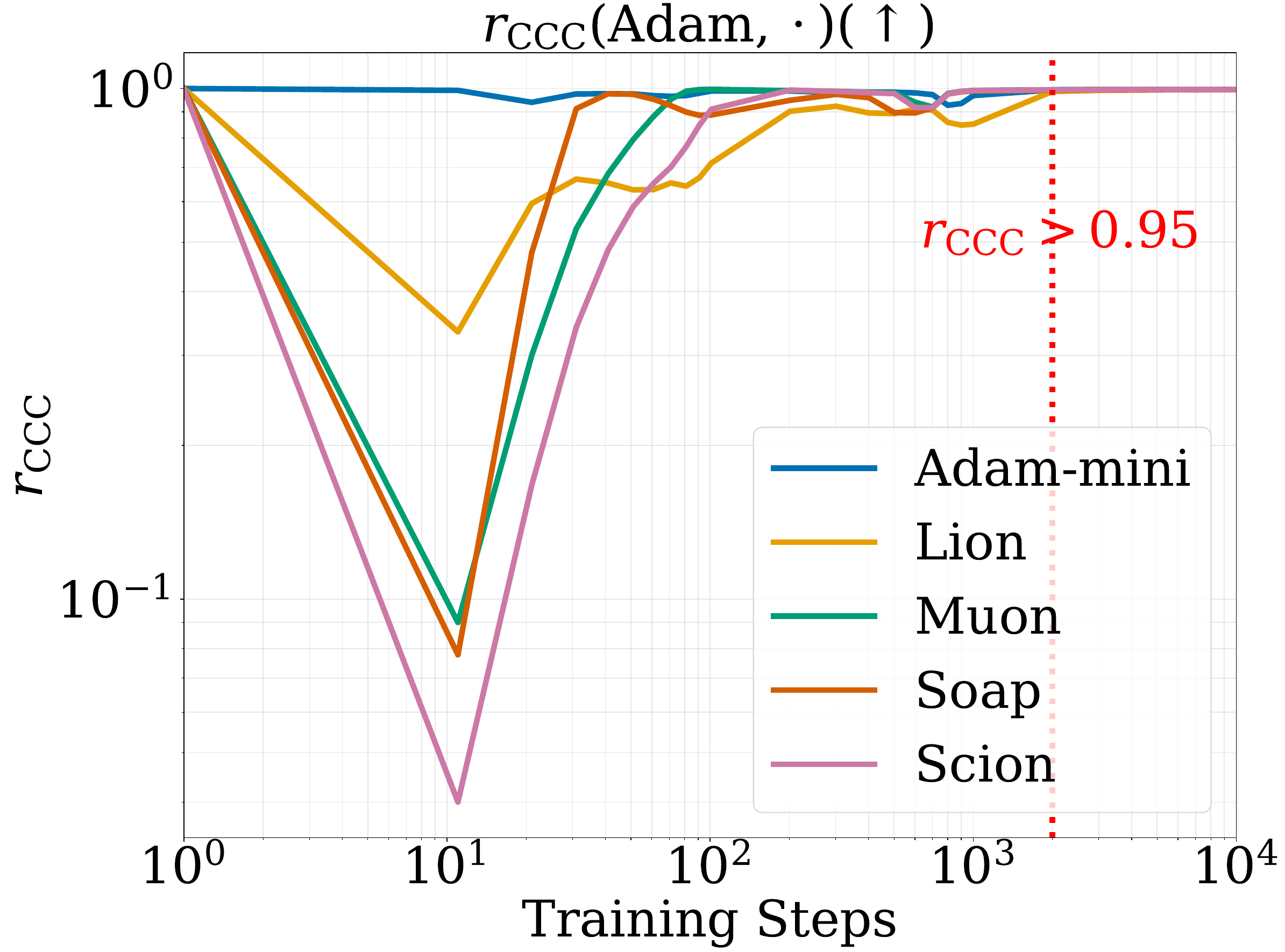}}
    \subfigure[$\cccerr$ for $\gamma=0.95$.]{ \includegraphics[width=0.23\textwidth]{figures/CCC_centered_fineweb_Adam_mom_0.95.pdf}}
    % % \vspace{-0.5em}
    \caption{Momentum ablations for $\differr$ and $\cccerr$. Panels (a)--(c) plot $\differr$ for momentumns $\gamma=0.85.0.9,0.95$, respectively. Panels (d)--(f) plot $\cccerr$ for them. The results show that approximate \ac{rgi} holds robustly across the tested momentum values.} 
\end{figure}

% \begin{figure}[H]
%     \centering
%     \subfigure[$\differr$ for $\gamma=0.85$.]{ \includegraphics[width=0.27\textwidth]{figures/CCC_centered_fineweb_Adam_mom_0.85.pdf}}
%     \subfigure[$\differr$ for $\gamma=0.9$.]{ \includegraphics[width=0.27\textwidth]{figures/CCC_centered_fineweb_Adam_mom_0.9.pdf}}
%     \subfigure[$\differr$ for $\gamma=0.95$.]{ \includegraphics[width=0.27\textwidth]{figures/CCC_centered_fineweb_Adam_mom_0.95.pdf}}
%     \caption{Momentum ablations for $\cccerr$. Panels (a)--(c) plot $\cccerr$ for optimizers configured with $\gamma=0.85$, $0.9$, and $0.95$, respectively. Approximate \ac{rgi} holds robustly across the tested momentum values.} 
% \end{figure}

To evaluate the effect of the learning-rate scheduler, we ablate the number of warmup steps in the Warmup-Stable-Decay scheduler. The experiments in Section~\ref{sec:opt} use $700$ warmup steps. We additionally evaluate all optimizers with $100$ and $1{,}400$ warmup steps. Figure~\ref{fig:scheduler} reports the corresponding values of $\differr$ and $\cccerr$. Across all tested warmup lengths, approximate \ac{rgi} remains consistently observed.
\begin{figure}[H]
    \centering
    \subfigure[$\differr$ with $100$ warmup steps.]{ \includegraphics[width=0.23\textwidth]{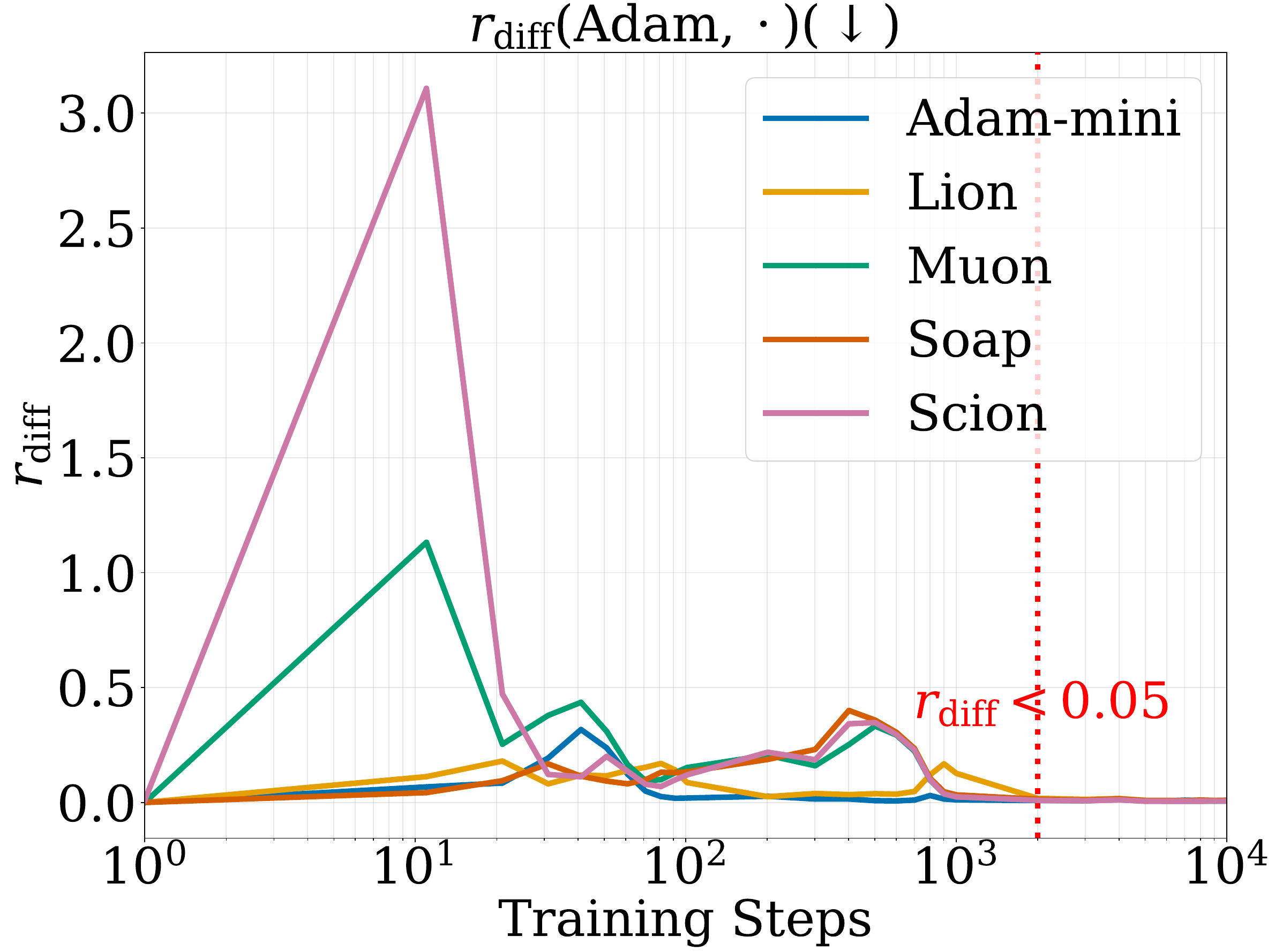}}
    \subfigure[$\differr$ with $1,400$ warmup steps.]{ \includegraphics[width=0.23\textwidth]{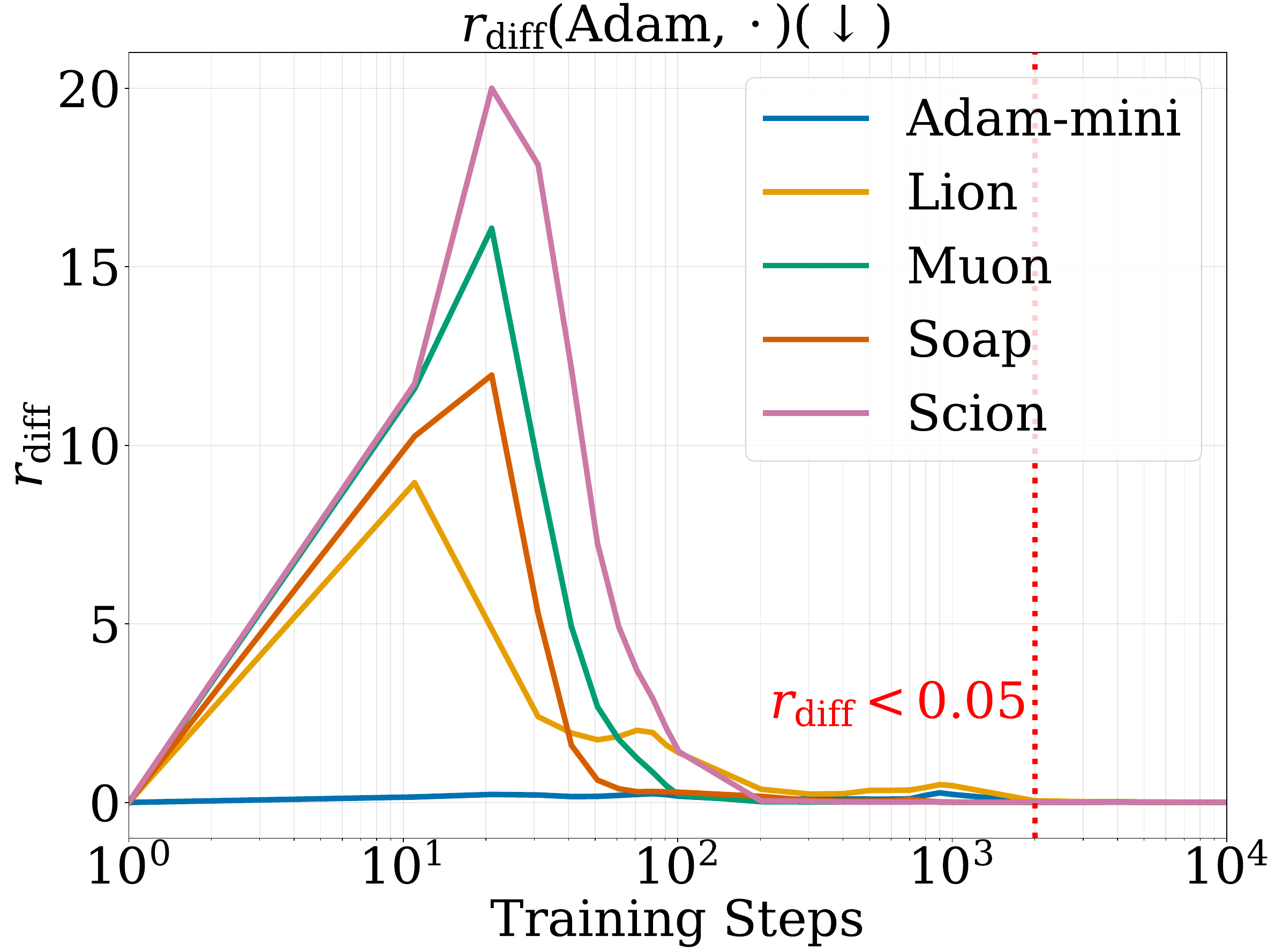}}
    \subfigure[$\cccerr$ with $100$ warmup steps.]{ \includegraphics[width=0.23\textwidth]{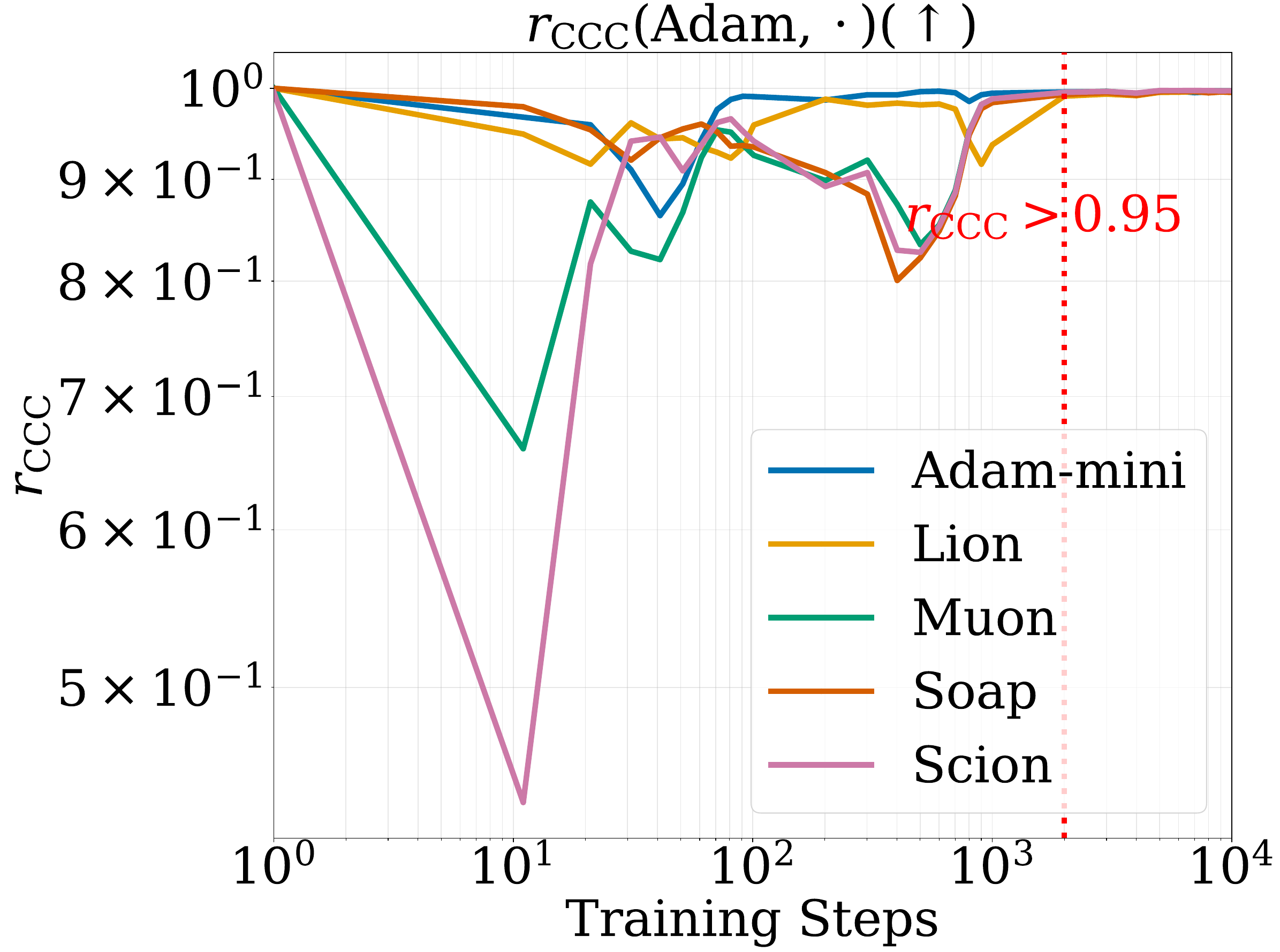}}
    \subfigure[$\cccerr$ with $1,400$ warmup steps.]{ \includegraphics[width=0.23\textwidth]{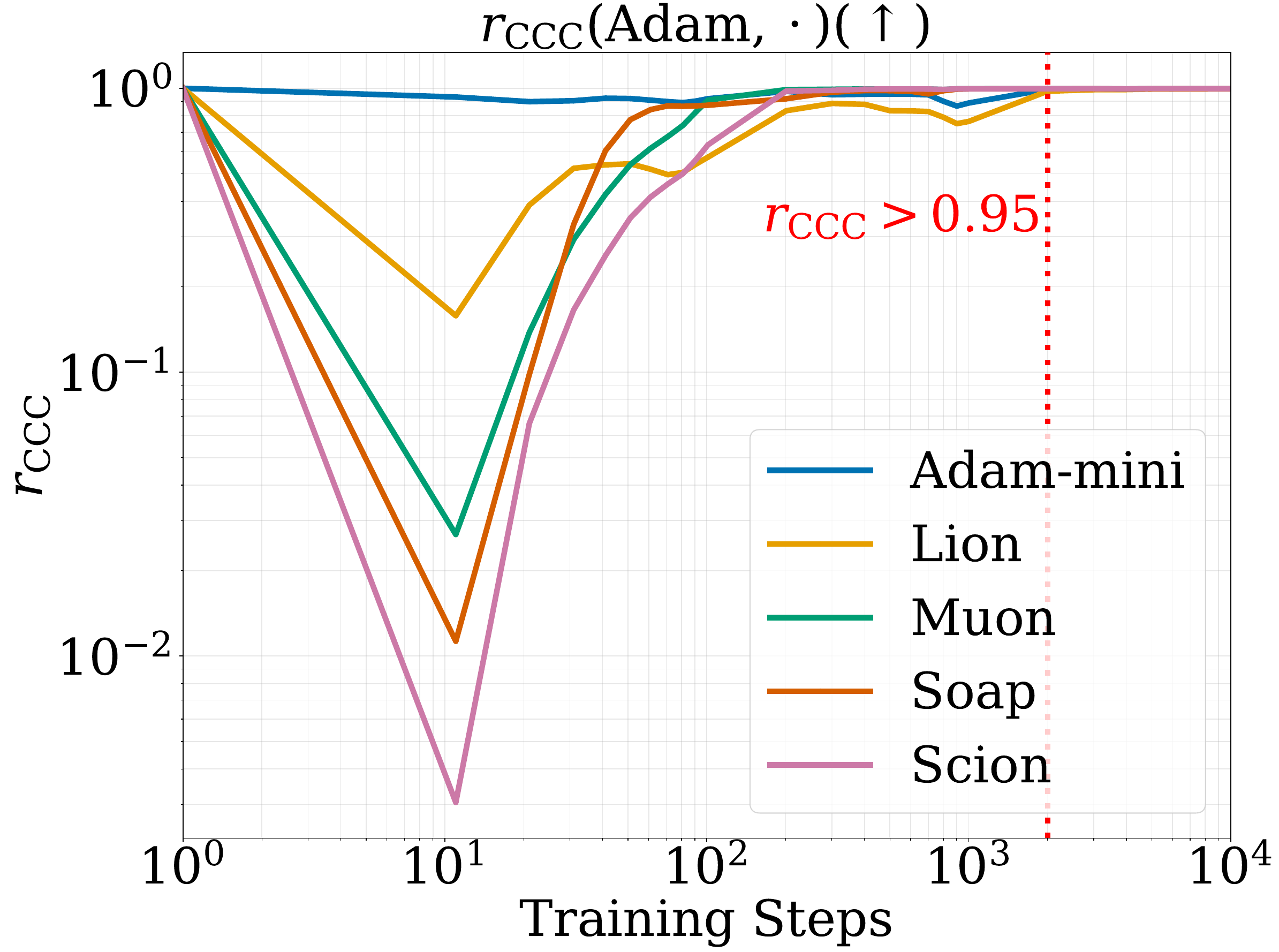}}
    \caption{Warmup-step ablation. Panels (a) and (b) plot $\differr$ for $100$ and $1{,}400$ warmup steps, while Panels (c) and (d) plot $\cccerr$ for the same warmup settings. Approximate \ac{rgi} holds robustly across the tested scheduler warmup lengths.} 
    \label{fig:scheduler}
\end{figure}

The experiments in Section~\ref{sec:opt} use weight decay as $0$. To evaluate the effect of the  weight decay, we additionally evaluate all optimizers with weight decay $0.05$ and $0.1$. Figure~\ref{fig:wd} reports the corresponding values of $\differr$ and $\cccerr$. Across all tested weight-decay values, approximate \ac{rgi} remains consistently observed.
\begin{figure}[H]
    \centering
    \subfigure[$\differr$ with weight decay value $0.05$.]{ \includegraphics[width=0.23\textwidth]{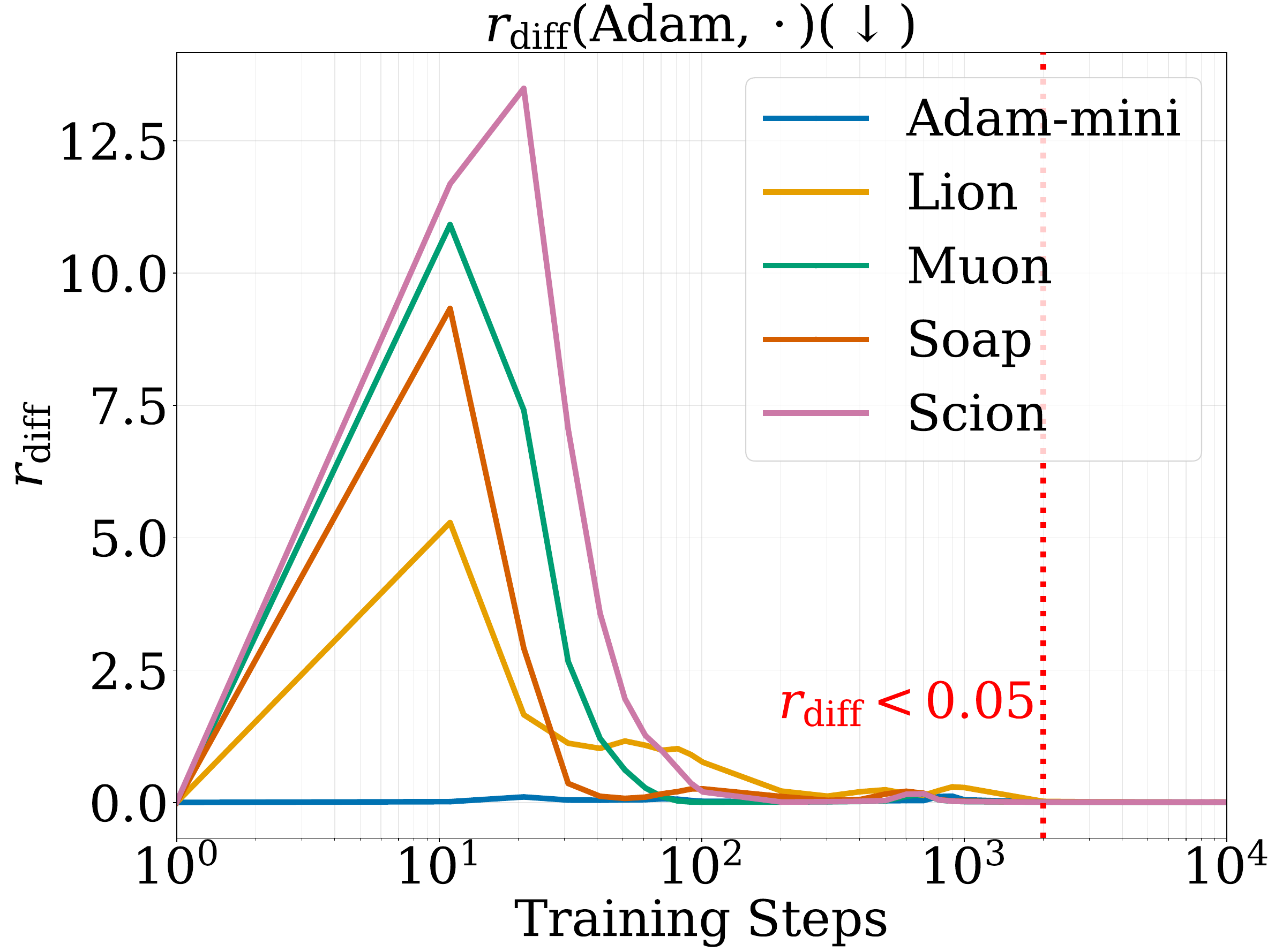}}
    \subfigure[$\differr$ with weight decay value $0.1$.]{ \includegraphics[width=0.23\textwidth]{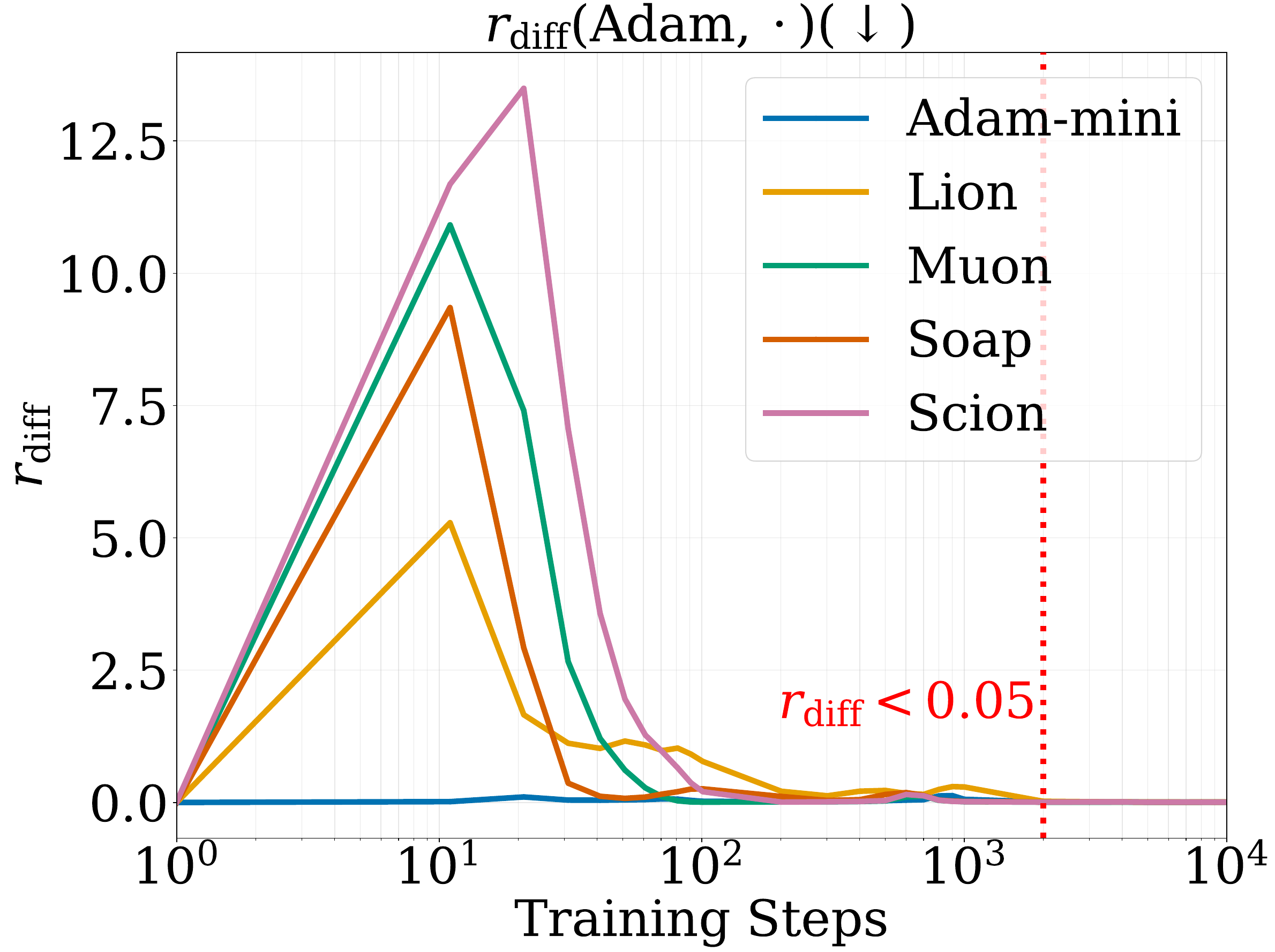}}
    \subfigure[$\cccerr$ with weight decay value $0.05$.]{ \includegraphics[width=0.23\textwidth]{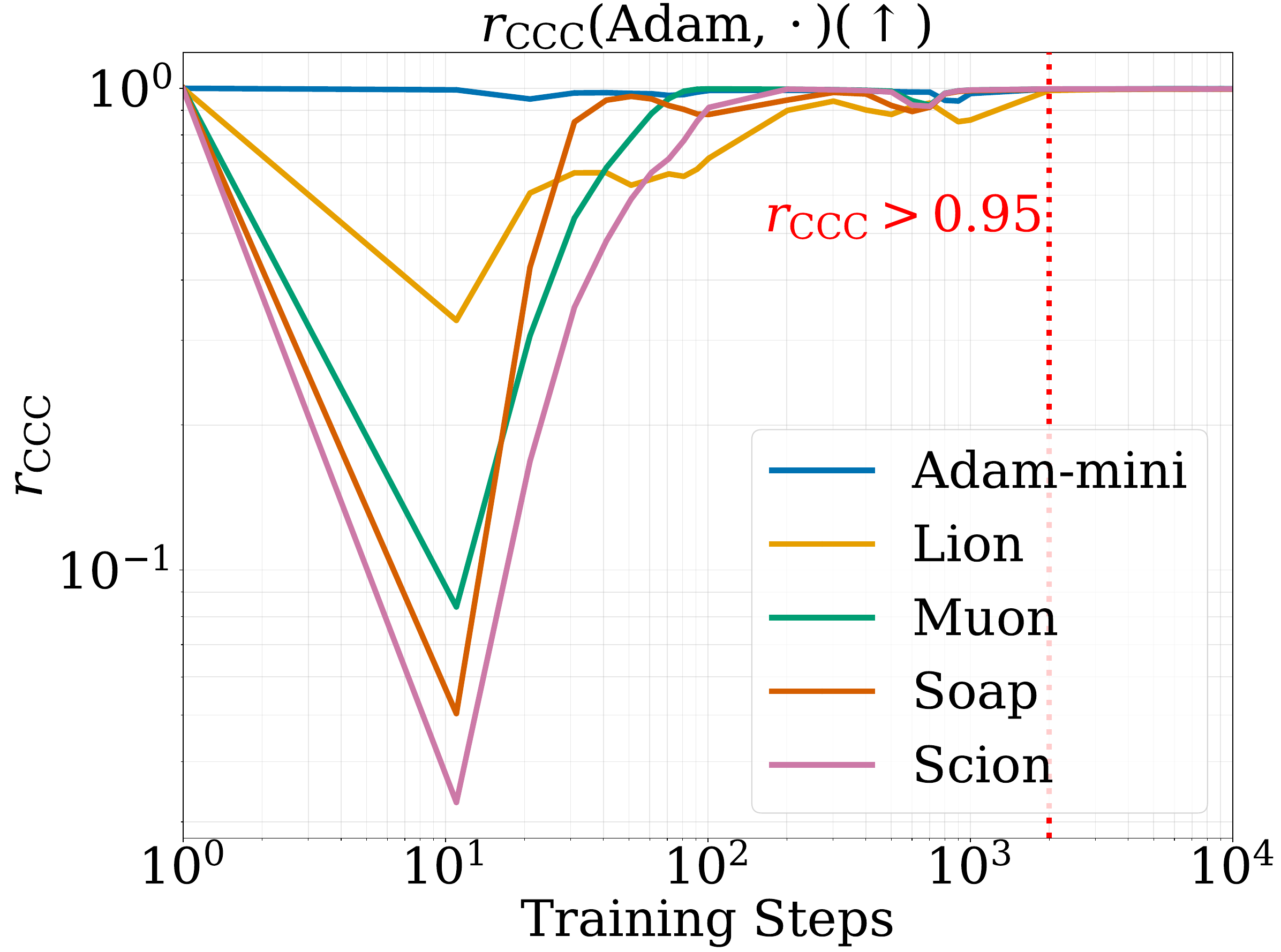}}
    \subfigure[$\cccerr$ with weight decay value $0.1$.]{ \includegraphics[width=0.23\textwidth]{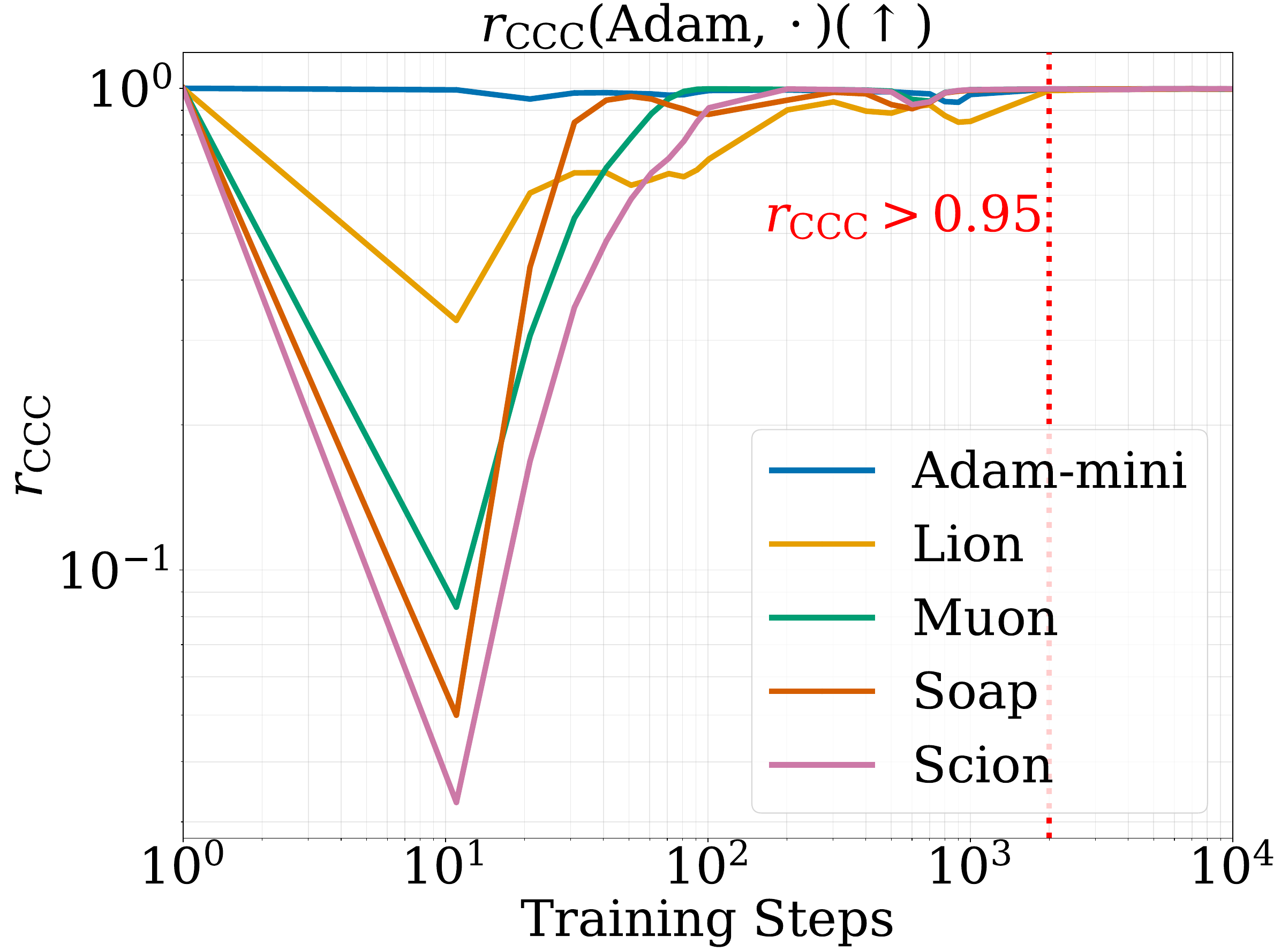}}
    \caption{Weight-decay ablation. Panels (a) and (b) plot $\differr$ for $0.05$ and $0.1$ weight decay, while Panels (c) and (d) plot $\cccerr$ for the same weight decay settings. Approximate \ac{rgi} holds robustly across the tested weight decay values.} 
    \label{fig:wd}
\end{figure}
\subsubsection{Additional Results for Section~\ref{sec:arch}}

\begin{figure}[H]
    \centering
    % % \vspace{-1em}
    \subfigure[$\differr$ under different initial parameters.]{
        \includegraphics[width=0.31\textwidth]{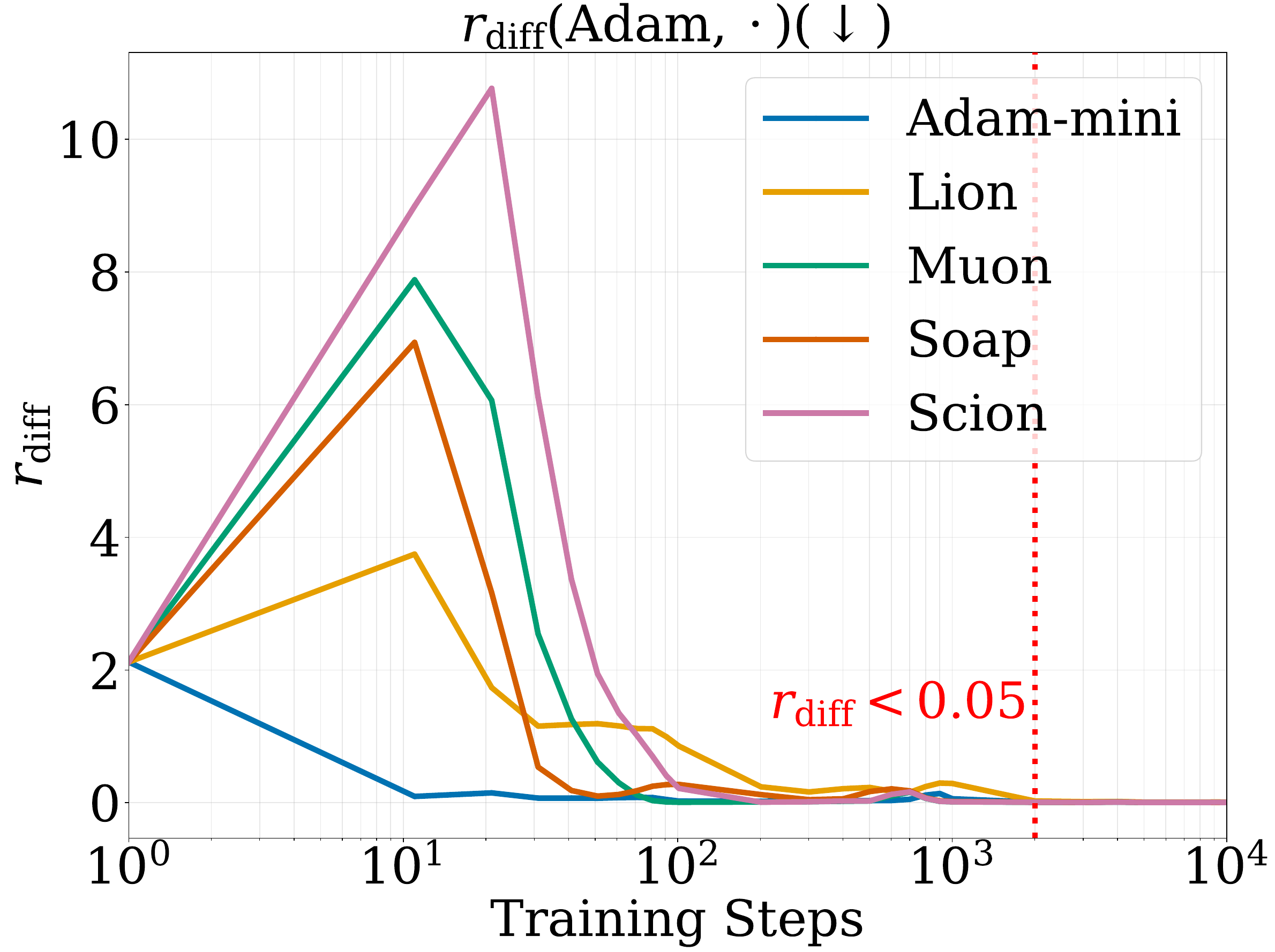}
        \label{fig:r_diff_init}
    }
    \subfigure[$\differr$ for hidden state dimensions $768$ and $384$.]{
        \includegraphics[width=0.31\textwidth]{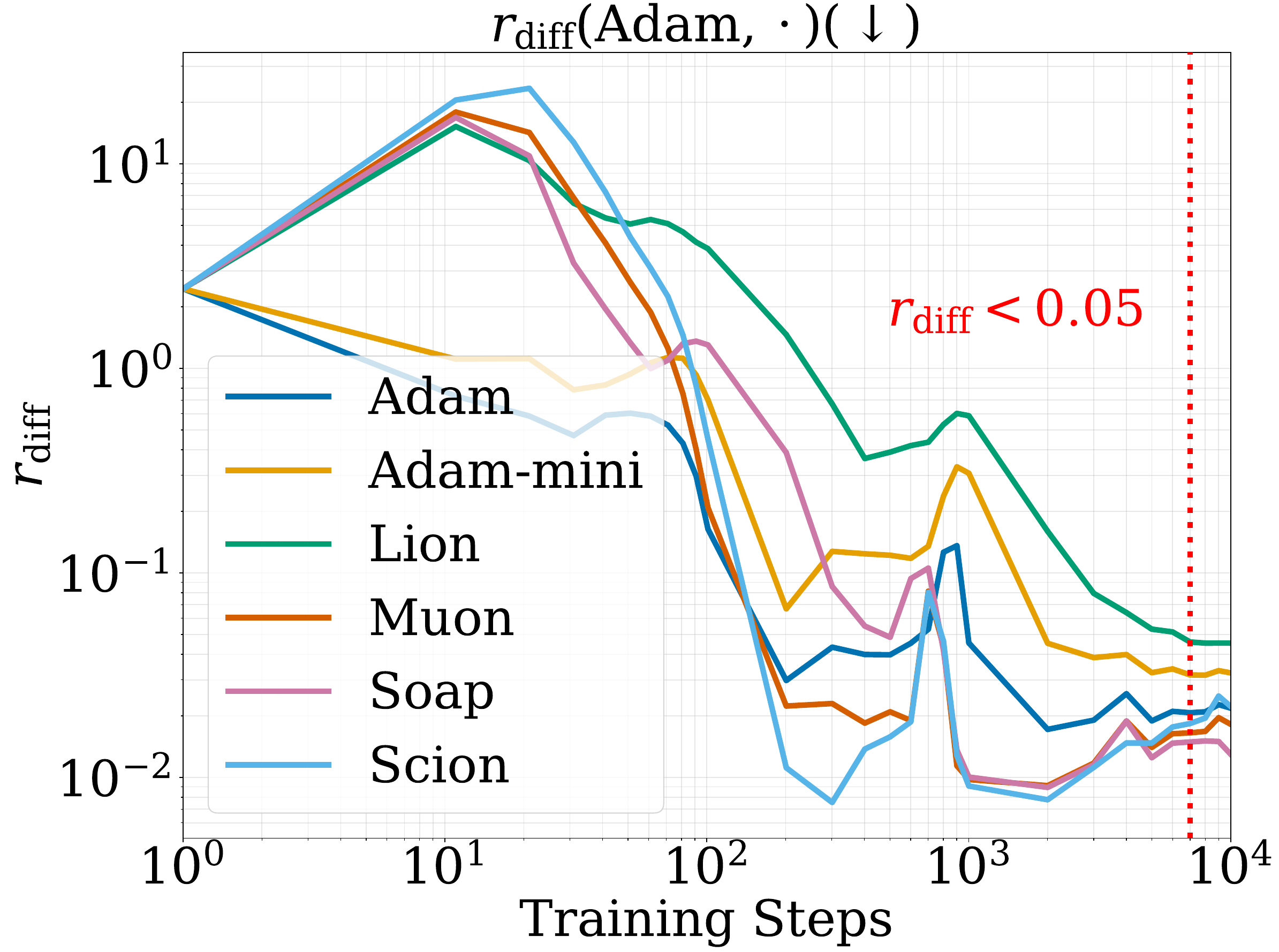}
        \label{fig:r_diff_cross_384}
    }
    \subfigure[$\differr$ for hidden state dimensions $768$ and $64$.]{
        \includegraphics[width=0.31\textwidth]{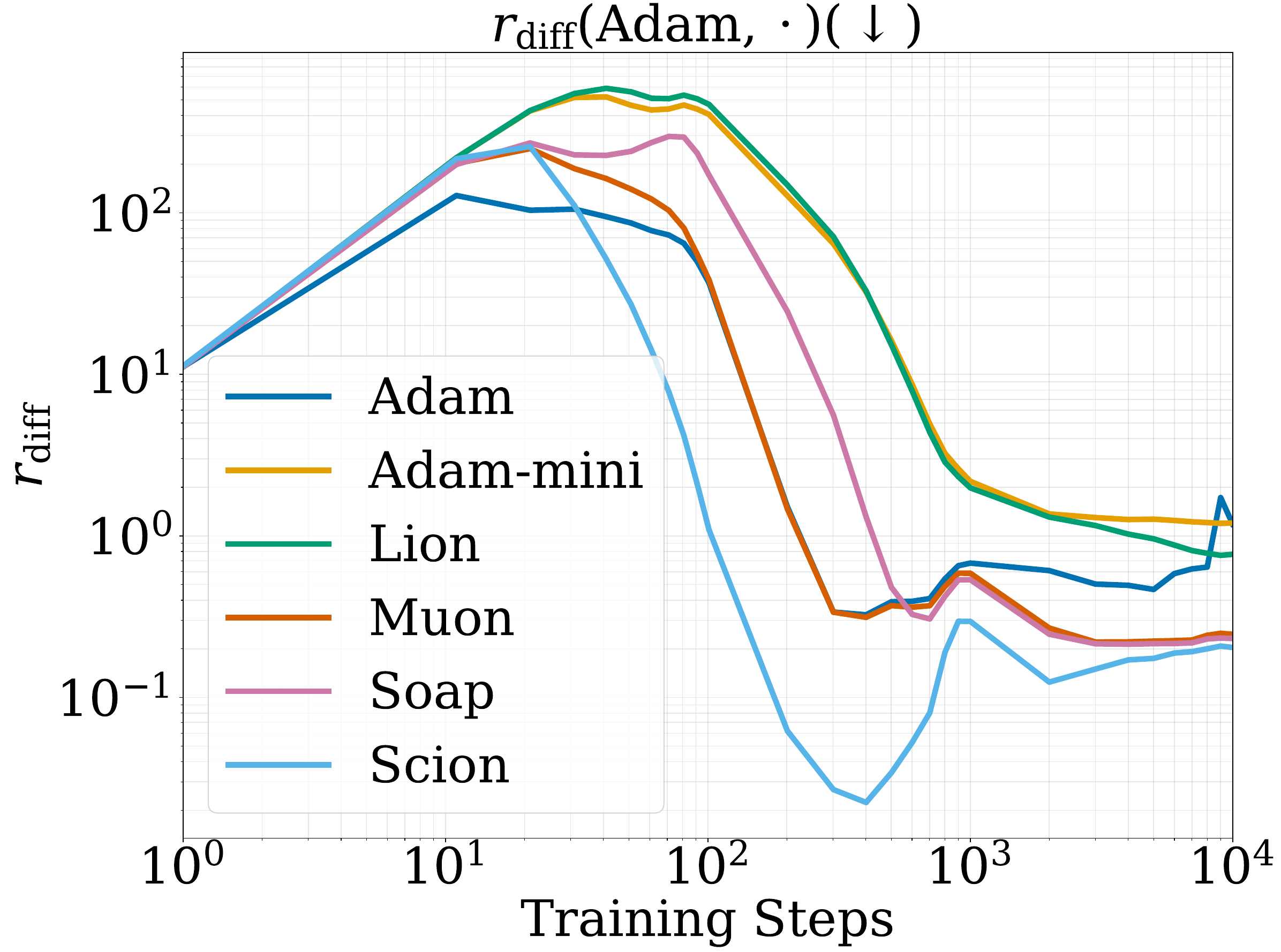}
        \label{fig:r_diff_cross_64}
    }
    
    \caption{\ac{rgi} across optimizers, initializations, and model sizes. Panel (a) plots $\differr$ across optimizers with different initializations. Panels (b) and (c) plot $\differr$ across models with hidden-state dimensions $768$, $384$, and $64$. \ac{rgi} continues to hold across different initializations and moderate changes in model size.}
    \label{fig:rdiff_same_arch}
    % % \vspace{-2.0em}
\end{figure}

\begin{figure}[H]
    \centering
    \subfigure[Values of $\differr$ between \ac{fa} models with RoPE and NoPE.]{ \includegraphics[width=0.23\textwidth]{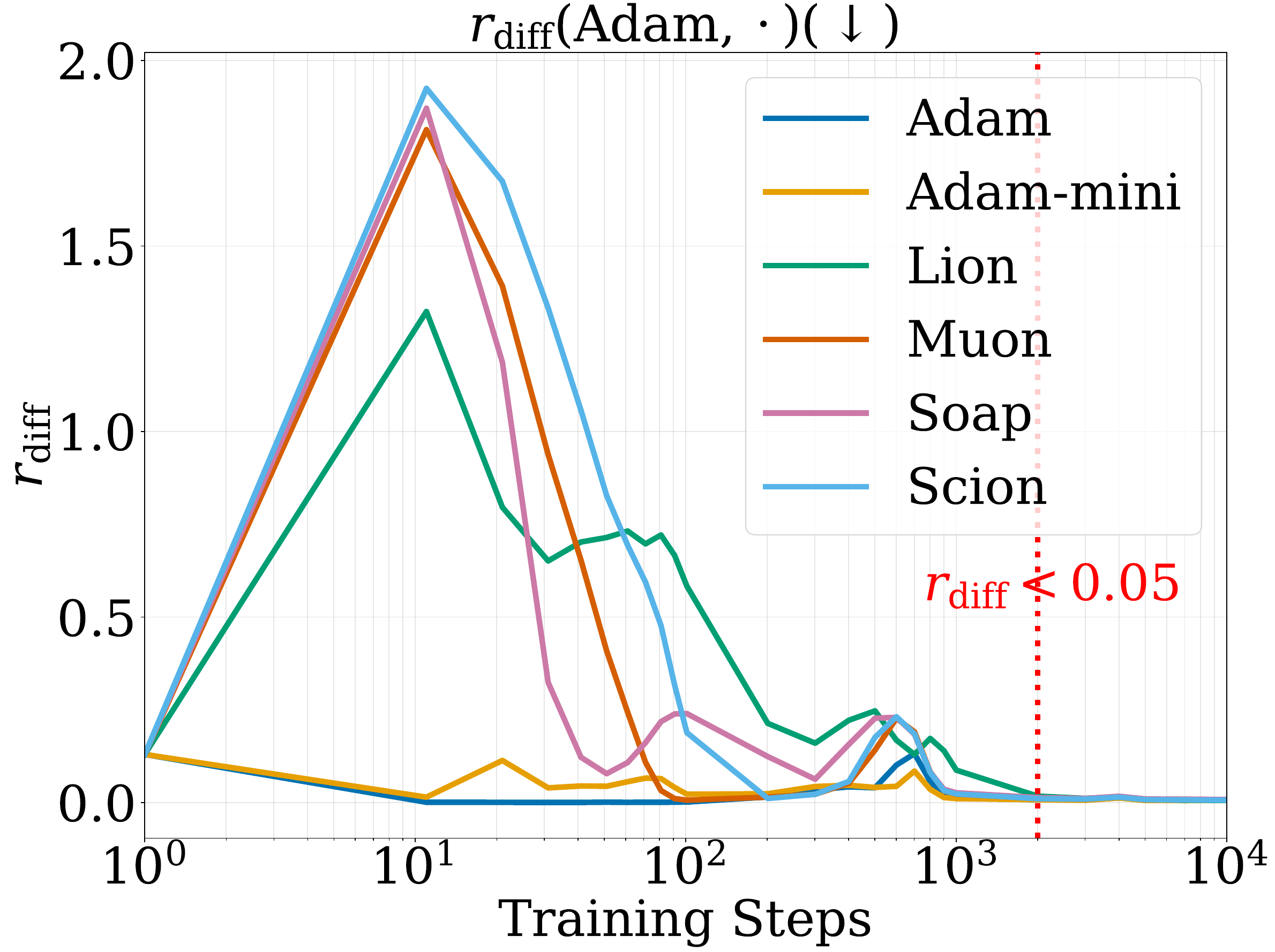}
    \label{fig:r_diff_nope}
    }
    \subfigure[Values of $\differr$ between \ac{fa} models with RoPE and ALiBi.]{ \includegraphics[width=0.23\textwidth]{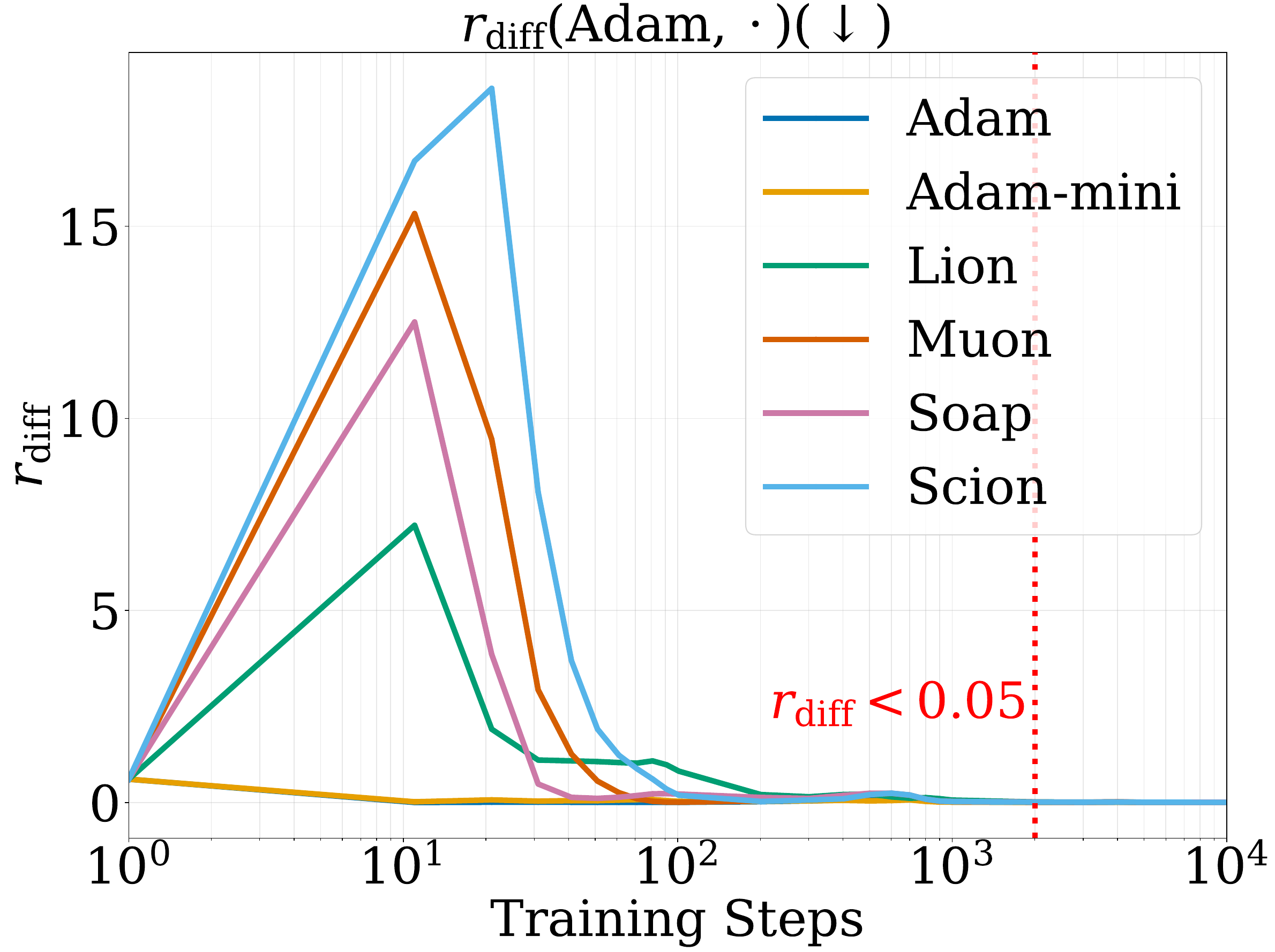}
    \label{fig:r_diff_alibi}
    }
    \subfigure[Values of $\differr$ between \ac{fa} and \ac{swa} models.]{ \includegraphics[width=0.23\textwidth]{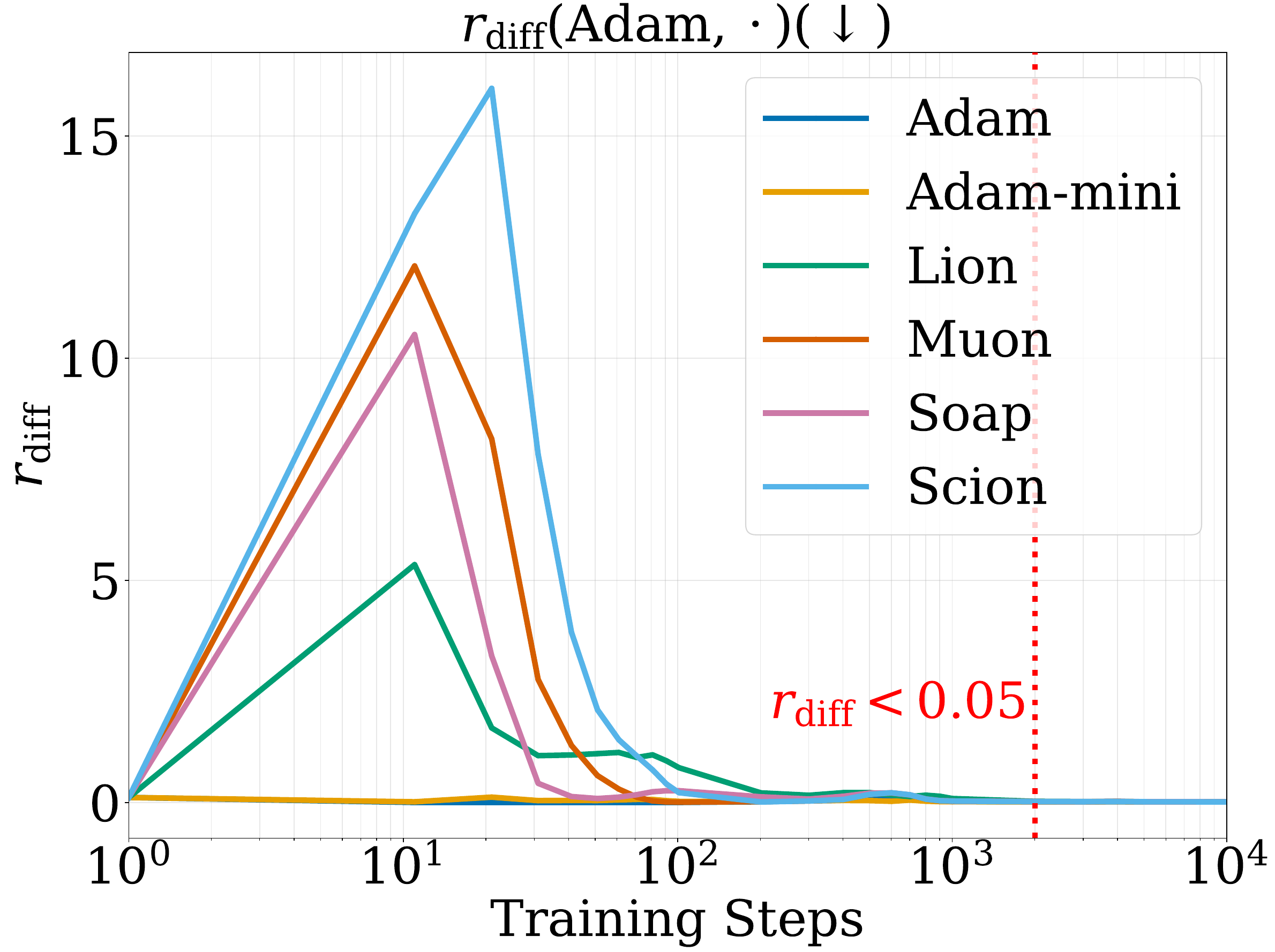}
    \label{fig:r_diff_swa}
    }
    \subfigure[Values of $\differr$ between models with GELU and SReLU.]{ \includegraphics[width=0.23\textwidth]{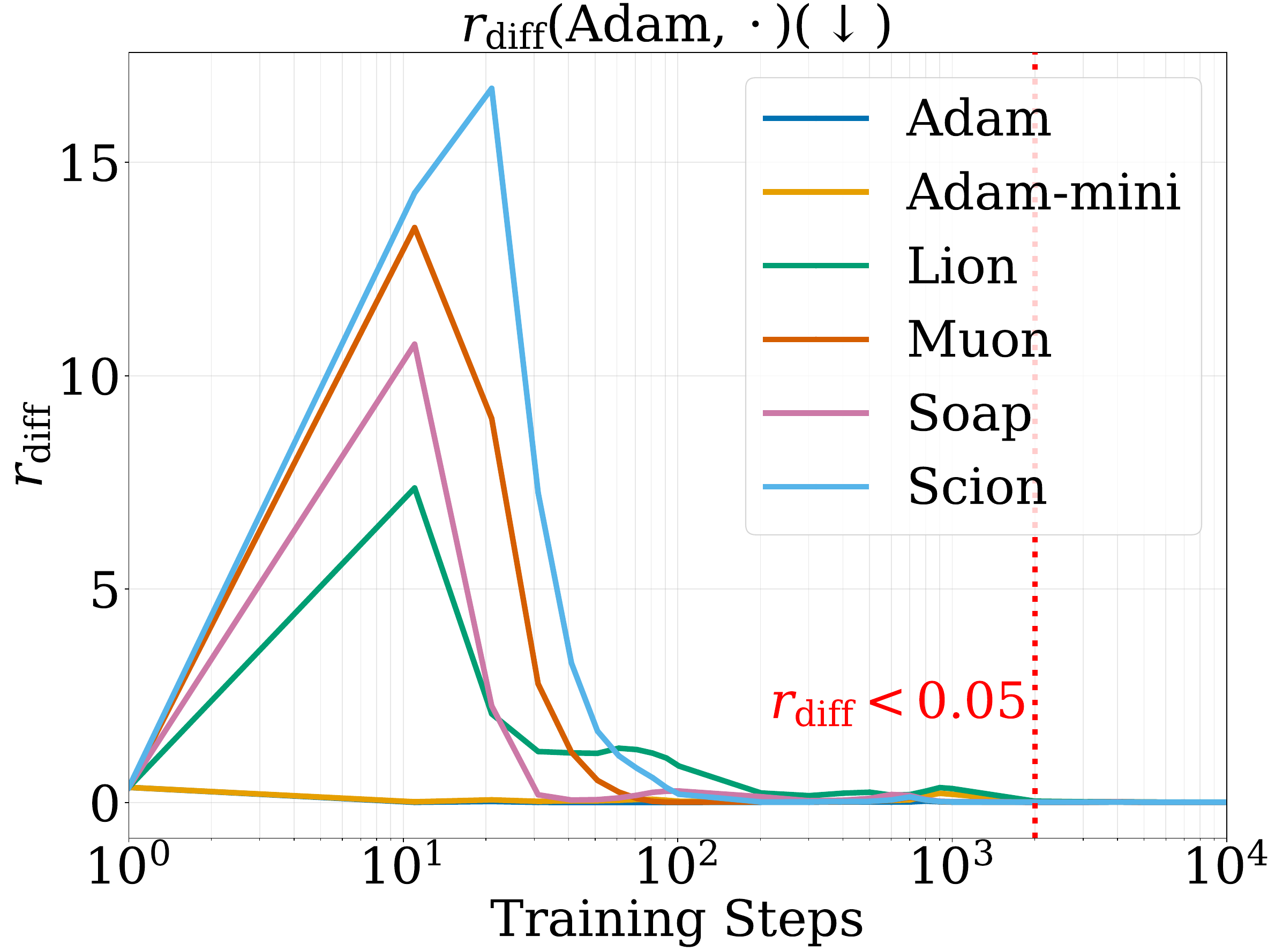}
    \label{fig:r_diff_srelu}
    }
    % % \vspace{-1.0em}
    \caption{\ac{rgi} across optimizers and models with different attention and \ac{ffn} designs. Panels (a)--(c) plot $\differr$ between the RoPE \ac{fa} baseline trained with $\adam$ and NoPE \ac{fa}, ALiBi \ac{fa}, and RoPE \ac{swa} models, respectively, each trained with all optimizers. Panel (d) plots $\differr$ between models using GELU and SReLU activations. \ac{rgi} persists across different attention and \ac{ffn} designs.} 
    % \vspace{-1.0em}
    \label{fig:rdiff_diff_arch}
\end{figure}

\begin{figure}[H]
    \centering
    \subfigure[Values of $\differr$ between models with GELU and SiLU.]{ \includegraphics[width=0.23\textwidth]{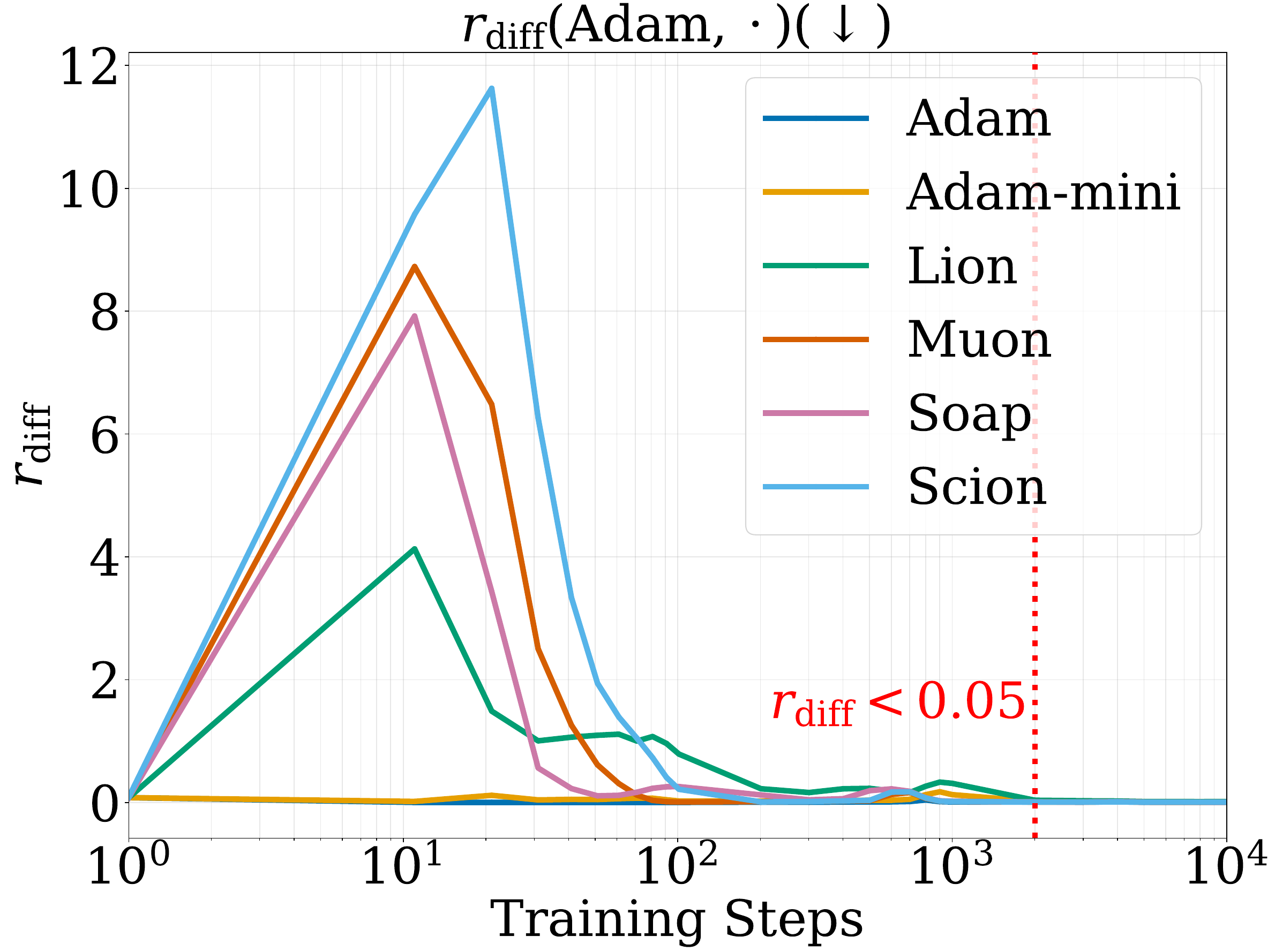}
    }
    \subfigure[Values of $\cccerr$ between models with GELU and SiLU.]{ \includegraphics[width=0.23\textwidth]{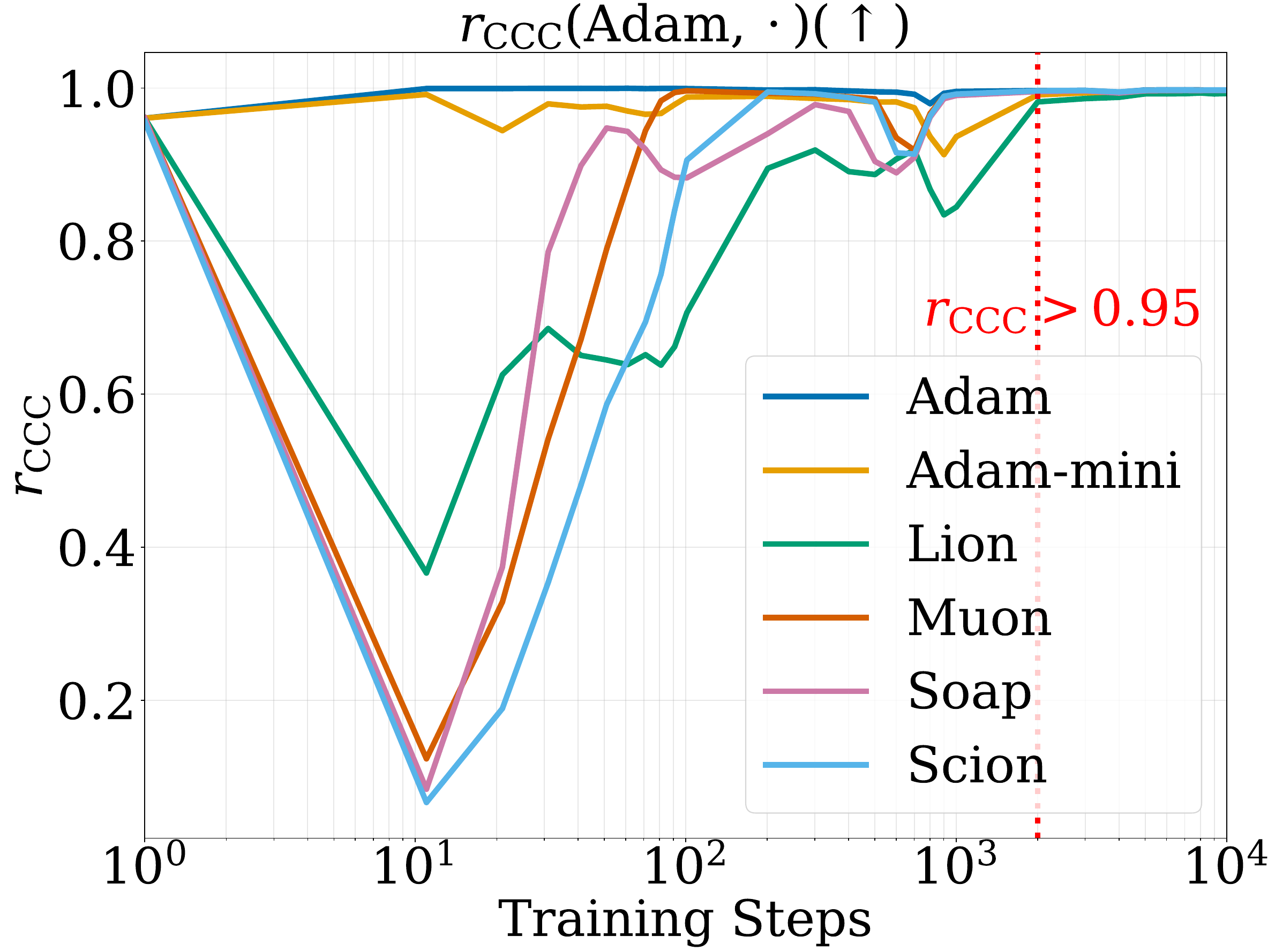}
    }
    \subfigure[Values of $\differr$ between models with gated and non-gated \ac{ffn}.]{ \includegraphics[width=0.23\textwidth]{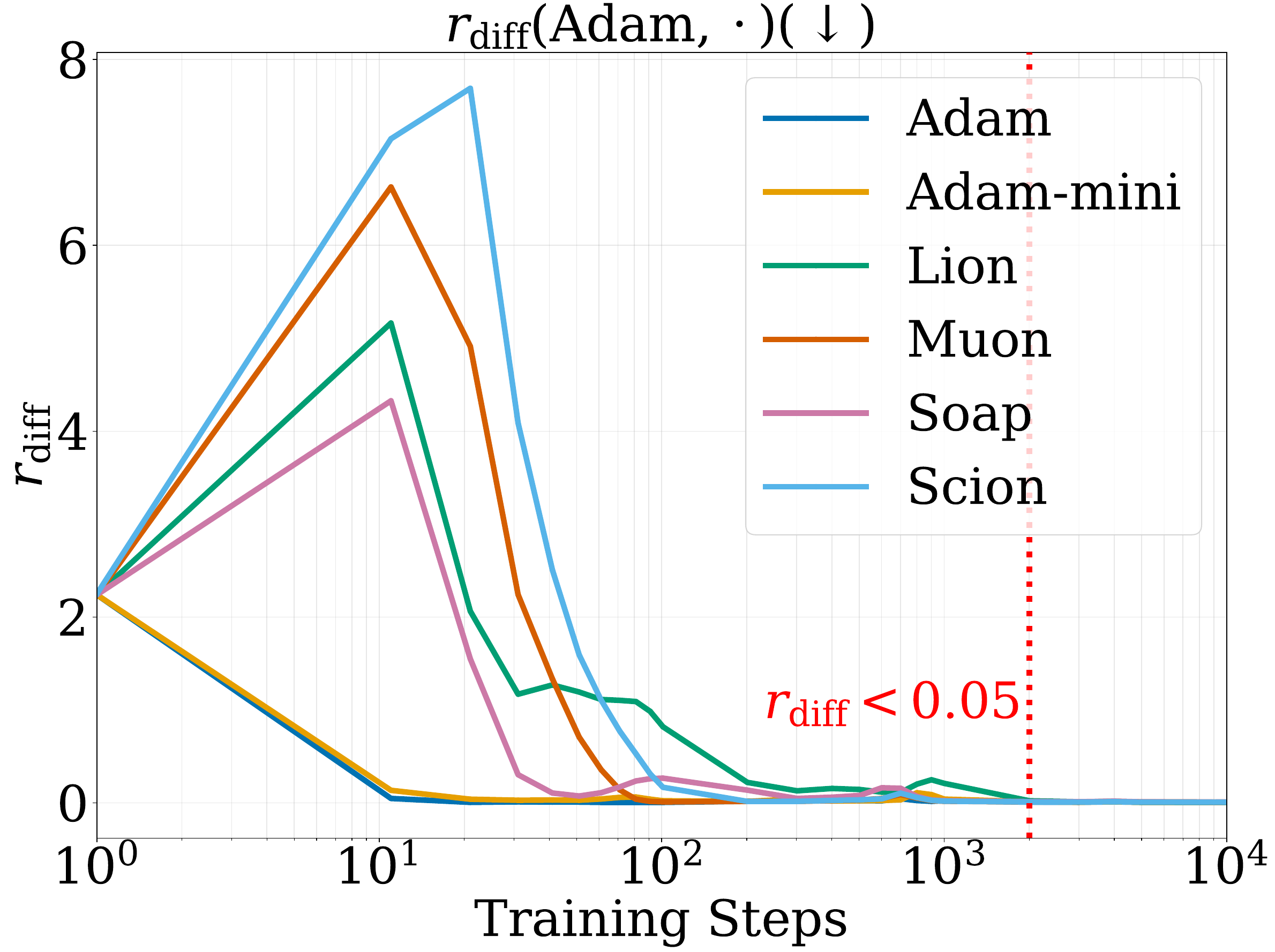}
    }
    \subfigure[Values of $\cccerr$ between models with gated and non-gated \ac{ffn}.]{ \includegraphics[width=0.23\textwidth]{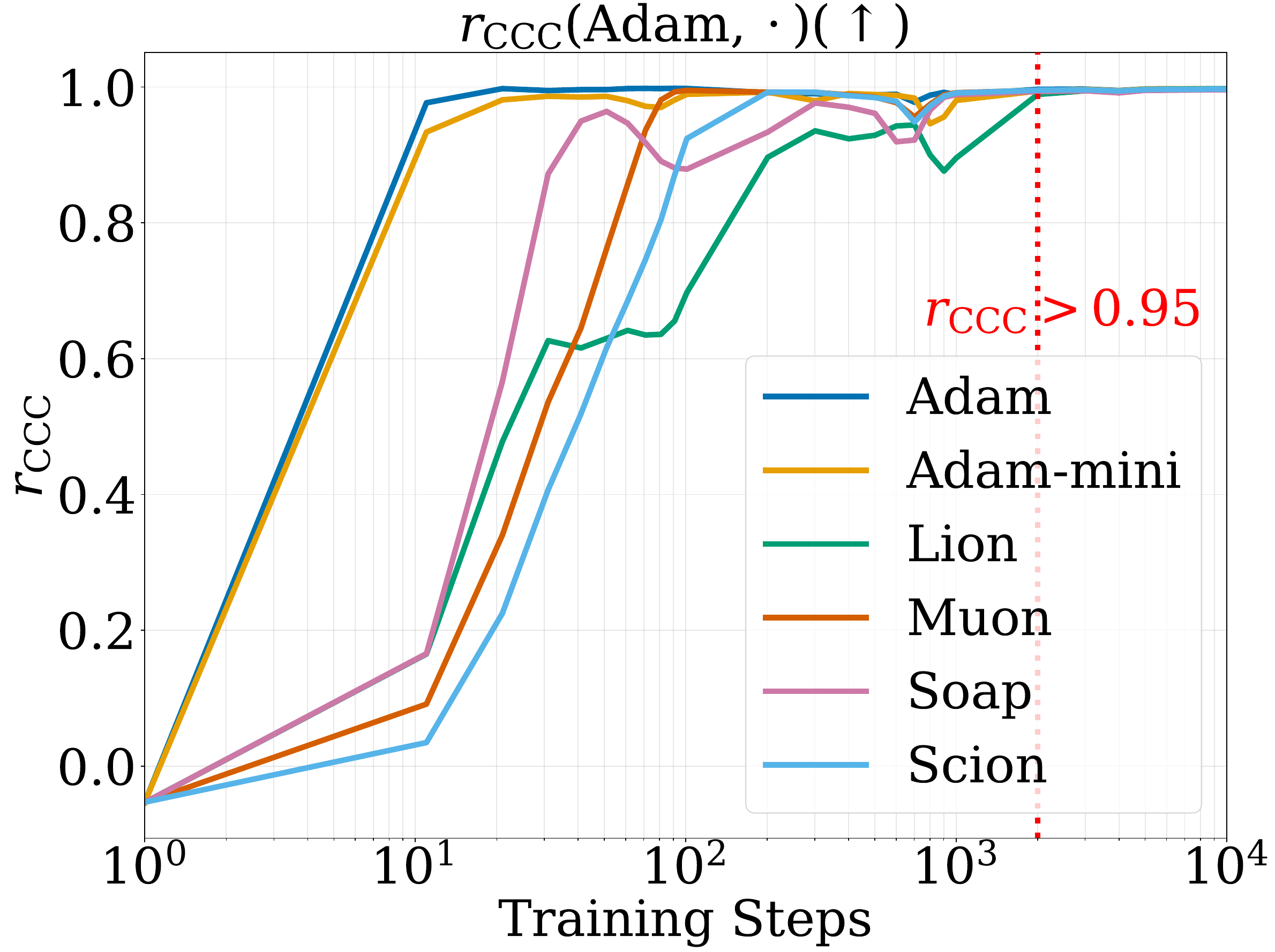}
    }
    \caption{\ac{rgi} across optimizers and models with different \ac{ffn} designs. Panels (a) and (b) plot $\differr$ and $\cccerr$, respectively, between the GELU baseline trained with Adam and SiLU models trained with other optimizers. Panels (c) and (d) plot $\differr$ and $\cccerr$, respectively, between the non-gated \ac{ffn} baseline trained with Adam and gated \ac{ffn} models trained with other optimizers. \ac{rgi} persists across different \ac{ffn} designs.} 
    % % \vspace{-1.0em}
    \label{fig:silu_gate}
\end{figure}

\subsubsection{Additional Results for Section~\ref{sec:train_data}}

\begin{figure}[H]
    \centering
    % % \vspace{-1.0em}
    \subfigure[Values of $\differr$ with $D_l^{\prime}$ from C4.]{ \includegraphics[width=0.31\textwidth]{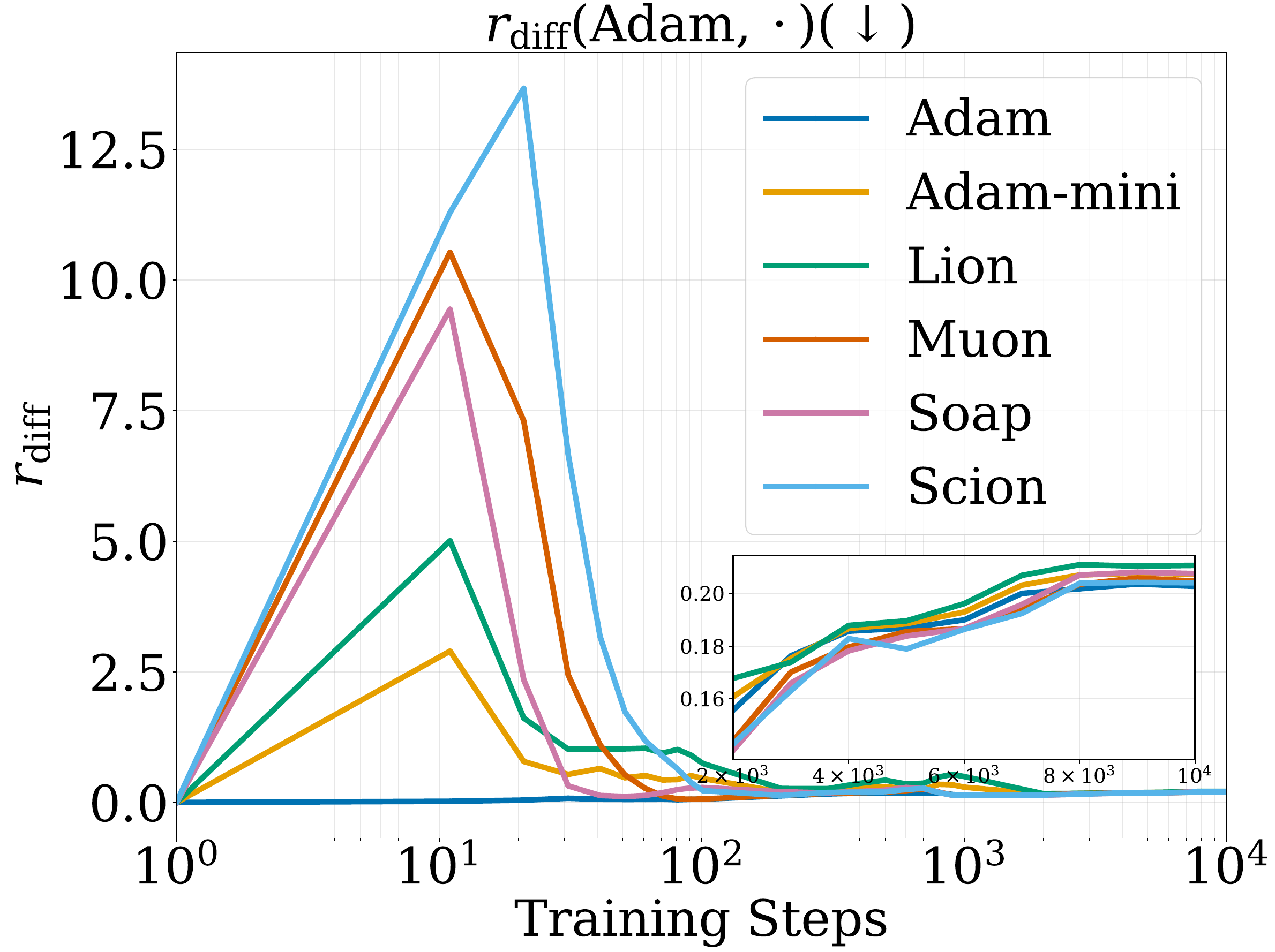}\label{fig:r_diff_c4}}
    % \hspace{2em}
    \subfigure[Values of $\differr$ with $D_l^{\prime}$ from ArXiv.]{ \includegraphics[width=0.31\textwidth]{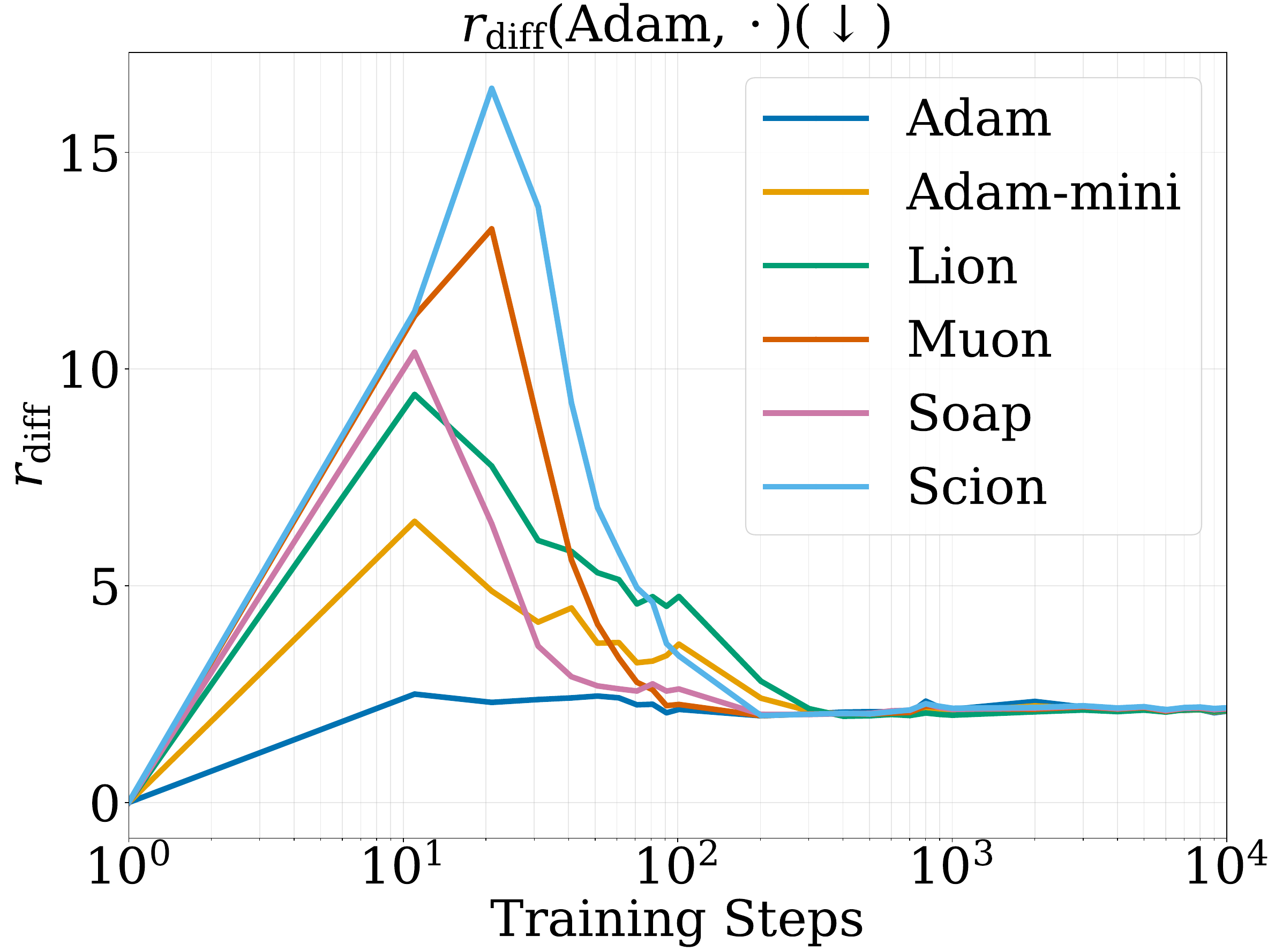}\label{fig:r_diff_arxiv}}
    % \hspace{2em}
    \subfigure[Values of $\differr$ with $D_l^{\prime}$ from CodeParrot.]{ \includegraphics[width=0.31\textwidth]{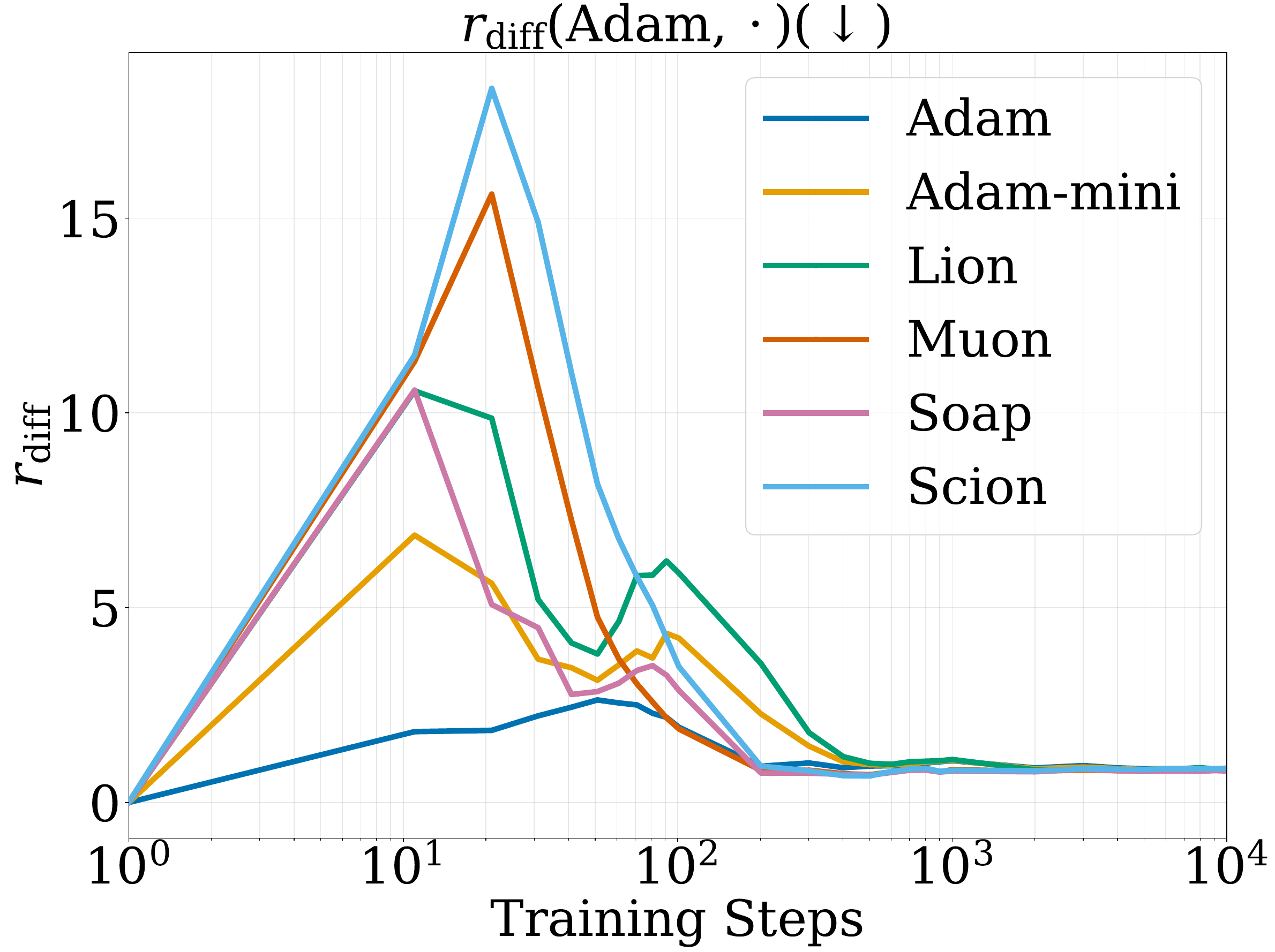}\label{fig:r_diff_code}}
    % % \vspace{-1.0em}
    \caption{\ac{rgi} across optimizers trained on different data streams. The panels plot $\differr$ when Adam is trained on FineWeb and the compared optimizer is trained on C4 (Panel (a)), ArXiv (Panel (b)), or CodeParrot (Panel (c)). The results show that \ac{rgi} is strongly affected by the training data stream: changing the data stream can substantially weaken, or even eliminate, \ac{rgi} across optimizers.} 
    % % \vspace{-1.0em}
    % \label{fig:dataset}
\end{figure}

\subsubsection{Additional Results for Section~\ref{sec:val}}

\begin{figure}[H]
    \centering
    \subfigure[Values of $\differr$ with $D_{\val}$ from C4.]{ \includegraphics[width=0.23\textwidth]{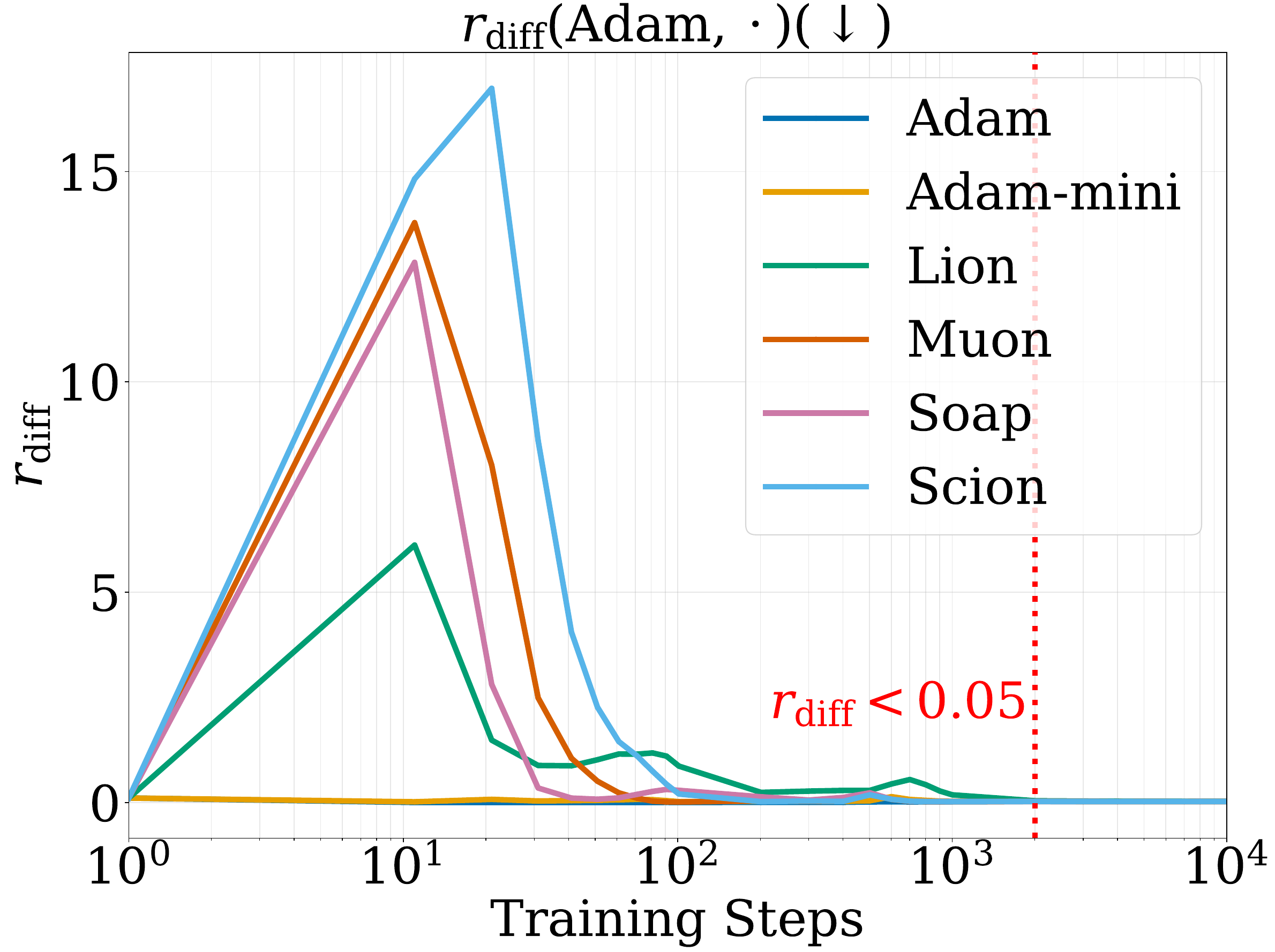}}
    \subfigure[Values of $\differr$ with $D_{\val}$ from ArXiv.]{ \includegraphics[width=0.23\textwidth]{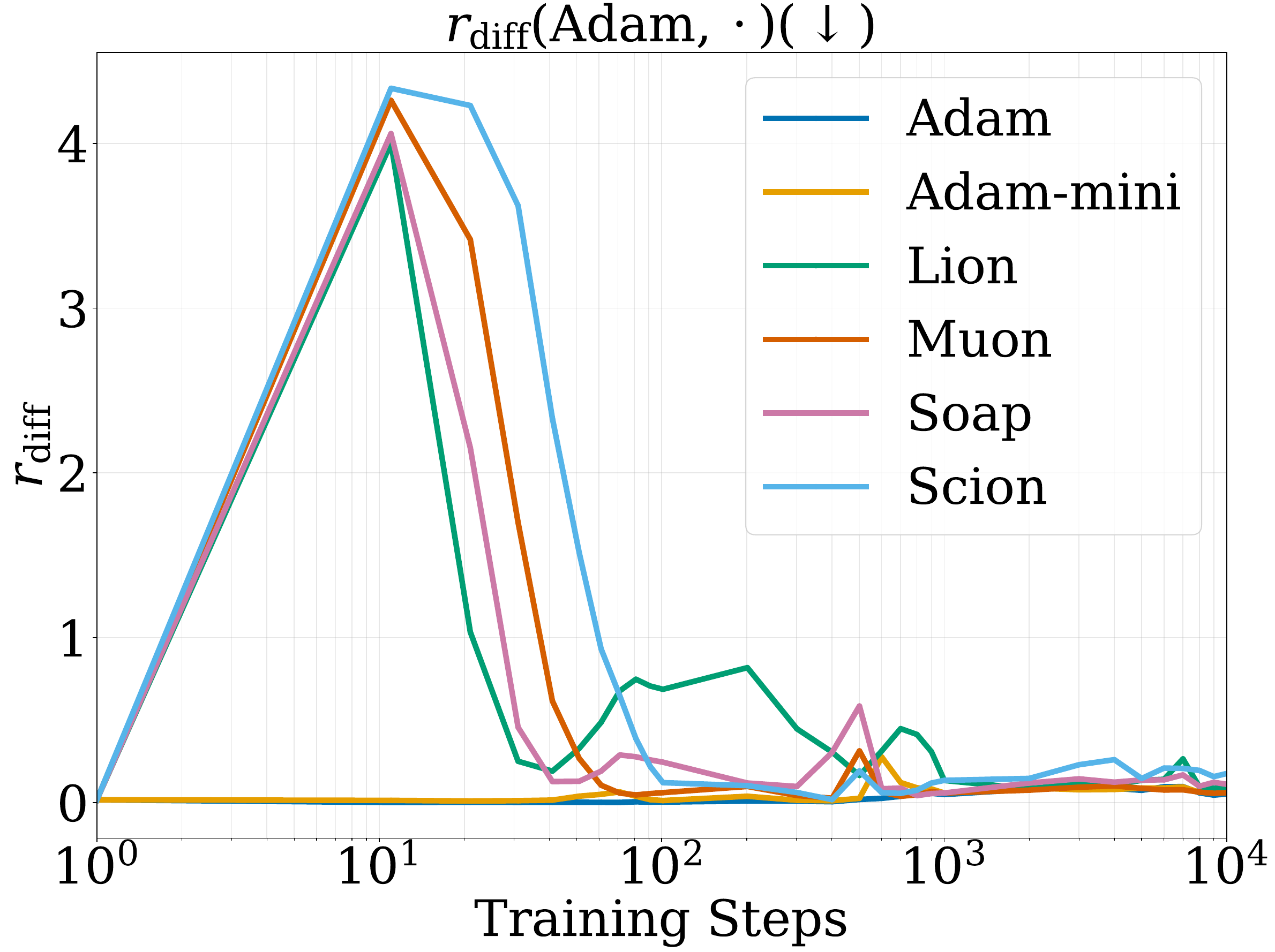}}
    \subfigure[Values of $\differr$ with $D_{\val}$ from CodeParrot.]{ \includegraphics[width=0.23\textwidth]{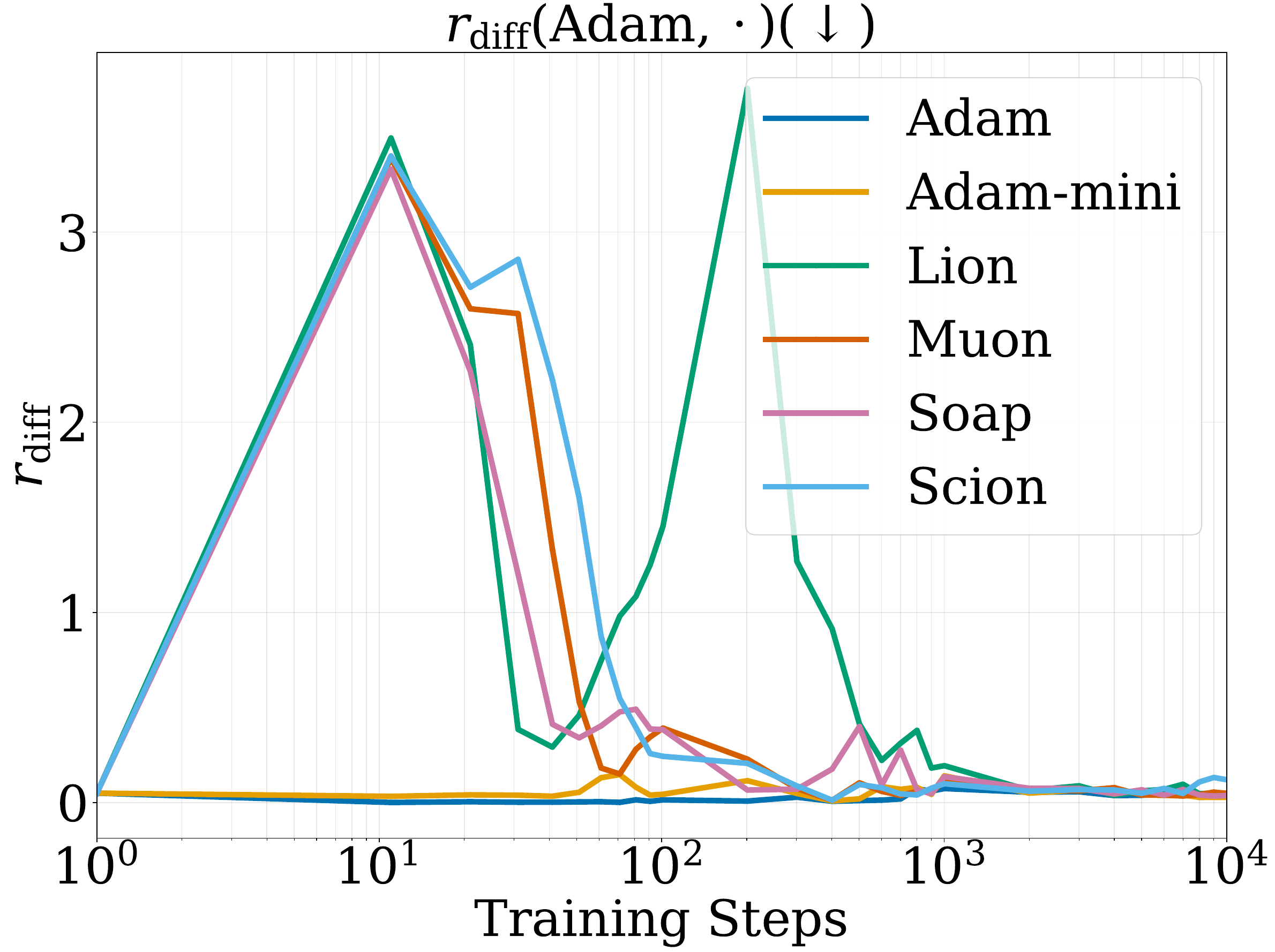}}
    \subfigure[Values of $\differr$ with $D_{\val}$ from Random.]{ \includegraphics[width=0.23\textwidth]{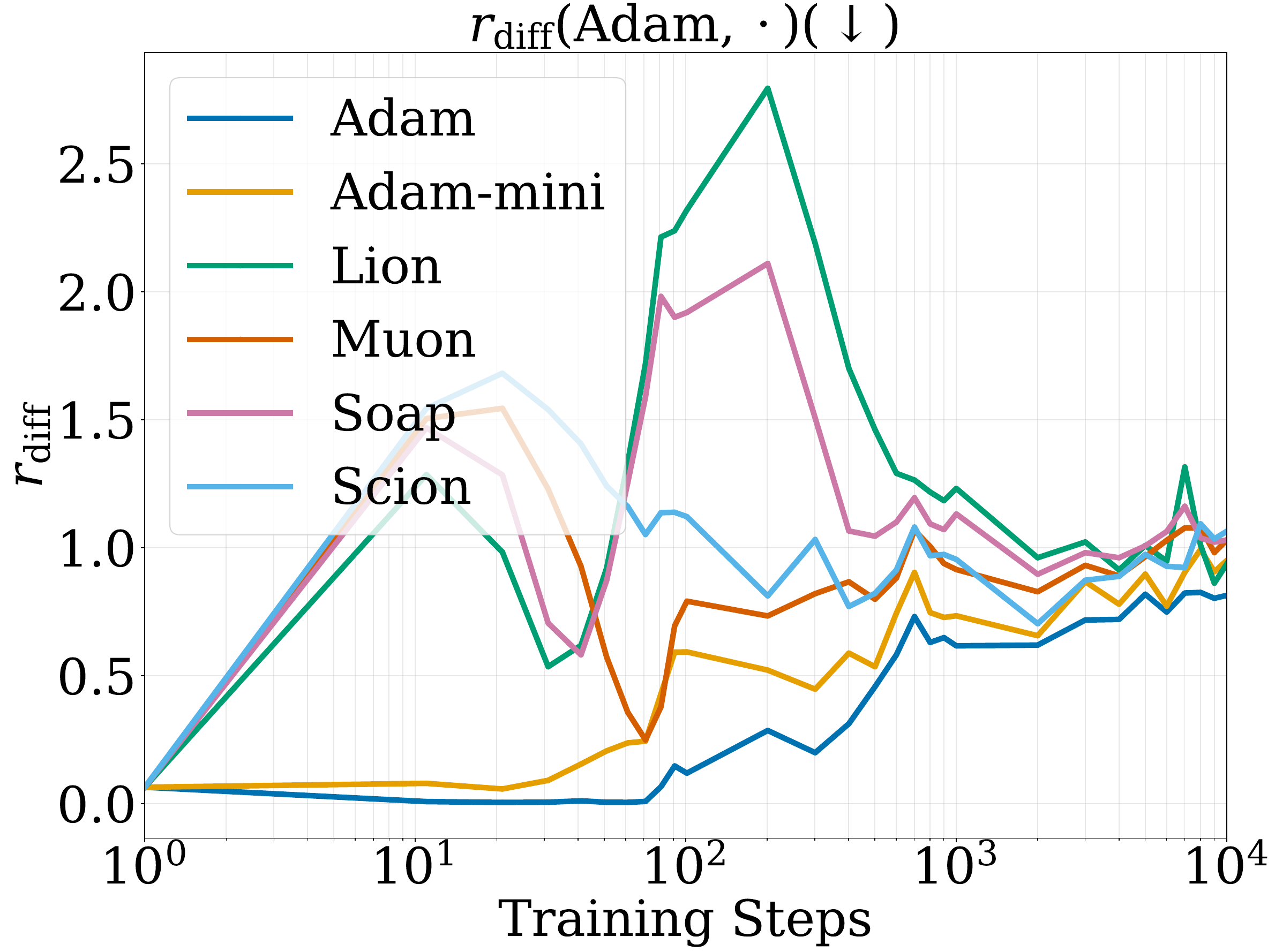}}
    \caption{ \ac{rgi} across optimizers and \ac{fa}/\ac{swa} models evaluated on different validation datasets. The panels plot $\differr$ evaluated on C4, ArXiv, CodeParrot, and Random. These results show that \ac{rgi} can transfer to related validation distributions, but becomes weaker on more distributionally distinct validation datasets.} 
    % % \vspace{-1.0em}
    % \label{fig:val}
\end{figure}

\begin{figure}[H]
    \centering
    \subfigure[Values of $\differr$ with $D_{\val}$ from C4.]{ \includegraphics[width=0.23\textwidth]{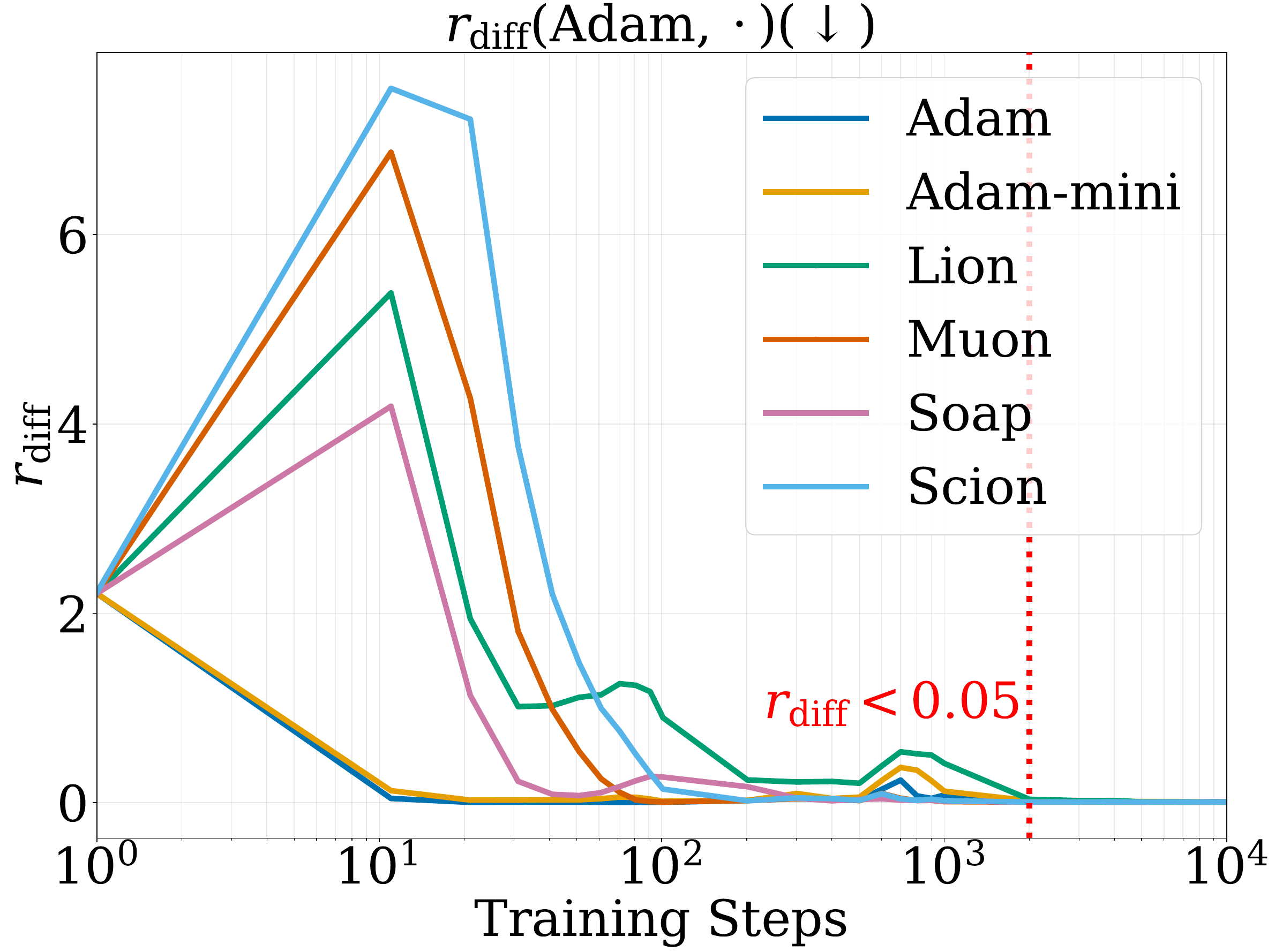}}
    \subfigure[Values of $\differr$ with $D_{\val}$ from ArXiv.]{ \includegraphics[width=0.23\textwidth]{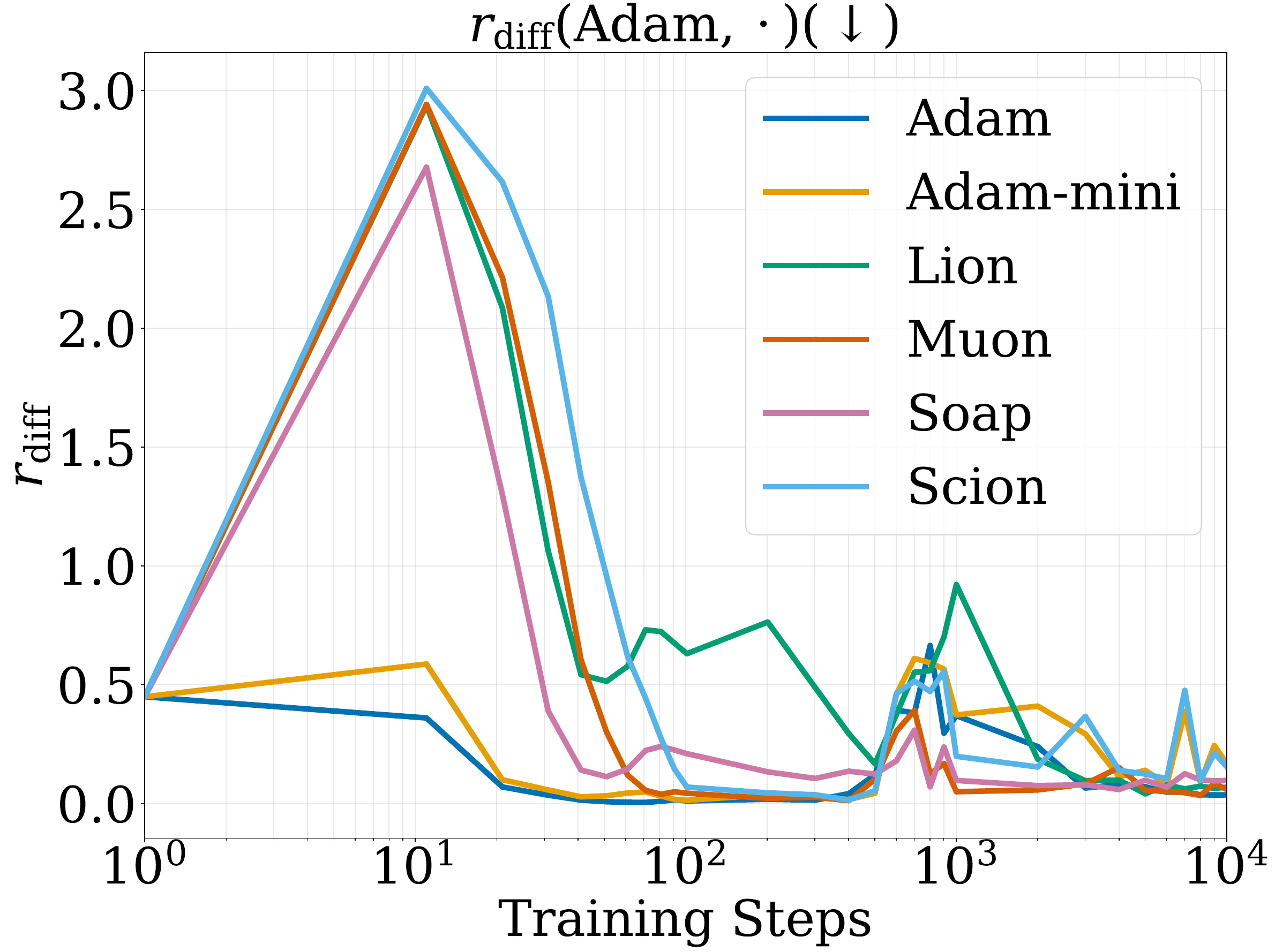}}
    \subfigure[Values of $\differr$ with $D_{\val}$ from CodeParrot.]{ \includegraphics[width=0.23\textwidth]{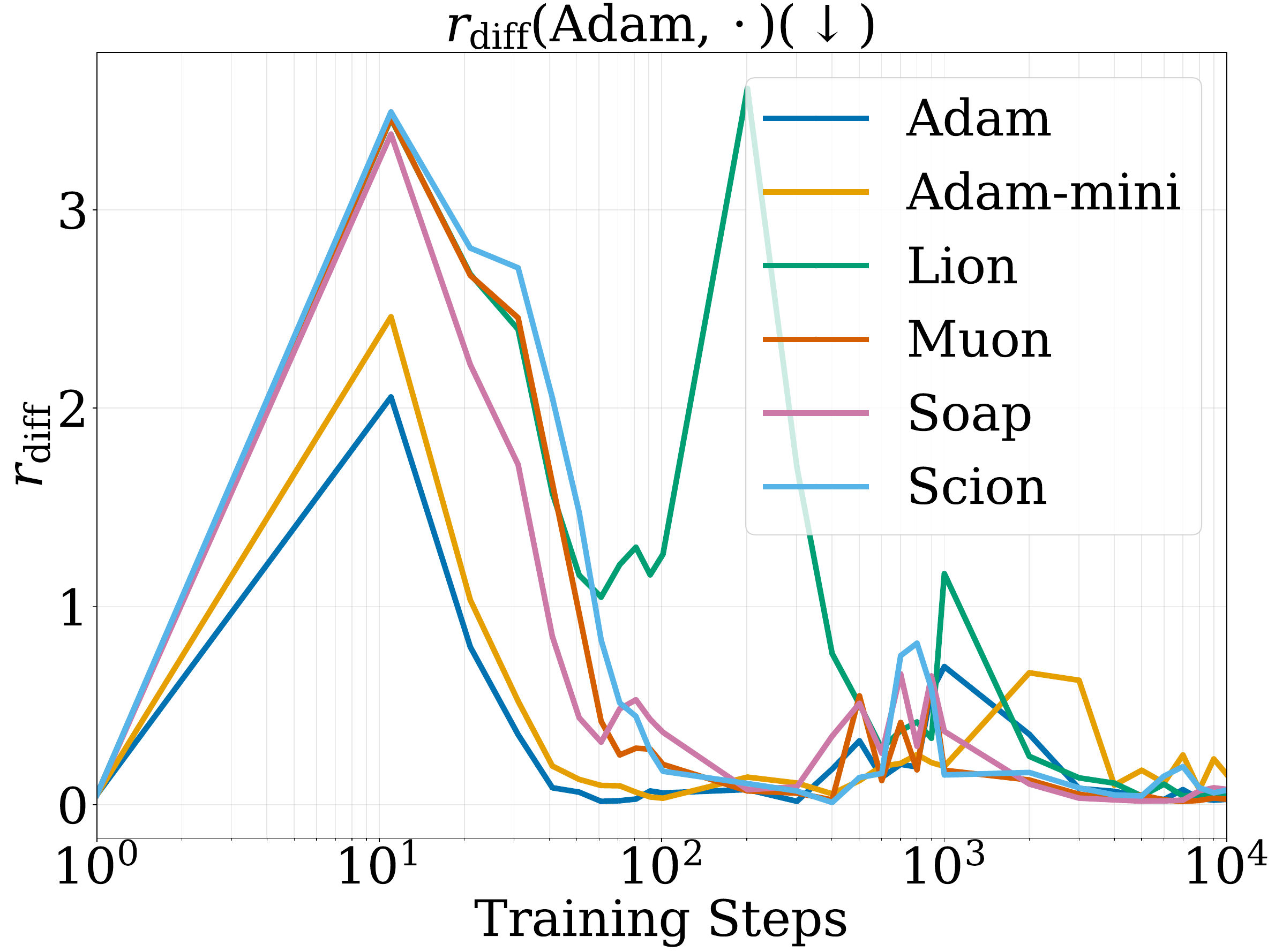}}
    \subfigure[Values of $\differr$ with $D_{\val}$ from Random.]{ \includegraphics[width=0.23\textwidth]{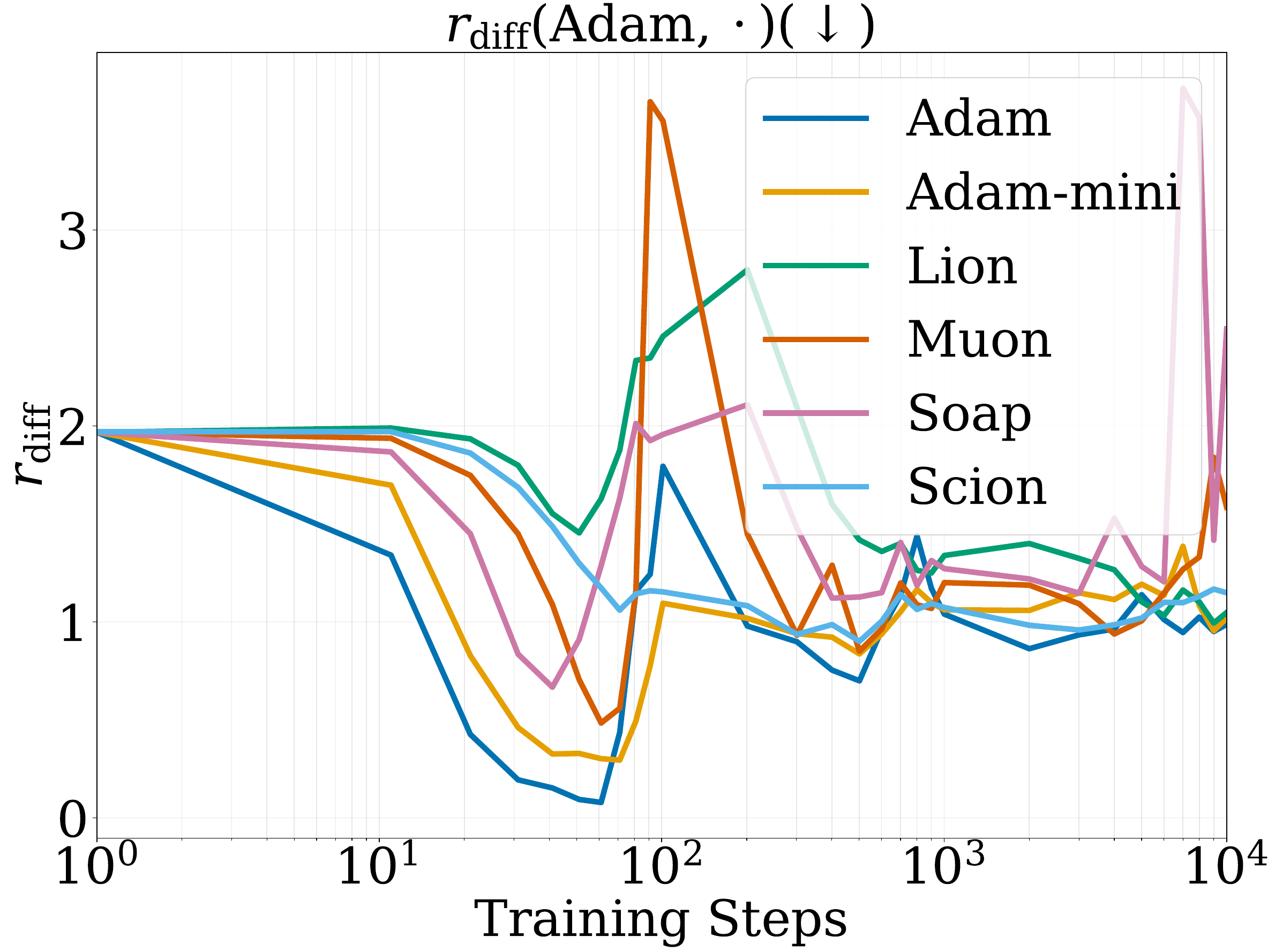}}

    \subfigure[Values of $\cccerr$ with $D_{\val}$ from C4.]{ \includegraphics[width=0.23\textwidth]{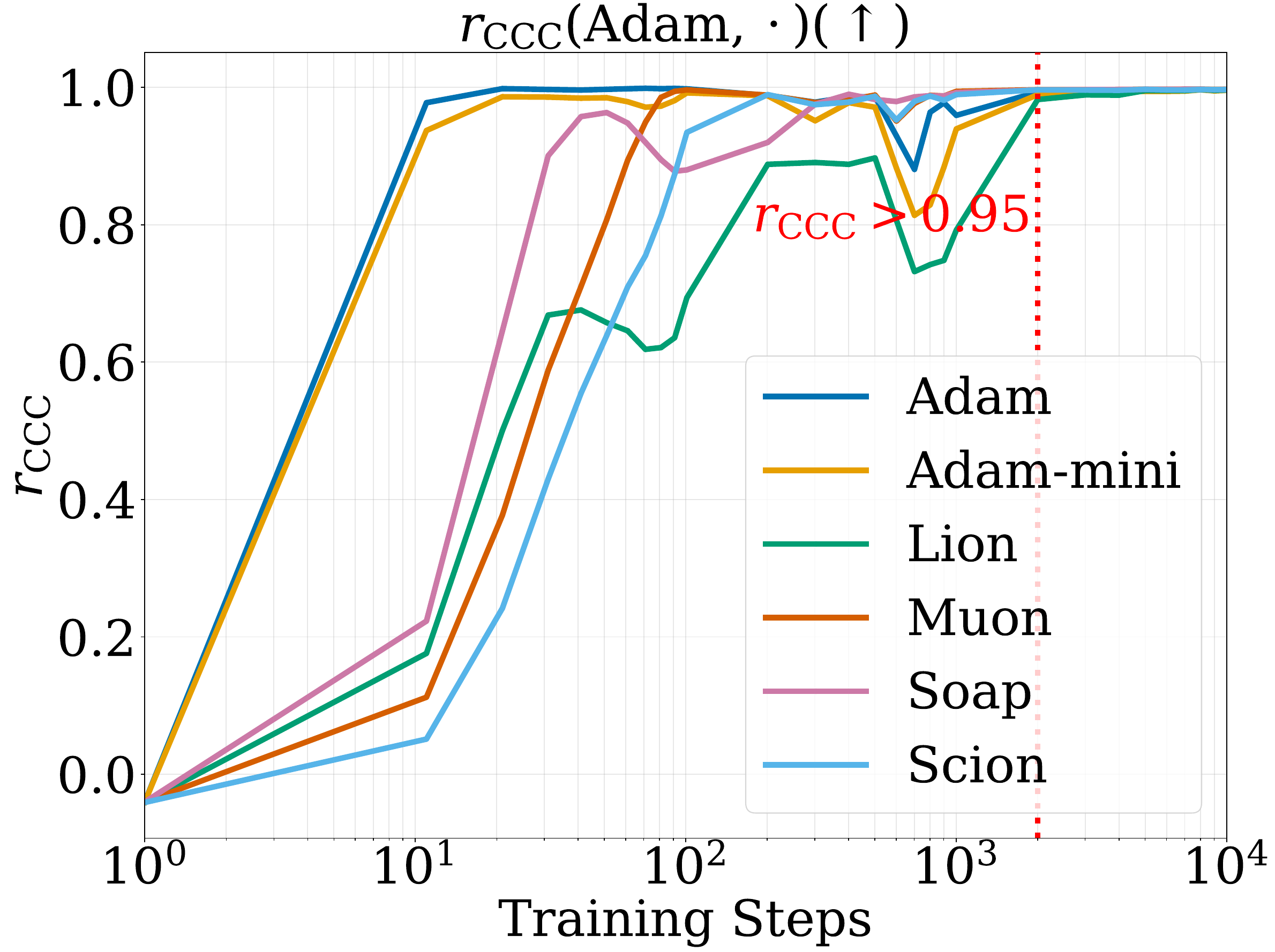}}
    \subfigure[Values of $\cccerr$ with $D_{\val}$ from ArXiv.]{ \includegraphics[width=0.23\textwidth]{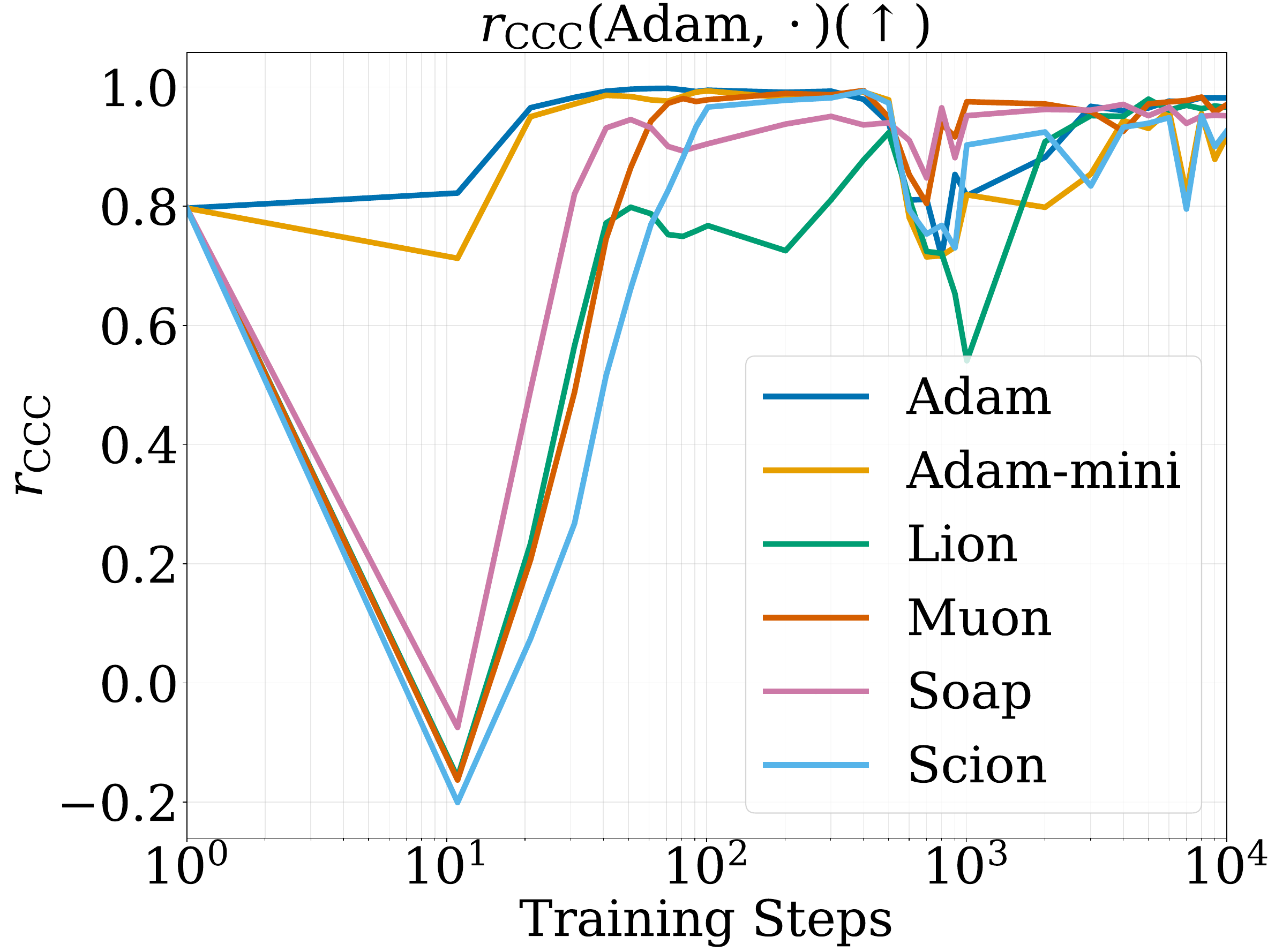}}
    \subfigure[Values of $\cccerr$ with $D_{\val}$ from CodeParrot.]{ \includegraphics[width=0.23\textwidth]{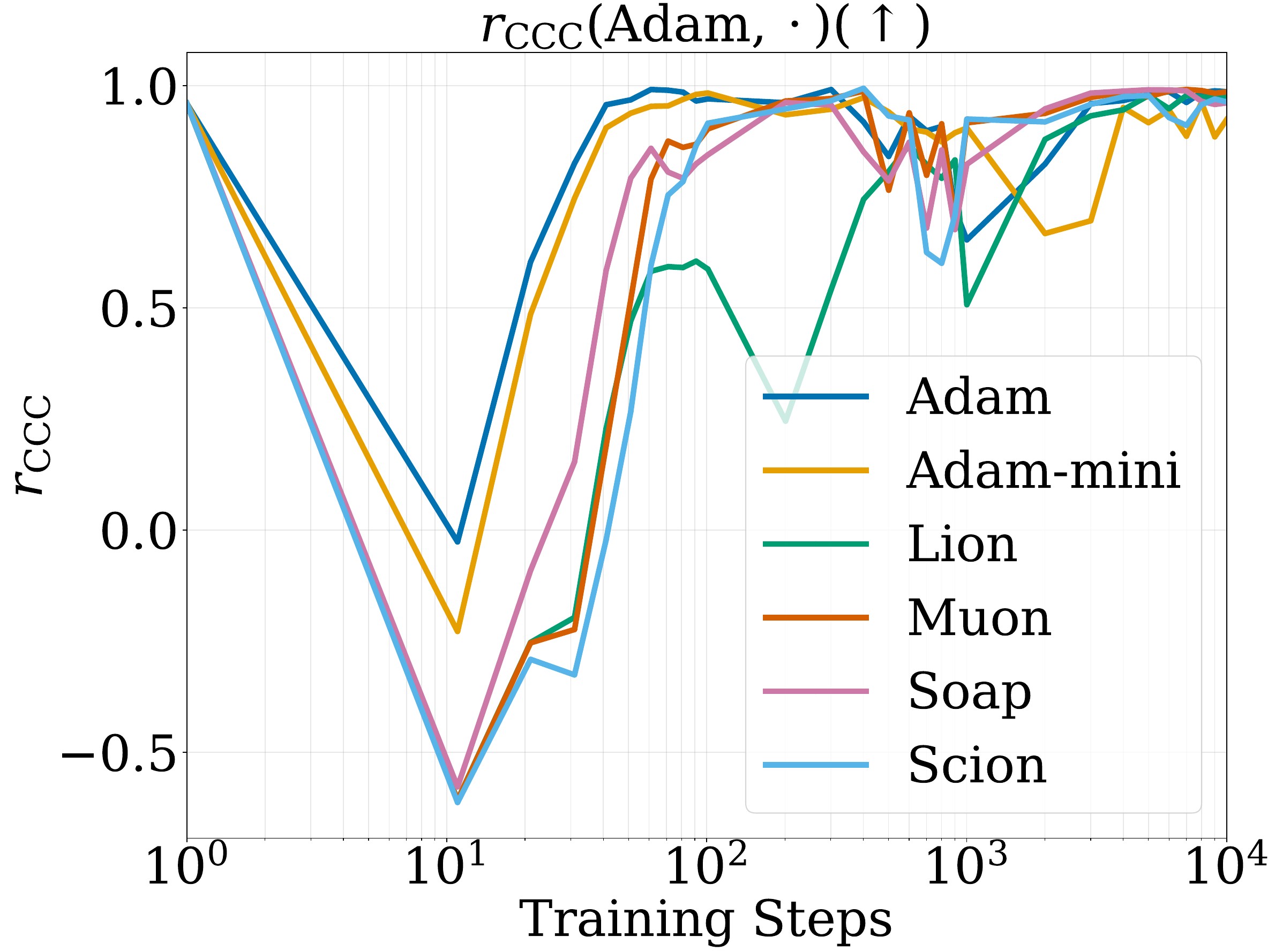}}
    \subfigure[Values of $\cccerr$ with $D_{\val}$ from Random.]{ \includegraphics[width=0.23\textwidth]{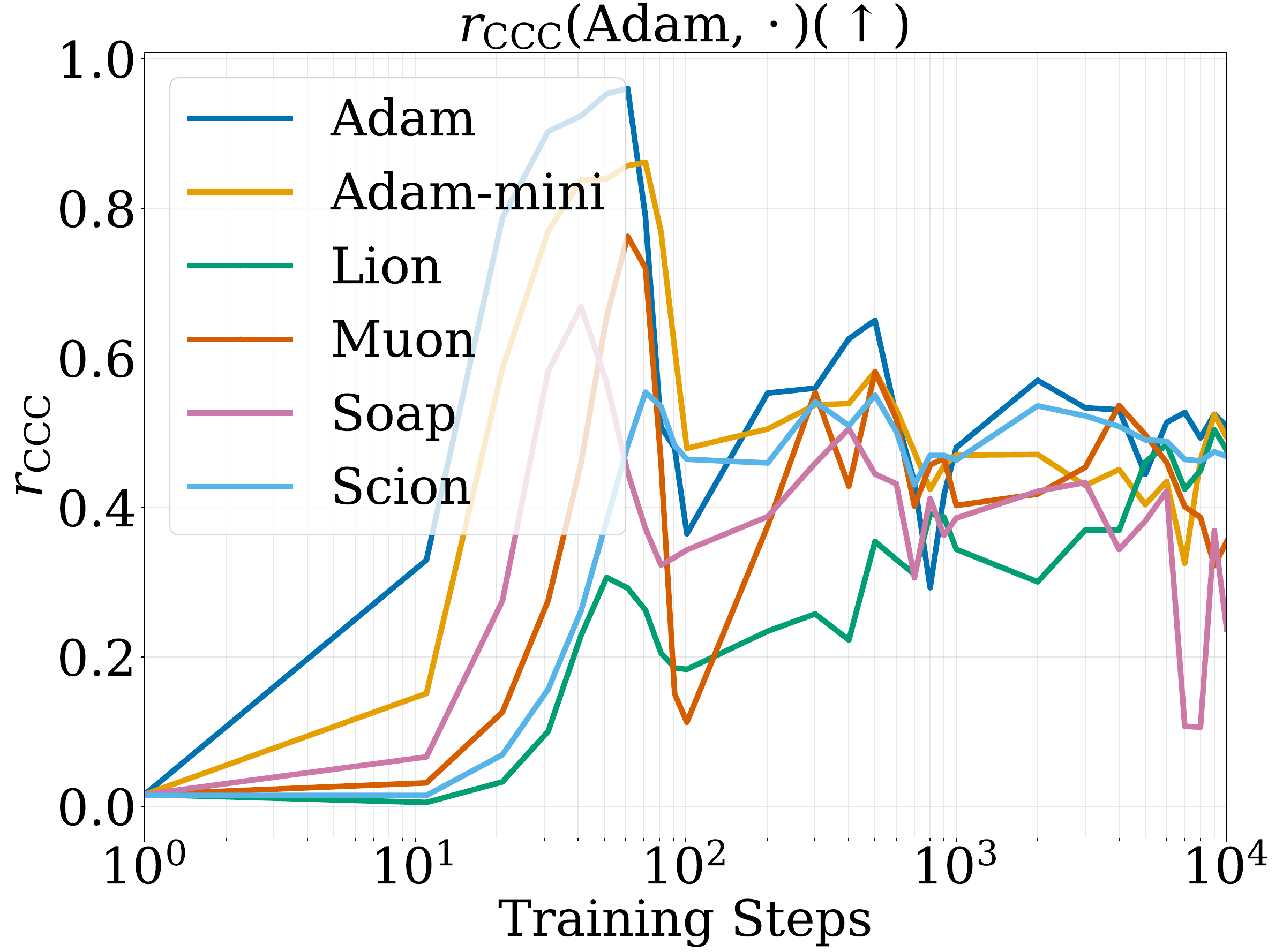}}
    \caption{ \ac{rgi} across optimizers and gated/non-gated models evaluated on different validation datasets. The panels plot $\differr$ and $\cccerr$ evaluated on C4, ArXiv, CodeParrot, and Random. These results show that \ac{rgi} can transfer to related validation distributions, but becomes weaker on more distributionally distinct validation datasets.} 
    % % \vspace{-1.0em}
    % \label{fig:val}
\end{figure}

% \begin{figure}[H]
%     \centering
%     \subfigure[Values of $\cccerr$ with $D_{\val}$ from C4.]{ \includegraphics[width=0.23\textwidth]{figures/CCC_centered_c4_Adam_ood.pdf}}
%     \subfigure[Values of $\cccerr$ with $D_{\val}$ from ArXiv.]{ \includegraphics[width=0.23\textwidth]{figures/CCC_centered_arxiv_Adam_ood.pdf}}
%     \subfigure[Values of $\cccerr$ with $D_{\val}$ from CodeParrot.]{ \includegraphics[width=0.23\textwidth]{figures/CCC_centered_codeparrot_Adam_ood.pdf}}
%     \subfigure[Values of $\cccerr$ with $D_{\val}$ from Random.]{ \includegraphics[width=0.23\textwidth]{figures/CCC_centered_random_Adam_ood.pdf}}
%     \caption{validation dataset ablation for $\cccerr$. Panels (a)--(d) plot $\cccerr$ evaluated on C4, ArXiv, CodeParrot, and Random, respectively. \ac{rgi} transfers best to the related C4 validation distribution and weakens on more distributionally distinct validation sets.} 
% \end{figure}

\subsection{Additional Results for Section~\ref{sec:relation}}

\subsubsection{Additional Results for Section~\ref{sec:ntk}}

\begin{figure}[H]
    \centering
    \subfigure[Values of $\relerr$ between parameters.]{ \includegraphics[width=0.35\textwidth]{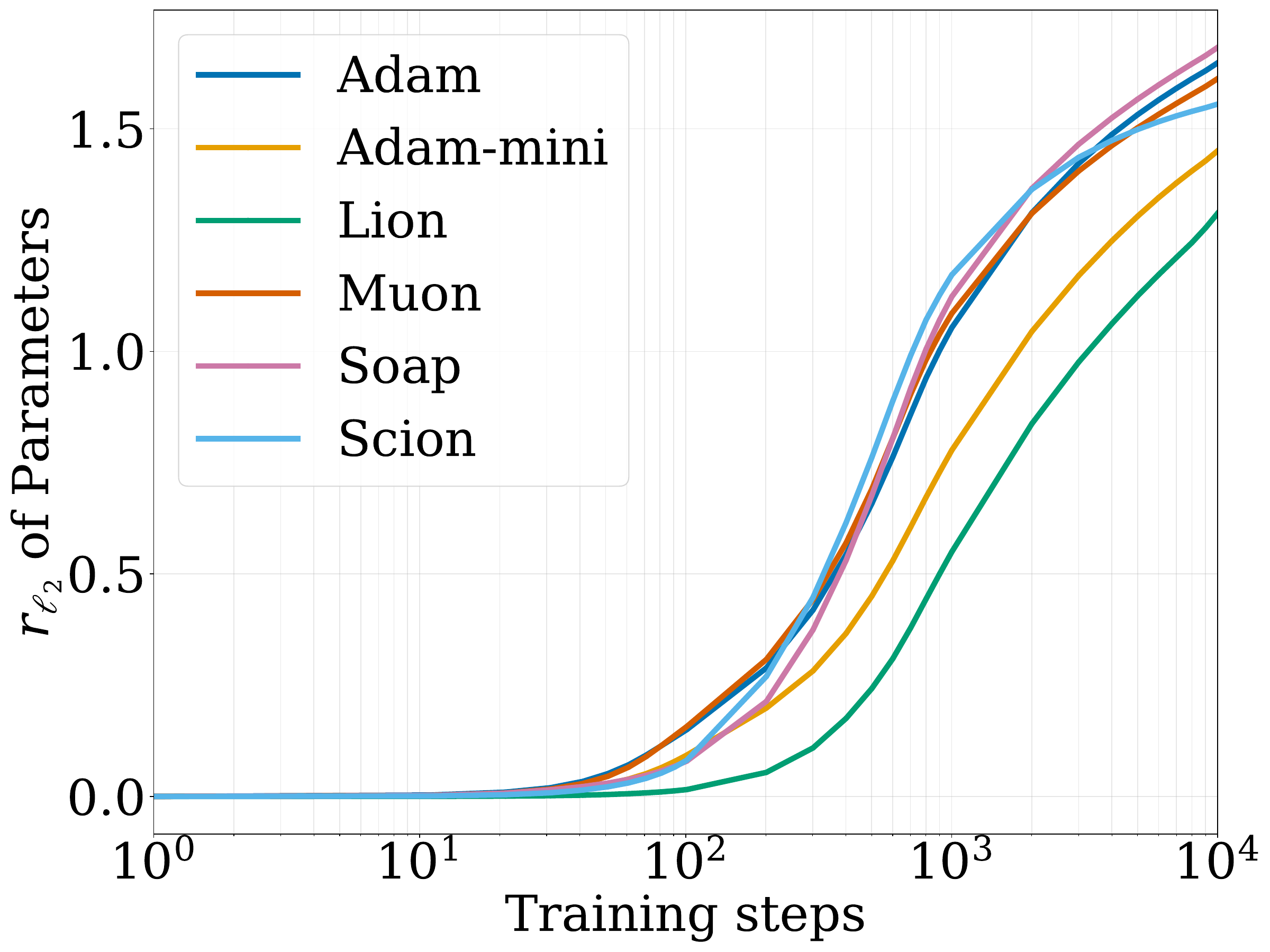}\label{fig:l2_self_p}}
    \subfigure[Values of $\relerr$ between hidden states.]{ \includegraphics[width=0.35\textwidth]{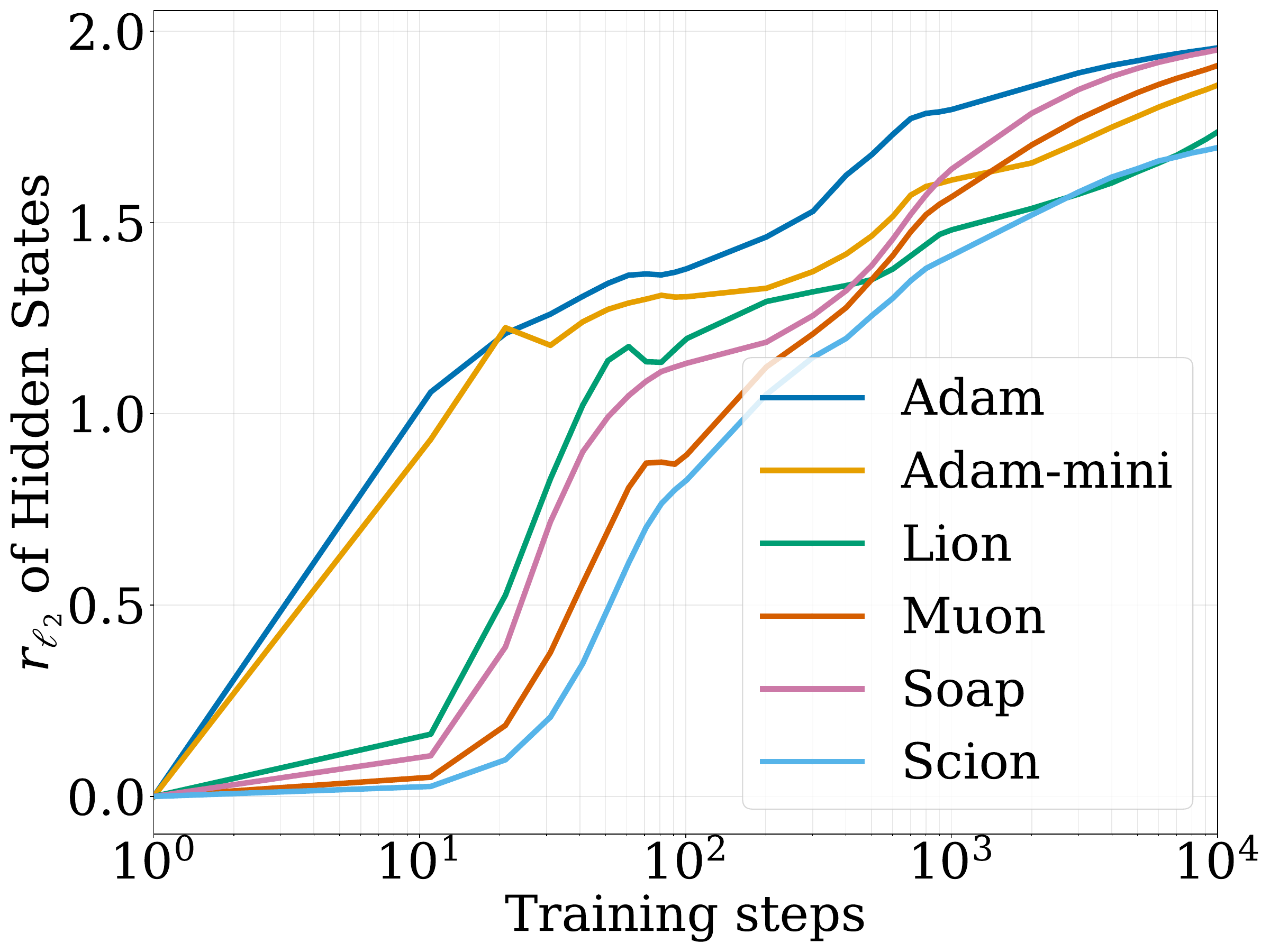}\label{fig:l2_self_h}}
    
    \caption{Movement from initialization in parameter and hidden-state space. Panels (a) and (b) plot $\relerr(\downarrow)$ between initialization and later checkpoints for parameters and hidden states, respectively. The large relative distances show substantial movement during training, providing evidence against an \ac{ntk}-style lazy-training explanation.} 
    \label{fig:self_ntk}
\end{figure}

\begin{figure}[H]
    \centering
    \subfigure[Values of $\relerr$ between parameters.]{ \includegraphics[width=0.35\textwidth]{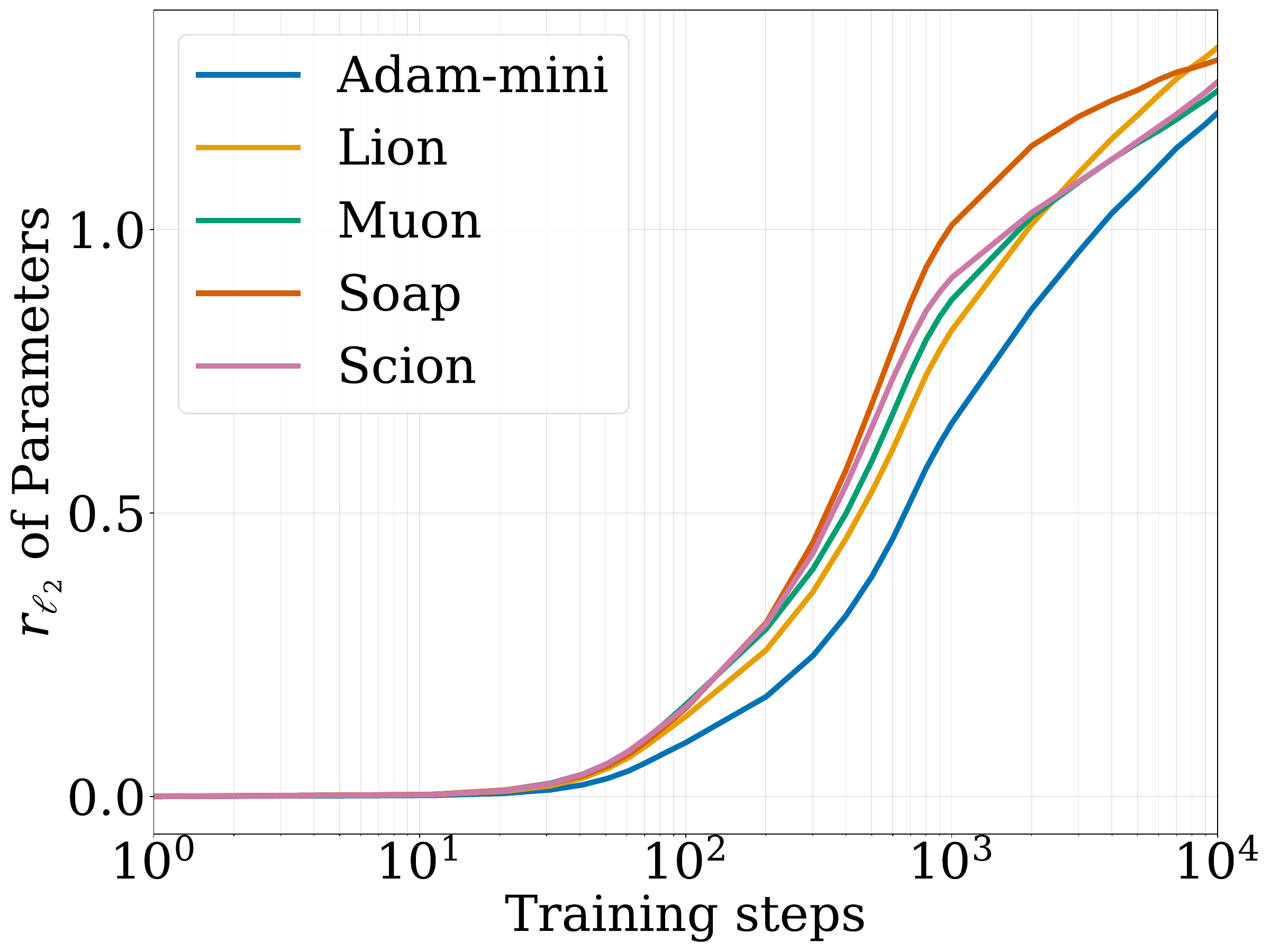}\label{fig:l2_cross_p}}
    \subfigure[Values of $\relerr$ between hidden states.]{ \includegraphics[width=0.35\textwidth]{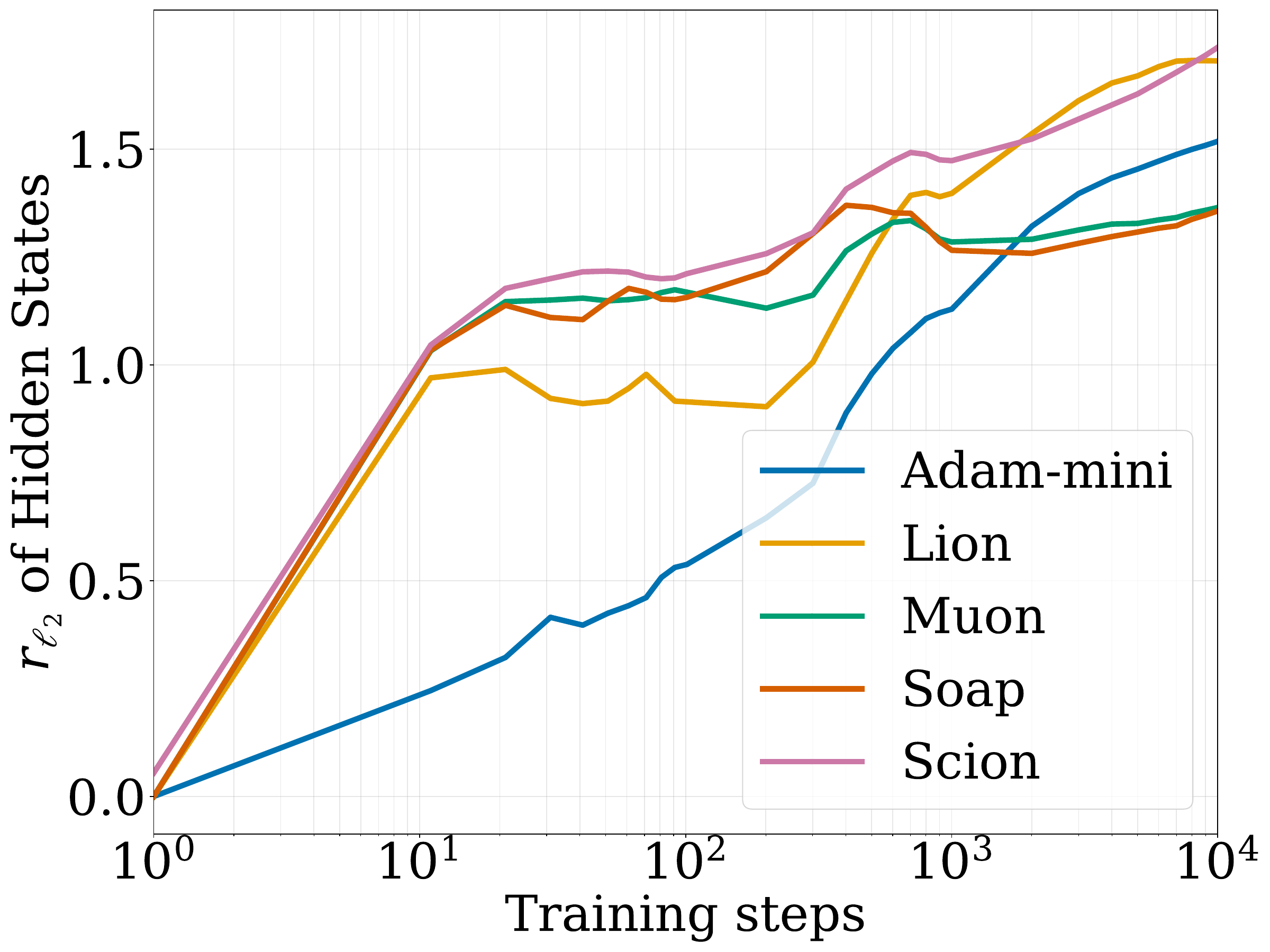}\label{fig:l2_cross_h}}
    \caption{Cross-optimizer distances in parameter and hidden-state space. Panels (a) and (b) plot $\relerr(\downarrow)$ between Adam and the other optimizers at matched training steps for parameters and hidden states, respectively. The large relative distances show that different optimizers need not share similar parameters or hidden states even when their validation-loss differences exhibit \ac{rgi}.} 
    \label{fig:cross_ntk}
\end{figure}

\subsubsection{Additional Results for Section~\ref{sec:mf}}
\begin{figure}[H]
    \centering
    \subfigure[Values of $\differr$ at width $576$.]{ \includegraphics[width=0.23\textwidth]{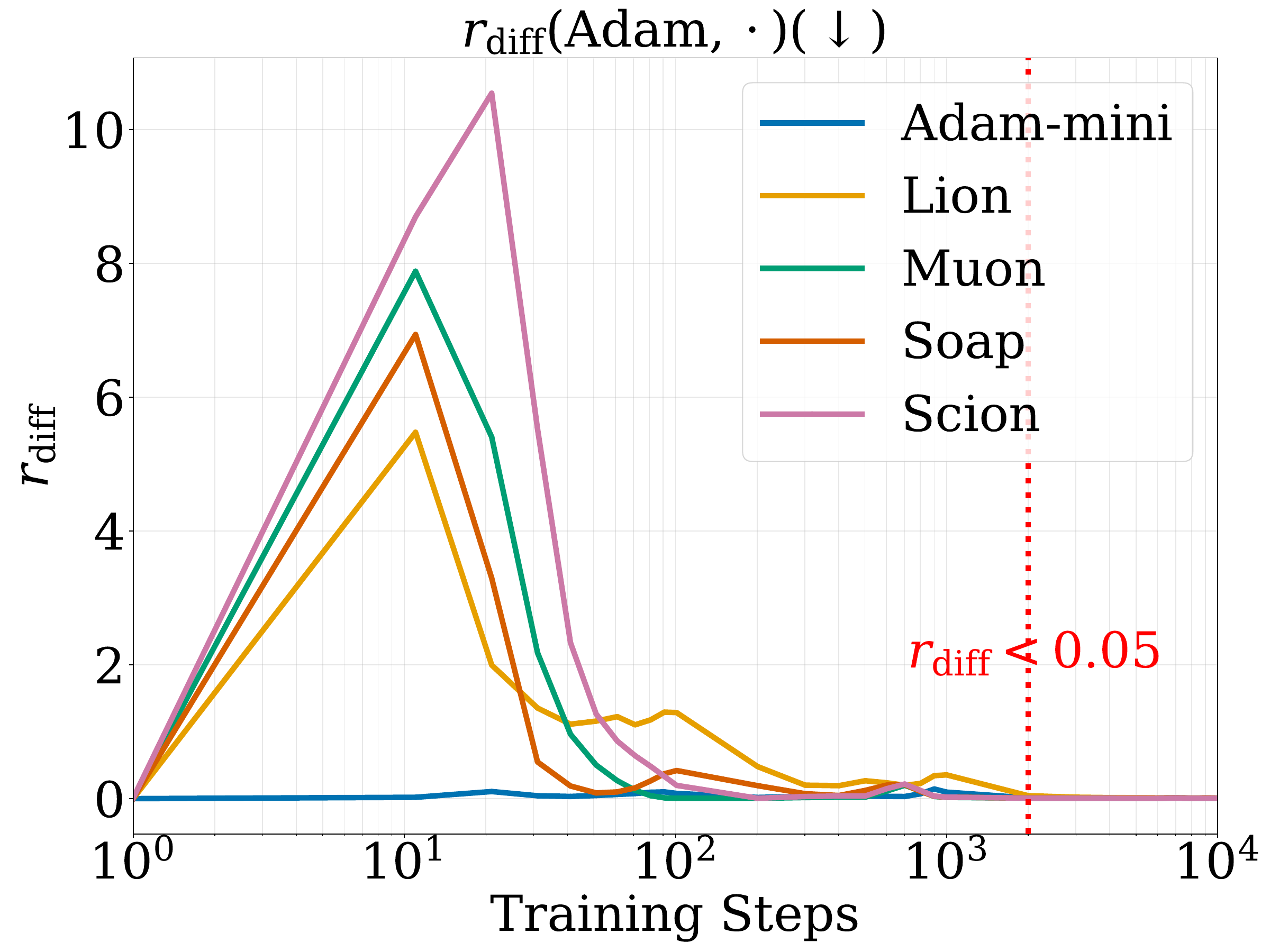}\label{fig:r_diff_576}}
    \subfigure[Values of $\differr$ at width $384$.]{ \includegraphics[width=0.23\textwidth]{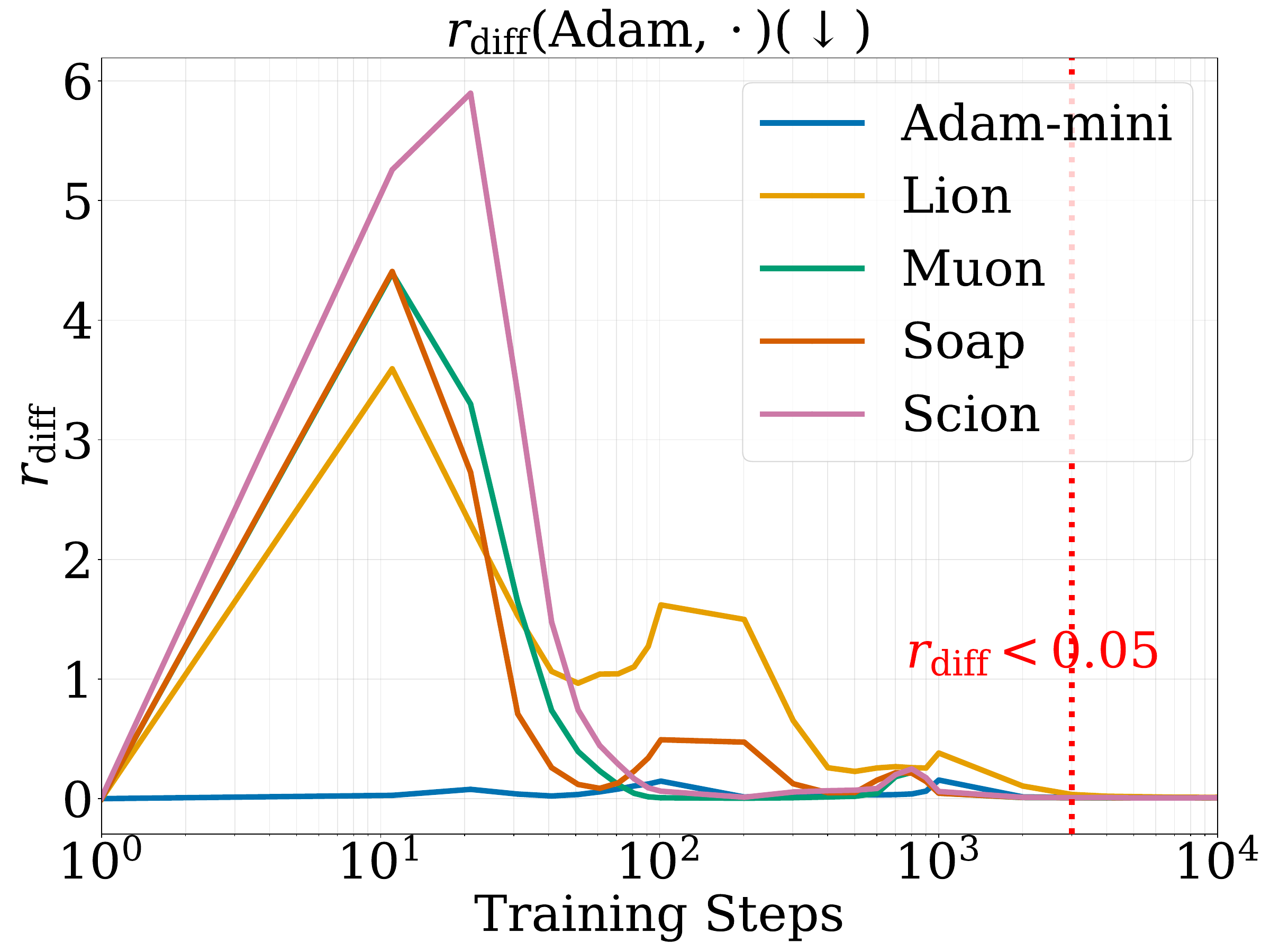}\label{fig:r_diff_384}}
    \subfigure[Values of $\differr$ at width $192$.]{ \includegraphics[width=0.23\textwidth]{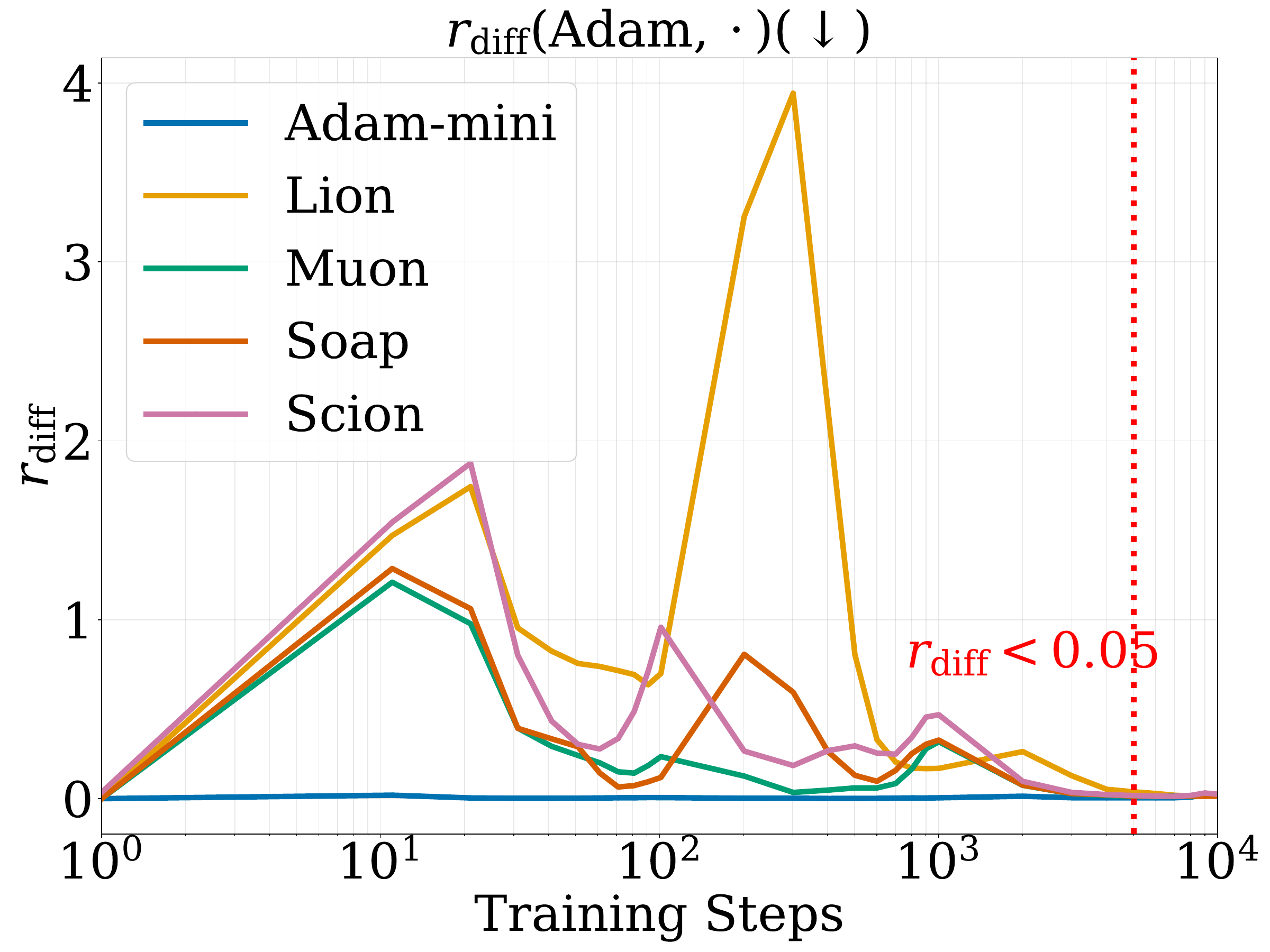}\label{fig:r_diff_192}}
    \caption{Model-width ablations for $\differr$. Panels (a)--(c) plot $\differr$ for model widths $576$, $384$, and $192$, respectively. \ac{rgi} holds across all tested widths, supporting the conclusion that the phenomenon is not restricted to the standard model width.} 
    \label{fig:mean_field_ccc}
\end{figure}

\section{Discussions of Alternative Explanations for \ac{rgi}}\label{app:alter}
In this section, we test and falsify several hypotheses that could potentially explain \ac{rgi}. Specifically, we consider two hypotheses:
\begin{itemize}[leftmargin=2em]
\item \textbf{Intrinsic data-difficulty hypothesis.} Suppose the loss decomposes as $L(\theta,d)=H(d)+\varepsilon(\theta,d)$, where $\varepsilon(\theta,d)$ is much smaller than $H(d)$. Then the relative loss structure is dominated by the intrinsic difficulty $H(d)$ of the data and depends only weakly on the learned parameters $\theta$.
\item \textbf{Different convergence-speed hypothesis.} \ac{rgi} emerges because different optimizers generate the same learning trajectory but with different speed.
\end{itemize}

We first consider the intrinsic data-difficulty hypothesis. For a validation example $d=(x_{\mathrm{prefix}},x_{\mathrm{pred}})$, our work focuses on the token-level loss as $L(\theta,d)=-\log P_{\theta}(x_{\mathrm{pred}}\mid x_{\mathrm{prefix}}).$ Importantly, RGI does not average over $x_{\mathrm{pred}}$; it holds pointwise for each $d\in\mathcal{D}_{\mathrm{val}}$. 

Consider any plausible loss decomposition $L(\theta,d)=H(d)+\varepsilon(\theta,d)$ such that $H(d),\varepsilon(\theta,d)\geq 0$ for every $d$ and $\theta$, where $H(d)$ is independent of the model. We have that
\begin{align*}
0\leq H(d)
=\inf_{\theta}\bigl[L(\theta,d)-\varepsilon(\theta,d)\bigr]
\leq \inf_{\theta}L(\theta,d).
\end{align*}
Therefore, the minimum achievable loss on $d$ provides an upper bound on any model-independent component $H(d)$. To estimate this bound, we directly optimize the LLM parameters $\theta$ on each data $d$ and define the achieved loss as $\hat{H}(d)=L(\hat{\theta}_d,d)$, where $\hat{\theta}_{d}$ denotes the optimized parameters. Since $\inf_{\theta}L(\theta,d)\leq \hat{H}(d)$, we obtain $H(d)\leq \hat{H}(d)$. We then report the average values of $\hat{H}(d)$ and $L(\theta_t,d)$, as well as their ratio, at $t=10000$, when $L(\theta_t,d)$ attains its minimum along the training trajectory. Table~\ref{tab:optimizer_metrics} shows that the model-independent component $H(d)$ is negligible, i.e., about $0.4\%$, relative to the smallest observed loss $L(\theta,d)$ along the trajectory. Thus, this falsifies the intrinsic data-difficulty hypothesis.

\begin{table}[H]
\centering
\caption{Comparison of $H(d)$ and $L(\theta_t,d)$ across optimizers. The results show that $H(d)$ is negligible relative to $L(\theta_t,d)$.}
\label{tab:optimizer_metrics}
\begin{tabular}{lcccccc}
\toprule
Metric & Adam-mini & Adam & Lion & Muon & Scion & Soap \\
\midrule
Average $\hat{H}(d)$ 
    & 0.001 & 0.001 & 0.001 & 0.001 & 0.001 & 0.001 \\
Average $L(\theta_{10000}, d)$ 
    & 3.500 & 3.457 & 3.554 & 3.409 & 3.459 & 3.403 \\
Average $\hat{H}(d)/L(\theta_{10000}, d)$ 
    & 0.003 & 0.004 & 0.003 & 0.004 & 0.004 & 0.004 \\
\bottomrule
\end{tabular}
\end{table}

We then consider the different convergence-speed hypothesis. We falsify this hypothesis from two perspectives. 

First, we view this from the representation space. If the convergence-speed explanation were correct, models trained with different optimizers should reach similar parameters when matched at the same validation-loss level. 
\begin{table}[H]
\centering
\caption{Parameter cosine similarity across optimizers with aligned validation loss.}
\label{tab:param_cosine_similarity}
\begin{tabular}{ccccccccc}
\toprule
Adam step & Validation loss & Adam-mini & Lion & Muon & Scion & Shampoo & SOAP \\
\midrule
0    & 11.023 & 1.000 & 1.000 & 1.000 & 1.000 & 1.000 & 1.000 \\
1000 & 4.261  & 0.773 & 0.753 & 0.641 & 0.612 & 0.396 & 0.559 \\
2000 & 3.815  & 0.658 & 0.604 & 0.475 & 0.469 & 0.186 & 0.400 \\
3000 & 3.654  & 0.580 & 0.520 & 0.425 & 0.411 & 0.150 & 0.303 \\
4000 & 3.588  & 0.526 & 0.465 & 0.377 & 0.383 & 0.130 & 0.257 \\
\bottomrule
\end{tabular}
\end{table}
Table~\ref{tab:param_cosine_similarity} compares models trained with different optimizers after aligning them to the validation loss achieved by Adam at each training step. We report results only for validation-loss levels greater than or equal to $3.55$, because some optimizers do not achieve a loss below $3.55$ during training. The results show that the similarity between validation-loss-aligned models decreases as training progresses. This finding provides evidence against the convergence-speed explanation.

Second, we view the loss difference between optimizers. Table~\ref{tab:loss_difference_adam} reports the average loss difference between each optimizer $A$ and Adam, defined as $C_{A,\mathrm{Adam}}(t)=L_A(t)-L_{\mathrm{Adam}}(t)$. We observe that $C_{A,\mathrm{Adam}}(t)$ at $t=2000$, when \ac{rgi} emerges, differs substantially from its value at convergence, i.e., at $t=10000$. This finding provides evidence against the convergence-speed explanation.
\begin{table}[H]
\centering
\caption{Loss difference relative to Adam, where
$C_{A,\mathrm{Adam}}(t)=L_A(t)-L_{\mathrm{Adam}}(t)$.}
\label{tab:loss_difference_adam}
\begin{tabular}{rrrrrr}
\toprule
Step & Adam-mini & Lion & Muon & SOAP & Scion \\
\midrule
0     & 0.000 & 0.000 & 0.000  & 0.000  & 0.000 \\
1000  & 0.321 & 0.738 & -0.275 & -0.341 & -0.226 \\
\textbf{2000}  & \textbf{0.129} & \textbf{0.276} & \textbf{-0.138} & \textbf{-0.164} & \textbf{-0.085} \\
3000  & 0.092 & 0.181 & -0.099 & -0.117 & -0.046 \\
4000  & 0.073 & 0.130 & -0.089 & -0.100 & -0.033 \\
5000  & 0.067 & 0.107 & -0.072 & -0.083 & -0.018 \\
6000  & 0.056 & 0.089 & -0.064 & -0.075 & -0.010 \\
7000  & 0.044 & 0.073 & -0.062 & -0.070 & -0.008 \\
8000  & 0.041 & 0.078 & -0.061 & -0.065 & -0.007 \\
9000  & 0.041 & 0.090 & -0.048 & -0.053 & -0.007 \\
\textbf{10000} & \textbf{0.041} & \textbf{0.081} & \textbf{-0.040} & \textbf{-0.041} & \textbf{-0.007} \\
\bottomrule
\end{tabular}
\end{table}

\section{Empirical Verification of Assumptions~\ref{ass:H} and \ref{ass:x}}\label{app:verify}
 % We also conduct experiments to empirically verify Assumptions~\ref{ass:H} and \ref{ass:x}. 

To empirically assess this Assumptions~\ref{ass:H}, we estimate the target transformation at each layer as the best linear mapping between the hidden states of two adjacent layers after training. We compute $\widehat{H}^{l}=\arg\min_A\sum_i\|f _i^{l+1}-Af _i^l\|^2$, where $f_i^l$ denotes the normalized hidden state of token $i$ at layer $l$. We then calculate the angle $\widehat{a}_{i,j}^l$ between columns $\widehat{H}_{:,i}^{,l}$ and $\widehat{H}_{:,j}^{,l}$. The average angle at layer $l$ is $\widehat{a}^{l}=\sum_{i\neq j}\widehat{a}_{i,j}^l/(r^2-r)$, and the layer-averaged angle is $\widehat{a}=\sum_{l\in[L]}\widehat{a}^{l}/L$. An average angle close to $90^\circ$ will provide empirical evidence that the estimated transformation approximately satisfies Assumption~\ref{ass:H}. Figure~\ref{fig:assH} verifies this by showing that $\hat{a}$ for various optimizers is close to $90^\circ$ at the end of training.

\begin{figure}[H]
    \centering
    \subfigure[Values of $\hat{a}$ for models trained with different optimizers after training.]{ \includegraphics[width=0.44\textwidth]{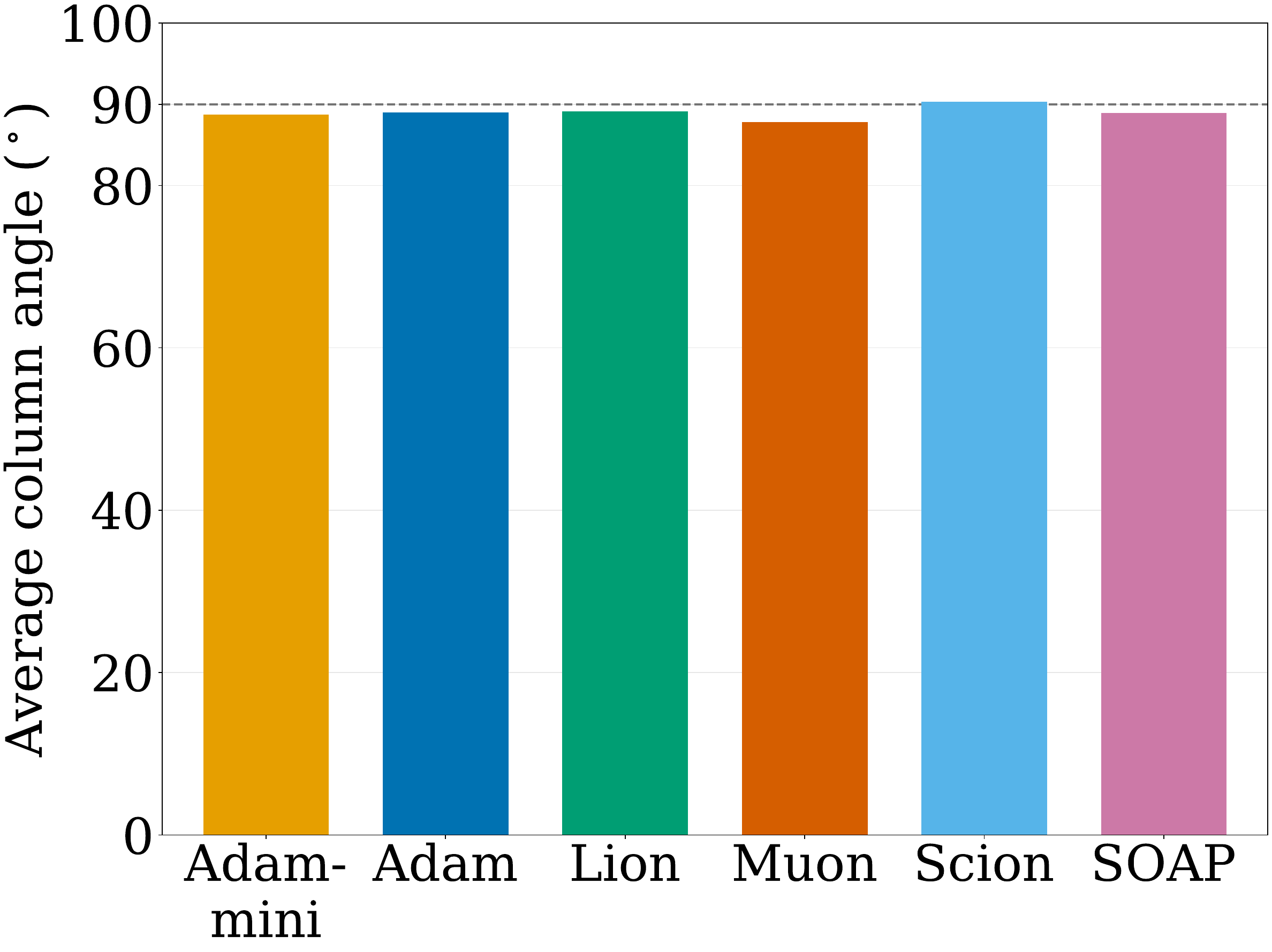}\label{fig:assH}}
    \subfigure[Values of $\hat{b}$ for different optimizers at different training steps.]{ \includegraphics[width=0.44\textwidth]{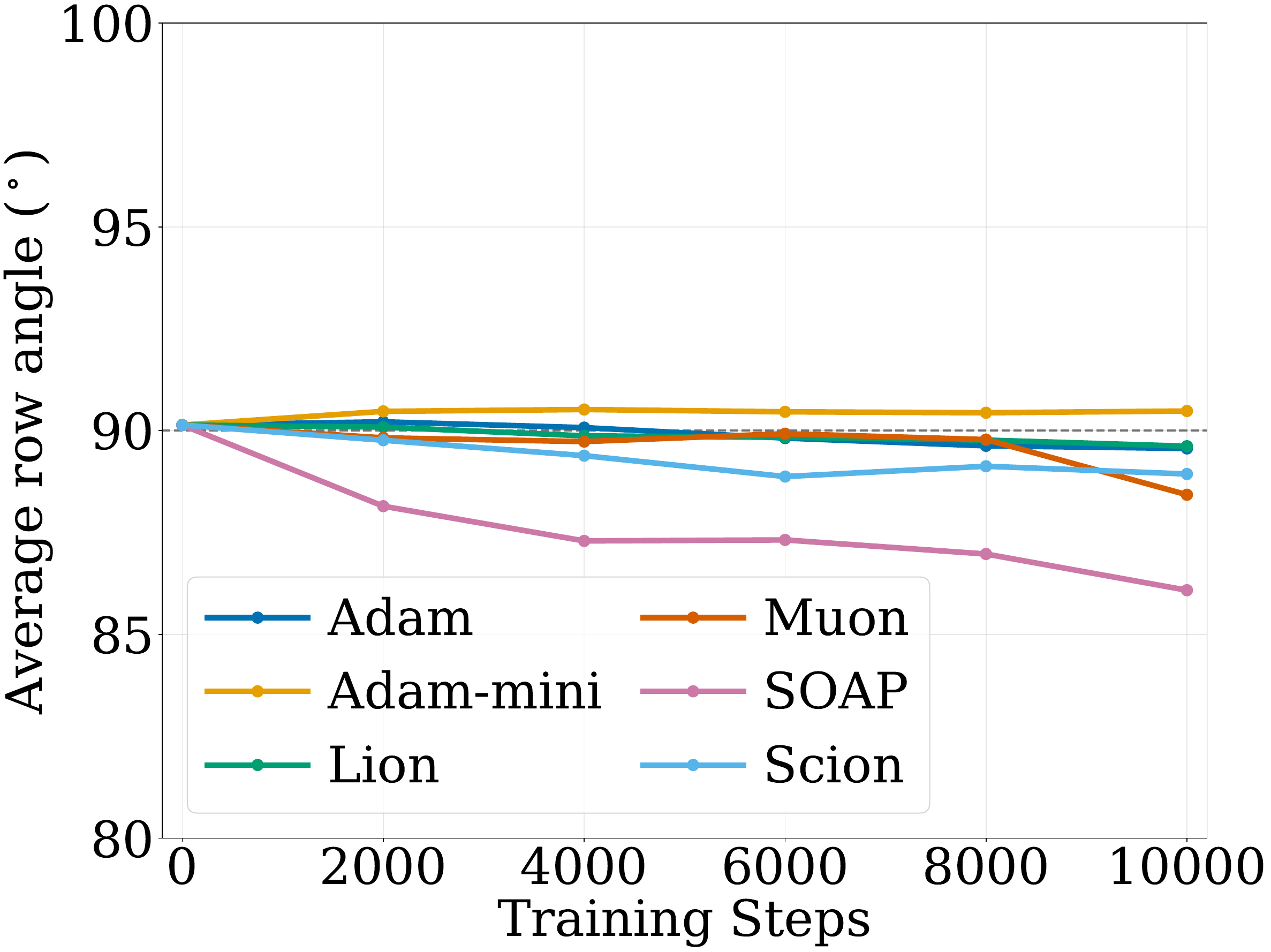}\label{fig:assx}}
    \caption{Empirical assessment of Assumptions~\ref{ass:H} and~\ref{ass:x}. Panel (a) shows that the columns of the targeted linear mappings are nearly orthogonal, while panel (b) shows that the hidden states of different tokens remain nearly orthogonal throughout training across all optimizers. These results provide empirical support for Assumptions~\ref{ass:H} and~\ref{ass:x} in practical \ac{llm} training.
    % \shuche{Panel (a) is columns. But in the figure, it is written as average row angle.}
} 
    \label{fig:assume_verify}
\end{figure}

Assumption~\ref{ass:x} essentially requires the inputs to each layer to be isotropic within the subspace they span. Our proof continues to hold when $\sum_{i=1}^{n} x_i x_i^{\top}/n=Q\operatorname{diag}(\sigma_x^2 I_r,0)Q^{\top}$,
where $r$ is the dimension of the subspace spanned by the inputs and $Q\in\bbR^{d\times d}$ is an orthogonal matrix. A sufficient condition is that, up to an orthogonal change of coordinates, the input vectors are mutually orthogonal and have equal norms.
\begin{proposition}
Suppose $x_1,\ldots,x_n\in\mathbb{R}^{d}$ are nonzero, mutually orthogonal vectors with the same norm. Then
$\sum_{i=1}^{n}x_ix_i^{\top}/n
=Q\operatorname{diag}(\sigma_x^2 I_r,0)Q^{\top}$
for some $\sigma_x>0$ and orthogonal matrix $Q\in\bbR^{d\times d}$, where $r=n$.
\end{proposition}
Since the input to each layer is normalized by RMSNorm, we focus on empirically assessing the directional orthogonality of the hidden states. Let $f_i^l$ denote the normalized hidden state of token $i$ at layer $l$. We measure the pairwise directional similarity by
$\widehat b_{i,j}^l$, which is the angle between $f_i^l$ and $f_j^l$, and average $\widehat b_{i,j}^l$ over distinct token pairs and layers to obtain the overall orthogonality measure $\widehat b$. An average angle close to $90^\circ$ indicates greater directional orthogonality. Figure~\ref{fig:assx} confirms this along the entire training courses of various optimizers.

\section{Proof of Proposition~\ref{prop:equivalence}}\label{app:prop_equivalence}
We first prove that \ac{rgi} is sufficient for constant generalization gap. From Definition~\ref{def:rgi}, we have that
\begin{align}
    \!\!L\big(\theta_t^{\frakT},d\big)\!-\!L\big(\theta_t^{\frakT^{\prime}},d\big) \!=\!L\big(\theta_t^{\frakT},d^*\big)\!-\!L\big(\theta_t^{\frakT^{\prime}},d^*\big),\label{eq:1}
\end{align}
where we fix $d^{\prime}$ as any $d^*$. Since the optimizers $\sfA,\sfA^\prime$ are deterministic, the right-hand side of Eqn.~\eqref{eq:1} is a deterministic function of $\theta_0,\theta_0^{\prime},D_{1:t},D_{1:t}^{\prime}$. We write it as $C(\frakT,\frakT^{\prime},t)$. Thus, we have that 
\begin{align*}
    L\big(\theta_t^{\frakT},d\big)-L\big(\theta_t^{\frakT^{\prime}},d\big) =C(\frakT,\frakT^{\prime},t).
\end{align*}

Then we prove that the constant generalization gap is sufficient for \ac{rgi}. This is direct, since the following holds
\begin{align*}
    L\big(\theta_t^{\frakT},d\big)-L\big(\theta_t^{\frakT^{\prime}},d\big) =C(\frakT,\frakT^{\prime},t)=L\big(\theta_t^{\frakT},d^{\prime}\big)-L\big(\theta_t^{\frakT^{\prime}},d^{\prime}\big).
\end{align*}
Thus, we conclude the proof of Proposition~\ref{prop:equivalence}.

% From Definition~\ref{def:rgi}, we have that
% \begin{align*}
%     L\big(\theta_t^{\frakT},d\big) &=L\big(\theta_t^{\frakT},d^{\prime}\big) + L\big(\theta_t^{\frakT^{\prime}},d\big)-L\big(\theta_t^{\frakT^{\prime}},d^{\prime}\big) \\
%     &= L\big(\theta_t^{\frakT},d^{\prime}\big) +J(\theta_0,D_{1:t}, d,d^{\prime}),
% \end{align*}
% where the function $J$ is defined as
% \begin{align*}
%     J(\theta_0,D_{1:t}, d,d^{\prime}) &= L\big(\theta_t^{\frakT^{\prime}},d\big)-L\big(\theta_t^{\frakT^{\prime}},d^{\prime}\big) \\
%     & =L\big(\theta_t^{\frakT},d\big)-L\big(\theta_t^{\frakT},d^{\prime}\big)
% \end{align*}
% The definition of $J$ does not depend on $\sfA$.

\section{Proof of Theorem~\ref{thm:quadratic_case_study}}\label{app:quadratic_case_study}

We prove Theorem~\ref{thm:quadratic_case_study} in three steps.
\begin{itemize}
    \item \textbf{Step 1}: Derive closed-form expressions for the optimizer trajectories.
    \item \textbf{Step 2}: Use these expressions to analyze the resulting losses.
    \item \textbf{Step 3}: Characterize the cosine similarity between the learned weights.
\end{itemize}

\textbf{Step 1: Derive closed-form expressions for the optimizer trajectories}

We start by calculating the gradients of $A$ and $B$. With Assumption~\ref{ass:x}, the loss can be simplified as follows.

\begin{align*}
    \frac1{2n}\sum_{i=1}^n\|(H-AB)x_i\|_2^2
    =\frac{1}{2}
    \operatorname{tr}\bigg[(H-AB)\bigg(\frac1n\sum_{i=1}^n x_ix_i^\top\bigg)(H-AB)^\top
    \bigg]=\frac{\sigma_x^2}{2}\|H-AB\|_F^2 .
\end{align*}
Therefore, the gradients of $A$ and $B$ are
\begin{align}
    \nabla_A L
    =\sigma_x^2(AB-H)B^\top,
    \qquad
    \nabla_B L
    =\sigma_x^2A^\top(AB-H).
    \label{eq:app_quadratic_gradients}
\end{align}
To analyze the trajectories of $\gd$, $\adam$, and $\muon$, we decompose the hidden dimension $s$ of parameters $A\in\bbR^{q\times s}, B\in\bbR^{s\times r}$ as $\mathbb R^s=\mathbb R^r\oplus\mathbb R^r\oplus
\mathbb R^{s-2r}$.
Accordingly, the parameters are decomposed as
\begin{align*}
    A=[A_1,A_2,A_3],
    \qquad
    B=
    \begin{bmatrix}
        B_1\\
        B_2\\
        B_3
    \end{bmatrix},
\end{align*}
where $A_1,A_2\in\mathbb R^{q\times r}$, $A_3\in\mathbb R^{q\times(s-2r)}$, $B_1,B_2\in\mathbb R^{r\times r}$, and $
B_3\in\mathbb R^{(s-2r)\times r}$. To define the trajectories of different optimizers, we introduce two auxiliary matrices 
$G,J\in\mathbb{R}^{q\times r}$ with orthonormal columns in the orthogonal complement of 
the target column space $\operatorname{col}(H)$. Specifically, we choose $G$ and $J$ such that
\begin{align*}
    G^\top G=J^\top J=I_r,
    \qquad
    H^\top G=H^\top J=0,
    \qquad
    G^\top J=0.
\end{align*}
Such matrices exist provided that $q-\operatorname{rank}(H)\ge 2r$. In addition, since 
$s\ge 3r$, we can choose a matrix $E\in\mathbb R^{(s-2r)\times r}$ with orthonormal columns:
\begin{align*}
    E^\top E=I_r .
\end{align*}

Then we present the following proposition about the trajectories of different optimizers.
\begin{proposition}\label{prop:param_structure}
For each optimizer $\sfA\in\{\gd,\adam,\muon\}$, initialize the parameters by
\begin{equation*}
    A_0^{\sfA}=0,
    \qquad
    B_0^{\sfA}
    =
    \begin{bmatrix}
        b_0^{\sfA}I_r\\
        0\\
        0
    \end{bmatrix},
    \qquad b_0^{\sfA}>0 .
\end{equation*}
For any positive step sizes $\eta_t,\gamma_t,\tau_t>0$ for $\gd$, $\adam$, and $\muon$,
respectively, under the specified tie-breaking rules, the iterates remain in the following
low-dimensional invariant subspaces. Namely, for every $t\ge 0$, there exist scalars
$a_t^{\gd},b_t^{\gd},a_t^{\adam},b_t^{\adam},\lambda_t^{\adam},
a_t^{\muon},b_t^{\muon},\lambda_t^{\muon}\in\mathbb R$ such that
\begin{align}
    A_t^{\gd}
    &=
    [a_t^{\gd}H,0,0],
    \qquad
    B_t^{\gd}
    =
    \begin{bmatrix}
        b_t^{\gd}I_r\\
        0\\
        0
    \end{bmatrix},\label{eq:gd}
    \\
    A_t^{\adam}
    &=
    [a_t^{\adam}H,-\lambda_t^{\adam}G,0],
    \qquad
    B_t^{\adam}
    =
    \begin{bmatrix}
        b_t^{\adam}I_r\\
        0\\
        0
    \end{bmatrix},\label{eq:adam}
    \\
    A_t^{\muon}
    &=
    [a_t^{\muon}H,0,-\lambda_t^{\muon}JE^\top],
    \qquad
    B_t^{\muon}
    =
    \begin{bmatrix}
        b_t^{\muon}I_r\\
        0\\
        0
    \end{bmatrix}.\label{eq:muon}
\end{align}

\end{proposition}
\begin{proof}
    To prove Proposition~\ref{prop:param_structure}, we introduce the structured family of parameters that will serve as an
invariant class for the optimizers' trajectories. 
\begin{align}
    A=[aH,\,-\lambda_GG,\,-\lambda_JJE^\top],
    \qquad
    B=
    \begin{bmatrix}
        bI_r\\
        0\\
        0
    \end{bmatrix},
    \label{eq:app_quadratic_structured_class}
\end{align}
where $a,b,\lambda_G,\lambda_J\in\bbR$ are real numbers. Then, with $c=ab$, we have that $AB-H=(c-1)H.$
Substituting this into Eqn.~\eqref{eq:app_quadratic_gradients} gives
\begin{align}
    \nabla_A L
    =
    \sigma_x^2(c-1)b\,[H,0,0].
    \label{eq:app_quadratic_grad_a}
\end{align}
For the gradient of $B$, the orthogonality conditions on $H,G,J$ derive that
\begin{align*}
    A^\top H
    =
    \begin{bmatrix}
        aH^\top H\\
        -\lambda_GG^\top H\\
        -\lambda_JEJ^\top H
    \end{bmatrix}
    =
    \begin{bmatrix}
        a\mu^2I_r\\
        0\\
        0
    \end{bmatrix},
\end{align*}
where the parameter $\mu$ results from Assumption~\ref{ass:H}. Thus, the gradient of $B$ is
\begin{align}
    \nabla_B L
    =
    \sigma_x^2(c-1)
    \begin{bmatrix}
        a\mu^2I_r\\
        0\\
        0
    \end{bmatrix}.
    \label{eq:app_quadratic_grad_b}
\end{align}
In the following, we will prove that $\gd$, $\adam$, and $\muon$ maintain structured weights defined in Proposition~\ref{prop:param_structure} with induction. We start with $\gd$. The initialization of $\gd$ satisfies Eqn.~\eqref{eq:gd} with $a_t^{\gd}=0$ and $b_t^{\gd}=b_0^{\gd}$ for $t=0$. We assume that Eqn.~\eqref{eq:gd} holds for $l=0,\cdots, t$. The parameters at step $t+1$ is updated as
\begin{align*}
    A_{t+1}^{\gd}
    =
    A_t^{\gd}
    -
    \eta_t\nabla_A\mathcal L(A_t^{\gd},B_t^{\gd}),
    \qquad
    B_{t+1}^{\gd}
    =
    B_t^{\gd}
    -
    \eta_t\nabla_B\mathcal L(A_t^{\gd},B_t^{\gd}).
\end{align*}

Denote $c_t^{\gd}=a_t^{\gd}b_t^{\gd}$, by Eqns.~\eqref{eq:app_quadratic_grad_a} and~\eqref{eq:app_quadratic_grad_b}, we have that $A_{t+1}^{\gd}$ and $B_{t+1}^{\gd}$ satisfy Eqn.~\eqref{eq:gd} with
\begin{align*}
    a_{t+1}^{\gd}
    &=
    a_t^{\gd}
    -
    \eta_t\sigma_x^2(c_t^{\gd}-1)b_t^{\gd},
    \\
    b_{t+1}^{\gd}
    &=
    b_t^{\gd}
    -
    \eta_t\sigma_x^2\mu^2(c_t^{\gd}-1)a_t^{\gd} .
\end{align*}
Thus, we prove the desired result for $\gd$. We then analyze $\adam$. Since the $\adam$ optimizer takes the element-wise sign of gradient, we consider the following tie-breaking rule
\begin{align*}
    \mathsf{Sgn}(M)
    =
    \left\{
    Z:\;
    Z_{ij}=\operatorname{sign}(M_{ij})\ \text{if }M_{ij}\neq0,
    \quad
    Z_{ij}\in[-1,1]\ \text{if }M_{ij}=0
    \right\},
\end{align*}
where $\operatorname{sign}(x)$ denotes the sign of $x$, with $\operatorname{sign}(0)=0$. Then the update of $\adam$ is 
\begin{align*}
    A_{t+1}^{\adam}
    =
    A_t^{\adam}-\gamma_t\Xi_t^A,
    \qquad
    B_{t+1}^{\adam}
    =
    B_t^{\adam}-\gamma_t\Xi_t^B,
\end{align*}
where $\Xi_t^A$ and $\Xi_t^B$ are plausible updates, i.e., 
\begin{align*}
    \Xi_t^A\in\mathsf{Sgn}(\nabla_A\mathcal L),
    \qquad
    \Xi_t^B\in\mathsf{Sgn}(\nabla_B\mathcal L).
\end{align*}
We then prove that this update rule satisfies Eqn.~\eqref{eq:adam}. We first note that the initialization satisfies Eqn.~\eqref{eq:adam} with $a_t^{\adam}=\lambda_t^{\adam}=0$, $b_t^{\adam}=b_0^{\adam}$ for $t=0$. Then we assume that Eqn.~\eqref{eq:adam} holds for $l=0,\cdots, t$. Eqs.~\eqref{eq:app_quadratic_grad_a} and~\eqref{eq:app_quadratic_grad_b}.
Using $\operatorname{sign}(H)=\kappa H$ imply that we can choose
\begin{align*}
    \Xi_t^A
    =
    \operatorname{sign}((c_t^{\adam}-1)b_t^{\adam})
    \kappa[H,0,0]
    +
    [0,G,0], \text{ and }\Xi_t^B
    =
    \operatorname{sign}((c_t^{\adam}-1)a_t^{\adam})
    \begin{bmatrix}
        I_r\\
        0\\
        0
    \end{bmatrix}.
\end{align*}
The update $[0,G,0]$ term in $\Xi_t^A$ is admissible because it is selected on a zero-gradient
block and the entries of $G$ lie in $[-1,1]$. Denote $c_t^{\adam}=a_t^{\adam}b_t^{\adam}$, we have that $A_{t+1}^{\adam}$ and $B_{t+1}^{\adam}$ satisfy Eqn.~\eqref{eq:adam} with
\begin{align}
    a_{t+1}^{\adam}
    &=
    a_t^{\adam}
    -
    \gamma_t\kappa
    \operatorname{sign}((c_t^{\adam}-1)b_t^{\adam}),\label{eq:adam_update_1}
    \\
    b_{t+1}^{\adam}
    &=
    b_t^{\adam}
    -
    \gamma_t
    \operatorname{sign}((c_t^{\adam}-1)a_t^{\adam}),\label{eq:adam_update_2}
    \\
    \lambda_{t+1}^{\adam}
    &=
    \lambda_t^{\adam}+\gamma_t .\label{eq:adam_update_3}
\end{align}
Thus, we prove the desired result for $\adam$. Finally, we consider $\muon$. Given a matrix $M\in\bbR^{m\times n}$ with rank $r$, let $M = U\Sigma V^\top$
be its compact SVD over the nonzero singular values, where
$\Sigma\in\bbR^{r\times r}$, $U\in\bbR^{m\times r}$, and
$V\in\bbR^{n\times r}$. We define the spectral normalization operator by
\begin{align*}
    \mathsf{Spec}(M)
    =
    \left\{
    Z\in\bbR^{m\times n}:\; ZV=U,\quad Z^\top U=V,\quad \|Z\|_2\le 1
    \right\},
\end{align*}
where $\|\cdot\|_2$ denotes the spectral norm of matrices. Specifically, when $M=0$, this definition gives $\mathsf{Spec}(0)=\{Z:\|Z\|_2\le 1\}$.
The $\muon$ optimizer then updates the parameters as 
\begin{align*}
    A_{t+1}^{\muon}
    =
    A_t^{\muon}-\tau_tT_t^A,
    \qquad
    B_{t+1}^{\muon}
    =
    B_t^{\muon}-\tau_tT_t^B,
\end{align*}
where $T_t^A$ and $T_t^B$ are plausible spectrally normalized updates, i.e.,
\begin{align*}
    T_t^A\in\mathsf{Spec}(\nabla_A\mathcal L),
    \qquad
    T_t^B\in\mathsf{Spec}(\nabla_B\mathcal L).
\end{align*}
In the following, we will prove that the parameters updated by $\muon$ satisfiy Eqn.~\eqref{eq:muon} via induction. The intialization of $\muon$ satifies Eqn.~\eqref{eq:muon} with $a_t^{\muon}=\lambda_t^{\muon}=0$, $b_t^{\muon}=b_0^{\muon}$ for $t=0$. We assume that Eqn.~\eqref{eq:muon} holds for $l=0,\cdots, t$. Since $H^\top H=\mu^2I_r$, the normalized active direction associated with
$H$ is $H/\mu$. Then we can choose $T_t^A$ and $T_t^B$ as follows.
\begin{align*}
    T_t^A
    =
    \operatorname{sign}((c_t^{\muon}-1)b_t^{\muon})\!\!
    \left[\frac{1}{\mu}H,0,0\right]
    \!+\!
    [0,0,JE^\top], \text{ and }T_t^B
    =
    \operatorname{sign}((c_t^{\muon}-1)a_t^{\muon})\!\!
    \begin{bmatrix}
        I_r\\
        0\\
        0
    \end{bmatrix}.
\end{align*}
The completion $[0,0,JE^\top]$ acts only on a zero singular subspace. Moreover,
$H^\top J=0$ and $E^\top E=I_r$ imply that
$[H/\mu,0,0]$ and $[0,0,JE^\top]$ have orthogonal left
and right singular subspaces. Therefore the selected $T_t^A$ has spectral norm one and is admissible. With $c_t^{\muon}=a_t^{\muon}b_t^{\muon}$, $A_{t+1}^{\muon}$ and $B_{t+1}^{\muon}$ satisfy Eqn.~\eqref{eq:muon} with
\begin{align*}
    a_{t+1}^{\muon}
    &=
    a_t^{\muon}
    -
    \frac{\tau_t}{\mu}
    \operatorname{sign}((c_t^{\muon}-1)b_t^{\muon}),
    \\
    b_{t+1}^{\muon}
    &=
    b_t^{\muon}
    -
    \tau_t
    \operatorname{sign}((c_t^{\muon}-1)a_t^{\muon}),
    \\
    \lambda_{t+1}^{\muon}
    &=
    \lambda_t^{\muon}+\tau_t .
\end{align*}

Thus, we conclude the proof of Proposition~\ref{prop:param_structure}
\end{proof}

\textbf{Step 2: Use these expressions to analyze the resulting losses.}

Proposition~\ref{prop:param_structure} shows that for any $\sfA\in\{\gd,\adam,\muon\}$, we have that
\begin{align*}
    A_t^{\sfA}B_t^{\sfA}
    =
    a_t^{\sfA} b_t^{\sfA} H.
\end{align*}
Thus, with $c_t^{\sfA}=a_t^{\sfA}b_t^{\sfA}$, we obtain
\begin{align*}
    \mathcal L_t^{\sfA}(x)
    =
    \frac12
    (1-c_t^{\sfA})^2
    \|Hx\|_2^2 .
\end{align*}
Since Assumption~\ref{ass:H} requires that $H^\top H=\mu^2 I_r$, every input point satisfying
$
    \|x\|_2=\rho.
$
Hence
\begin{align*}
    \mathcal L_t^{\sfA}(x)
    =
    \frac{\mu^2\rho^2}{2}
    (1-c_t^{\sfA})^2 .
\end{align*}
The right-hand side may depend on the optimizer $\sfA$ and on time $t$, but it is
constant over all points on the sphere $\|x\|_2=\rho$. Therefore, for any
two points $x,x'$ with $\|x\|_2=\|x'\|_2=\rho$,
and any two optimizers ${\sfA},{\sfA}'\in\{\gd,\adam,\muon\}$, we have
\begin{align*}
    L_t^{\sfA}(x)- L_t^{\sfA}(x')
    =
    L_t^{{\sfA}'}(x)- L_t^{{\sfA}'}(x').
\end{align*}
This proves exact \ac{rgi} on the validation sphere.

\textbf{Step 3: Characterize the cosine similarity between the learned weights.}

We denote the vectorized parameters as $\theta_t^{\frakT}=(\operatorname{vec}(A_t^{\sfA}),\operatorname{vec}(B_t^{\sfA}))$, where $\operatorname{vec}(\cdot)$ vectorizes the matrix parameter. Then the cosine similarity between parameters learned by different optimizers can be calculated as follows.
\begin{proposition}\label{prop:cossim}
    Under the same conditions of Proposition~\ref{prop:param_structure}, the cosine similarities between parameters learned by $\gd$, $\adam$, and $\muon$ are
\begin{align*}
    \cossim(\theta_t^{\gd},\theta_t^{\adam})
    &=
    \frac{
        \mu^2 a_t^{\gd}a_t^{\adam}
        +
        b_t^{\gd}b_t^{\adam}
    }{
        w_t^{\gd}\cdot w_t^{\adam}
    },\\
    \cossim(\theta_t^{\gd},\theta_t^{\muon})
    &=
    \frac{
        \mu^2 a_t^{\gd}a_t^{\muon}
        +
        b_t^{\gd}b_t^{\muon}
    }{
        w_t^{\gd}\cdot w_t^{\muon}
    },\\
    \cossim(\theta_t^{\adam},\theta_t^{\muon})
    &=
    \frac{
        \mu^2 a_t^{\adam}a_t^{\muon}
        +
        b_t^{\adam}b_t^{\muon}
    }{
        w_t^{\adam}\cdot w_t^{\muon}
    },
\end{align*}
where the denominators are defined as $w_t^{\gd}=\sqrt{
            \mu^2(a_t^{\gd})^2
            +
            (b_t^{\gd})^2
        }$, $w_{t}^{\adam}=\sqrt{
            \mu^2(a_t^{\adam})^2
            +
            (b_t^{\adam})^2
            +
            (\lambda_t^{\adam})^2
        }$, and $w_t^{\muon}=\sqrt{
            \mu^2(a_t^{\muon})^2
            +
            (b_t^{\muon})^2
            +
            (\lambda_t^{\muon})^2
        }$.
        Each displayed ratio applies when its denominator is nonzero;
otherwise, the corresponding similarity is zero by convention.
    
\end{proposition}
\begin{proof}
To compute the cosine similarity between two parameter vectors, we normalize their inner product by the product of their norms. We first record several basic identities for the corresponding inner products and norms. Specifically, we have
\begin{align*}
    \left\langle
        [H,0,0],
        [0,G,0]
    \right\rangle_F
    =
    0,\quad 
    \left\langle
        [H,0,0],
        \big[0,0,JE^\top\big]
    \right\rangle_F
    =
    0,\quad
    \left\langle
        [0,G,0],
        \big[0,0,JE^\top\big]
    \right\rangle_F
    =
    0.
\end{align*}
Moreover, since
$ H^\top H=\mu^2I_r,
    G^\top G=I_r,
    J^\top J=I_r,
    E^\top E=I_r$,
we have that
\begin{align*}
    \|[H,0,0]\|_F^2
    =
    \mu^2r,\quad
    \|[0,G,0]\|_F^2
    =
    \|G\|_F^2
    =
    r,\quad
    \|[0,0,JE^\top]\|_F^2
    =
    r.
\end{align*}

Thus, for $\gd$ and $\adam$, the inner product between parameters are
\begin{align*}
    \left\langle
        \theta_t^{\gd},
        \theta_t^{\adam}
    \right\rangle
    =
    r
    \left(
        \mu^2 a_t^{\gd}a_t^{\adam}
        +
        b_t^{\gd}b_t^{\adam}
    \right).
\end{align*}
Moreover, the norms of the parameters are
\begin{align*}
    \|\theta_t^{\gd}\|_2^2
    =
    r
    \left(
        \mu^2(a_t^{\gd})^2
        +
        (b_t^{\gd})^2
    \right),\quad
    \|\theta_t^{\adam}\|_2^2
    =
    r
    \left(
        \mu^2(a_t^{\adam})^2
        +
        (b_t^{\adam})^2
        +
        (\lambda_t^{\adam})^2
    \right).
\end{align*}
Combining these, we have that
\begin{align*}
    \cossim(\theta_t^{\gd},\theta_t^{\adam})
    =
    \frac{
        \mu^2 a_t^{\gd}a_t^{\adam}
        +
        b_t^{\gd}b_t^{\adam}
    }{
        w_t^{\gd}\cdot w_t^{\adam}
    }.
\end{align*}

Similarly, the cosine similarity between the parameters of $\gd$ and $\muon$ is
\begin{align*}
    \cossim(\theta_t^{\gd},\theta_t^{\muon})
    =
    \frac{
        \mu^2 a_t^{\gd}a_t^{\muon}
        +
        b_t^{\gd}b_t^{\muon}
    }{w_t^{\gd}\cdot w_t^{\muon}}.
\end{align*}

For the cosine similarity between the parameters learned by $\adam$ and $\muon$, we have that
\begin{align*}
    \left\langle
        [0,G,0],
        [0,0,JE^\top]
    \right\rangle_F
    =
    0.
\end{align*}
Then we calculate the cosine similarity as
\begin{align*}
    \cossim(\theta_t^{\adam},\theta_t^{\muon})
    =
    \frac{
        \mu^2 a_t^{\adam}a_t^{\muon}
        +
        b_t^{\adam}b_t^{\muon}
    }{w_t^{\adam}\cdot w_t^{\muon}
    }.
\end{align*}
Thus, we conclude the proof of Proposition~\ref{prop:cossim}.
\end{proof}

In the following, we will prove that $a_t$ and $b_t$ of $\adam$ and $\muon$ are uniformly bounded,
whereas $\lambda_t$ of them grows linearly. We use constant step sizes for them, i.e., $\gamma_t=\gamma,\tau_t=\tau.$

First, we consider $\adam$. 
% Recall that
% \begin{align*}
%     c_t^{\adam}
%     =
%     a_t^{\adam}b_t^{\adam},
% \end{align*}
% and
% \begin{align*}
%     a_{t+1}^{\adam}
%     =
%     a_t^{\adam}
%     -
%     \gamma\kappa
%     \operatorname{sign}
%     \left(
%         (c_t^{\adam}-1)b_t^{\adam}
%     \right),
% \end{align*}
% \begin{align*}
%     b_{t+1}^{\adam}
%     =
%     b_t^{\adam}
%     -
%     \gamma
%     \operatorname{sign}
%     \left(
%         (c_t^{\adam}-1)a_t^{\adam}
%     \right),
% \end{align*}
% \begin{align*}
%     \lambda_{t+1}^{\adam}
%     =\lambda_t^{\adam}+\gamma.
% \end{align*}
Recall that we initialize the parameter with $a_0^{\adam}=0,\lambda_0^{\adam}=0,
b_0^{\adam}>0$. Since $c_0^{\adam}
=a_0^{\adam}b_0^{\adam}=0$, the first update gives
\begin{align*}
    a_1^{\adam}
    &=
    a_0^{\adam}
    -
    \gamma\kappa
    \operatorname{sign}
    \left(
        (c_0^{\adam}-1)b_0^{\adam}
    \right)
    =
    \kappa\gamma,\\
    b_1^{\adam}&=b_0^{\adam}-
\gamma\operatorname{sign}
\left((c_0^{\adam}-1)a_0^{\adam}\right)=b_0^{\adam},
\end{align*}
where the equalities result from Eqn.~\eqref{eq:adam_update_1} and \eqref{eq:adam_update_2}.
Thus, the first update leads to $a_1^{\adam}>0, b_1^{\adam}>0$. We then define
\begin{align*}
    \Delta_{\adam}
    =
    a_1^{\adam}
    -
    \kappa b_1^{\adam}
    =
    \kappa(\gamma-b_0^{\adam}),
    \quad
    D_{\adam}
    =
    |\Delta_{\adam}|,\quad r_{\adam}
    =\frac{1}{D_{\adam}+\sqrt{\kappa}
    }.
\end{align*}
Then we have the following results about $\adam$.

\begin{proposition}\label{prop:adam}
    With the step-size condition $\gamma_t=\gamma<\min\left\{b_0^{\adam},(\kappa b_0^{\adam}+\sqrt{\kappa})^{-1}\right\}$, we have that  $a_t^{\adam}>0, b_t^{\adam}>0$ and
$a_t^{\adam}-\kappa b_t^{\adam}=\Delta_{\adam}
$ hold for all $t\geq 1$.
\end{proposition}
\begin{proof}
    
We prove the desired results by induction. The results trivially hold for $t=1$. Suppose the claims hold at time $t\geq 1$.
Since $a_t^{\adam}>0$ and $b_t^{\adam}>0$, we have
\begin{align*}
    \operatorname{sign}
    \left(
        (c_t^{\adam}-1)b_t^{\adam}
    \right)
    =
    \operatorname{sign}(c_t^{\adam}-1),\quad
    \operatorname{sign}
    \left(
        (c_t^{\adam}-1)a_t^{\adam}
    \right)
    =
    \operatorname{sign}(c_t^{\adam}-1).
\end{align*}
Therefore, Eqn.~\eqref{eq:adam_update_1} and \eqref{eq:adam_update_2} show that $a_{t+1}^{\adam}=a_t^{\adam}-\gamma\kappa \operatorname{sign}(c_t^{\adam}-1)$, and $b_{t+1}^{\adam}=b_t^{\adam}-\gamma\operatorname{sign}(c_t^{\adam}-1)$. It follows that
\begin{align*}
    a_{t+1}^{\adam}
    -
    \kappa b_{t+1}^{\adam}
    =
    a_t^{\adam}
    -
    \kappa b_t^{\adam}
    =
    \Delta_{\adam}.
\end{align*}
It remains to show that the next iterate stays in the positive quadrant. Since $a_t^{\adam}=\kappa b_t^{\adam}+\Delta_{\adam}$, if \(0<b_t^{\adam}\le r_{\adam}\), then
\begin{align*}
    c_t^{\adam}
    =
    \left(
        \kappa b_t^{\adam}
        +
        \Delta_{\adam}
    \right)b_t^{\adam}
    \le
    \kappa(b_t^{\adam})^2
    +
    D_{\adam}b_t^{\adam}
    \le
    \kappa r_{\adam}^2
    +
    D_{\adam}r_{\adam}.
\end{align*}
By the definition of $r_{\adam}$, we have
\begin{align*}
    \kappa r_{\adam}^2
    +D_{\adam}r_{\adam}=
    \frac{\kappa}{(D_{\adam}+\sqrt{\kappa})^2
    }
    +
    \frac{D_{\adam}}{D_{\adam}+\sqrt{\kappa}
    }\le 1.
\end{align*}
Hence, we can show that if 
$c_t^{\adam}>1$, we have $b_t^{\adam}>r_{\adam}$. Similarly, since $b_t^{\adam}=(
a_t^{\adam}-\Delta_{\adam})/\kappa$, if \(0<a_t^{\adam}\le \kappa r_{\adam}\), then
\begin{align*}
    c_t^{\adam}
    =
    \frac{
        a_t^{\adam}
        \left(
            a_t^{\adam}
            -
            \Delta_{\adam}
        \right)
    }{\kappa}
    \le
    \frac{
        a_t^{\adam}
        \left(
            a_t^{\adam}
            +
            D_{\adam}
        \right)
    }{\kappa}
    \le
    \kappa r_{\adam}^2
    +
    D_{\adam}r_{\adam}
    \le
    1.
\end{align*}
Hence, we also show that if $c_t^{\adam}>1$, we have 
$a_t^{\adam}>\kappa r_{\adam}$.

Now consider two cases. If $c_t^{\adam}\le1$,
then $\operatorname{sign}(c_t^{\adam}-1)\le0$, by the update rule in Eqn.~\eqref{eq:adam_update_1} and \eqref{eq:adam_update_2}, we have 
\begin{align*}
    a_{t+1}^{\adam}
    \ge
    a_t^{\adam}
    >
    0,
    \qquad
    b_{t+1}^{\adam}
    \ge
    b_t^{\adam}
    >
    0.
\end{align*}
If $c_t^{\adam}>1$, then we have that $b_t^{\adam}>r_{\adam}, a_t^{\adam}>\kappa r_{\adam}$. Since step-size condition $\gamma_t=\gamma<\min\left\{b_0^{\adam},(\kappa b_0^{\adam}+\sqrt{\kappa})^{-1}\right\}$ implies
$0<\gamma<r_{\adam}$, we have
\begin{align*}
    b_{t+1}^{\adam}
    =b_t^{\adam}-\gamma
    >r_{\adam}-\gamma
    >0,\quad
    a_{t+1}^{\adam}
    =
    a_t^{\adam}-\kappa\gamma
    >
    \kappa r_{\adam}-\kappa\gamma
    >
    0.
\end{align*}
Thus, the positive quadrant is invariant for $\adam$. We conclude the proof of Proposition~\ref{prop:adam}.
\end{proof}
In the following, we will show that $a_t^{\adam}$ and $b_{t}^{\adam}$ are bounded. Define
\begin{align*}
    R_{\adam}
    =
    \max\left\{
        \frac{2D_{\adam}}{\kappa},
        \sqrt{\frac{2}{\kappa}}
    \right\},\quad M_{\adam}=\max\left\{
b_0^{\adam}, R_{\adam}+\gamma\right\}.
\end{align*}
We prove by induction that $b_t^{\adam}\le M_{\adam}$ for all $t$. The claim holds at $t=0$. Suppose $b_t^{\adam}\le M_{\adam}$. If $b_t^{\adam}\ge R_{\adam}$, then we have that
\begin{align*}
 a_t^{\adam}
    =\kappa b_t^{\adam}
    +\Delta_{\adam}
    \ge \kappa b_t^{\adam}
    -D_{\adam}
    \ge\frac{\kappa}{2}b_t^{\adam},
\end{align*}
where the first equality results from Proposition~\ref{prop:adam}. Therefore, we have $c_t^{\adam}=a_t^{\adam}b_t^{\adam}\ge \kappa/2\cdot(b_t^{\adam})^2\ge 1$.
Hence, we have if $b_t^{\adam}\ge R_{\adam}$, then 
$b_{t+1}^{\adam}\le b_t^{\adam}\le M_{\adam}$. If $b_t^{\adam}<R_{\adam}$, then we have that 
\begin{align*}
    b_{t+1}^{\adam}=
b_t^{\adam}-\gamma\operatorname{sign}(c_t^{\adam}-1)
    \le
    b_t^{\adam}+\gamma
    <
    R_{\adam}+\gamma
    \le
    M_{\adam}.
\end{align*}
Therefore, we show that
$b_t^{\adam}\le M_{\adam}$
for all $t$. Furthermore, since $a_t^{\adam}=\kappa b_t^{\adam}+\Delta_{\adam}$, we also have
\begin{align*}
    0<a_t^{\adam}
    \le
    \kappa M_{\adam}
    +
    D_{\adam}.
\end{align*}
Thus, we show that $a_t^{\adam},b_t^{\adam}=O(1)$ with $t$ increasing.

Second, we consider $\muon$. 
% Recall that
% \begin{align*}
%     c_t^{\muon}
%     =a_t^{\muon}b_t^{\muon},
% \end{align*}
% and
% \begin{align*}
%     a_{t+1}^{\muon}
%     =
%     a_t^{\muon}
%     -
%     \frac{\tau}{\mu}
%     \operatorname{sign}
%     \left(
%         (c_t^{\muon}-1)b_t^{\muon}
%     \right),
% \end{align*}
% \begin{align*}
%     b_{t+1}^{\muon}
%     =
%     b_t^{\muon}
%     -
%     \tau
%     \operatorname{sign}
%     \left(
%         (c_t^{\muon}-1)a_t^{\muon}
%     \right),
% \end{align*}
% \begin{align*}
%     \lambda_{t+1}^{\muon}
%     =
%     \lambda_t^{\muon}+\tau.
% \end{align*}
The parameters are initialized with $a_0^{\muon}=0, \lambda_t^{\muon}=0, b_0^{\muon}>0$. We then define 
\begin{align*}
\Delta_{\muon}=a_1^{\muon}-\frac1\mu b_1^{\muon}, \quad D_{\muon}=|\Delta_{\muon}|,\quad
    r_{\muon}
    =\frac{1}{D_{\muon}+\mu^{-1/2}}.
\end{align*}
The step-size condition of $\tau_t=\tau<\min\left\{b_0^{\muon},( b_0^{\muon}/\mu+\mu^{-1/2})^{-1}\right\}$ implies $0<\tau<r_{\muon}$. Similar to Proposition~\ref{prop:adam} for $\adam$, we can also show by induction that  $a_t^{\muon}>0, b_t^{\muon}>0$ and $a_t^{\muon}-\mu^{-1} b_t^{\muon}=\Delta_{\muon}$ for all $t\geq 1$. For the boundness, we define
\begin{align*}
    R_{\muon}
    =
    \max\left\{
        2\mu D_{\muon},
        \sqrt{2\mu}
    \right\},\quad M_{\muon}
    =
    \max\left\{
        b_0^{\muon},
        R_{\muon}+\tau
    \right\}.
\end{align*}
Similar to our proof for $\adam$, we can show by induction that $b_t^{\muon}\le M_{\muon}$ for all $t\geq 1$. Since $a_t^{\muon}=\mu^{-1} b_t^{\muon}+\Delta_{\muon}$, we also have
$0<a_t^{\muon} \le\mu^{-1} M_{\muon}+D_{\muon}$. Thus, we have  $a_t^{\muon},b_t^{\muon}=O(1)$ when $t$ increases.

In the following, we then derive the bounds for the cosine similarities between parameters. Since $\lambda_t^{\adam}
=\lambda_0^{\adam}+\gamma t$ and $\lambda_t^{\muon}
=\lambda_0^{\muon}+\tau t$, we have that
\begin{align*}
|\lambda_t^{\adam}|=\Theta(t),
    \quad
|\lambda_t^{\muon}|=\Theta(t).
\end{align*}
From the boundedness proved above,
there exist finite constants $C_{\adam},C_{\muon}$ such that for all $t$,
\begin{align*}
    \mu^2(a_t^{\adam})^2
    +(b_t^{\adam})^2
    \le
    C_{\adam}^2,\quad
    \mu^2(a_t^{\muon})^2
    +
    (b_t^{\muon})^2
    \le
    C_{\muon}^2.
\end{align*}
Thus, for $\gd$ and $\adam$, we have
\begin{align*}
    \left|\cossim(\theta_t^{\gd},\theta_t^{\adam})
    \right|
    \le
    \frac{\sqrt{\mu^2(a_t^{\adam})^2+(b_t^{\adam})^2}
    }{\sqrt{\mu^2(a_t^{\adam})^2+(b_t^{\adam})^2
    +(\lambda_t^{\adam})^2
    }}
    \le\frac{C_{\adam}
    }{|\lambda_t^{\adam}|
    }=O(t^{-1}),
\end{align*}
where the first inequality results from the Cauchy-Schwarz inequality, and the second inequality results from the boundedness of $a_t^{\adam}$ and $b_t^{\adam}$. Similarly,  we have that
\begin{align*}
    \left|
        \cossim(\theta_t^{\gd},\theta_t^{\muon})
    \right|
    \!\le\!
    \frac{
        C_{\muon}
    }{
        |\lambda_t^{\muon}|
    }
    =
    O(t^{-1}),\,
    \left|
        \cossim(\theta_t^{\adam},\theta_t^{\muon})
    \right|
    \!\le\!
    \frac{
        C_{\adam}C_{\muon}
    }{
        |\lambda_t^{\adam}|
        |\lambda_t^{\muon}|
    }
    =
    O(t^{-2}).
\end{align*}
Thus, we conclude the proof of Theorem~\ref{thm:quadratic_case_study}.

\end{document}